\documentclass{article}

\PassOptionsToPackage{numbers,compress}{natbib}

\usepackage[workshop]{nips_2018}

\usepackage[utf8]{inputenc} 
\usepackage[T1]{fontenc}    
\usepackage{hyperref}       
\usepackage{url}            
\usepackage{booktabs}       
\usepackage{amsfonts}       
\usepackage{nicefrac}       
\usepackage{microtype}      
\usepackage{nameref}

\usepackage{float}
\usepackage{graphicx}
\usepackage{tocstyle}

\usepackage{color}
\usepackage{amsmath}
\usepackage{amssymb}
\usepackage{amsthm}
\usepackage{mathtools}
\usepackage[mathscr]{eucal}
\usepackage{bm}
\usepackage{bbm}
\usepackage{relsize}
\usepackage{graphics}
\usepackage{wrapfig}
\usepackage{multirow}
\usepackage{enumitem} 

\usepackage{array}

\usepackage{pgf}
\usepackage{xinttools}
\usepackage{tikz}
\usetikzlibrary{arrows.meta}
\usepackage{textcomp}
\usetikzlibrary{calc,shapes,arrows,positioning,automata,trees}
\usepackage{placeins} 
\usepackage{tabularx}
\usepackage{graphicx}
\usepackage{subcaption} 
\usepackage{listings}
\usepackage{xargs,lipsum,caption,changepage,ifthen}

\usepackage[toc,page]{appendix}

\DeclarePairedDelimiter\floor{\lfloor}{\rfloor}

\newtheorem{theorem}{Theorem}
\newtheorem{definition}{Definition}
\newtheorem*{theorem*}{Theorem}
\newtheorem*{definition*}{Definition}

\newtheorem{theoremA}{Theorem}
\newtheorem{definitionA}{Definition}
\newtheorem{corollaryA}{Corollary}%
\newtheorem{propositionA}{Proposition}%
\newtheorem{lemmaA}{Lemma}%
\newcommand\Ba{\bm{a}}
\newcommand\Bb{\bm{b}}
\newcommand\Bc{\bm{c}}

\newcommand\Bf{\bm{f}}
\newcommand\Bg{\bm{g}}
\newcommand\Bh{\bm{h}}
\newcommand\Bi{\bm{i}}

\newcommand\Bo{\bm{o}}
\newcommand\Bp{\bm{p}}
\newcommand\Bq{\bm{q}}
\newcommand\Br{\bm{r}}
\newcommand\Bs{\bm{s}}

\newcommand\Bw{\bm{w}}
\newcommand\Bx{\bm{x}}
\newcommand\By{\bm{y}}
\newcommand\Bz{\bm{z}}
\newcommand\BA{\bm{A}}
\newcommand\BB{\bm{B}}
\newcommand\BC{\bm{C}}
\newcommand\BD{\bm{D}}

\newcommand\BI{\bm{I}}
\newcommand\BJ{\bm{J}}

\newcommand\BP{\bm{P}}
\newcommand\BQ{\bm{Q}}
\newcommand\BR{\bm{R}}
\newcommand\BS{\bm{S}}

\newcommand\BV{\bm{V}}
\newcommand\BW{\bm{W}}
\newcommand\BX{\bm{X}}

\newcommand\Bbe{\bm{\beta}}

\newcommand\Bth{\bm{\theta}}

\newcommand\Bom{\bm{\omega}}
\newcommand\BOn{\bm{1}}
\newcommand\BZe{\bm{0}}

 \newcommand{\dR}{\mathbb{R}}

 \newcommand{\rT}{\mathrm{T}}

\newcommand{\cA}{\mathcal{A}}

\newcommand{\cG}{\mathcal{G}}

 \newcommand{\cP}{\mathcal{P}}
 \newcommand{\cR}{\mathcal{R}}
\newcommand{\cS}{\mathcal{S}}

\newcommand{\sA}{\mathscr{A}}

 \newcommand{\sR}{\mathscr{R}}
\newcommand{\sS}{\mathscr{S}}

\newcommand{\Rc}{\mathrm{c}} \newcommand{\Rd}{\mathrm{d}} 
 \newcommand{\Rf}{\mathrm{f}}

\newcommand\EXP{\mathbf{\mathrm{E}}}
\newcommand\PR{\mathbf{\mathrm{Pr}}}
\newcommand\VAR{\mathbf{\mathrm{Var}}}
\newcommand\COV{\mathbf{\mathrm{Cov}}}
\newcommand\TR{\mathbf{\mathrm{Tr}}}

\newcommand\argmax{\mathop{\mathrm{argmax}\,}}

\newcommand\nn{\mathrm{new}}

\newcommand\net{\mathrm{net}}

\newcommand{\mse}{\mathop{\bf mse}}
\newcommand{\bias}{\mathop{\bf bias}}
\newcommand{\var}{\mathop{\bf var}}

\newcommand{\ABS}[1]{{{\left| #1 \right|}}} 

\renewcommand{\leq}{\leqslant}
\renewcommand{\geq}{\geqslant}

\newcommand{\figpath}{figures/}

\newcommand{\returnrealization}{\mathbf{g}}

\usepackage{pdflscape} 

\newcolumntype{R}[1]{>{\raggedright\arraybackslash}p{#1}}
\newcolumntype{C}[1]{>{\centering\arraybackslash}p{#1}}
\newcolumntype{L}[1]{>{\raggedleft\arraybackslash}p{#1}}

\usepackage{bigdelim}
\usepackage{colortbl} 
\definecolor{mColor1}{rgb}{0.95,0.95,0.95}

\newcommandx{\mycaptionminipage}[3][3=c,usedefault]{%
    \begin{minipage}[#3]{#1}%
        \ifthenelse{\equal{#3}{b}}{\captionsetup{aboveskip=0pt}}{}
        \ifthenelse{\equal{#3}{t}}{\captionsetup{belowskip=0pt}}{}
        \vspace{0pt}\centering\captionsetup{width=\textwidth} 
        #2%
    \end{minipage}%
}%
\newcommandx{\mysidecaption}[4][4=c,usedefault]{%
    \checkoddpage%
    \ifoddpage%
        \mycaptionminipage{\dimexpr\linewidth-#1\linewidth-\intextsep\relax}{#3}[#4]%
        \hfill%
        \mycaptionminipage{#1\linewidth}{#2}[#4]%
    \else%
        \mycaptionminipage{#1\linewidth}{#2}[#4]%
        \hfill%
        \mycaptionminipage{\dimexpr\linewidth-#1\linewidth-\intextsep\relax}{#3}[#4]%
    \fi%
}%

\hypersetup{
   colorlinks=true,
   linkcolor=blue,
   citecolor=green,
   urlcolor=magenta,
   pdfborder=0 0 0,
   pdftitle=RUDDER: Return Decomposition for Delayed Rewards,
   pdfsubject=reinforcement learning; delayed reward; reward redistribution; return decomposition; bias-variance; credit assignment; LSTM ,
   pdfkeywords={reinforcement learning, delayed reward, reward redistribution, return decomposition, bias-variance, credit assignment, LSTM },
   pdfauthor=anonymous,%
   pdfstartview=FitH
}

\title{RUDDER: Return Decomposition for Delayed Rewards}

\author{\vspace{0.1cm} Jose A. Arjona-Medina\thanks{authors contributed equally} \quad Michael Gillhofer\footnotemark[1] \quad Michael Widrich\footnotemark[1] \\ {\vspace{0.1cm}\bf Thomas Unterthiner \quad Johannes Brandstetter \quad Sepp Hochreiter\footnotemark[2]} \\
  LIT AI Lab \\ Institute for Machine Learning\\
  Johannes Kepler University Linz, Austria \\
  \footnotemark[2]~~also at Institute of Advanced Research in Artificial Intelligence (IARAI)
}

\DeclareUnicodeCharacter{2212}{-}

\begin{document}


\maketitle

\begin{abstract}
We propose RUDDER,
a novel reinforcement learning approach
for delayed rewards in finite Markov decision processes (MDPs).
In MDPs the $Q$-values are equal to the
expected immediate reward plus the expected future rewards.
The latter are related to bias problems 
in temporal difference (TD) learning and to
high variance problems in Monte Carlo (MC) learning.
Both problems are even more severe when rewards are delayed.
RUDDER aims at making the expected future rewards
zero, which simplifies $Q$-value estimation to
computing the mean of the immediate reward.
We propose the following two new concepts to push the
expected future rewards toward zero.
(i) Reward redistribution that leads to return-equivalent
decision processes with the same optimal policies and, when optimal,
zero expected future rewards.
(ii) Return decomposition via contribution analysis which
transforms the reinforcement learning task 
into a regression task at which deep learning excels. 
On artificial tasks with delayed rewards,
RUDDER is significantly faster than 
MC and exponentially faster than Monte Carlo Tree Search (MCTS), TD($\lambda$),
and reward shaping approaches.
At Atari games, 
RUDDER on top of a Proximal Policy Optimization (PPO) baseline improves the scores, 
which is most prominent at games with delayed rewards.
Source code is available at \url{https://github.com/ml-jku/rudder} and demonstration videos at \url{https://goo.gl/EQerZV}.
\end{abstract}

\section{Introduction}

Assigning credit for a received reward to past actions
is central to reinforcement learning \cite{Sutton:18book}.
A great challenge is to learn long-term credit assignment for delayed 
rewards \cite{Ke:18,Hung:18,Hernandez-Leal:18,Sahni:18}.
Delayed rewards are often episodic or sparse
and common in real-world problems \cite{Rahmandad:09,Luoma:17}.
For Markov decision processes (MDPs), 
the $Q$-value is equal to the expected immediate reward plus 
the expected future reward.
For $Q$-value estimation,
the expected future reward leads to biases in temporal difference (TD) and
high variance in Monte Carlo (MC) learning.
For delayed rewards, 
TD requires exponentially many updates to correct the bias, 
where the number of updates is exponential in the number of delay steps.
For MC learning the number of states affected by a delayed reward can 
grow exponentially with the number of delay steps. 
(Both statements are proved after 
theorems \ref{th:AexponDecay} and \ref{th:Aaffect} in the appendix.)
An MC estimate of the expected future reward 
has to average over all possible future trajectories,
if rewards, state transitions, or policies are probabilistic.
Delayed rewards make an MC estimate much harder.
    
The main goal of our approach is 
to construct an MDP that has {\bf expected future rewards equal to zero}. 
If this goal is achieved, 
$Q$-value estimation simplifies to computing the mean of the immediate rewards.
To push the expected future rewards to zero, we require two new concepts.
The first new concept is {\bf reward redistribution} to create {\bf return-equivalent MDPs}, 
which are characterized by having the same optimal policies. 
An optimal reward redistribution should transform a delayed reward MDP into  
a return-equivalent MDP with zero expected future rewards.
However, expected future rewards equal to zero are in general not possible for MDPs.
Therefore, we introduce sequence-Markov decision processes (SDPs),
for which reward distributions need not to be Markov.
We construct a reward redistribution that leads to a return-equivalent
SDP with a second-order Markov reward distribution and 
expected future rewards that are equal to zero.
For these return-equivalent 
SDPs, $Q$-value estimation simplifies to computing the mean. 
Nevertheless, the $Q$-values or advantage functions 
can be used for learning optimal policies.
The second new concept is
{\bf return decomposition} and its realization via {\bf contribution analysis}.
This concept serves to efficiently construct a proper reward redistribution,
as described in the next section.
Return decomposition transforms a reinforcement learning task 
into a regression task, where the sequence-wide 
return must be predicted from the whole state-action sequence.
The regression task identifies which state-action pairs
contribute to the return prediction and, therefore, receive a redistributed reward. 
Learning the regression model uses only completed episodes as training set,
therefore avoids problems with unknown future state-action trajectories.
Even for sub-optimal reward redistributions, we obtain an enormous speed-up
of $Q$-value learning if relevant reward-causing state-action pairs are identified.
We propose RUDDER (RetUrn Decomposition for DElayed Rewards) for learning with 
reward redistributions that are obtained via return decompositions.

{\bf To get an intuition} for our approach,
assume you repair pocket watches and then sell them.
For a particular brand of watch you have to decide
whether repairing pays off.
The sales price is known,
but you have unknown costs, 
i.e.\ negative rewards, caused by repair and delivery.
The advantage function is
the sales price
minus the expected immediate repair costs
minus the expected future delivery costs.
Therefore, you want to know whether the advantage function is positive.
--- Why is zeroing the expected future costs beneficial? ---
If the average delivery costs are known,
then they can be added to the repair costs
resulting in zero future costs.
Using your repairing experiences,
you just have to average over the repair costs
to know whether repairing pays off.
--- Why is return decomposition so efficient? ---
Because of pattern recognition.
For zero future costs, you have to estimate
the expected brand-related delivery costs,
which are e.g.\ packing costs.
These brand-related costs are superimposed by 
brand-independent general delivery costs 
for shipment (e.g.\ time spent for delivery).
Assume that general delivery costs
are indicated by patterns, e.g.\ weather conditions, which delay delivery.
Using a training set of completed deliveries,
supervised learning can identify these patterns
and attribute costs to them.
This is return decomposition.
In this way, only brand-related delivery costs remain 
and, therefore, can be estimated more efficiently than by MC.

{\bf Related Work.}
Our new learning algorithm is gradually changing 
the reward redistribution during learning, 
which is known as shaping \cite{Skinner:58,Sutton:18book}.
In contrast to RUDDER, potential-based shaping like 
reward shaping \cite{Ng:99}, look-ahead advice, and
look-back advice \cite{Wiewiora:03icml} use a fixed reward redistribution. 
Moreover, since these methods keep the original reward, 
the resulting reward redistribution is not optimal,
as described in the next section,
and learning can still be exponentially slow. 
A monotonic positive reward transformation \cite{Peters:07}
also changes the reward distribution but is neither assured 
to keep optimal policies nor to have expected future rewards of zero.
Disentangled rewards keep optimal policies but are neither environment 
nor policy specific, therefore can in general not achieve 
expected future rewards being zero \cite{Fu:18}.
Successor features decouple environment and policy 
from rewards, but changing the reward changes the  
optimal policies \cite{Barreto:17,Barreto:18}. 
Temporal Value Transport (TVT) uses 
an attentional memory mechanism to learn
a value function that serves as
fictitious reward \cite{Hung:18}. 
However, expected future rewards are not close to zero
and optimal policies are not guaranteed to be kept.
Reinforcement learning tasks have been changed
into supervised tasks \cite{Schaal:99,Barto:04,Schmidhuber:15}. 
For example, a model that predicts the return can supply update 
signals for a policy by sensitivity analysis. 
This is known as ``backpropagation through a model''  \cite{Munro:87,Robinson:89,RobinsonFallside:89,Werbos:90,Schmidhuber:91nips,Bakker:02,Bakker:07}.
In contrast to these approaches, 
(i) we use contribution analysis instead of 
sensitivity analysis,
and (ii) we use the whole state-action sequence to predict its
associated return.

\section{Reward Redistribution and Novel Learning Algorithms}

Reward redistribution is the main new concept
to achieve expected future rewards equal to zero.
We start by introducing MDPs,  
return-equivalent sequence-Markov decision processes (SDPs),
and reward redistributions.
Furthermore, optimal reward redistribution is defined and
novel learning algorithms based on reward redistributions 
are introduced.

\paragraph{MDP Definitions and 
           Return-Equivalent Sequence-Markov Decision Processes (SDPs).}
\label{c:def}
A finite Markov decision process (MDP) $\cP$
is 5-tuple $\cP=(\sS,\sA,\sR,p,\gamma )$ of finite sets
$\sS$ of states $s$ (random variable $S_t$ at time $t$),
$\sA$ of actions $a$ (random variable $A_t$),
and $\sR$ of rewards $r$ (random variable $R_{t+1}$).
Furthermore, $\cP$ has
transition-reward distributions
$p(S_{t+1}=s',R_{t+1}=r \mid S_t=s,A_t=a)$
conditioned on state-actions, and a discount factor $\gamma \in [0, 1]$.
The marginals are $p(r\mid s,a) = \sum_{s'}p(s',r\mid s,a)$ and 
$p(s'\mid s,a) = \sum_{r}p(s',r\mid s,a)$.
The expected reward is $r(s,a) = \sum_{r} r  p(r\mid s,a)$.
The return $G_t$ is $G_t =  \sum_{k=0}^{\infty}  \gamma^k  R_{t+k+1}$,
while for finite horizon MDPs with sequence
length $T$ and $\gamma=1$ it is $G_t = \sum_{k=0}^{T-t}  R_{t+k+1}$.
A Markov policy is given as
action distribution $\pi(A_t=a \mid S_t=s)$ conditioned on
states.
We often equip an MDP $\cP$ with a policy $\pi$ 
without explicitly mentioning it. 
The action-value function $q^{\pi}(s,a)$ for policy $\pi$ is 
$q^{\pi}(s,a) = \EXP_{\pi} \left[ G_t \mid S_t=s, A_t=a \right]$.
The goal of learning is to maximize the expected return at time $t=0$,
that is $v^{\pi}_0=\EXP_{\pi} \left[G_0\right]$.
The optimal policy $\pi^{*}$ is $\pi^{*} = \argmax_{\pi}[v_0^{\pi}]$.
A {\em sequence-Markov decision process} (SDP) is defined 
as a decision process which is equipped with a Markov policy
and has Markov transition probabilities
but a reward that is not required to be Markov.
Two SDPs $\tilde{\cP}$ and $\cP$
are {\em return-equivalent} if
(i) they differ only in their reward distribution  
and (ii) they
have the same expected return at $t=0$ for each policy $\pi$:
$\tilde{v}^{\pi}_0=v^{\pi}_0$. 
They are {\em strictly return-equivalent} if
they have the same expected return for every episode and
for each policy $\pi$.
Strictly return-equivalent SDPs are return-equivalent.
Return-equivalent SDPs have the same optimal policies.
For more details see Section~\ref{sec:AReturnEquivalent} in the appendix.

\paragraph{Reward Redistribution.}
Strictly return-equivalent SDPs $\tilde{\cP}$ and $\cP$
can be constructed by reward redistributions.
A {\em reward redistribution} given an SDP $\tilde{\cP}$ 
is a procedure that redistributes 
for each sequence $s_0,a_0,\ldots,s_T,a_T$ the realization of 
the sequence-associated return variable 
$\tilde{G}_0 = \sum_{t=0}^{T}  \tilde{R}_{t+1}$ 
or its expectation along the sequence. 
Later we will introduce a reward redistribution that 
depends on the SDP $\tilde{\cP}$.
The reward redistribution creates a new SDP $\cP$ with
the redistributed reward $R_{t+1}$ at time $(t+1)$ and the return variable
$G_0 = \sum_{t=0}^{T} R_{t+1}$. 
A reward redistribution is second order Markov if
the redistributed reward $R_{t+1}$ depends only on $(s_{t-1},a_{t-1},s_t,a_t)$. 
If the SDP $\cP$ is obtained from the 
SDP $\tilde{\cP}$ by reward redistribution, 
then $\tilde{\cP}$ and $\cP$ are strictly return-equivalent.
The next theorem states that the optimal policies are still the same 
for $\tilde{\cP}$ and $\cP$ (proof after Section Theorem~S2). 
\begin{theorem}
\label{th:EquivT}
Both the SDP $\tilde{\cP}$ 
with delayed reward $\tilde{R}_{t+1}$
and the SDP $\cP$ with redistributed reward $R_{t+1}$ 
have the same optimal policies.
\end{theorem}

\paragraph{Optimal Reward Redistribution with 
           Expected Future Rewards Equal to Zero.}
We move on to the main goal of this paper: 
to derive an SDP via reward redistribution 
that has expected future rewards equal to zero and, therefore,
no delayed rewards.
At time $(t-1)$ the immediate reward is $R_t$ with expectation
$r(s_{t-1},a_{t-1})$.
We define the expected future rewards $\kappa(m,t-1)$ at time $(t-1)$
as the expected sum of future rewards from $R_{t+1}$ to $R_{t+1+m}$.
\begin{definition}
For $1 \leq t\leq T$ and $0\leq m\leq T-t$,  
the expected sum of delayed rewards at time $(t-1)$ 
in the interval $[t+1,t+m+1]$ is defined as
$\kappa(m,t-1)  =  \EXP_{\pi} \left[ \sum_{\tau=0}^{m} R_{t+1+\tau}  
  \mid s_{t-1}, a_{t-1} \right]$.
\end{definition}
For every time point $t$, the expected future rewards $\kappa(T-t-1,t)$ 
given $(s_t,a_t)$
is the expected sum of future rewards until sequence end, 
that is, in the interval $[t+2,T+1]$. 
For MDPs, the Bellman equation for $Q$-values becomes
$q^\pi(s_t,a_t) =   r(s_t,a_t)  + \kappa(T-t-1,t)$.
We aim to derive an MDP with $\kappa(T-t-1,t)=0$,
which gives $q^\pi(s_t,a_t) =  r(s_t,a_t)$.
In this case, learning the $Q$-values 
simplifies to estimating the expected immediate reward
$r(s_t,a_t) = \EXP \left[ R_{t+1} \mid s_t,a_t\right]$.
Hence, the reinforcement learning task reduces to computing 
the mean, e.g.\ the arithmetic mean, for each
state-action pair $(s_t,a_t)$.
A reward redistribution is defined to be {\em optimal},  
if $\kappa(T-t-1,t) = 0$ for $0\leq t \leq T-1$.
In general, an optimal reward redistribution violates
the Markov assumptions and the Bellman equation does
not hold (proof after Theorem~\ref{th:Aviolate} in the appendix).
Therefore, we will consider SDPs in the following.
The next theorem states that 
a delayed reward MDP $\tilde{\cP}$ 
with a particular policy $\pi$
can be transformed into a return-equivalent SDP $\cP$
with an optimal reward redistribution.
\begin{theorem}
\label{th:zeroExp}
We assume a delayed reward MDP $\tilde{\cP}$, 
where the accumulated reward is given at sequence end.
A new SDP $\cP$ is obtained by a 
second order Markov reward redistribution,
which ensures that $\cP$ is return-equivalent to $\tilde{\cP}$.
For a specific $\pi$, the following two
statements are equivalent:
(I) ~~$\kappa(T-t-1,t) = 0$, i.e.\ the reward redistribution is optimal, 
\begin{flalign}
\label{eq:diffQ}
\text{(II)~~}  \EXP \left[ R_{t+1} \mid s_{t-1},a_{t-1},s_t,a_t \right] 
   \ &= \ \tilde{q}^\pi(s_t,a_t) \ - \
    \tilde{q}^\pi(s_{t-1},a_{t-1}) \ .&& 
\end{flalign}
An optimal reward redistribution
fulfills for $1 \leq t\leq T$ and $0\leq m\leq T-t$:
  $\kappa(m,t-1)= 0$. 
\end{theorem}
The proof can be found after Theorem~\ref{th:AzeroExp} in the appendix.
Equation $\kappa(T-t-1,t)= 0$ implies that the new SDP $\cP$
has no delayed rewards, that is, 
$\EXP_{\pi} \left[ R_{t+1+\tau}  \mid  s_{t-1}, a_{t-1} \right] =  0$,
for $0\leq \tau \leq T-t-1$ 
(Corollary~\ref{th:ApropDelay} in the appendix).
The SDP $\cP$ has no delayed rewards since no state-action pair
can increase or decrease the expectation of a future reward.
Equation~\eqref{eq:diffQ} shows that for an optimal reward redistribution
the expected reward has to be the difference of
consecutive $Q$-values of the original delayed reward.
The optimal reward redistribution is
second order Markov since the expectation of $R_{t+1}$ at time 
$(t+1)$ depends on $(s_{t-1},a_{t-1},s_t,a_t)$.

The next theorem states the major advantage of an
optimal reward redistribution:
$\tilde{q}^\pi(s_t,a_t)$ can be estimated with an offset that 
depends only on $s_t$ 
by estimating the expected immediate redistributed reward.
Thus, $Q$-value estimation becomes trivial and the
the advantage function of the MDP $\tilde{\cP}$ can be readily computed.
\begin{theorem}
\label{th:OptReturnDecomp}
If the reward redistribution is 
optimal, then the $Q$-values 
of the SDP $\cP$ are given by 
\begin{align}
\label{eq:qvalue}
   q^\pi(s_t,a_t) \ &= \   r(s_t,a_t) \ = \  
   \tilde{q}^\pi(s_t,a_t) \ - \ 
    \EXP_{s_{t-1},a_{t-1}} \left[ \tilde{q}^\pi(s_{t-1},a_{t-1}) \mid s_t \right]
    \ = \ \tilde{q}^\pi(s_t,a_t) \ - \ \psi^\pi(s_t) \ .
\end{align} \vspace{-0.7cm}

The SDP $\cP$ and the original MDP $\tilde{\cP}$ 
have the same advantage function.
Using a behavior policy 
$\breve{\pi}$ the expected immediate reward is
\begin{align}
\label{eq:behavior}
   \EXP_{\breve{\pi}} \left[ R_{t+1} \mid s_t,a_t \right] \ &= \
    \tilde{q}^\pi(s_t,a_t) \ - \ \psi^{\pi,\breve{\pi}}(s_t) \ .
  \end{align}
\end{theorem}
The proof can be found after Theorem~\ref{th:AOptReturnDecomp} in the appendix. 
If the reward redistribution is not optimal, then 
$\kappa(T-t-1,t)$ measures the deviation of the $Q$-value 
from $r(s_t,a_t)$. 
This theorem justifies several learning methods based on
reward redistribution presented in the next paragraph.

\paragraph{Novel Learning Algorithms Based on Reward Redistributions.}
\label{c:novel}
We assume $\gamma=1$ and a finite horizon or an absorbing state
original MDP $\tilde{\cP}$ with delayed rewards.
For this setting we introduce new reinforcement learning algorithms.
They are gradually changing 
the reward redistribution during learning and are
based on the estimations in Theorem~\ref{th:OptReturnDecomp}.
These algorithms are also valid for non-optimal reward redistributions,
since the optimal policies are kept (Theorem~\ref{th:EquivT}).
Convergence of RUDDER learning
can under standard assumptions be proven by the stochastic 
approximation for 
two time-scale update rules \cite{Borkar:97,Karmakar:17}.
Learning consists of an LSTM and a $Q$-value update.
Convergence proofs to an optimal policy are 
difficult, since locally stable attractors 
may not correspond to optimal policies.

According to Theorem~\ref{th:EquivT}, reward redistributions
keep the optimal policies. 
Therefore, even non-optimal reward redistributions ensure correct learning. 
However, an optimal reward redistribution speeds up learning considerably.
Reward redistributions can be combined 
with methods that use $Q$-value ranks or advantage functions.
We consider 
(A) $Q$-value estimation, 
(B) policy gradients, and 
(C) $Q$-learning.
Type (A) methods estimate $Q$-values and are divided 
into variants (i), (ii), and (iii).
Variant (i) assumes an optimal reward redistribution
and estimates $\tilde{q}^\pi(s_t,a_t)$ with an offset
depending only on $s_t$.
The estimates are based on Theorem~\ref{th:OptReturnDecomp}
either by on-policy direct $Q$-value estimation according to Eq.~\eqref{eq:qvalue}
or by off-policy immediate reward estimation according to Eq.~\eqref{eq:behavior}.
Variant (ii) methods assume a non-optimal reward redistribution and 
correct Eq.~\eqref{eq:qvalue} by estimating $\kappa$.
Variant (iii) methods use eligibility traces for the redistributed reward.
RUDDER learning can be based on policies like
``greedy in the limit with infinite exploration'' (GLIE) or
``restricted rank-based randomized'' (RRR) \cite{Singh:00}. 
GLIE policies change toward greediness with respect to the $Q$-values
during learning.
For more details on these learning approaches 
see Section~\ref{sec:QestimateA} in the apendix.

Type (B) methods replace in the expected updates  
$\EXP_{\pi}\left[ \nabla_{\theta} \log \pi(a \mid s;\Bth)  
q^\pi(s,a) \right]$ of policy gradients the value $q^\pi(s,a)$ 
by an estimate of $r(s,a)$ or by
a sample of the redistributed reward. 
The offset $\psi^{\pi}(s)$ in Eq.~\eqref{eq:qvalue}
or  $\psi^{\pi,\breve{\pi}}(s)$ in Eq.~\eqref{eq:behavior}
reduces the variance as baseline normalization does. 
These methods can be extended to Trust Region Policy
Optimization (TRPO) \cite{Schulman:15icml} as used in 
Proximal Policy Optimization (PPO) \cite{Schulman:17}.
The type (C) method is $Q$-learning with the redistributed reward. 
Here, $Q$-learning is justified if
immediate and future reward are drawn together,
as typically done.

\section{Constructing Reward Redistributions by Return Decomposition}

We now propose methods to construct reward redistributions. 
Learning with non-optimal reward redistributions {\em does work} since the 
optimal policies do not change according to Theorem~\ref{th:EquivT}.
However, reward redistributions that are optimal considerably speed up learning,
since future expected rewards introduce 
biases in TD methods and high variances in MC methods.
The expected optimal redistributed reward is 
the difference of $Q$-values according to Eq.~\eqref{eq:diffQ}. 
The more a reward redistribution deviates from these differences,
the larger are the absolute $\kappa$-values and, in turn, the less optimal
the reward redistribution gets.
Consequently, to construct a reward redistribution which is close to optimal
we aim at identifying the largest $Q$-value differences.

\paragraph{Reinforcement Learning as Pattern Recognition.}
We want to transform the reinforcement learning problem into
a pattern recognition task to employ deep learning approaches.
The sum of the $Q$-value differences gives the 
difference between expected return at sequence begin and
the expected return at sequence end (telescope sum).
Thus, $Q$-value differences allow to predict the 
expected return of the whole state-action sequence.
Identifying the largest $Q$-value differences 
reduces the prediction error most.
$Q$-value differences are assumed to be associated with
patterns in state-action transitions.
The largest $Q$-value differences 
are expected to be found more frequently in sequences
with very large or very low return.
The resulting task is to predict the expected return
from the whole sequence and identify which 
state-action transitions have contributed the most to the prediction.
This pattern recognition task serves to
construct a reward redistribution, where the redistributed reward
corresponds to the different contributions.
The next paragraph shows how the return is decomposed and redistributed
along the state-action sequence.

\paragraph{Return Decomposition.}
\label{para:returnDecomposition}
The {\em return decomposition} idea is 
that a function $g$ predicts the expectation
of the return
for a given state-action sequence (return for the whole sequence).
The function $g$ is neither a value nor an action-value function
since it predicts the expected return when the whole sequence is given.
With the help of $g$ either the predicted value or
the realization of the return is redistributed over
the sequence. 
A state-action pair receives as redistributed reward
its contribution to the prediction, which
is determined by contribution analysis.
We use contribution analysis since sensitivity analysis has serious drawbacks:
local minima, instabilities, exploding or vanishing
gradients, and proper exploration
\cite{Hochreiter:90,Schmidhuber:90diff}.
The major drawback is that
the relevance of actions is missed
since sensitivity analysis does not consider the contribution of actions to 
the output,
but only their effect on the output when slightly perturbing them.
Contribution analysis
determines how much a state-action pair contributes to the final prediction.
We can use any contribution analysis method, but
we specifically consider three methods:
(A) differences of return predictions,
(B) integrated gradients (IG) \cite{Sundararajan:17}, and
(C) layer-wise relevance propagation (LRP) \cite{Bach:15}.
For (A), $g$ must try to predict the  
sequence-wide return at every time step.
The redistributed reward is given by 
the difference of consecutive predictions. 
The function $g$ can be decomposed into
past, immediate, and future contributions to the return.
Consecutive predictions share the same past and the same
future contributions except for two immediate state-action pairs.
Thus, in the difference of consecutive predictions 
contributions cancel except for the two immediate state-action pairs.
Even for imprecise predictions of future contributions to the return, 
contribution analysis is more precise, 
since prediction errors cancel out.
Methods (B) and (C) rely on information later in the sequence for
determining the contribution and thereby may introduce a non-Markov reward.
The reward can be viewed to be probabilistic
but is prone to have high variance. 
Therefore, we prefer method (A).

\paragraph{Explaining Away Problem.}
\label{para:explainingAway}
We still have to tackle the problem that reward causing actions do not receive redistributed rewards
since they are explained away by later states.
To describe the problem, assume an MDP $\tilde{\cP}$ with the only 
reward at sequence end.
To ensure the Markov property, states in $\tilde{\cP}$ have to store 
the reward contributions of previous state-actions;
e.g.\ $s_T$ has to store all previous contributions such that the expectation $\tilde{r}(s_T,a_T)$
is Markov.
The explaining away problem is that later states
are used for return prediction, while reward causing
earlier actions are missed.
To avoid explaining away,
we define a difference function $\Delta(s_{t-1},a_{t-1},s_t,a_t)$
between a state-action pair $(s_t,a_t)$ and 
its predecessor $(s_{t-1},a_{t-1})$.
That $\Delta$ is a function of $(s_t,a_t,s_{t-1},a_{t-1})$
is justified by Eq.~\eqref{eq:diffQ}, which ensures that such
$\Delta$s allow an optimal reward redistribution.
The sequence of differences is 
$\Delta_{0:T} := \big(\Delta(s_{-1},a_{-1},s_0,a_0),\ldots,
 \Delta(s_{T-1},a_{T-1},s_T,a_T)\big)$. 
The components $\Delta$ are assumed 
to be statistically independent from each other, therefore
$\Delta$ cannot store reward contributions of previous $\Delta$.
The function $g$ should predict the return by 
$g(\Delta_{0:T}) = \tilde{r}(s_T,a_T)$ and
can be decomposed into $g(\Delta_{0:T}) = \sum_{t=0}^T h_t$. 
The contributions are
$h_t = h(\Delta(s_{t-1},a_{t-1},s_t,a_t))$
for $0 \leq t \leq T$.
For the redistributed rewards $R_{t+1}$, we ensure 
$\EXP \left[ R_{t+1} \mid s_{t-1},a_{t-1},s_t,a_t \right]  =  h_t$.
The reward $\tilde{R}_{T+1}$ of $\tilde{\cP}$
is probabilistic and 
the function $g$ might not be perfect,
therefore neither $g(\Delta_{0:T}) = \tilde{r}_{T+1}$ for the return
realization $\tilde{r}_{T+1}$ nor
$g(\Delta_{0:T}) = \tilde{r}(s_T,a_T)$ for the expected return
holds.
Therefore, we need to introduce the compensation
$\tilde{r}_{T+1} \ - \ 
  \sum_{\tau=0}^T h(\Delta(s_{\tau-1},a_{\tau-1},s_{\tau},a_{\tau}))$
as an extra reward $R_{T+2}$ at time $T+2$
to ensure strictly return-equivalent SDPs.
If $g$ was perfect, then it would predict the expected return which
could be redistributed.
The new redistributed rewards
$R_{t+1}$ are based on the return decomposition, since they 
must have the contributions $h_t$ as mean:\newline
$\EXP \left[ R_1 \mid s_0,a_0 \right]   =  h_0$ , 
$\EXP \left[ R_{t+1} \mid s_{t-1},a_{t-1},s_t,a_t \right] 
   =  h_t, 0 < t \leq T$  , 
$R_{T+2}  = \tilde{R}_{T+1}  -  \sum_{t=0}^T h_t$,
where the realization $\tilde{r}_{T+1}$ is replaced by its
random variable $\tilde{R}_{T+1}$.
If the prediction of $g$ is perfect, then we can
redistribute the expected return via the prediction.
Theorem~\ref{th:zeroExp} holds also for
the correction $R_{T+2}$ (see Theorem~\ref{th:AzeroExpCorr} in the appendix). 
A $g$ with zero prediction errors 
results in an optimal reward redistribution.
Small prediction errors lead to reward redistributions 
close to an optimal one.

\paragraph{RUDDER: Return Decomposition using LSTM.}
RUDDER uses a Long Short-Term Memory (LSTM) network for 
return decomposition and the resulting reward redistribution.
RUDDER consists of three phases. 
{\bf (I) Safe exploration.}
Exploration sequences should generate LSTM training samples 
with delayed rewards by avoiding 
low $Q$-values during a particular time interval. 
Low $Q$-values hint at states where the agent gets stuck. 
Parameters comprise starting time, length, and $Q$-value threshold.
{\bf (II) Lessons replay buffer for training the LSTM.}
If RUDDER's safe exploration discovers 
an episode with unseen delayed rewards,
it is secured in a lessons replay buffer \cite{Lin:93}. 
Unexpected rewards are indicated by a large prediction error of the LSTM.
For LSTM training, episodes with larger errors are sampled more often 
from the buffer, similar to prioritized
experience replay \cite{Schaul:15}.
{\bf (III) LSTM and return decomposition.}
An LSTM learns to predict sequence-wide return at 
every time step and, thereafter, 
return decomposition uses differences of return predictions
(contribution analysis method (A)) to construct a reward redistribution. 
For more details see Section~\ref{sec:ALSTMadjust} in the appendix.

\paragraph{Feedforward Neural Networks (FFNs) vs.\ LSTMs.}
In contrast to LSTMs, 
FNNs are not suited for processing sequences. 
Nevertheless, FNNs can learn a action-value function, which 
enables contribution analysis by 
differences of predictions. 
However, this leads to serious problems by spurious contributions
that hinder learning.
For example, any contributions would be incorrect
if the true expectation of the return did not change. 
Therefore, prediction errors might falsely cause contributions 
leading to spurious rewards. 
FNNs are prone to such prediction errors since they
have to predict the expected return
again and again from each different state-action pair and 
cannot use stored information.
In contrast, the LSTM is less prone to produce spurious 
rewards:
(i) The LSTM will only learn to store information 
if a state-action pair has a strong evidence 
for a change in the expected return. 
If information is stored, then internal states and,
therefore, also the predictions change, otherwise the predictions
stay unchanged.
Hence, storing events receives a contribution and
a corresponding reward, while by default nothing is stored and 
no contribution is given.
(ii) The LSTM tends to have smaller prediction errors since it can 
reuse past information for predicting the expected return. 
For example, key events can be stored.
(iii) Prediction errors of LSTMs are much more likely to cancel 
via prediction differences than those of FNNs. 
Since consecutive predictions of LSTMs rely on the same
internal states, they usually have highly correlated errors.

\paragraph{Human Expert Episodes.} They are an alternative to 
exploration and can 
serve to fill the lessons replay buffer.
Learning can be sped up considerably when LSTM identifies
human key actions.
Return decomposition will reward human key actions even for episodes
with low return since other actions that thwart high
returns receive negative reward.
Using human demonstrations in reinforcement learning
led to a huge improvement on some Atari
games like Montezuma's Revenge \cite{Pohlen:18,Aytar:18}.

\paragraph{Limitations.}
In all of the experiments reported in this manuscript, we show that RUDDER significantly outperforms other methods for delayed reward problems. However, RUDDER might not be effective when the reward is not delayed since LSTM learning takes extra time and has problems with very long sequences. Furthermore, reward redistribution may introduce disturbing spurious reward signals.

\section{Experiments}
\label{sec:Mexperiments}
RUDDER is evaluated on three artificial tasks with delayed rewards.
These tasks are designed to show problems of TD, MC, 
and potential-based reward shaping.
RUDDER overcomes these problems.
Next, we demonstrate that RUDDER also works for more complex tasks
with delayed rewards.
Therefore, we compare RUDDER with 
a Proximal Policy Optimization (PPO) baseline
on 52 Atari games. 
All experiments use finite time horizon or 
absorbing states MDPs with $\gamma=1$ and reward at episode end.
For more information see Section~\ref{sec:Aexp} in the appendix.

{\bf Artificial Tasks (I)--(III).} 
Task (I) shows that TD methods have problems with vanishing information 
for delayed rewards.
Goal is to learn that a delayed reward is
larger than a distracting immediate reward.
Therefore, the correct expected future reward must be  
assigned to many state-action pairs. 
Task (II) is a variation of the introductory pocket watch example
with delayed rewards.
It shows that MC methods have problems with the 
high variance of future unrelated rewards.
The expected future reward that is caused by the first action
has to be estimated.
Large future rewards that are not associated 
with the first action impede MC estimations. 
Task (III) shows that potential-based reward shaping methods have
problems with delayed rewards.
For this task, only the first two actions are relevant, to which
the delayed reward has to be propagated back.

The tasks have different delays, 
are tabular ($Q$-table),
and use an $\epsilon$-greedy policy with $\epsilon = 0.2$.
We compare RUDDER, MC, and TD($\lambda$) 
on all tasks, and Monte Carlo Tree Search (MCTS) on task (I).
Additionally, on task (III), 
SARSA($\lambda$) and reward shaping are compared.
We use $\lambda=0.9$ as suggested \cite{Sutton:18book}.
Reward shaping methods are the original method, look-forward advice,
and look-back advice with three different potential functions.
RUDDER uses an LSTM without output and forget gates,
no lessons buffer, and no safe exploration. 
For all tasks contribution analysis is performed 
with difference of return predictions.  
A $Q$-table is learned 
by an exponential moving average of 
the redistributed reward (RUDDER's $Q$-value estimation) 
or by $Q$-learning. 
Performance is measured by 
the learning time to achieve 90\% of the maximal expected return.
A Wilcoxon signed-rank test determines 
the significance of performance differences between RUDDER
and other methods.

{\bf(I) Grid World} 
shows problems of TD methods with delayed rewards.
The task illustrates a time bomb that
explodes at episode end.
The agent has to defuse the bomb
and then run away as far as possible since
defusing fails with a certain probability.
Alternatively, the agent can immediately run away, 
which, however, leads to less reward on average.
The Grid World is a $31\times 31$ grid with
{\em bomb} at coordinate $[30,15]$ and
{\em start} at $[30-d, 15]$, where $d$ is the delay of the task.
The agent can move 
{\em up}, {\em down}, {\em left}, and {\em right} as long as
it stays on the grid.
At the end of the episode, after $\floor*{1.5 d}$ steps, 
the agent receives a reward of 1000
with probability of 0.5, 
if it has visited {\em bomb}.
At each time step, the agent receives 
an immediate reward of $c\cdot t \cdot h$, 
where $c$ depends on the chosen action, 
$t$ is the current time step, and $h$ is 
the Hamming distance to {\em bomb}.
Each move toward the {\em bomb}, 
is immediately penalized with $c=-0.09$. 
Each move away from the {\em bomb}, 
is immediately rewarded with $c=0.1$. 
The agent must learn the $Q$-values precisely to
recognize that directly running away is not optimal.
Figure~\ref{fig:test}(I) shows the 
learning times 
to solve the task vs.\ the delay of the reward averaged over 100 trials.
For all delays, RUDDER is significantly faster than all other methods
with $p$-values $<10^{-12}$.
Speed-ups vs.\ MC and MCTS,
suggest to be exponential with delay time.
RUDDER is exponentially faster with increasing 
delay than $Q$($\lambda$), 
supporting Theorem~\ref{th:AexponDecay} in the appendix. 
{\bf RUDDER significantly outperforms all other methods.}

\begin{figure}[!t]
 \centering%
 \resizebox{\linewidth}{!}{%
    \input{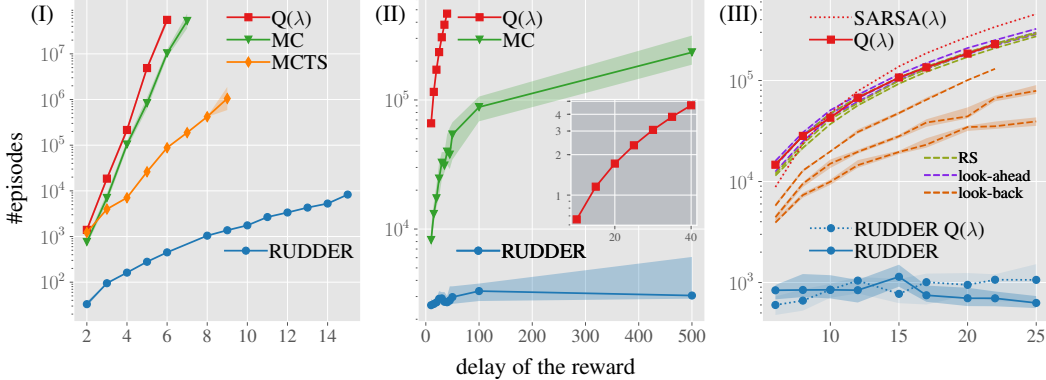}
 }
\caption{Comparison of RUDDER and other methods on artificial tasks 
with respect to the learning time in episodes (median of 100 trials)
vs.\ the delay of the reward.
The shadow bands indicate the $40\%$ and $60\%$ quantiles.
In (II), the y-axis of the inlet is scaled by $10^{5}$.
In (III), reward shaping (RS), 
look-ahead advice (look-ahead), 
and look-back advice (look-back) use three different potential functions.
In (III), the dashed blue line
represents RUDDER with $Q$($\lambda$), 
in contrast to RUDDER with $Q$-estimation. 
In all tasks, RUDDER significantly outperforms all other methods.\label{fig:test}}%
\end{figure}%

{\bf (II) The Choice}
shows problems of MC methods with delayed rewards.
This task has probabilistic state transitions, which
can be represented as a tree with states as nodes.
The agent traverses the tree from the root (initial state) 
to the leafs (final states).
At the root, the agent has to choose between 
the left and the right subtree,
where one subtree has a higher expected reward.
Thereafter, it traverses the tree randomly 
according to the transition probabilities. 
Each visited node adds its fixed share to the final reward. 
The delayed reward is given as accumulated shares at a leaf.
The task is solved when 
the agent always chooses the subtree with higher expected reward.
Figure~\ref{fig:test}(II) shows the 
learning times 
to solve the task vs.\ the delay of the reward averaged over 100 trials.
For all delays, RUDDER is significantly faster than all other methods
with $p$-values $<10^{-8}$.
The speed-up vs.\ MC, 
suggests to be exponential with delay time.
RUDDER is exponentially faster with increasing delay 
than $Q$($\lambda$), 
supporting Theorem~\ref{th:AexponDecay} in the appendix. 
{\bf RUDDER significantly outperforms all other methods.}

{\bf (III) Trace-Back} 
shows problems of potential-based reward shaping methods
with delayed rewards.
We investigate how fast information about delayed rewards 
is propagated back by
RUDDER, $Q$($\lambda$), SARSA($\lambda$), 
and potential-based reward shaping.
MC is skipped since it does not 
transfer back information.
The agent can move in a 15$\times$15 grid
to the 4 adjacent positions as long as it remains on the grid.
Starting at $(7,7)$, the number of moves per episode is $T=20$. 
The optimal policy moves the agent up in $t=1$ and  
right in $t=2$, 
which gives immediate reward of $-50$ at $t=2$, 
and a delayed reward of 150 at the end $t=20=T$. 
Therefore, the optimal return is 100.
For any other policy, 
the agent receives only an immediate reward of 50 at $t=2$.
For $t\leq 2$, state transitions are deterministic, while
for $t>2$ they are uniformly distributed and 
independent of the actions.
Thus, the return does not depend on actions at $t>2$.
We compare RUDDER, original reward shaping, 
look-ahead advice, and look-back advice. 
As suggested by the authors, we use SARSA instead of $Q$-learning 
for look-back advice. 
We use three different potential functions for reward shaping,
which are all based on the reward redistribution 
(see appendix). 
At $t=2$, there is a distraction since the immediate 
reward is $-50$ for the optimal and 50 for other actions. 
RUDDER is significantly faster than all other methods
with $p$-values $<10^{-17}$.
Figure~\ref{fig:test}(III) shows the 
learning times averaged over 100 trials.
{\bf RUDDER is exponentially faster than 
all other methods and significantly outperforms them.}

\paragraph{Atari Games.}
\label{para:Atari}
\label{c:Atari}
RUDDER is evaluated with respect to its learning time and 
achieves scores on Atari games of the
Arcade Learning Environment (ALE) \cite{Bellemare:13}
and OpenAI Gym \cite{Brockman:16}.
RUDDER is used on top of the TRPO-based \cite{Schulman:15icml} 
policy gradient method PPO that uses GAE \cite{Schulman:15}.
Our PPO baseline differs from the original 
PPO baseline \cite{Schulman:17} in two aspects.
(i) Instead of using the sign function of the rewards, 
rewards are scaled by their current maximum.
In this way, the ratio between different rewards
remains unchanged and the advantage of large delayed rewards
can be recognized. 
(ii) The safe exploration strategy of RUDDER is used.
The entropy coefficient is replaced by 
Proportional Control \cite{Bolton:15,Berthelot:17}. 
A coarse hyperparameter optimization is performed for the PPO baseline.
For all 52 Atari games, RUDDER uses 
the same architectures, losses, and hyperparameters, 
which were optimized for the baseline.
The only difference to the PPO baseline is that
the policy network predicts the value function 
of the redistributed reward to integrate reward redistribution
 into the PPO framework.
Contribution analysis uses an LSTM with
differences of return predictions.
Here $\Delta$ is the pixel-wise
difference of two consecutive frames augmented 
with the current frame.
LSTM training and reward redistribution are restricted to
sequence chunks of 500 frames.
Source code is provided upon publication.




%

\begin{table}
\begin{center}
\begin{tabular}{lcccr}
\toprule
& RUDDER & baseline & delay & delay-event\\
Bowling  & 192 & 56 & 200 & strike pins\\
Solaris  & 1,827 & 616 & 122 & navigate map\\
Venture  & 1,350 & 820 & 150 & find treasure\\
Seaquest  & 4,770 & 1,616 & 272 & collect divers\\
\bottomrule
\\ 
\end{tabular}
\caption{Average scores over 3 random seeds with 10 trials each for 
delayed reward Atari games.
"delay": frames between reward and first related action.  
RUDDER considerably improves the PPO baseline on delayed reward games.%
\label{tab:atarires}}%
\end{center}
\end{table}

Policies are trained with no-op starting condition
for 200M game frames using every 4th frame.
Training episodes end with losing a life 
or at maximal 108K frames.
All scores are averaged over 3 different random seeds 
for network and ALE initialization.
We asses the performance by the learning time and
the achieved scores. 
First, we compare RUDDER to the baseline by
average scores per game throughout training, 
to assess learning speed \cite{Schulman:17}.
For 32 (20) games RUDDER (baseline) learns on average faster.
Next, we compare the average scores of the last 10 training games.
For 29 (23) games RUDDER (baseline) has higher average scores.
In the majority of games RUDDER, improves the scores of the PPO baseline.
To compare RUDDER and the baseline on Atari games that are 
characterize by delayed rewards, 
we selected the games Bowling, Solaris, Venture, and Seaquest.
In these games, high scores are achieved by learning the delayed reward, 
while learning the immediate reward and 
extensive exploration (like for Montezuma's revenge)
is less important.
The results are presented in Table~\ref{tab:atarires}.
For more details and further results see Section~\ref{sec:Aatari} in the appendix.
Figure~\ref{fig:bowlingExample}
displays how RUDDER redistributes rewards to key events 
in Bowling.
{\bf At delayed reward Atari games,
RUDDER considerably increases the scores compared to the PPO baseline.}

\begin{figure}[htp]
\begin{center}
\includegraphics[angle=0,width=0.7\textwidth]{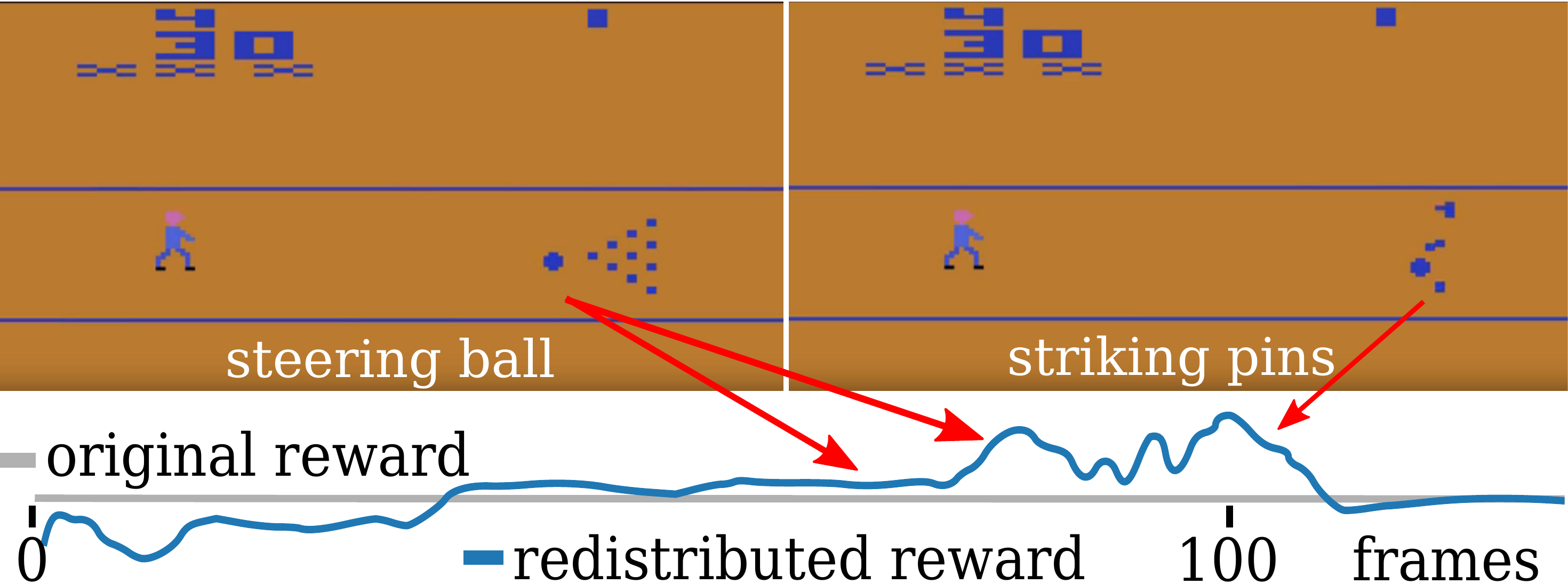} 
\caption{RUDDER redistributes rewards to key events in the Atari game Bowling.
Originally, rewards are delayed and only given at episode end.
The first 120 out of 200 frames of the episode are shown.
RUDDER identifies key actions that steer the ball to hit all pins.
\label{fig:bowlingExample}}
\end{center}
\end{figure}





{\bf Conclusion.} 
We have introduced RUDDER, a novel reinforcement learning 
algorithm based on the new concepts of reward redistribution
and return decomposition.
On artificial tasks,
RUDDER significantly outperforms TD($\lambda$), MC, MCTS and reward shaping methods,
while on Atari games it improves a PPO baseline on average but 
most prominently on long delayed rewards games.

\subsubsection*{Acknowledgments}
This work was supported by 
NVIDIA Corporation,
Merck KGaA, 
Audi.JKU Deep Learning Center, 
Audi Electronic Venture GmbH, 
Janssen Pharmaceutica (madeSMART), 
TGW Logistics Group,
ZF Friedrichshafen AG, UCB S.A.,
FFG grant 871302,
LIT grant DeepToxGen and AI-SNN,
and
FWF grant P 28660-N31.

\subsubsection*{References}
References are provided in Section~\ref{sec:Areferences} in the appendix.



\newpage

\begin{appendices}

\newlength{\auxparskip}
\setlength{\auxparskip}{\parskip} 
\setlength{\parskip}{0pt}
\tableofcontents

\newpage

\renewcommand{\thesection}{A\arabic{section}}
\renewcommand{\thefigure}{A\arabic{figure}}
\renewcommand{\thetable}{A\arabic{table}}
\renewcommand{\theequation}{A\arabic{equation}}

\setcounter{figure}{0}
\setcounter{table}{0}
\setcounter{equation}{0}

\setcounter{theorem}{0}
\setcounter{definition}{0}
\setcounter{corollary}{0}
\setcounter{proposition}{0}
\setcounter{lemma}{0}

\maketitle
\section{Definition of Finite Markov Decision Processes}
\label{sec:AMDP}

We consider a finite Markov decision process (MDP) $\cP$,
which is a 5-tuple $\cP=(\sS,\sA,\sR,p,\gamma )$:
\begin{itemize}
\item $\sS$ is a finite set of states;
  $S_t$ is the random variable
  for states at time $t$ with value
  $s \in \sS$. $S_t$ has a discrete probability distribution.
\item    $\sA$  is a finite set of actions (sometimes state-dependent
  $\sA(s)$);  $A_t$ is the random variable
  for actions at time $t$ with value
  $a \in \sA$. $A_t$ has a discrete probability distribution.
\item $\sR$  is a finite set of rewards; $R_{t+1}$ is the random variable
  for rewards at time $(t+1)$ with value
  $r \in \sR$. $R_t$ has a discrete probability distribution.
\item $p(S_{t+1}=s',R_{t+1}=r \mid S_t=s,A_t=a)$ are the
  transition and reward distributions over states and rewards, respectively, conditioned on
  state-actions,
\item  $\gamma \in [0, 1]$ is a discount factor for the reward.
\end{itemize}
The Markov policy $\pi$ is a distribution over
actions given the state: $\pi(A_t=a \mid S_t=s)$.
We often equip an MDP $\cP$ with a policy $\pi$ 
without explicitly mentioning it. 
At time $t$, the random variables give the states, actions, and
rewards of the MDP, while low-case letters give possible values.
At each time $t$, the environment is in some state $s_t \in \sS$. The policy
$\pi$ takes an action $a_t \in \sA$, which causes a transition of
the environment to state $s_{t+1}$ and a reward $r_{t+1}$ for the
policy.
Therefore, the MDP creates a sequence
\begin{align}
\left(S_0,A_0,R_1,S_1,A_1,R_2,S_2,A_2,R_3,\ldots \right) \ .
\end{align}
The marginal probabilities for 
\begin{align}
p(s',r\mid s,a) \ &= \ \PR\left[S_{t+1}=s',R_{t+1}=r \mid S_t=s,A_t=a
  \right] \   
\end{align}
are:
\begin{align}
  p(r\mid s,a) \ &= \ \PR\left[R_{t+1}=r \mid S_t=s,A_t=a
  \right] \ = \ \sum_{s'}p(s',r\mid s,a) \ , \\
  p(s'\mid s,a) \ &= \ \PR\left[S_{t+1}=s' \mid S_t=s,A_t=a
  \right] \ = \ \sum_{r}p(s',r\mid s,a) \ .
\end{align}

We use a sum convention: $\sum_{a,b}$ goes over all possible values of
$a$ and $b$, that is, all combinations which fulfill the constraints
on $a$ and $b$. If $b$ is a function of $a$ (fully determined by $a$),
then $\sum_{a,b}=\sum_{a}$.

We denote expectations:
\begin{itemize}
\item $\EXP_{\pi}$ is the expectation where
 the random variable is an MDP sequence of states, actions, and rewards
 generated with policy $\pi$.
\item $\EXP_{s}$ is the expectation where
 the random variable is $S_t$ with values $s \in \sS$.
\item $\EXP_{a}$ is the expectation where
 the random variable is $A_t$ with values $a \in \sA$.
\item $\EXP_{r}$ is the expectation where
 the random variable is $R_{t+1}$ with values $r \in \sR$.
\item $\EXP_{s,a,r,s',a'}$ is the expectation where
  the random variables are $S_{t+1}$ with values $s' \in \sS$,
  $S_t$ with values $s \in \sS$, $A_t$ with values $a \in \sA$, 
  $A_{t+1}$ with values $a' \in \sA$,  and $R_{t+1}$ with values $r \in \sR$.
  If more or fewer random variables are used, the notation is
  consistently adapted. 
\end{itemize}

The return $G_t$ is the accumulated reward starting from $t+1$:
\begin{align}
  G_t \ &= \ \sum_{k=0}^{\infty}  \gamma^k \ R_{t+k+1} \ .
\end{align}
The discount factor $\gamma$ determines how much immediate rewards are
favored over more delayed rewards.
For $\gamma=0$ the return (the objective) is determined
as the largest expected immediate reward, while
for $\gamma=1$ the return is determined by the
expected sum of future rewards if the sum exists.

\paragraph{State-Value and Action-Value Function.}
The state-value function $v^{\pi}(s)$ for policy $\pi$ and state $s$
is defined as
\begin{align}
  v^{\pi}(s) \ &= \ \EXP_{\pi} \left[ G_t \mid S_t=s \right]
   \ = \ \EXP_{\pi} \left[\sum_{k=0}^{\infty}  \gamma^k
     \ R_{t+k+1} \mid S_t=s \right] \ .  
\end{align}
Starting at $t=0$:
\begin{align}
  v_0^{\pi} \ &= \ \EXP_{\pi} \left[\sum_{t=0}^{\infty} \gamma^{t}
    \ R_{t+1}\right] \ = \ \EXP_{\pi} \left[G_0\right] \ ,
\end{align}
the optimal state-value function $v_{*}$ and policy $\pi_{*}$ are
\begin{align}
  v_{*}(s) \ &= \ \max_{\pi} v^{\pi}(s) \ , \\
 \pi_{*} \ &= \ \arg\max_{\pi} v^{\pi}(s) \ \text{ for all } s \ .
\end{align}

The action-value function $q^{\pi}(s,a)$ for policy $\pi$ is the
expected return when starting from $S_t=s$, taking action $A_t=a$,
and following policy $\pi$:
\begin{align}
  q^{\pi}(s,a) \ &= \ \EXP_{\pi} \left[ G_t \mid S_t=s, A_t=a \right]
   \ = \ \EXP_{\pi} \left[\sum_{k=0}^{\infty}  \gamma^k
   \ R_{t+k+1}  \mid S_t=s, A_t=a \right] \ .  
\end{align}
The optimal action-value function $q_{*}$ and policy $\pi_{*}$ are
\begin{align}
  q_{*}(s,a) \ &= \ \max_{\pi} q^{\pi}(s,a) \ , \\
  \pi_{*} \ &= \ \arg\max_{\pi} q^{\pi}(s,a) \ \text{ for all } (s,a) \ .
\end{align}
The optimal action-value function $q_{*}$ can be expressed via the optimal value
function  $v_{*}$:
\begin{align}
  q_{*}(s,a) \ &= \ \EXP \left[ R_{t+1} \ + \ \gamma \ v_{*}(S_{t+1})
   \mid S_t=s,A_t=a\right] \ .  
\end{align}
The optimal state-value function $v_{*}$
can be expressed via the optimal action-value
function  $q_{*}$ using the optimal policy $\pi_{*}$: 
\begin{align}
  v_{*}(s) \ &= \ \max_{a} q^{\pi_{*}}(s,a) \ = \ 
   \max_{a} \EXP_{\pi_{*}} \left[ G_t \mid S_t=s,A_t=a\right] \ = \\ \nonumber
   &\max_{a} \EXP_{\pi_{*}} \left[ R_{t+1} \ + \
       \gamma \ G_{t+1}\mid S_t=s,A_t=a\right]  \ = \\ \nonumber
   &\max_{a} \EXP \left[ R_{t+1} \ + \ \gamma \ v_{*}(S_{t+1}) \mid S_t=s,A_t=a\right] \ .  
\end{align}

\paragraph{Finite time horizon and no discount.}
We consider a {\bf finite} time horizon, that is, we consider only
episodes of length $T$, but may receive reward $R_{T+1}$ at episode end at time $T+1$.
The finite time horizon MDP creates a sequence
\begin{align}
\left(S_0,A_0,R_1,S_1,A_1,R_2,S_2,A_2,R_3,\ldots,S_{T-1},A_{T-1},R_T,S_T,A_T,R_{T+1}\right) \ .
\end{align}
Furthermore, we do not discount future rewards, that is, we set $\gamma=1$.
The return $G_t$ from time $t$ to $T$ is the sum of rewards:
\begin{align}
    G_t \ &= \   \sum_{k=0}^{T-t}  R_{t+k+1} \ .
\end{align}

The state-value function $v$ for policy $\pi$ is
\begin{align}
  v^{\pi}(s) \ &= \ \EXP_{\pi} \left[ G_t \mid S_t=s \right]
  \ = \ \EXP_{\pi} \left[\sum_{k=0}^{T-t}  R_{t+k+1}
  \mid S_t=s \right] 
\end{align}
and the action-value function $q$ for policy $\pi$ is
\begin{align}
  \label{eq:bellmanA}
  q^{\pi}(s,a) \ &= \ \EXP_{\pi} \left[ G_t \mid S_t=s, A_t=a \right]
  \ = \ \EXP_{\pi} \left[\sum_{k=0}^{T-t}   R_{t+k+1}
    \mid S_t=s, A_t=a \right] \\  \nonumber &= \
  \EXP_{\pi} \left[R_{t+1} \ + \   G_{t+1} \mid S_t=s, A_t=a  \right]
  \\ \nonumber &= \
  \sum_{s',r} p(s',r \mid s,a) \left[r \ + \  \sum_{a'}
    \pi(a'\mid s') \ q^{\pi}(s',a')  \right] \ .  
\end{align}

From the Bellman equation Eq.~\eqref{eq:bellmanA}, we obtain:
\begin{align}
  \label{eq:bellmanExp}
  \sum_{s'} p(s'\mid s,a) \ \sum_{a'} \pi(a' \mid s') \ q^\pi(s',a') \ &= \ 
  q^\pi(s,a) \ - \ \sum_{r} r \ p(r \mid s,a)  \ , \\
  \EXP_{s',a'} \left[ q^\pi(s',a') \mid s,a \right] \
  &= \  q^\pi(s,a) \ - \   r(s,a)  \ .
\end{align} 

The expected return at time $t=0$ for policy $\pi$ is
\begin{align}
  v_0^{\pi} \ &= \   \EXP_{\pi} \left[G_0\right]  \ = \ \EXP_{\pi}
  \left[\sum_{t=0}^{T} R_{t+1}\right] \ , \\ \nonumber
  \pi^{*} \ &= \ \underset{\pi}{\argmax} \ v_0^{\pi}  \ .
\end{align}
The agent may start in a particular starting state $S_0$ which is
a random variable. Often $S_0$ has only one value $s_0$.

\paragraph{Learning.}
The {\bf goal} of learning is to find the policy $\pi^{*}$ that
maximizes the expected future discounted reward (the return)
if starting at $t=0$. Thus, the optimal policy $\pi^{*}$ is
\begin{align}
\pi^{*} \ &= \ \underset{\pi}{\argmax} \ v_0^{\pi}  \ .
\end{align}
We consider two learning approaches for $Q$-values: Monte Carlo and
temporal difference. 

\paragraph{Monte Carlo (MC).}
To estimate $q^\pi(s,a)$, MC computes the arithmetic mean of all observed
returns $(G_t \mid S_t=s,A_t=a)$ in the data.
When using Monte Carlo for learning a policy we use an 
exponentially weighted arithmetic mean since the policy steadily changes.

For the $i$th update Monte Carlo tries to minimize $\frac{1}{2}M(s_t,a_t)^2$ with the residual $M(s_t,a_t)$
\begin{align}
\label{eq:mc-error}
  M(s_t,a_t) \ &= \ (q^\pi)^i(s_t,a_t) \ - \
   \sum_{\tau=0}^{T-t-1} \gamma^{\tau}  r_{t+1+\tau} \ ,
\end{align}
such that the update of the action-value $q$ at state-action $(s_t,a_t)$ is
\begin{align}
\label{eq:AMCpolicy}
 (q^\pi)^{i+1}(s_t,a_t) \ &= \ (q^\pi)^i(s_t,a_t) \ - \
 \alpha \ M(s_t,a_t) \ .
\end{align} 
This update is called {\em constant-$\alpha$ MC} \cite{Sutton:18book}.
 
\paragraph{Temporal difference (TD) methods.}
TD updates are based on the Bellman equation.
If $r(s,a)$ and
$\EXP_{s',a'} \left[\hat{q}^\pi(s', a') \mid s,a\right]$ have been estimated,
the $Q$-values can be updated according to the Bellman equation:
\begin{align}
  \left(\hat{q}^\pi\right)^{\nn}(s,a) \ &= \
  r(s,a) \ + \ \gamma \
  \EXP_{s',a'} \left[\hat{q}^\pi(s', a') \mid s,a\right]\ . 
\end{align}
The update is applying the Bellman operator with estimates
$\EXP_{s',a'} \left[\hat{q}^\pi(s', a') \mid s,a\right]$ and
$r(s,a)$ to $\hat{q}^\pi$ to obtain
$\left(\hat{q}^\pi\right)^{\nn}$.
The new estimate $\left(\hat{q}^\pi\right)^{\nn}$ is closer to
the fixed point $q^\pi$ of the Bellman operator,
since the Bellman operator is a contraction
(see Section~\ref{sec:ApropPolyCon2}
and Section~\ref{sec:ApropPolyCon}).

Since the estimates $\EXP_{s',a'} \left[\hat{q}^\pi(s', a') \mid
  s,a\right]$ and $r(s,a)$ are not known,
TD methods try to minimize $\frac{1}{2}B(s,a)^2$ with the Bellman residual $B(s,a)$:
\begin{align}
\label{eq:bellman-error}
  B(s,a) \ &= \ \hat{q}^\pi(s,a) \ - \ r(s,a) \ - \
  \gamma \ \EXP_{s',a'} \left[\hat{q}^\pi(s', a')\right] \ .
\end{align}
TD methods use an estimate $\hat{B}(s,a)$ of $B(s,a)$ and a learning
rate $\alpha$ to make an update
\begin{align}
  \hat{q}^\pi(s,a)^\nn \ \leftarrow \  \hat{q}^\pi(s,a) \ - \ \alpha \
  \hat{B}(s,a) \ .
\end{align}
For all TD methods $r(s,a)$ is estimated by $R_{t+1}$ and $s'$ by
$S_{t+1}$, while
$\hat{q}^\pi(s', a')$ does not change with the
current sample, that is, it is fixed for the estimate.
However, the sample determines which $(s',a')$ is chosen.
The TD methods differ in how they select $a'$.
{\bf SARSA}~\cite{Rummery:94} selects $a'$ by sampling from the policy:
\begin{align} \nonumber
  \EXP_{s',a'} \left[\hat{q}^\pi(s', a')\right] \ &\approx \
  \hat{q}^\pi(S_{t+1}, A_{t+1})
\end{align}
and {\bf expected SARSA}~\cite{John:94} averages over selections
\begin{align} \nonumber
  \EXP_{s',a'} \left[\hat{q}^\pi(s', a')\right] \ &\approx \
  \sum_{a} \pi(a \mid S_{t+1}) \ \hat{q}^\pi(S_{t+1}, a).
\end{align}
It is possible to estimate $r(s,a)$ separately via an unbiased 
minimal variance estimator
like the arithmetic mean and then 
perform TD updates with the Bellman error
using the estimated $r(s,a)$ \cite{Romoff:18}.
{\bf $\BQ$-learning} \cite{Watkins:89} is  an
off-policy TD algorithm which is proved
to converge \cite{Watkins:92,Dayan:92}. The proofs were later generalized 
\cite{Jaakkola:94,Tsitsiklis:94}.
$Q$-learning uses
\begin{align}
  \EXP_{s',a'} \left[\hat{q}^\pi(s', a')\right] \ &\approx \
  \max_{a}  \hat{q}(S_{t+1},a) \ .
\end{align}
The action-value function $q$, which is learned by $Q$-learning, approximates $q_{*}$
independently of the policy that is followed. More precisely,
with $Q$-learning $q$ converges with probability 1 to the optimal $q_{*}$.
However, the policy still
determines which state-action pairs are encountered during learning.
The convergence only requires that all action-state pairs are visited and
updated infinitely often.

\section{Reward Redistribution, Return-Equivalent SDPs, Novel Learning Algorithms, 
         and Return Decomposition}
\label{c:RR}

\subsection{State Enriched MDPs}
\label{sec:AStateEnriched}

For MDPs with a delayed reward the
states have to code the reward. 
However, for an immediate reward the
states can be made more compact by removing the reward information.
For example, states with memory of a delayed reward can be mapped
to states without memory.
Therefore, in order to compare MDPs, we introduce the concept of homomorphic MDPs.
We first need to define a partition of a set induced by a function.
Let $B$ be a partition of a set $X$. For any $x \in X$,
we denote $[x]_B$ the block of $B$ to which $x$ belongs. 
Any function $f$ from a set $X$ to a set $Y$ 
induces a partition (or equivalence relation) on $X$, 
with $[x]_f = [x']_f$ if and only if $f(x) = f(x')$.
We now can define homomorphic MDPs.
\begin{definitionA}[Ravindran and Barto \cite{Ravindran:01,Ravindran:03}]
An MDP homomorphism $h$ from
an MDP $\cP=(\sS,\sA,\sR,p,\gamma)$
to an MDP
$\tilde{\cP}=(\tilde{\sS},\tilde{\sA},\tilde{\sR},\tilde{p},\tilde{\gamma})$
is a a tuple of surjections $(f,g_1,g_2,\ldots,g_n)$ ($n$ is number of
states), with $h(s,a)=(f(s),g_s(a))$,
where $f: \sS \to \tilde{\sS}$ and $g_s: \sA_s \to
\tilde{\sA}_{f(s)}$ for $s \in \sS$ ($\sA_s$ are the admissible actions in state $s$
and $\tilde{\sA}_{f(s)}$ are the admissible actions in state $\tilde{s}$).
Furthermore, for all $s,s' \in \sS, a \in \sA_s$:
\begin{align}
\tilde{p}(f(s') \mid f(s),g_s(a)) \ &= \ \sum_{s'' \in [s']_f} p(s''
\mid s,a) \ , \\
\tilde{p}(\tilde{r} \mid f(s),g_s(a)) \ &= \ p(r \mid s,a) \ .
\end{align}
We use $[s]_f=[s']_f$ if and only if $f(s)=f(s')$.
\end{definitionA}

We call
$\tilde{\cP}$ the {\em homomorphic image} of $\cP$ under $h$.
For homomorphic images the optimal $Q$-values and the optimal 
policies are the same. 
\begin{lemmaA}[Ravindran and Barto \cite{Ravindran:01}]
\label{th:Arav}
If $\tilde{\cP}$ is a homomorphic image of $\cP$, then
the optimal $Q$-values are the same and
a policy that is optimal in $\tilde{\cP}$ can be transformed to
an optimal policy in $\cP$ by normalizing the number of actions $a$
that are mapped to the same action $\tilde{a}$.
\end{lemmaA}
Consequently, the original MDP
can be solved by solving a homomorphic image.

Similar results have been obtained by
Givan~et~al.\ using stochastically bisimilar MDPs:
``Any stochastic bisimulation used for aggregation preserves the
optimal value and action sequence properties as well as the optimal
policies of the model'' \cite{Givan:03}.
Theorem 7 and Corollary 9.1 in Givan~et~al.\ show the facts of
Lemma~\ref{th:Arav}. 
Li~et~al.\ give an overview over state abstraction and state aggregation for
Markov decision processes, which covers homomorphic MDPs \cite{Li:06}.

A Markov decision process $\tilde{\cP}$ is state-enriched compared to
an MDP $\cP$ if $\tilde{\cP}$ has the same states, actions, transition
probabilities, and reward probabilities as
$\cP$ but with additional information in its states.
We define state-enrichment as follows:
\begin{definitionA}
  A Markov decision process $\tilde{\cP}$ is {\em state-enriched} compared to
  a Markov decision process $\cP$ if $\cP$ is a homomorphic image of
  $\tilde{\cP}$, where $g_{\tilde{s}}$ is the identity and 
  $f(\tilde{s})=s$ is not bijective. 
\end{definitionA}
Being not bijective means that there exist $\tilde{s}'$
and $\tilde{s}''$ with $f(\tilde{s}')=f(\tilde{s}'')$, that is, 
$\tilde{\sS}$ has more elements than $\sS$.
In particular, state-enrichment does not change the optimal policies nor
the $Q$-values in the sense of Lemma~\ref{th:Arav}.
\begin{propositionA}
\label{th:Aenrich}
If an MDP $\tilde{\cP}$ is {\em state-enriched} compared to
an MDP $\cP$, then both MDPs have the same 
optimal $Q$-values and the same optimal policies. 
\end{propositionA}
\begin{proof}
According to the definition $\cP$ is a homomorphic image of
$\tilde{\cP}$. The statements of Proposition~\ref{th:Aenrich} 
follow directly from Lemma~\ref{th:Arav}.
\end{proof}
Optimal policies of the 
state-enriched MDP $\tilde{\cP}$
can be transformed to optimal policies of the original MDP $\cP$ 
and, vice versa, 
each optimal policy of the original MDP $\cP$ 
corresponds to at least one optimal
policy of the state-enriched MDP $\tilde{\cP}$.

\subsection{Return-Equivalent Sequence-Markov Decision Processes (SDPs)}
\label{sec:AReturnEquivalent}

Our goal is to compare Markov decision processes (MDPs)
with delayed rewards to decision processes (DPs) without delayed rewards.
The DPs without delayed rewards can but need not to be Markov in the rewards.
Toward this end, we consider two DPs $\tilde{\cP}$ and $\cP$ which
differ only in their (non-Markov) reward distributions. 
However for each policy $\pi$ the DPs $\tilde{\cP}$ and $\cP$
have the same expected return at $t=0$, 
that is, $\tilde{v}^{\pi}_0=v^{\pi}_0$, 
or they have the same expected return for every episode.

\subsubsection{Sequence-Markov Decision Processes (SDPs)}
\label{sec:Asdps}

We first define decision processes that
are Markov except for the reward, which is not required to be Markov.
\begin{definitionA}
  \label{def:Asdp}
A sequence-Markov decision process (SDP) is defined as a finite decision process which is equipped with 
a Markov policy and has Markov transition probabilities 
but a reward distribution that is not required to be Markov.
\end{definitionA}

\begin{propositionA}
\label{th:AsdpMDP}
  Markov decision processes are sequence-Markov decision processes.
\end{propositionA}
\begin{proof}
MDPs have Markov transition probabilities
and are equipped with Markov policies.
\end{proof}

\begin{definitionA}
  \label{def:AseqEqi}
We call two sequence-Markov decision processes $\cP$ and $\tilde{\cP}$  
that have the same Markov transition probabilities
and are equipped with the same Markov policy {\em sequence-equivalent}.
\end{definitionA}

\begin{lemmaA}
\label{th:AsdpSeq}
Two sequence-Markov decision processes that are 
sequence-equivalent have the same probability to generate 
state-action sequences $(s_0,a_0,\ldots,s_t,a_t)$, $0 \leq t \leq T$. 
\end{lemmaA}
\begin{proof}
Sequence generation only depends on transition probabilities
and policy. Therefore the probability of generating a particular
sequences is the same for both SDPs.
\end{proof}

\subsubsection{Return-Equivalent SDPs}
\label{sec:AEquiSDPs}

We define return-equivalent SDPs which can be shown to
have the same optimal policies.
\begin{definitionA}
  \label{def:AreturnEqSDP}
  Two sequence-Markov decision processes $\tilde{\cP}$ and $\cP$
  are {\em return-equivalent} if
  they differ only in their reward but
  for each policy $\pi$ have the same expected return
  $\tilde{v}^{\pi}_0=v^{\pi}_0$.
  $\tilde{\cP}$ and $\cP$ 
  are {\em strictly return-equivalent} if
  they have the same expected return for every episode and
  for each policy $\pi$:
  \begin{align}
   \EXP_{\pi} \left[
    \tilde{G}_0 \mid s_0,a_0,\ldots,s_T,a_T \right] \ &= \
   \EXP_{\pi} \left[
    G_0 \mid s_0,a_0,\ldots,s_T, a_T\right] \ .
  \end{align}
\end{definitionA}
The definition of return-equivalence can be generalized to 
strictly monotonic functions $f$
for which $\tilde{v}^{\pi}_0=f(v^{\pi}_0)$. Since strictly monotonic functions
do not change the ordering of the returns, maximal returns stay
maximal after applying the function $f$.

Strictly return-equivalent SDPs are return-equivalent as the
next proposition states.
\begin{propositionA}
\label{th:AstrictToEqsdp}
  Strictly return-equivalent sequence-Markov decision processes
  are return-equivalent.
\end{propositionA}
\begin{proof}
 The expected return at $t=0$ given a policy
  is the sum of
  the probability of generating a sequence times the expected reward
  for this sequence. Both expectations are the same for 
  two strictly return-equivalent sequence-Markov decision processes.
 Therefore the expected return at time $t=0$ is the same.
\end{proof}

The next proposition states that return-equivalent SDPs  
have the same optimal policies.
\begin{propositionA}
\label{th:AsdpPol}
  Return-equivalent sequence-Markov decision processes 
  have the same optimal policies.
\end{propositionA}
\begin{proof}
  The optimal policy is defined as maximizing the expected return at
  time $t=0$. 
  For each policy the 
  expected return at
  time $t=0$ is the same for return-equivalent decision processes.
  Consequently, the optimal policies are the same.
\end{proof}

Two strictly return-equivalent SDPs have the same 
expected return for each state-action sub-sequence
$(s_0,a_0,\ldots,s_t,a_t)$, $0 \leq t \leq T$.
\begin{lemmaA}
\label{th:AreturnSub}
Two strictly return-equivalent SDPs $\tilde{\cP}$ and $\cP$
have the same expected return for each
state-action sub-sequence
$(s_0,a_0,\ldots,s_t,a_t)$, $0 \leq t \leq T$:
\begin{align}
   \EXP_{\pi} \left[
    \tilde{G}_0 \mid s_0,a_0,\ldots,s_t,a_t \right] \ &= \
   \EXP_{\pi} \left[
    G_0 \mid s_0,a_0,\ldots,s_t, a_t\right] \ .
  \end{align}
\end{lemmaA}
\begin{proof}
Since the SDPs are strictly return-equivalent, we have
\begin{align}
   &\EXP_{\pi} \left[
    \tilde{G}_0 \mid s_0,a_0,\ldots,s_t,a_t \right] \\ \nonumber
    &= \ \sum_{s_{t+1},a_{t+1},\ldots,s_T,a_T} 
    p_{\pi} (s_{t+1},a_{t+1},\ldots,s_T,a_T \mid s_t,a_t) \
    \EXP_{\pi} \left[
    \tilde{G}_0 \mid s_0,a_0,\ldots,s_T,a_T \right] \\ \nonumber
    &= \ \sum_{s_{t+1},a_{t+1},\ldots,s_T,a_T} 
    p_{\pi} (s_{t+1},a_{t+1},\ldots,s_T,a_T \mid s_t,a_t) \
    \EXP_{\pi} \left[
    G_0 \mid s_0,a_0,\ldots,s_T,a_T \right] \\ \nonumber
    &= \    \EXP_{\pi} \left[
    G_0 \mid s_0,a_0,\ldots,s_t, a_t\right] \ .
  \end{align}
We used the marginalization of the full probability and
the Markov property of the state-action sequence.
\end{proof}

We now give the analog definitions and
results for MDPs which are SDPs.
\begin{definitionA}
  \label{def:AreturnEq}
  Two Markov decision processes $\tilde{\cP}$ and $\cP$
  are {\em return-equivalent} if
  they differ only in $p(\tilde{r} \mid s,a)$ and $p(r \mid s,a)$ but
  have the same expected return
  $\tilde{v}^{\pi}_0=v^{\pi}_0$ for each policy $\pi$.
  $\tilde{\cP}$ and $\cP$ 
  are {\em strictly return-equivalent} if
  they have the same expected return for every episode and
  for each policy $\pi$:
  \begin{align}
   \EXP_{\pi} \left[
    \tilde{G}_0 \mid s_0,a_0,\ldots,s_T,a_T \right] \ &= \
   \EXP_{\pi} \left[
    G_0 \mid s_0,a_0,\ldots,s_T, a_T\right] \ .
  \end{align}
\end{definitionA}

Strictly return-equivalent MDPs are return-equivalent as the
next proposition states.
\begin{propositionA}
\label{th:AstrictToEq}
  Strictly return-equivalent decision processes are return-equivalent.
\end{propositionA}
\begin{proof}
Since MDPs are SDPs, the proposition follows from Proposition~\ref{th:AstrictToEqsdp}.
\end{proof}

\begin{propositionA}
\label{th:AthRE}
  Return-equivalent Markov decision processes have the same optimal policies.
\end{propositionA}
\begin{proof}
Since MDPs are SDPs, the proposition follows from Proposition~\ref{th:AsdpPol}.
\end{proof}

For strictly return-equivalent MDPs the expected return is the same if 
a state-action sub-sequence is given. 
\begin{propositionA}
\label{th:AstrictSub}
Strictly return-equivalent MDPs $\tilde{\cP}$ and $\cP$
have the same expected return for a given state-action sub-sequence
$(s_0,a_0,\ldots,s_t,a_t)$, $0 \leq t \leq T$:
\begin{align}
   \EXP_{\pi} \left[
    \tilde{G}_0 \mid s_0,a_0,\ldots,s_t,a_t \right] \ &= \
   \EXP_{\pi} \left[
    G_0 \mid s_0,a_0,\ldots,s_t, a_t\right] \ .
  \end{align}
\end{propositionA}
\begin{proof}
Since MDPs are SDPs, the proposition follows from Lemma~\ref{th:AreturnSub}.
\end{proof}

\subsection{Reward Redistribution for Strictly Return-Equivalent SDPs}
\label{sec:rewardRedist}

Strictly return-equivalent SDPs $\tilde{\cP}$ and $\cP$
can be constructed by a reward redistribution.

\subsubsection{Reward Redistribution}

We define reward redistributions for SDPs.
\begin{definitionA}
A {\em reward redistribution} given an SDP $\tilde{\cP}$ 
is a fixed procedure that redistributes 
for each state-action sequence $s_0,a_0,\ldots,s_T,a_T$ the realization of 
the associated return variable 
$\tilde{G}_0 = \sum_{t=0}^{T}  \tilde{R}_{t+1}$ 
or its expectation $\EXP \left[\tilde{G}_0 \mid s_0,a_0,\ldots,s_T,a_T \right]$
along the sequence. 
The redistribution creates a new SDP $\cP$ with
redistributed reward $R_{t+1}$ at time $(t+1)$ and return variable
$G_0 = \sum_{t=0}^{T} R_{t+1}$.
The redistribution procedure 
ensures for each sequence either $\tilde{G}_0 = G_0$
or 
\begin{align}
   \EXP_{\pi} \left[
    \tilde{G}_0 \mid s_0,a_0,\ldots,s_T,a_T \right] \ &= \
   \EXP_{\pi} \left[
    G_0 \mid s_0,a_0,\ldots,s_T, a_T\right] \ .
\end{align}
\end{definitionA}

Reward redistributions can be very general.
A special case is if the return can be deduced from the past sequence, 
which makes the return causal. 
\begin{definitionA}
A reward redistribution is {\em causal} if for the redistributed
reward $R_{t+1}$ the following holds:
\begin{align}
   \EXP \left[ R_{t+1} \mid s_0,a_0,\ldots,s_T,a_T \right] \ &= \ 
   \EXP \left[ R_{t+1} \mid s_0,a_0,\ldots,s_t,a_t \right]\ .
\end{align}
\end{definitionA}

For our approach we only need reward redistributions that are
second order Markov.
\begin{definitionA}
A causal reward redistribution is {\em second order Markov} if
\begin{align}
   \EXP \left[ R_{t+1} \mid s_0,a_0,\ldots,s_t,a_t \right] \ &= \ 
   \EXP \left[ R_{t+1} \mid s_{t-1},a_{t-1},s_t,a_t \right] \ .
\end{align}
\end{definitionA}

\subsection{Reward Redistribution Constructs Strictly Return-Equivalent SDPs}

\begin{theoremA}
\label{th:AsdpEquiv}
If the SDP $\cP$ is obtained by reward redistribution from the 
SDP $\tilde{\cP}$, then $\tilde{\cP}$ and $\cP$ are strictly return-equivalent.
\end{theoremA}

\begin{proof}
For redistributing the reward we have for 
each state-action sequence $s_0,a_0,\ldots,s_T,a_T$ the same return
$\tilde{G}_0 = G_0$, therefore  
\begin{align}
   \EXP_{\pi} \left[
    \tilde{G}_0 \mid s_0,a_0,\ldots,s_T,a_T \right] \ &= \
   \EXP_{\pi} \left[
    G_0 \mid s_0,a_0,\ldots,s_T, a_T\right] \ .
\end{align}
For redistributing the expected return the last equation holds 
by definition.
The last equation is the definition of 
strictly return-equivalent SDPs.
\end{proof}

The next theorem states that the optimal policies are still the same 
when redistributing the reward.
\begin{theoremA}
\label{th:AEquivT}
If the SDP $\cP$ is obtained by reward redistribution from the 
SDP $\tilde{\cP}$, then both SDPs have the same optimal policies.
\end{theoremA}

\begin{proof}
\label{c:t1p}
According to Theorem~\ref{th:AsdpEquiv}, 
the SDP $\cP$ is strictly return-equivalent to the SDP $\tilde{\cP}$.
According to Proposition~\ref{th:AstrictToEqsdp} and Proposition~\ref{th:AsdpPol}
the SDP $\cP$ and the SDP $\tilde{\cP}$ have  
the same optimal policies. 
\end{proof}

\subsubsection{Special Cases of Strictly Return-Equivalent Decision Processes: 
  Reward Shaping, Look-Ahead Advice, and Look-Back Advice}
\label{sec:AshapingEquiv}

Redistributing the reward via
reward shaping \cite{Ng:99,Wiewiora:03}, look-ahead advice, and
look-back advice \cite{Wiewiora:03icml} is 
a special case of reward redistribution that leads to MDPs 
which are strictly return-equivalent to the original MDP.
We show that reward shaping is a special case of
reward redistributions that lead to MDPs 
which are strictly return-equivalent to the original MDP. 
First, we subtract from the potential the constant 
$c=(\Phi(s_0,a_0) - \gamma^T \Phi(s_T,a_T) )/ (1-\gamma^T)$,
which is the potential of the initial state minus the 
discounted potential in the last state divided by a fixed divisor. 
Consequently, the sum of additional rewards in reward shaping, 
look-ahead advice, or look-back advice from $1$ to $T$ is zero.
The original sum of additional rewards is
\begin{align}
 \sum_{i=1}^{T} \gamma^{i-1} \ (\gamma \Phi(s_i,a_i) \ - \  \Phi(s_{i-1},a_{i-1}) )
 \ &= \ \gamma^T \Phi(s_T,a_T) \ - \ \Phi(s_0,a_0) \ .
\end{align}
If we assume $\gamma^T \Phi(s_T,a_T)=0$ and $\Phi(s_0,a_0)=0$,
then reward shaping does not change the return and 
the shaping reward is a reward 
redistribution leading to an MDP that is 
strictly return-equivalent to the original MDP.
For $T\to \infty$ only
$\Phi(s_0,a_0)=0$ is required.
The assumptions can always be fulfilled by adding a single new initial
state and a single new final state to the original MDP.

Without the assumptions $\gamma^T \Phi(s_T,a_T)=0$ and $\Phi(s_0,a_0)=0$, 
we subtract $c=(\Phi(s_0,a_0) - \gamma^T \Phi(s_T,a_T) )/ (1-\gamma^T)$
from all
potentials $\Phi$, and obtain
\begin{align}
 \sum_{i=1}^{T} \gamma^{i-1} \ (\gamma (\Phi(s_i,a_i) \ - \ c) 
 \ - \  (\Phi(s_{i-1},a_{i-1}) \ - \ c ) )
 \ &= \ 0 \ .
\end{align}
Therefore, the potential-based shaping function (the additional reward)
added to the original reward does not change the return, which means
that the shaping reward is a reward 
redistribution that leads to an MDP that is 
strictly return-equivalent to the original MDP.
Obviously, reward shaping is a special case of reward
redistribution that leads to a strictly 
return-equivalent MDP.
Reward shaping does not change the general learning behavior if a constant $c$
is subtracted from the potential function $\Phi$. 
The $Q$-function of the original reward shaping and the $Q$-function of
the reward shaping, which has a constant $c$
subtracted from the potential function $\Phi$, differ by $c$
for every $Q$-value \cite{Ng:99,Wiewiora:03}.
For infinite horizon MDPs with $\gamma<1$, 
the terms $\gamma^T$ and $\gamma^T \Phi(s_T,a_T)$ 
vanish, therefore it is sufficient to subtract $c=\Phi(s_0,a_0)$ from the
potential function.

Since TD based reward shaping methods keep the original reward, 
they can still be exponentially slow for delayed rewards.
Reward shaping methods like reward shaping, look-ahead advice, and
look-back advice rely on the Markov property of the original reward, 
while an optimal reward redistribution is not Markov.
In general, reward shaping does not lead to 
an optimal reward redistribution according to 
Section~\ref{sec:Aopt_rew_red}.

As discussed in Paragraph~\ref{sec:Aremark},
the optimal reward redistribution does not 
comply to the Bellman equation.
Also look-ahead advice does not comply to the 
Bellman equation.
The return for the look-ahead advice reward $\tilde{R}_{t+1}$ is
\begin{align}
  G_t \ &= \  
  \sum_{i=0}^{\infty} \tilde{R}_{t+i+1} 
\end{align}
with expectations for the reward $\tilde{R}_{t+1}$
\begin{align}
   \EXP_{\pi} \left[ \tilde{R}_{t+1} \mid  s_{t+1},a_{t+1},s_t,a_t \right]
   \ &= \ \tilde{r}(s_{t+1},a_{t+1},s_t,a_t)  \ &= \ 
  \gamma \Phi(s_{t+1},a_{t+1}) \ - \  \Phi(s_t,a_t) \ .
\end{align}
The expected reward $\tilde{r}(s_{t+1},a_{t+1},s_t,a_t)$ depends on 
future states $s_{t+1}$ and, more importantly, on future actions $a_{t+1}$.
It is a non-causal reward redistribution.
Therefore look-ahead advice cannot be directly used for selecting the 
optimal action at time $t$.
For look-back advice we have
\begin{align}
   \EXP_{\pi} \left[ \tilde{R}_{t+1} \mid  s_t,a_t,s_{t-1},a_{t-1} \right]
   \ &= \ \tilde{r}(s_t,a_t,s_{t-1},a_{t-1})  \ &= \ 
  \Phi(s_t,a_t) \ - \ \gamma^{-1}  \Phi(s_{t-1},a_{t-1}) \ .
\end{align}
Therefore look-back advice introduces a 
second-order Markov reward like the optimal reward redistribution.

\subsection{Transforming an Immediate Reward MDP to a Delayed Reward MDP}
\label{sec:Aequiv}

We assume to have a Markov decision process $\cP$ with immediate reward. 
The MDP $\cP$ is transformed into an MDP $\tilde{\cP}$ with delayed
reward, where the reward is given at sequence end.
The reward-equivalent MDP $\tilde{\cP}$ with delayed reward 
is state-enriched,
which ensures that it is an MDP. 

The state-enriched MDP $\tilde{\cP}$ has
\begin{itemize}
\item reward:
\begin{align}
  \tilde{R}_t  \ &= \
  \begin{cases}
    0 \ , & \text{for } t \leq T  \\
    \sum_{k=0}^{T}  R_{k+1} \ , & \text{for } t = T+1 \ .
  \end{cases} 
\end{align}
 \item state:
\begin{align}
  \tilde{s}_t \ &= \ (s_t,\rho_t) \ , \\
  \rho_t \ &= \ \sum_{k=0}^{t-1}  r_{k+1} \ , \ \text{ with } R_{k+1}=r_{k+1}\ .
\end{align} 
\end{itemize}
Here we assume that $\rho$ can only take a finite number of values to assure
that the enriched states $\tilde{s}$ are finite. 
If the original reward was continuous, then $\rho$ can represent the accumulated reward 
with any desired precision if the sequence length is $T$ and the original reward was
bounded. We assume that $\rho$ is sufficiently precise to distinguish the 
optimal policies, which are deterministic, from sub-optimal deterministic policies. 
The random variable $R_{k+1}$ is distributed according to
$p(r \mid s_k,a_k)$.
We assume that the time $t$ is coded in $s$ in order to know when the
episode ends and reward is no longer received, otherwise we introduce
an additional state variable $\tau=t$ that codes the time.

\begin{propositionA}
\label{th:AreturnEqu}
If a Markov decision process $\cP$ with immediate reward is
transformed by above defined $\tilde{R}_t$ and $\tilde{s}_t$ 
to a Markov decision process $\tilde{\cP}$ with delayed
reward, where the reward is given at sequence end, then:
(I) the optimal policies do not change, and (II)
for $\tilde{\pi}(a \mid \tilde{s}) =   \pi(a \mid s)$
\begin{align}
 \tilde{q}^{\tilde{\pi}}(\tilde{s},a) \ &= \ q^\pi(s,a) \ + \
 \sum_{k=0}^{t-1}  r_{k+1} \ ,
\end{align}
for $\tilde{S}_t=\tilde{s}$, $S_t=s$, and $A_t=a$.
\end{propositionA}

\begin{proof}
For (I) we first perform an state-enrichment of $\cP$ by
$\tilde{s}_t =  (s_t,\rho_t)$ with $\rho_t=  \sum_{k=0}^{t-1}
r_{k+1}$ for $R_{k+1}=r_{k+1}$ leading to an intermediate MDP.
We assume that the finite-valued $\rho$ is 
sufficiently precise to distinguish the 
optimal policies, which are deterministic, 
from sub-optimal deterministic policies. 
Proposition~\ref{th:Aenrich} ensures that 
neither the optimal $Q$-values 
nor the optimal policies change between the original MDP $\cP$
and the intermediate MDP.
Next, we redistribute the original reward $R_{t+1}$ 
according to the redistributed reward $\tilde{R}_t$. 
The new MDP $\tilde{\cP}$ with state enrichment and reward redistribution
is strictly return-equivalent to 
the intermediate MDP with state enrichment but the original reward.
The new MDP $\tilde{\cP}$ is Markov since the enriched state ensures that
$\tilde{R}_{T+1}$ is Markov.
Proposition~\ref{th:AstrictToEq} and Proposition~\ref{th:AthRE} ensure that
the optimal policies are the same.

For (II) we show a proof without Bellman equation and a proof using
the Bellman equation.\\
{\bf Equivalence without Bellman equation.}
We have $\tilde{G}_0=G_0$.
The Markov property ensures that the future reward is independent of
the already received reward:
\begin{align}
  \EXP_{\pi} \left[ \sum_{k=t}^{T}  R_{k+1} \mid S_t=s, A_t=
  a, \rho=\sum_{k=0}^{t-1}  r_{k+1}\right] \ &= \ \EXP_{\pi}
  \left[ \sum_{k=t}^{T}  R_{k+1} \mid S_t=s, A_t=a \right] \ .
\end{align}
We assume $\tilde{\pi}(a \mid \tilde{s})=\pi(a \mid s)$.

We obtain
\begin{align}
 \tilde{q}^{\tilde{\pi}}(\tilde{s},a) \ &= \ \EXP_{\tilde{\pi}} \left[
 \tilde{G}_0 \mid \tilde{S}_t=\tilde{s}, A_t=a \right] \\ \nonumber
 &= \ \EXP_{\tilde{\pi}} \left[ \sum_{k=0}^{T}  R_{k+1}
 \mid S_t=s,  \rho=\sum_{k=0}^{t-1}  r_{k+1} , A_t= a \right] \\ \nonumber
 &= \ \EXP_{\tilde{\pi}} \left[ \sum_{k=t}^{T}  R_{k+1}
 \mid S_t=s, \rho=\sum_{k=0}^{t-1}  r_{k+1} , A_t= a \right] \ + \
  \sum_{k=0}^{t-1}   r_{k+1}\\ \nonumber
 &= \ \EXP_{\pi} \left[ \sum_{k=t}^{T}  R_{k+1}
 \mid S_t=s, A_t= a \right] \ + \ \sum_{k=0}^{t-1}  r_{k+1}\\ \nonumber
 &= \ q^{\pi}(s,a) \ + \ \sum_{k=0}^{t-1}  r_{k+1} \ .
\end{align} 
We used $\EXP_{\tilde{\pi}} =\EXP_{\pi}$, which is ensured
since reward probabilities, transition probabilities, and
the probability of choosing an action by the policy correspond to each
other in both settings.

Since the optimal policies do not change for
reward-equivalent and state-enriched processes,
we have
\begin{align}
 \tilde{q}^*(\tilde{s},a) \ &= \ q^*(s,a) \ + \ \sum_{k=0}^{t-1}  r_{k+1}\ .
\end{align}

{\bf Equivalence with Bellman equation.}
With $q^\pi(s,a)$ as optimal action-value function for the
original Markov decision process, we define a new Markov
decision process with action-state function $\tilde{q}^{\tilde{\pi}}$.
For $\tilde{S}_t=\tilde{s}$, $S_t=s$, and $A_t=a$ we have
\begin{align}
 \tilde{q}^{\tilde{\pi}}(\tilde{s},a) \ &:= \ q^\pi(s,a) \ + \
 \sum_{k=0}^{t-1}  r_{k+1} \ , \\
 \tilde{\pi}(a \mid \tilde{s}) \ &:= \  \pi(a \mid s) \ .
\end{align} 

Since $\tilde{s}'=(s',\rho')$, $\rho'=r+\rho$, and $\tilde{r}$ is
constant, the values $\tilde{S}_{t+1}=\tilde{s}'$ and
$\tilde{R}_{t+1}=\tilde{r}$
can be computed from
$R_{t+1}=r$, $\rho$, and $S_{t+1}=s'$. Therefore, we have 
\begin{align}
\tilde{p}(\tilde{s}',\tilde{r}\mid s,\rho,a)\ &= \ \tilde{p}(s',\rho',\tilde{r}\mid s,\rho,a)\ = \ p(s',r\mid s,a) \ .
\end{align}

For $t < T$, we have $\tilde{r}=0$ and $\rho'=r+\rho$, where we set
$r=r_{t+1}$:
\begin{align}
  \tilde{q}^{\tilde{\pi}}(\tilde{s},a) \ &= \ q^\pi(s,a) \ + \ \sum_{k=0}^{t-1} 
  r_{k+1} \\ \nonumber
  &= \ \sum_{s',r} p(s',r\mid s,a) \ 
  \left[r \ + \ \sum_{a'} \pi(a' \mid s') \ q^\pi(s',a')  \right]  \ +
  \
  \sum_{k=0}^{t-1}  r_{k+1} \\ \nonumber
  &= \ \sum_{s',\rho'} \tilde{p}(s',\rho', \tilde{r}\mid s,\rho,a) \ 
  \left[r \ + \ \sum_{a'} \pi(a' \mid s') \ q^\pi(s',a')  \right]
  \ + \ \sum_{k=0}^{t-1}  r_{k+1} \\ \nonumber
  &= \ \sum_{\tilde{s}', \tilde{r}} \tilde{p}(\tilde{s}', \tilde{r}\mid \tilde{s},a) \ 
  \left[ r \ + \ \sum_{a'} \pi(a' \mid s') \ q^\pi(s',a')   \ + \
    \sum_{k=0}^{t-1}  r_{k+1}\right]  \\ \nonumber
  &= \  \sum_{\tilde{s}', \tilde{r}} \tilde{p}(\tilde{s}', \tilde{r}\mid
  \tilde{s},a) \ 
  \left[\tilde{r} \ + \ \sum_{a'} \pi(a' \mid s') \ q^\pi(s',a')   \ + \
    \sum_{k=0}^{t}  r_{k+1} \right]  \\ \nonumber
  &= \ \sum_{\tilde{s}', \tilde{r}} \tilde{p}(\tilde{s}', \tilde{r}\mid
  \tilde{s},a) \ 
  \left[\tilde{r} \ + \ \sum_{a'} \tilde{\pi}(a' \mid \tilde{s}') \
  \tilde{q}^{\tilde{\pi}}(\tilde{s}',a')
  \right]  \ .
\end{align} 

For $t=T$ we have $\tilde{r}=\sum_{k=0}^{T}  r_{k+1}=\rho'$ and
$q^\pi(s',a')=0$ as well as $\tilde{q}^{\tilde{\pi}}(\tilde{s}',a')=0$.
Both $q$ and $\tilde{q}$ must be zero for $t\geq T$ since after time $t=T+1$
there is no more reward.
We obtain for $t=T$ and $r=r_{T+1}$:
\begin{align}
  \tilde{q}^{\tilde{\pi}}(\tilde{s},a) \ &= \ q^\pi(s,a) \ + \ \sum_{k=0}^{T-1} 
  r_{k+1} \\ \nonumber
  &= \ \sum_{s',r} p(s',r\mid s,a) \ 
  \left[r \ + \ \sum_{a'} \pi(a' \mid s') \ q^\pi(s',a')  \right]  \ + \ 
  \sum_{k=0}^{T-1}  r_{k+1} \\ \nonumber
  &= \ \sum_{s',\rho', r} \tilde{p}(s',\rho'\mid s,\rho,a) \ 
  \left[ r \ + \ \sum_{a'} \pi(a' \mid s') \ q^\pi(s',a')  \right]  \ + \ 
  \sum_{k=0}^{T-1}  r_{k+1} \\ \nonumber
  &= \ \sum_{s',\rho', r} \tilde{p}(s',\rho'\mid s,\rho,a) \ 
  \left[\sum_{k=0}^{T}  r_{k+1}  \ + \ \sum_{a'} \pi(a' \mid s') \ q^\pi(s',a')  \right]  \\ \nonumber
  &= \ \sum_{\tilde{s}', \rho'} \tilde{p}(\tilde{s}'\mid \tilde{s},a) \ 
  \left[\rho'  \ + \ \sum_{a'} \pi(a' \mid s') \ q^\pi(s',a')  \right]  \\ \nonumber
  &= \ \sum_{\tilde{s}', \rho'} \tilde{p}(\tilde{s}'\mid \tilde{s},a) \ 
  \left[\rho' \ + \ 0  \right]  \\ \nonumber
  &= \  \sum_{\tilde{s}', \tilde{r}} \tilde{p}(\tilde{s}'\mid \tilde{s},a) \ 
  \left[\tilde{r}  \ + \ \sum_{a'} \tilde{\pi}(a' \mid \tilde{s}') \
  \tilde{q}^{\tilde{\pi}}(\tilde{s}',a')
  \right] \ .
\end{align}

Since $\tilde{q}^{\tilde{\pi}}(\tilde{s},a)$ fulfills the Bellman 
equation, it is the action-value function for $\tilde{\pi}$.

\end{proof}

\subsection{Transforming an Delayed Reward MDP to an Immediate Reward SDP}
\label{sec:AdelayIm}

Next we consider the opposite direction, where the delayed reward
MDP $\tilde{\cP}$ is given and we want to find an immediate reward
SDP $\cP$ that is return-equivalent to $\tilde{\cP}$.
We assume an episodic reward for $\tilde{\cP}$, that is, 
reward is only given at sequence end.
The realization of final reward, that is the realization of the return,
$\tilde{r}_{T+1}$ is redistributed 
to previous time steps.
Instead of redistributing the realization $\tilde{r}_{T+1}$ 
of the random variable $\tilde{R}_{T+1}$, also its expectation
$\tilde{r}(s_T,a_T)=\EXP \left[ \tilde{R}_{T+1} \mid s_T,a_T \right]$ 
can be redistributed since $Q$-value estimation 
considers only the mean.  
We used the Markov property 
\begin{align}
  \EXP_{\pi} \left[ \tilde{G}_0 \mid s_0,a_0,\ldots,s_T,a_T \right] \ &= \
  \EXP_{\pi} \left[ \sum_{t=0}^{T} \tilde{R}_{t+1} \mid s_0,a_0,\ldots,s_T,a_T \right] \\ \nonumber
  &= \
  \EXP \left[ \tilde{R}_{T+1} \mid s_0,a_0,\ldots,s_T,a_T \right] \\ \nonumber
  &= \
  \EXP \left[ \tilde{R}_{T+1} \mid s_T,a_T \right] \ .
\end{align}
Redistributing the expectation reduces the variance of estimators since
the variance of the random variable is already factored out.

We assume a delayed reward MDP $\tilde{\cP}$ with reward
\begin{align}
  \tilde{R}_t  \ &= \
  \begin{cases}
    0 \ , & \text{for } t \leq T  \\
    \tilde{R}_{T+1} \ , & \text{for } t = T+1 \ ,
  \end{cases} 
\end{align}
where $\tilde{R}_t=0$ means that the random variable $\tilde{R}_t$ is always zero.
The expected reward at the last time step is
\begin{align}
 \tilde{r}(s_T,a_T) \ &= \ \EXP\left[\tilde{R}_{T+1} \mid s_T,a_T\right] \ ,
\end{align}
which is also the expected return.
Given a state-action sequence $(s_0,a_0,\ldots,s_T,a_T)$,
we want to redistribute either the realization $\tilde{r}_{T+1}$ of the
random variable $\tilde{R}_{T+1}$ or its expectation $\tilde{r}(s_T,a_T)$,

\subsubsection{Optimal Reward Redistribution}
\label{sec:Aopt_rew_red}

The main goal in this paper is to derive 
an SDP via reward redistribution 
that has zero expected future rewards.
Consequently the SDP has no delayed rewards.
To measure the amount of delayed rewards,
we define the expected sum of delayed rewards $\kappa(m,t-1)$.
\begin{definitionA}
For $1 \leq t\leq T$ and $0\leq m\leq T-t$,  
the expected sum of delayed rewards at time $(t-1)$ 
in the interval $[t+1,t+m+1]$ is defined as
\begin{align}
 \kappa(m,t-1) \ = \ \EXP_{\pi} \left[ \sum_{\tau=0}^{m} R_{t+1+\tau}  
  \mid s_{t-1}, a_{t-1} \right]  \ .
\end{align}
\end{definitionA}
The Bellman equation for $Q$-values becomes
\begin{align}
    q^\pi(s_t,a_t) \ &= \   r(s_t,a_t) \ + \ \kappa(T-t-1,t) \ ,
\end{align} 
where $\kappa(T-t-1,t)$ 
is the expected sum of future rewards until sequence end given $(s_t,a_t)$, 
that is, in the interval $[t+2,T+1]$. 
We aim to derive an MDP with $\kappa(T-t-1,t)=0$,
which gives $q^\pi(s_t,a_t) =  r(s_t,a_t)$.
In this case, 
learning the $Q$-values reduces to estimating the average immediate reward
$r(s_t,a_t) = \EXP \left[ R_{t+1} \mid s_t,a_t\right]$.
Hence, the reinforcement learning task reduces to computing 
the mean, e.g.\ the arithmetic mean, for each
state-action pair $(s_t,a_t)$.
Next, we define an optimal reward redistribution.
\begin{definitionA}
A reward redistribution is optimal,  
if  $\kappa(T-t-1,t) = 0$ for $0\leq t \leq T-1$.
\end{definitionA}

Next theorem states that in general an MDP with optimal reward
redistribution does not exist, which
is the reason why we will consider SDPs in the following.
\begin{theoremA}
\label{th:Aviolate}
In general, an optimal reward redistribution violates
the assumption that the reward distribution is Markov, 
therefore the Bellman equation does not hold.
\end{theoremA}

\begin{proof}
We assume an MDP $\tilde{\cP}$ with
$\tilde{r}(s_T,a_T)\not=0$ and which has policies
that lead to different expected returns at time $t=0$.
If all reward is given at time $t=0$, 
all policies have the same expected return at time $t=0$.
This violates our assumption, therefore not all reward can 
be given at $t=0$.
In vector and matrix notation the Bellman equation is
\begin{align}
  \Bq^\pi_t    \ &= \ \Br_t \ + \  \BP_{t\rightarrow t+1} \ \Bq^\pi_{t+1}  \ ,
\end{align}
where $\BP_{t\rightarrow t+1}$ is the row-stochastic matrix with
$p(s_{t+1}\mid s_t,a_t)\pi(a_{t+1} \mid s_{t+1})$ 
at positions $((s_t,a_t),(s_{t+1},a_{t+1}))$.
An optimal reward redistribution requires the 
expected future rewards to be zero:
\begin{align}
  \BP_{t\rightarrow t+1} \ \Bq^\pi_{t+1}  \ &= \ \BZe 
\end{align}
and, since optimality requires $\Bq^\pi_{t+1}=\Br_{t+1}$,
we have
\begin{align}
  \BP_{t\rightarrow t+1} \ \Br_{t+1}  \ &= \ \BZe \ , 
\end{align}
where $\Br_{t+1}$ is the vector 
with components $\tilde{r}(s_{t+1},a_{t+1})$.
Since (i) the MDPs are return-equivalent, 
(ii) $\tilde{r}(s_T,a_T)\not=0$, 
and (iii) not all reward is given at $t=0$,
an $(t+1)$ exists with $\Br_{t+1}\not=\BZe$.
We can construct an MDP $\tilde{\cP}$ which has
(a) at least as many state-action pairs $(s_t,a_t)$ 
as pairs $(s_{t+1},a_{t+1})$ and (b) the transition matrix
$\BP_{t\rightarrow t+1}$ has full rank.
$\BP_{t\rightarrow t+1}\Br_{t+1}=\BZe$ 
is now a contradiction to
$\Br_{t+1}\not=0$ and 
$\BP_{t\rightarrow t+1}$ has full rank.
Consequently, simultaneously ensuring Markov properties 
and ensuring zero future return
is in general not possible.
\end{proof}

For a particular $\pi$,
the next theorem states that an optimal reward redistribution,
that is $\kappa=0$, is equivalent to 
a redistributed reward which expectation is the difference of
consecutive $Q$-values of the original delayed reward.
The theorem states that an optimal reward redistribution exists but
we have to assume an SDP $\cP$ that has a
second order Markov reward redistribution.
\begin{theoremA}
\label{th:AzeroExp}
We assume a delayed reward MDP $\tilde{\cP}$, 
where the accumulated reward is given at sequence end.
An new SDP $\cP$ is obtained by a 
second order Markov reward redistribution,
which ensures that $\cP$ is return-equivalent to $\tilde{\cP}$.
For a specific $\pi$, the following two
statements are equivalent:
(I) ~~$\kappa(T-t-1,t) = 0$, i.e.\ the reward redistribution is optimal, 
\begin{flalign}
\label{eq:AdiffQ}
\text{(II)~~}  \EXP \left[ R_{t+1} \mid s_{t-1},a_{t-1},s_t,a_t \right] 
   \ &= \ \tilde{q}^\pi(s_t,a_t) \ - \
    \tilde{q}^\pi(s_{t-1},a_{t-1}) \ .&& 
\end{flalign}
 
Furthermore, an optimal reward redistribution
fulfills for $1 \leq t\leq T$ and $0\leq m\leq T-t$:
\begin{align}
\label{eq:AoptimalCon1}
  \kappa(m,t-1) \ &= \  0 \ . 
\end{align} 
\end{theoremA}

\begin{proof}
\label{c:t2p}
PART (I): we assume that the reward redistribution is optimal, that is,
\begin{align}
  \kappa(T-t-1,t) \ &= \  0 \ .
\end{align} 
The redistributed reward $R_{t+1}$ is second order Markov.
We abbreviate the expected $R_{t+1}$ by $h_t$:
 \begin{align}
    \EXP \left[ R_{t+1} \mid s_{t-1},a_{t-1},s_t,a_t \right] 
    \ &= \ h_t \ .
\end{align} 
  
The assumptions of Lemma~\ref{th:AreturnSub} hold for 
for the delayed reward MDP $\tilde{\cP}$ and
the redistributed reward SDP $\cP$. 
Therefore for a given state-action sub-sequence
$(s_0,a_0,\ldots,s_t,a_t)$, $0 \leq t \leq T$:
\begin{align}
   \EXP_{\pi} \left[
    \tilde{G}_0 \mid s_0,a_0,\ldots,s_t,a_t \right] \ &= \
   \EXP_{\pi} \left[
    G_0 \mid s_0,a_0,\ldots,s_t, a_t\right] 
  \end{align}
with
$G_0=\sum_{\tau=0}^{T}  R_{\tau+1}$ and $\tilde{G}_0=\tilde{R}_{T+1}$.
The Markov property of the MDP $\tilde{\cP}$
ensures that the future reward  from $t+1$ on is independent of
the past sub-sequence $s_0,a_0,\ldots,s_{t-1},a_{t-1}$:
\begin{align}
 \EXP_{\pi} \left[
 \sum_{\tau=0}^{T-t} \tilde{R}_{t+1+\tau} \mid s_t,a_t \right] \ &= \ 
 \EXP_{\pi} \left[
 \sum_{\tau=0}^{T-t} \tilde{R}_{t+1+\tau} \mid 
 s_0,a_0,\ldots,s_t,a_t \right] \ .
 \end{align}
The second order Markov property of the SDP $\cP$
ensures that the future reward from $t+2$ on is independent of
the past sub-sequence $s_0,a_0,\ldots,s_{t-1},a_{t-1}$:
\begin{align}
\EXP_{\pi} \left[
 \sum_{\tau=0}^{T-t-1} R_{t+2+\tau} \mid s_t,a_t \right] \ &= \ 
 \EXP_{\pi} \left[
 \sum_{\tau=0}^{T-t-1} R_{t+2+\tau} \mid s_0,a_0,\ldots,s_t,a_t \right] \ .
\end{align}

Using these properties we obtain
\begin{align}
 \tilde{q}^{\pi}(s_t,a_t) \ 
 &= \ \EXP_{\pi} \left[
 \sum_{\tau=0}^{T-t} \tilde{R}_{t+1+\tau} \mid s_t,a_t \right] \\ \nonumber
 &= \ \EXP_{\pi} \left[
 \sum_{\tau=0}^{T-t} \tilde{R}_{t+1+\tau} \mid s_0,a_0,\ldots,s_t,a_t \right] \\ \nonumber
 &= \ \EXP_{\pi} \left[
 \tilde{R}_{T+1} \mid s_0,a_0,\ldots,s_t,a_t \right] \\ \nonumber
 &= \ \EXP_{\pi} \left[
 \sum_{\tau=0}^{T} \tilde{R}_{\tau+1} \mid s_0,a_0,\ldots,s_t,a_t \right] \\ \nonumber
  &= \ \EXP_{\pi} \left[
 \tilde{G}_0 \mid s_0,a_0,\ldots,s_t,a_t \right] \\ \nonumber
 &= \ \EXP_{\pi} \left[
 G_0 \mid s_0,a_0,\ldots,s_t, a_t\right] \\ \nonumber
  &= \ \EXP_{\pi} \left[
 \sum_{\tau=0}^{T} R_{\tau+1}  \mid s_0,a_0,\ldots,s_t, a_t\right] \\ \nonumber
   &= \ \EXP_{\pi} \left[
 \sum_{\tau=0}^{T-t-1} R_{t+2+\tau}  \mid s_0,a_0,\ldots,s_t, a_t\right] \ + \ 
 \sum_{\tau=0}^{t}  h_{\tau} \\ \nonumber
   &= \ \EXP_{\pi} \left[
 \sum_{\tau=0}^{T-t-1} R_{t+2+\tau}  \mid s_t, a_t\right] \ + \ 
 \sum_{\tau=0}^{t}  h_{\tau} \\ \nonumber
  &= \ \kappa(T-t-1,t)  \ + \ 
 \sum_{\tau=0}^{t} h_{\tau} \\ \nonumber
  &= \ \sum_{\tau=0}^{t}  h_{\tau} \ .
\end{align}
We used
\begin{align}
 \kappa(T-t-1,t) \ = \ \EXP_{\pi} \left[ \sum_{\tau=0}^{T-t-1} R_{t+2+\tau}  
  \mid s_t, a_t \right]  \ = \ 0 \ .
\end{align}

It follows that 
\begin{align}
 &\EXP \left[ R_{t+1} \mid s_{t-1},a_{t-1},s_t,a_t \right] 
    \ = \ h_t   \\ \nonumber 
    &= \ \tilde{q}^\pi(s_t,a_t) \ - \ \tilde{q}^\pi(s_{t-1},a_{t-1}) \ .
\end{align} 

~~\newline
~~\newline

PART (II): we assume that
\begin{align}
 &\EXP \left[ R_{t+1} \mid s_{t-1},a_{t-1},s_t,a_t \right] 
    \ = \ h_t   \\ \nonumber 
    &= \ \tilde{q}^\pi(s_t,a_t) \ - \ \tilde{q}^\pi(s_{t-1},a_{t-1}) \ .
\end{align} 

The expectations
$\EXP_{\pi}\left[.\mid  s_{t-1},a_{t-1}\right]$
like
$\EXP_{\pi}\left[\tilde{R}_{T+1}\mid  s_{t-1},a_{t-1}\right]$
are expectations over all episodes starting in $(s_{t-1},a_{t-1})$
and ending in some $(s_T,a_T)$.

First, we consider $m=0$ and  $1 \leq t\leq T$, therefore
$\kappa(0,t-1)=\EXP_{\pi} \left[R_{t+1}  \mid s_{t-1}, a_{t-1} \right]$.
Since  $\tilde{r}(s_{t-1},a_{t-1})=0$ for $1 \leq t\leq T$, we have
\begin{align}
  \tilde{q}^\pi(s_{t-1},a_{t-1}) \ &= \  
  \tilde{r}(s_{t-1},a_{t-1})\ + \ \sum_{s_t,a_t} p(s_t,a_t \mid s_{t-1}, a_{t-1}) \ 
  \tilde{q}^\pi(s_t,a_t) \\ \nonumber
  &= \ 
  \sum_{s_t,a_t} p(s_t,a_t \mid s_{t-1}, a_{t-1}) \ \tilde{q}^\pi(s_t,a_t) \ .
\end{align} 
Using this equation we obtain for $1 \leq t\leq T$:
\begin{align}
  \kappa(0,t-1) \ &= \
  \EXP_{s_t,a_t,R_{t+1}} \left[ R_{t+1}  \mid s_{t-1}, a_{t-1} \right] \\ \nonumber
  &= \ \EXP_{s_t,a_t} \left[ \tilde{q}^\pi(s_t,a_t) \ - \  
  \tilde{q}^\pi(s_{t-1},a_{t-1})  \mid s_{t-1}, a_{t-1} \right] \\ \nonumber
  &= \ \sum_{s_t,a_t} p(s_t,a_t \mid s_{t-1}, a_{t-1}) \ 
  \left( \tilde{q}^\pi(s_t,a_t) \ - \  
  \tilde{q}^\pi(s_{t-1},a_{t-1}) \right) \\ \nonumber
  &= \ \tilde{q}^\pi(s_{t-1},a_{t-1}) \ - \ 
  \sum_{s_t,a_t} p(s_t,a_t \mid s_{t-1}, a_{t-1}) \  
  \tilde{q}^\pi(s_{t-1},a_{t-1}) \\ \nonumber
  &= \ \tilde{q}^\pi(s_{t-1},a_{t-1}) \ - \ \tilde{q}^\pi(s_{t-1},a_{t-1})  
  \ = \  0 \ .
\end{align} 

Next, we consider the expectation of $\sum_{\tau=0}^{m} R_{t+1+\tau}$
for $1 \leq t\leq T$ and $1\leq m\leq T-t$ (for $m>0$)
\begin{align}
  \kappa(m,t-1) \ &= \ 
  \EXP_{\pi} \left[ \sum_{\tau=0}^{m} R_{t+1+\tau}  
  \mid s_{t-1},a_{t-1}\right]\\ \nonumber
  &= \
  \EXP_{\pi} \left[\sum_{\tau=0}^{m}
  \left(\tilde{q}^\pi(s_{\tau+t},a_{\tau+t}) \ - \
    \tilde{q}^\pi(s_{\tau+t-1},a_{\tau+t-1}) \right)  \mid
  s_{t-1},a_{t-1}\right] \\ \nonumber
  &= \
  \EXP_{\pi} \left[\tilde{q}^\pi(s_{t+m},a_{t+m})  \ - \
    \tilde{q}^\pi(s_{t-1},a_{t-1})
    \mid s_{t-1},a_{t-1}\right]\\ \nonumber
   &= \ \EXP_{\pi}\left[\EXP_{\pi}\left[ \sum_{\tau=t+m}^{T}
       \tilde{R}_{\tau+1} \mid s_{t+m},a_{t+m} \right]  \mid 
  s_{t-1},a_{t-1}\right]  \\ \nonumber
  &- \
   \EXP_{\pi}\left[\EXP_{\pi}\left[ \sum_{\tau=t-1}^{T}
       \tilde{R}_{\tau+1}  \mid s_{t-1},a_{t-1}\right] \mid
       s_{t-1},a_{t-1}\right] \\ \nonumber
 &= \ \EXP_{\pi}\left[ \tilde{R}_{T+1} \mid 
  s_{t-1},a_{t-1}\right]   \ - \
   \EXP_{\pi}\left[ \tilde{R}_{T+1}  \mid s_{t-1},a_{t-1}\right] \\ \nonumber
  &= \ 0 \ .
\end{align}
We used that $\tilde{R}_{t+1}=0$ for $t <T$.

For $t=\tau+1$ and $m=T-t=T-\tau-1$ we have
\begin{align}
  \kappa(T-\tau-1,\tau) \ &= \ 0 \ ,
\end{align}
which characterizes an optimal reward redistribution.

\end{proof}
Thus, an SDP with an optimal reward redistribution 
has a expected future rewards that are zero.
Equation $\kappa(T-t-1,t)= 0$ means that the new SDP $\cP$
has no delayed rewards as shown in next corollary.
\begin{corollaryA}
\label{th:ApropDelay}
An SDP with an optimal reward redistribution
fulfills for $0\leq \tau \leq T-t-1$ 
\begin{align}
  \EXP_{\pi} \left[ R_{t+1+\tau}  \mid
  s_{t-1}, a_{t-1} \right] \ &= \ 0 \ .
\end{align}
The SDP has no delayed rewards since no state-action pair
can increase or decrease the expectation of a future reward.
\end{corollaryA}

\begin{proof}
For $\tau=0$ we use $\kappa(m,t-1) = 0$ from Theorem~\ref{th:AzeroExp}
with $m=0$:
\begin{align}
  \EXP_{\pi} \left[ R_{t+1}  \mid
  s_{t-1}, a_{t-1} \right] \ = \  \kappa(0,t-1) \ = \ 0 \ . 
\end{align}

For $\tau>0$,
we also use $\kappa(m,t-1) = 0$ from Theorem~\ref{th:AzeroExp}:
\begin{align}
  \EXP_{\pi} \left[ R_{t+1+\tau}  \mid
  s_{t-1}, a_{t-1} \right] \ &= \ 
   \EXP_{\pi} \left[ \sum_{k=0}^{\tau} R_{t+1+k} \ - \ 
    \sum_{k=0}^{\tau-1} R_{t+1+k} \mid
  s_{t-1}, a_{t-1} \right] \\ \nonumber
  &= \    \EXP_{\pi} \left[ \sum_{k=0}^{\tau} R_{t+1+k} \mid
  s_{t-1}, a_{t-1} \right] \ - \ 
    \EXP_{\pi} \left[ \sum_{k=0}^{\tau-1} R_{t+1+k} \mid
  s_{t-1}, a_{t-1} \right] \\ \nonumber
  &= \ \kappa(\tau,t-1) \ - \ \kappa(\tau-1,t-1) \ = \ 0 \ - \ 0 \ = \ 0 \ .
\end{align}

\end{proof}

A related approach is to ensure zero return by 
reward shaping if the exact value function is known \cite{Schulman:15}.

The next theorem states the major advantage of an
optimal reward redistribution:
$\tilde{q}^\pi(s_t,a_t)$ can be estimated with an offset that 
depends only on $s_t$ 
by estimating the expected immediate redistributed reward.
Thus, $Q$-value estimation becomes trivial and the
the advantage function of the MDP $\tilde{\cP}$ can be readily computed.
\begin{theoremA}
\label{th:AOptReturnDecomp}
If the reward redistribution is 
optimal, then the $Q$-values 
of the SDP $\cP$ are given by 
\begin{align}
\label{eq:Aqvalue}
   q^\pi(s_t,a_t) \ &= \   r(s_t,a_t) \ = \  
   \tilde{q}^\pi(s_t,a_t) \ - \ 
    \EXP_{s_{t-1},a_{t-1}} \left[ \tilde{q}^\pi(s_{t-1},a_{t-1}) \mid s_t \right] \\ \nonumber
    &= \ \tilde{q}^\pi(s_t,a_t) \ - \ \psi^\pi(s_t) \ .
\end{align} 
The SDP $\cP$ and the original MDP $\tilde{\cP}$ 
have the same advantage function.
Using a behavior policy 
$\breve{\pi}$ the expected immediate reward is
\begin{align}
\label{eq:Abehavior}
   \EXP_{\breve{\pi}} \left[ R_{t+1} \mid s_t,a_t \right] \ &= \
    \tilde{q}^\pi(s_t,a_t) \ - \ \psi^{\pi,\breve{\pi}}(s_t) \ .
  \end{align}
\end{theoremA}

\begin{proof}
The expected reward $r(s_t,a_t)$ is computed for $0\leq t \leq T$, 
where $s_{-1},a_{-1}$ are states and actions, which are introduced 
for formal reasons at the beginning of an episode. 
The expected reward $r(s_t,a_t)$ 
is with $\tilde{q}^\pi(s_{-1},a_{-1})=0$:
\begin{align}
   r(s_t,a_t) \ &= \  \EXP_{r_{t+1}}  \left[R_{t+1} \mid s_t,a_t \right]
   \ = \ \EXP_{s_{t-1},a_{t-1}} \left[ \tilde{q}^\pi(s_t,a_t) \ - \ 
   \tilde{q}^\pi(s_{t-1},a_{t-1}) \mid s_t,a_t \right] \\ \nonumber 
   &= \ \tilde{q}^\pi(s_t,a_t) \ - \ 
   \EXP_{s_{t-1},a_{t-1}} \left[ \tilde{q}^\pi(s_{t-1},a_{t-1}) 
   \mid s_t,a_t \right] \ .
\end{align} 

The expectations
$\EXP_{\pi}\left[.\mid  s_t,a_t\right]$
like
$\EXP_{\pi}\left[\tilde{R}_{T+1}\mid  s_t,a_t\right]$
are expectations over all episodes starting in $(s_t,a_t)$
and ending in some $(s_T,a_T)$.

The $Q$-values for the SDP $\cP$
are defined for $0\leq t \leq T$ as:
\begin{align}
    q^\pi(s_t,a_t) \ &= \ 
    \EXP_{\pi}  \left[ \sum_{\tau=0}^{T-t} R_{t+1+\tau} \mid s_t,a_t \right] 
    \\\nonumber &= \ \EXP_{\pi}  \left[ \tilde{q}^\pi(s_T,a_T) \ - \  
    \tilde{q}^\pi(s_{t-1},a_{t-1}) \mid s_t,a_t \right] 
    \\\nonumber &= \ 
    \EXP_{\pi}  \left[ \tilde{q}^\pi(s_T,a_T) \mid s_t,a_t \right]  \ - \  
    \EXP_{\pi}  \left[ \tilde{q}^\pi(s_{t-1},a_{t-1}) \mid s_t,a_t \right] 
    \\\nonumber &= \ 
     \tilde{q}^\pi(s_t,a_t) \ - \  
    \EXP_{s_{t-1},a_{t-1}}  \left[ \tilde{q}^\pi(s_{t-1},a_{t-1}) 
    \mid s_t,a_t \right] 
    \\\nonumber &= \ r(s_t,a_t) \ .
\end{align} 
The second equality uses
\begin{align}
  \sum_{\tau=0}^{T-t} R_{t+1+\tau} \ &= \  \sum_{\tau=0}^{T-t}
  \tilde{q}^\pi(s_{t+\tau},a_{t+\tau}) \ - \    
   \tilde{q}^\pi(s_{t+\tau-1},a_{t+\tau-1}) \\ \nonumber
  &= \ \tilde{q}^\pi(s_T,a_T) \ - \  \tilde{q}^\pi(s_{t-1},a_{t-1}) \ .
\end{align} 

The posterior  $p(s_{t-1},a_{t-1} \mid s_t, a_t)$ is
\begin{align}
 p(s_{t-1},a_{t-1} \mid s_t, a_t)  \ &= \
  \frac{ p(s_t,a_t \mid s_{t-1}, a_{t-1}) \ p(s_{t-1},a_{t-1})}{
    p(s_t,a_t)} \\ \nonumber
  &= \  \frac{ p(s_t \mid s_{t-1}, a_{t-1}) \ p(s_{t-1},a_{t-1})}{
    p(s_t)}  \ = \ p(s_{t-1},a_{t-1} \mid s_t )  \ ,
\end{align} 
where we used $p(s_t,a_t \mid s_{t-1}, a_{t-1}) = \pi(a_t \mid s_t)
p(s_t \mid s_{t-1}, a_{t-1})$ and  $p(s_t,a_t)= \pi(a_t \mid s_t) p(s_t)$.
The posterior does no longer contain $a_t$.
We can express the mean of previous $Q$-values
by the posterior  $p(s_{t-1},a_{t-1} \mid s_t, a_t)$:
\begin{align}
  &\EXP_{s_{t-1},a_{t-1}} \left[ \tilde{q}^\pi(s_{t-1},a_{t-1}) \mid s_t, a_t \right] \ = \ 
  \sum_{s_{t-1},a_{t-1}} p(s_{t-1},a_{t-1} \mid s_t, a_t) \
  \tilde{q}^\pi(s_{t-1},a_{t-1}) \\ \nonumber
  &= \ \sum_{s_{t-1},a_{t-1}} 
   \ p(s_{t-1},a_{t-1} \mid s_t ) \
  \tilde{q}^\pi(s_{t-1},a_{t-1}) \ = \  
   \EXP_{s_{t-1},a_{t-1}} \left[ \tilde{q}^\pi(s_{t-1},a_{t-1}) \mid s_t \right]
   \ = \ \psi^\pi(s_t) \ , 
\end{align} 
with 
\begin{align}
  \psi^\pi(s_t) \ &= \
  \EXP_{s_{t-1},a_{t-1}} \left[ \tilde{q}^\pi(s_{t-1},a_{t-1}) \mid s_t \right] \ .
\end{align}

The SDP $\cP$ and the MDP
$\tilde{\cP}$ have the same advantage function, 
since the value functions are the expected $Q$-values across the actions
and follow the equation $v^\pi(s_t) = \tilde{v}^\pi(s_t)+\psi^\pi(s_t)$.
Therefore $\psi^\pi(s_t)$ cancels in the advantage function of the SDP $\cP$.

~~\newline

Using a behavior policy 
$\breve{\pi}$ the expected immediate reward is
\begin{align}
  \EXP_{\breve{\pi}} \left[ R_{t+1} \mid s_t,a_t \right] \ &= \
  \EXP_{r_{t+1},\breve{\pi}}  \left[R_{t+1} \mid s_t,a_t \right]
   \ = \ \EXP_{s_{t-1},a_{t-1},\breve{\pi}} \left[ \tilde{q}^\pi(s_t,a_t) \ - \ 
   \tilde{q}^\pi(s_{t-1},a_{t-1}) \mid s_t,a_t \right] \\ \nonumber 
   &= \ \tilde{q}^\pi(s_t,a_t) \ - \ 
   \EXP_{s_{t-1},a_{t-1},\breve{\pi}} \left[ \tilde{q}^\pi(s_{t-1},a_{t-1}) 
   \mid s_t,a_t \right] \ .
\end{align} 
The posterior  $p_{\breve{\pi}}(s_{t-1},a_{t-1} \mid s_t, a_t)$ is
\begin{align}
 p_{\breve{\pi}}(s_{t-1},a_{t-1} \mid s_t, a_t)  \ &= \
  \frac{ p_{\breve{\pi}}(s_t,a_t \mid s_{t-1}, a_{t-1}) \
  p_{\breve{\pi}}(s_{t-1},a_{t-1})}{
    p_{\breve{\pi}}(s_t,a_t)} \\ \nonumber
  &= \  \frac{ p(s_t \mid s_{t-1}, a_{t-1}) \ p_{\breve{\pi}}(s_{t-1},a_{t-1})}{
    p_{\breve{\pi}}(s_t)}  \ = \ p_{\breve{\pi}}(s_{t-1},a_{t-1} \mid s_t )  \ ,
\end{align} 
where we used $p_{\breve{\pi}}(s_t,a_t \mid s_{t-1}, a_{t-1}) = 
\breve{\pi}(a_t \mid s_t)
p(s_t \mid s_{t-1}, a_{t-1})$ and  $p_{\breve{\pi}}(s_t,a_t)= 
\breve{\pi}(a_t \mid s_t) p_{\breve{\pi}}(s_t)$.
The posterior does no longer contain $a_t$.
We can express the mean of previous $Q$-values
by the posterior  $p_{\breve{\pi}}(s_{t-1},a_{t-1} \mid s_t, a_t)$:
\begin{align}
  &\EXP_{s_{t-1},a_{t-1},\breve{\pi}} 
  \left[ \tilde{q}^\pi(s_{t-1},a_{t-1}) \mid s_t, a_t \right] \ = \ 
  \sum_{s_{t-1},a_{t-1}} p_{\breve{\pi}}(s_{t-1},a_{t-1} \mid s_t, a_t) \
  \tilde{q}^\pi(s_{t-1},a_{t-1}) \\ \nonumber
  &= \ \sum_{s_{t-1},a_{t-1}} 
   \ p_{\breve{\pi}}(s_{t-1},a_{t-1} \mid s_t ) \
  \tilde{q}^\pi(s_{t-1},a_{t-1}) \ = \  
   \EXP_{s_{t-1},a_{t-1},\breve{\pi}} 
   \left[ \tilde{q}^\pi(s_{t-1},a_{t-1}) \mid s_t \right]
   \ = \ \psi^{\pi,\breve{\pi}}(s_t) \ , 
\end{align} 
with 
\begin{align}
  \psi^{\pi,\breve{\pi}}(s_t) \ &= \
   \EXP_{s_{t-1},a_{t-1},\breve{\pi}} 
   \left[ \tilde{q}^\pi(s_{t-1},a_{t-1}) \mid s_t \right] \ .
\end{align} 
Therefore we have
\begin{align}
   \EXP_{\breve{\pi}} \left[ R_{t+1} \mid s_t,a_t \right] \ &= \
    \tilde{q}^\pi(s_t,a_t) \ - \ \psi^{\pi,\breve{\pi}}(s_t) \ .
  \end{align}

\end{proof}

\subsection{Novel Learning Algorithms based on Reward Redistributions}
\label{sec:Alearning}
We assume $\gamma=1$ and a finite horizon or absorbing state
original MDP $\tilde{\cP}$ with delayed reward.
According to Theorem~\ref{th:AOptReturnDecomp}, 
$\tilde{q}^\pi(s_t,a_t)$ can be estimated with an offset that 
depends only on $s_t$ 
by estimating the expected immediate redistributed reward.
Thus, $Q$-value estimation becomes trivial and the
the advantage function of the MDP $\tilde{\cP}$ can be readily computed.
All reinforcement learning methods like policy gradients that use 
$\arg\max_{a_t} \tilde{q}^\pi(s_t,a_t)$ or 
the advantage function 
$\tilde{q}^\pi(s_t,a_t) - \EXP_{a_t} \tilde{q}^\pi(s_t,a_t)$
of the original MDP $\tilde{\cP}$
can be used. These methods either rely on Theorem~\ref{th:AOptReturnDecomp}
and either estimate $q^\pi(s_t,a_t)$ according to Eq.~\eqref{eq:Aqvalue} 
or the expected immediate reward 
according to Eq.~\eqref{eq:Abehavior}. 
Both approaches estimate 
$\tilde{q}^\pi(s_t,a_t)$ with an offset 
that depends only on $s_t$ (either $\psi^{\pi}(s_t)$ 
or $\psi^{\pi,\breve{\pi}}(s_t)$).
Behavior policies like
``greedy in the limit with infinite exploration'' (GLIE) or
``restricted rank-based randomized'' (RRR) allow to prove
convergence of SARSA \cite{Singh:00}.
These policies can be used with reward redistribution.
GLIE policies can be realized by a softmax with exploration coefficient 
on the $Q$-values, 
therefore $\psi^\pi(s_t)$ or $\psi^{\pi,\breve{\pi}}(s_t)$ cancels.
RRR policies select actions probabilistically according to 
the ranks of their $Q$-values, where the 
greedy action has highest probability. Therefore $\psi(s_t)$
or $\psi^{\pi,\breve{\pi}}(s_t)$
is not required.
For function approximation, 
convergence of the $Q$-value estimation together 
with reward redistribution and GLIE or RRR policies 
can under standard assumptions be proven by the stochastic 
approximation theory for 
two time-scale update rules \cite{Borkar:97,Karmakar:17}.
Proofs for convergence to an optimal policy are in general 
difficult, since locally stable attractors 
may not correspond to optimal policies.

Reward redistribution can be used for
\begin{itemize}
\item
(A) $Q$-value estimation, 
\item
(B) policy gradients, and 
\item
(C) $Q$-learning.
\end{itemize}

\subsubsection{Q-Value Estimation}
\label{sec:QestimateA}
Like SARSA, RUDDER learning continually predicts 
$Q$-values to improve the policy. 
Type (A) methods estimate $Q$-values and are divided 
into variants (i), (ii), and (iii).
Variant (i) assumes an optimal reward redistribution
and estimates $\tilde{q}^\pi(s_t,a_t)$ with an offset
depending only on $s_t$.
The estimates are based on Theorem~\ref{th:AOptReturnDecomp}
either by on-policy direct $Q$-value estimation according to Eq.~\eqref{eq:Aqvalue}
or by off-policy immediate reward estimation according to Eq.~\eqref{eq:Abehavior}.
Variant (ii) methods assume a non-optimal reward redistribution and 
correct Eq.~\eqref{eq:Aqvalue} by estimating $\kappa$.
Variant (iii) methods use eligibility traces for the redistributed reward.

\paragraph{Variant (i): Estimation of $\tilde{q}^\pi(s_t,a_t)$ 
with an offset assuming optimality.}
Theorem~\ref{th:AOptReturnDecomp} justifies the estimation
of $\tilde{q}^\pi(s_t,a_t)$ with an offset
by on-policy direct $Q$-value estimation via Eq.~\eqref{eq:Aqvalue} or 
by off-policy immediate reward estimation via Eq.~\eqref{eq:Abehavior}.
RUDDER learning can be based on policies like
``greedy in the limit with infinite exploration'' (GLIE) or
``restricted rank-based randomized'' (RRR) \cite{Singh:00}. 
GLIE policies change toward greediness with respect to the $Q$-values
during learning.

\paragraph{Variant (ii): TD-learning of $\kappa$ and correction of the redistributed reward.}
For non-optimal reward redistributions $\kappa(T-t-1,t)$ can be estimated 
to correct the $Q$-values.
{\bf TD-learning of $\kappa$.}
The expected sum of delayed rewards $\kappa(T-t-1,t)$ can be formulated as
\begin{align}
  \kappa(T-t-1,t) \ &= \ 
  \EXP_{\pi} \left[\sum_{\tau=0}^{T-t-1} R_{t+2+\tau} \mid s_t,a_t\right]\\ \nonumber
  &= \ 
  \EXP_{\pi} \left[ R_{t+2} \ + \ \sum_{\tau=0}^{T-(t+1)-1} R_{(t+1)+2+\tau} \mid s_t,a_t\right]\\ \nonumber
  &= \ 
  \EXP_{s_{t+1},a_{t+1},r_{t+2}} \left[ R_{t+2}  \ + \ 
  \EXP_{\pi} \left[\sum_{\tau=0}^{T-(t+1)-1} R_{(t+1)+2+\tau} \mid s_{t+1},a_{t+1}\right] 
  \mid s_t,a_t\right]\\ \nonumber
  &= \ \EXP_{s_{t+1},a_{t+1},r_{t+2}} \left[ R_{t+2}  \ + \ \kappa(T-t-2,t+1) \mid s_t,a_t\right] \ .
\end{align}  
Therefore, $\kappa(T-t-1,t)$ can be estimated by $R_{t+2}$ and $\kappa(T-t-2,t+1)$, 
if the last two are drawn together, 
i.e.\ considered as pairs. 
Otherwise the expectations of $R_{t+2}$ and $\kappa(T-t-2,t+1)$ given
$(s_t,a_t)$ must be estimated.
We can use TD-learning if the immediate reward and 
the sum of delayed rewards are drawn as pairs, that is, simultaneously.  
The TD-error $\delta_{\kappa}$ becomes
\begin{align}
    \delta_{\kappa}(T-t-1,t) \ &= \ R_{t+2}  \ + \ \kappa(T-t-2,t+1) \ - \ \kappa(T-t-1,t) \ .
\end{align}  

We now define eligibility traces for $\kappa$.
Let the $n$-step return samples of $\kappa$ for $1\leq n \leq T-t$ be
\begin{align}
 \kappa^{(1)}(T-t-1,t) \ &= \  R_{t+2}  \ + \ \kappa(T-t-2,t+1)  \\ \nonumber
 \kappa^{(2)}(T-t-1,t) \ &= \  R_{t+2}  \ + \  R_{t+3}  \ + \ \kappa(T-t-3,t+2)  \\ \nonumber
 &\ldots \\ \nonumber
 \kappa^{(n)}(T-t,t) \ &= \  R_{t+2}  \ + \  R_{t+3}  \ + \ \ldots 
  \ + \  R_{t+n+1} \ + \ \kappa(T-t-n-1,t+n)  \ . 
\end{align}  
The $\lambda$-return for $\kappa$ is
\begin{align}
  \kappa^{(\lambda)} (T-t-1,t)\ &= \  (1 - \lambda) \  
  \sum_{n=1}^{T-t-1} \lambda^{n-1} \ \kappa^{(n)}(T-t-1,t)   \ + \ 
  \lambda^{T-t-1} \ \kappa^{(T-t)}(T-t-1,t)  \ . 
\end{align}  
We obtain
\begin{align}
  \kappa^{(\lambda)}(T-t-1,t) \ &= \  R_{t+2} \ + \ \kappa(T-t-2,t+1)  \\ \nonumber
  &+ \ \lambda \ \left(R_{t+3}  \ + \ \kappa(T-t-3,t+2) 
  \ - \ \kappa(T-t-2,t+1) \right) \\ \nonumber
  &+ \ \lambda^2 \ \left(R_{t+4}  \ + \ \kappa(T-t-4,t+3) 
  \ - \ \kappa(T-t-3,t+2) \right) \\ \nonumber
  \ldots \\ \nonumber
   &+ \ \lambda^{T-1-t} \ \left(R_{T+1}  \ + \ \kappa(0,T-1) 
  \ - \ \kappa(1,T-2) \right)  \ . 
\end{align}  
We can reformulate this as
\begin{align}
  \kappa^{(\lambda)}(T-t-1,t) \ &= \  \kappa(T-t-1,t) \ 
  + \ \sum_{n=0}^{T-t-1}  \lambda^{n} \  \delta_{\kappa}(T-t-n-1,t+n) \ .
\end{align}  
The $\kappa$ error $\Delta_{\kappa}$ is
\begin{align}
  \Delta_{\kappa}(T-t-1,t) \ &= \ \kappa^{(\lambda)}(T-t-1,t) \ - \  \kappa(T-t-1,t)  \
  = \ \sum_{n=0}^{T-t-1}  \lambda^{n} \  \delta_{\kappa}(T-t-n-1,t+n) \ .
\end{align}  
The derivative of 
\begin{align}
 1/2 \  \Delta_{\kappa}(T-t-1,t)^2 \ &= \ 1/2 \
 \left( \kappa^{(\lambda)}(T-t-1,t) \ - \  \kappa(T-t-1,t;\Bw) \right)^2  
\end{align}  
with respect to $\Bw$ is
\begin{align}
 &- \left(\kappa^{(\lambda)}(T-t-1,t) \ - \  \kappa(T-t-1,t;\Bw)  \right)
  \ \nabla_{w} \kappa(T-t-1,t;\Bw) \\ \nonumber 
  &= \
  - \ \sum_{n=0}^{T-t-1}  \lambda^{n} \  \delta_{\kappa}(T-t-n-1,t+n) \ 
  \nabla_{w} \kappa(T-t-1,t;\Bw) \ .
\end{align}  

The full gradient of the sum of $\kappa$ errors is
\begin{align}
 &1/2 \ \nabla_{w}  \sum_{t=0}^{T-1}\Delta_{\kappa}(T-t-1,t)^2 \\ \nonumber
 &= \  
  - \ \sum_{t=0}^{T-1} \ \sum_{n=0}^{T-t-1}  \lambda^{n} \  \delta_{\kappa}(T-t-n-1,t+n) \ 
  \nabla_{w} \kappa(T-t-1,t;\Bw) \\ \nonumber
   &= \ - \ \sum_{t=0}^{T-1} \ \sum_{\tau=t}^{T-1}  \lambda^{\tau-t} \  
  \delta_{\kappa}(T-\tau-1,\tau) \ \nabla_{w} \kappa(T-t-1,t;\Bw) 
  \\ \nonumber
  &= \ - \ \sum_{\tau=0}^{T-1} \  \delta_{\kappa}(T-\tau-1,\tau) \
    \sum_{t=0}^{\tau} \lambda^{\tau-t} \ \nabla_{w} \kappa(T-t-1,t;\Bw) \ .
\end{align} 
We set $n=\tau-t$, so that $n=0$ becomes $\tau=t$ and $n=T-t-1$ becomes 
$\tau=T-1$.
The recursion
\begin{align}
  f(t) \ &= \ \lambda \ f(t-1) \ + \ a_t \ , \qquad
  f(0) \ = \ 0 
 \end{align}  
can be written as
\begin{align}
  f(T) \ &= \ \sum_{t=1}^{T} \lambda^{T-t} \ a_t \ .
\end{align}  

Therefore, we can use following update rule for minimizing 
$\sum_{t=0}^{T-1}\Delta_{\kappa}(T,t)^2$ with respect to $\Bw$ with
$1\leq \tau \leq T-1$:
\begin{align}
  \Bz_{-1} \ &= \ 0 \\
  \Bz_{\tau} \ &= \ \lambda \ \Bz_{\tau-1} \ + \ 
  \nabla_{w} \kappa(T-\tau,\tau;\Bw) \\
  \delta_{\kappa}(T-\tau,\tau) \ &= \ R_{\tau+2}  \ + \ 
  \kappa(T-\tau-1,\tau+1;\Bw) \ - \ \kappa(T-\tau,\tau;\Bw) \\
  \Bw^{\nn} \ &= \ \Bw \ + \ \alpha \ \delta_{\kappa}(T-\tau,\tau) \ \Bz_{\tau} \ .
\end{align}

{\bf Correction of the reward redistribution.}
For correcting the redistributed reward, we apply 
a method similar to reward shaping or look-back advice.
This method ensures that the corrected redistributed reward 
leads to an SDP that is has the same return per sequence as the
SDP $\cP$.
The reward correction is 
\begin{align}
 F(s_t,a_t,s_{t-1},a_{t-1}) \ &= \ \kappa(m,t) \ - \ \kappa(m,t-1) \ ,
\end{align}
we define the corrected redistributed reward as
\begin{align}
  &R^{\Rc}_{t+1} \ = \ R_{t+1} \ + \ F(s_t,a_t,s_{t-1},a_{t-1}) \ = \ 
  R_{t+1} \ + \ \kappa(m,t) \ - \ \kappa(m,t-1) \ .
\end{align}
We assume that $\kappa(m,-1)=\kappa(m,T+1)=0$, therefore
\begin{align}
 \sum_{t=0}^{T+1} F(s_t,a_t,s_{t-1},a_{t-1})  \ &= \  
 \sum_{t=0}^{T+1} \kappa(m,t) \ - \ \kappa(m,t-1) \ = \ 
 \kappa(m,T+1) \ - \ \kappa(m,-1) \ = \  0 \ . 
\end{align} 
Consequently, the corrected redistributed reward $R^{\Rc}_{t+1}$ does not
change the expected return for a sequence, therefore, the resulting SDP has the
same optimal policies as the SDP without correction.

For a predictive reward of $\rho$ at time $t=k$, which
can be predicted from time $t=l<k$ to time $t=k-1$, we have:
\begin{align}
  \kappa(m,t)  \ &= \
  \begin{cases}
    0 \ , & \text{for } t < l \ , \\
    \rho \ , & \text{for } l \leq t < k \ ,\\
    0 \ , & \text{for } t \geq k \ .
  \end{cases} 
\end{align}
The reward correction is
\begin{align}
  F(s_t,a_t,s_{t-1},a_{t-1}) \ &= \
  \begin{cases}
    0 \ , & \text{for } t < l \ , \\
    \rho \ , & \text{for } t = l  \ , \\
    0 \ , & \text{for } l < t < k \ , \\
    -\rho \ , & \text{for } t = k  \ , \\
    0 \ , & \text{for } t > k \ .
  \end{cases} 
\end{align}

{\bf Using $\kappa$ as auxiliary task 
in predicting the return for return decomposition.}
A $\kappa$ prediction can serve as additional output of 
the function $g$ that predicts the return and
is the basis of the return decomposition. 
Even a partly prediction of $\kappa$ means that 
the reward can be distributed further back. 
If $g$ can partly predict $\kappa$, then $g$
has all information to predict the return earlier in the 
sequence. If the return is predicted 
earlier, then the reward will be distributed further back.
Consequently, the reward redistribution 
comes closer to an optimal reward redistribution.
However, at the same time, $\kappa$ can no longer be predicted.  
The function $g$ must find another $\kappa$ that can be predicted.
If no such $\kappa$ is found, then optimal reward redistribution is
indicated.

\paragraph{Variant (iii): Eligibility traces assuming optimality.}
We can use eligibility traces to further distribute the reward back.
For an optimal reward redistribution, we have
$\EXP_{s_{t+1}} [V(s_{t+1})]=0$. 
The new returns $\cR_t$ are given by the recursion
\begin{align}
\label{th:Aeligibility}
\cR_t \ &= \  r_{t+1} \ + \ \lambda \  \cR_{t+1} \ , \\
\cR_{T+2} \ &= \ 0 \ .
\end{align}
The expected policy gradient updates with the new returns $\cR$ are
$\EXP_{\pi}\left[ \nabla_{\theta} \log \pi(a_t \mid s_t;\Bth)
  \cR_t  \right]$.
To avoid an estimation of the value function $V(s_{t+1})$, 
we assume optimality,
which might not be valid. However, the error should be small if the 
return decomposition works well.
Instead of estimating a value function, we can use a correction 
as it is shown in
next paragraph.

\subsubsection{Policy Gradients}
\label{sec:PolicyA}
Type (B) methods are policy gradients.
 In the expected updates  
$\EXP_{\pi}\left[ \nabla_{\theta} \log \pi(a \mid s;\Bth)  
q^\pi(s,a) \right]$ of policy gradients, the value $q^\pi(s,a)$ 
is replaced by an estimate of $r(s,a)$ or by
samples of the redistributed reward. 
Convergence to optimal policies is guaranteed even with the
offset $\psi^{\pi}(s)$ in Eq.~\eqref{eq:Aqvalue} 
similar to baseline normalization for policy gradients.
With baseline normalization,
the baseline $b(s)=\EXP_{a}[r(s,a)]=\sum_a \pi(a\mid s) r(s,a)$ 
is subtracted from $r(s,a)$, which gives the policy gradient  
$\EXP_{\pi}\left[ \nabla_{\theta} \log \pi(a \mid s;\Bth)  (r(s,a) - b(s))
\right]$. 
With eligibility traces using $\lambda \in
[0,1]$ for $G_t^{\lambda}$ \cite{Sutton:18book}, we have 
the new returns $\cG_t =  r_t + \lambda \cG_{t+1}$ with $\cG_{T+2} = 0$.
The expected updates with the new returns $\cG$ are
$\EXP_{\pi}\left[ \nabla_{\theta} \log \pi(a_t \mid s_t;\Bth)
\cG_t  \right]$.

\subsubsection{Q-Learning}
\label{sec:qlearningA}
The type (C) method is $Q$-learning with the redistributed reward. 
Here, $Q$-learning is justified if
immediate and future reward are drawn together,
as typically done.
Also other temporal difference methods are justified when
immediate and future reward are drawn together.

\subsection{Return Decomposition to construct a Reward Redistribution}

We now propose methods to construct reward redistributions which 
ideally would be optimal. 
Learning with non-optimal reward redistributions {\em does work} since the 
optimal policies do not change according to Theorem~\ref{th:AEquivT}.
However reward redistributions that are optimal considerably speed up learning,
since future expected rewards introduce 
biases in TD-methods and the high variance in MC-methods.
The expected optimal redistributed reward is according to Eq.~\eqref{eq:AdiffQ} 
the difference of $Q$-values. 
The more a reward redistribution deviates from these differences,
the larger are the absolute $\kappa$-values and, in turn, the less optimal
is the reward redistribution.
Consequently we aim at identifying the largest $Q$-value differences to
construct a reward redistribution which is close to optimal.
Assume a grid world where you have to take a key to later open a door
to a treasure room. Taking the key increases the chances to receive the
treasure and, therefore, is associated with a large positive $Q$-value difference.
Smaller positive $Q$-value difference are steps toward the key location.

\paragraph{Reinforcement Learning as Pattern Recognition.}
We want to transform the reinforcement learning problem into
a pattern recognition problem to employ deep learning approaches.
The sum of the $Q$-value differences gives the 
difference between expected return at sequence begin and
the expected return at sequence end (telescope sum).
Thus, $Q$-value differences allow to predict the 
expected return of the whole state-action sequence.
Identifying the largest $Q$-value differences 
reduce the prediction error most.
$Q$-value differences are assumed to be associated with
patterns in state-action transitions 
like taking the key in our example. 
The largest $Q$-value differences 
are expected to be found more frequently in sequences
with very large or very low return.
The resulting task is to predict the expected return
from the whole sequence and identify which 
state-action transitions contributed most to the prediction.
This pattern recognition task is utilized to
construct a reward redistribution, where redistributed reward
corresponds to the contribution.

\subsubsection{Return Decomposition Idea}

The {\em return decomposition idea} is 
to predict the realization of the return or its
expectation by a function $g$ from the state-action sequence 
\begin{align}
 (s,a)_{0:T} \ &:= \  (s_0,a_0,s_1,a_1,\ldots,s_T,a_T) \ .
\end{align}
The return is the accumulated reward along the whole sequence
$(s,a)_{0:T}$.
The function $g$ depends on the policy $\pi$ that is used to 
generate the state-action sequences.
Subsequently, the prediction or the realization of the return
is distributed over
the sequence with the help of $g$.
One important advantage of a deterministic function $g$ is 
that it predicts with proper loss functions and if being perfect 
the expected return. Therefore, it 
removes the sampling variance of returns. 
In particular the variance of probabilistic rewards is averaged out. 
Even an imperfect function $g$ removes the variance as it is deterministic.
As described later, the sampling variance may be 
reintroduced when strictly return-equivalent SDPs are ensured.
We want to determine for each sequence element 
its contribution to the prediction of the function $g$. 
Contribution analysis computes the contribution of each 
state-action pair to the prediction, that is, the information of each
state-action pair about the prediction. 
In principle, we can use any contribution analysis method.
However, we prefer three methods:
(A) Differences in predictions.
If we can ensure that $g$ predicts the sequence-wide return
at every time step. 
The difference of two consecutive predictions is a measure of
the contribution of the current state-action pair to the return prediction.
The difference of consecutive predictions is the redistributed reward.
(B) Integrated gradients (IG) \cite{Sundararajan:17}.
(C) Layer-wise relevance propagation (LRP) \cite{Bach:15}.
The methods (B) and (C) use information later in the sequence for
determining the contribution of the current state-action pair. Therefore,
they introduce a non-Markov reward. 
However, the non-Markov reward can be viewed as probabilistic reward.
Since probabilistic reward increases the variance, we prefer method (A).

\paragraph{Explaining Away Problem.}
\label{para:AexplainingAway}
We still have to tackle the problem that reward causing actions do not receive redistributed rewards
since they are explained away by later states.
To describe the problem, assume an MDP $\tilde{\cP}$ with the only 
reward at sequence end.
To ensure the Markov property, states in $\tilde{\cP}$ have to store 
the reward contributions of previous state-actions;
e.g.\ $s_T$ has to store all previous contributions such that the expectation $\tilde{r}(s_T,a_T)$
is Markov.
The explaining away problem is that later states
are used for return prediction, while reward causing
earlier actions are missed.
To avoid explaining away,
between the state-action pair $(s_t,a_t)$ and its predecessor $(s_{t-1},a_{t-1})$, where
$(s_{-1},a_{-1})$ are introduced for starting an episode.
The sequence of differences is defined as
\begin{align}
  \Delta_{0:T} \ &:= \  
  \big(\Delta(s_{-1},a_{-1},s_0,a_0),\ldots,\Delta(s_{T-1},a_{T-1},s_T,a_T)\big) \ .
\end{align}
We assume that the differences $\Delta$ 
are mutually independent \cite{Hyvarinen:01}:
\begin{align}
  \label{eq:indept}
  &p \left(\Delta(s_{t-1},a_{t-1},s_t,a_t) \mid
    \Delta(s_{-1},a_{-1},s_0,a_0),\ldots,
    \Delta(s_{t-2},a_{t-2},s_{t-1},a_{t-1}), \right. \\ \nonumber
    &\left. \Delta(s_t,a_t,s_{t+1},a_{t+1})\ldots,
    \Delta(s_{T-1},a_{T-1},s_T,a_T)  \right) \ = \ 
    p\left(\Delta(s_{t-1},a_{t-1},s_t,a_t) \right) \ .
\end{align}
The function $g$ predicts the realization of the sequence-wide
return or its expectation from the sequence $\Delta_{0:T}$:
\begin{align}
  \label{eq:simpleA1}
   g\big(\Delta_{0:T}\big) \ &= \ 
  \EXP \left[\tilde{R}_{T+1} \mid s_T,a_T\right] \ = \ 
  \tilde{r}_{T+1}\ .
\end{align} 
{\bf Return decomposition} deconstructs $g$ into contributions
$h_t=h(\Delta(s_{t-1},a_{t-1},s_t,a_t)$ at time $t$:
\begin{align}
  \label{eq:simpleA}
  g\big(\Delta_{0:T}\big) \ &= \
  \sum_{t=0}^T  h(\Delta(s_{t-1},a_{t-1},s_t,a_t) )
   \ = \ \tilde{r}_{T+1}\ .
\end{align} 
If we can assume that $g$ can predict the return at every time step:
\begin{align}
  g\big(\Delta_{0:t}\big) \ &= \ 
  \EXP_{\pi} \left[\tilde{R}_{T+1} \mid s_t,a_t\right] \ ,
\end{align} 
then we use the contribution analysis method "differences of return predictions",
where the contributions are defined as:
\begin{align}
 h_0 \ &= \ h(\Delta(s_{-1},a_{-1},s_0,a_0) ) \ := \ g\big(\Delta_{0:0}\big) \\
 h_t \ &= \ h(\Delta(s_{t-1},a_{t-1},s_t,a_t) ) \ := \ 
 g\big(\Delta_{0:t}\big) \ - \ g\big(\Delta_{0:(t-1)}\big)\ .
\end{align} 

We assume that the sequence-wide return cannot be predicted from the
last state. The reason is that either immediate rewards are given only 
at sequence end without storing them in the states 
or information is removed from the states.
Therefore, a relevant event for predicting the final reward must be identified by
the function $g$. 
The prediction errors at the end of the episode become, in general, smaller since the
future is less random. Therefore, prediction errors later 
in the episode are up-weighted while early predictions ensure that information is
captured in $h_t$ for being used later.
The prediction at time $T$ has the largest weight and 
relies on information from the past.

If $g$ does predict the return at every time step,
contribution analysis decomposes $g$.
For decomposing a linear $g$ one can use the Taylor decomposition 
(a linear approximation) of $g$ with respect to the $h$
\cite{Bach:15,Montavon:17taylor}.
A non-linear $g$ can be decomposed by 
layerwise relevance propagation (LRP)
\cite{Bach:15,Montavon:17} or integrated gradients (IG)
\cite{Sundararajan:17}.

\subsubsection{Reward Redistribution based on Return Decomposition}

We assume a return decomposition
\begin{align}
  g\big(\Delta_{0:T}\big) \ &= \ \sum_{t=0}^T  h_t \ ,
\end{align} 
with
\begin{align}
  h_0 \ &= \ h(\Delta(s_{-1},a_{-1},s_0,a_0)) \ , \\
  h_t \ &= \ h(\Delta(s_{t-1},a_{t-1},s_t,a_t))
   \  \ \text{ for } 0 < t \leq T \ .
\end{align} 
We use these contributions for redistributing the reward. 
The reward redistribution is given by 
the random variable $R_{t+1}$ for the reward at time $t+1$.
These new redistributed rewards
$R_{t+1}$ must have the contributions $h_t$ as mean:
\begin{align}
\EXP \left[ R_{t+1} \mid s_{t-1},a_{t-1},s_t,a_t \right]  \ &=  \ h_t
\end{align} 

The reward $\tilde{R}_{T+1}$ of $\tilde{\cP}$
is probabilistic and 
the function $g$ might not be perfect,
therefore neither $g(\Delta_{0:T}) = \tilde{r}_{T+1}$ for the return
realization $\tilde{r}_{T+1}$ nor
$g(\Delta_{0:T}) = \tilde{r}(s_T,a_T)$ for the expected return
holds.
To assure strictly return-equivalent SDPs,
we have to compensate for both a probabilistic reward $\tilde{R}_{T+1}$ 
and an imperfect function $g$.
The compensation is given by
\begin{align}
  \tilde{r}_{T+1} \ - \ \sum_{\tau=0}^T h_t \ .
\end{align} 
We compensate with an extra reward $R_{T+2}$ at time $T+2$ 
which is immediately given
after $R_{T+1}$ at time $T+1$ after the state-action pair $(s_T,a_T)$.
The new redistributed reward $R_{t+1}$ is 
\begin{align}
   \EXP \left[ R_1 \mid s_0,a_0 \right] \ &= \ h_0  \ , \\
   \EXP \left[ R_{t+1} \mid s_{t-1},a_{t-1},s_t,a_t \right] 
   \ &= \ h_t  \  \ \text{ for } 0 < t \leq T \ , \\
   R_{T+2}  \ &= \ 
   \tilde{R}_{T+1} \ - \ \sum_{t=0}^T h_t  \ ,
\end{align}
where the realization $\tilde{r}_{T+1}$ is replaced by its
random variable $\tilde{R}_{T+1}$.
If the the prediction of $g$ is perfect, then we can set
$R_{T+2}=0$ and redistribute the expected return which is
the predicted return.
$R_{T+2}$ compensates for both a probabilistic reward $\tilde{R}_{T+1}$ 
and an imperfect function $g$.
Consequently all variance of sampling the return is moved to  $R_{T+2}$.
Only the imperfect function $g$ must be corrected while the variance
does not matter. 
However, we cannot distinguish, e.g.\ in early learning phases, 
between errors of $g$ and random reward. 
{\bf A perfect $g$ results in an optimal reward redistribution.}

Next theorem
shows that Theorem~\ref{th:AzeroExp} holds also for
the correction $R_{T+2}$. 
\begin{theoremA}
\label{th:AzeroExpCorr}
The optimality conditions
hold also for reward redistributions with corrections:
\begin{align}
\label{eq:AoptimalCon2}
  \kappa(T-t+1,t-1) \ &= \  0 \ . 
\end{align} 
\end{theoremA}

\begin{proof}
The expectation of 
$\kappa(T-t+1,t-1)=\sum_{\tau=0}^{T-t+1} R_{t+1+\tau}$, 
that is $\kappa(m,t-1)$ with $m=T-t+1$.
\begin{align}
  &\EXP_{\pi} \left[ \sum_{\tau=0}^{T-t+1} R_{t+1+\tau}  
  \mid s_{t-1},a_{t-1}\right]\\ \nonumber
  &= \
  \EXP_{\pi} \left[\tilde{R}_{T+1} \ - \ 
  \tilde{q}^\pi(s_T,a_T) \ + \ \sum_{\tau=0}^{T-t}
  \left(\tilde{q}^\pi(s_{\tau+t},a_{\tau+t}) \ - \
    \tilde{q}^\pi(s_{\tau+t-1},a_{\tau+t-1}) \right)  \mid
  s_{t-1},a_{t-1}\right] \\ \nonumber
  &= \
  \EXP_{\pi} \left[ \tilde{R}_{T+1}  \ - \
    \tilde{q}^\pi(s_{t-1},a_{t-1})
    \mid s_{t-1},a_{t-1}\right]\\ \nonumber
   &= \ \EXP_{\pi}\left[\tilde{R}_{T+1} \mid 
  s_{t-1},a_{t-1}\right]  \ - \
   \EXP_{\pi}\left[\EXP_{\pi}\left[ \sum_{\tau=t-1}^{T}
       \tilde{R}_{\tau+1} \mid s_{t-1},a_{t-1}\right] 
       \mid s_{t-1},a_{t-1}\right] \\ \nonumber
 &= \ \EXP_{\pi}\left[ \tilde{R}_{T+1} \mid 
  s_{t-1},a_{t-1}\right]   \ - \
   \EXP_{\pi}\left[ \tilde{R}_{T+1}  \mid s_{t-1},a_{t-1}\right] \\ \nonumber
  &= \ 0 \ .
\end{align}
If we substitute $t-1$ by $t$ ($t$ one step further and $m$ one step smaller)
it follows 
\begin{align}
  \kappa(T-t,t) \ &= \  0 \ . 
\end{align} 

Next, we consider the case $t=T+1$, that is $\kappa(0,T)$, 
which is the expected correction.
We will use following equality for the 
expected delayed reward at sequence end:
\begin{align}
 \tilde{q}^\pi(s_T,a_T) \ &= \   
 \EXP_{\tilde{R}_{T+1}} \left[ \tilde{R}_{T+1}  \mid s_T, a_T \right] 
 \ = \ \tilde{r}_{T+1}(s_T,a_T) \ ,
\end{align} 
since  $\tilde{q}^\pi(s_{T+1},a_{T+1})=0$.
For $t=T+1$ we obtain
\begin{align}
  &\EXP_{R_{T+2}} \left[ R_{T+2}  \mid s_T, a_T \right] \ = \ 
  \EXP_{\tilde{R}_{T+1}} \left[ \tilde{R}_{T+1} \ - \  
  \tilde{q}^\pi(s_T,a_T)  \mid s_T, a_T \right] \\ \nonumber
  &= \ \tilde{r}_{T+1}(s_T,a_T) \ - \ 
  \tilde{r}_{T+1}(s_T,a_T)  \ = \  0 \ .
\end{align} 

\end{proof}

In the experiments we also use a uniform compensation 
where each reward has the same
contribution to the compensation:
\begin{align}
  R_1 \ &= \ h_0 \ + \
 \frac{1}{T+1}  \
 \left(\tilde{R}_{T+1} - \sum_{\tau=0}^T
 h(\Delta(s_{\tau-1},a_{\tau-1},s_{\tau},a_{\tau})) \right) \\
  R_{t+1} \ &= \ h_t \ + \
  \frac{1}{T+1}  \
 \left(\tilde{R}_{T+1} - 
  \sum_{\tau=0}^T h(\Delta(s_{\tau-1},a_{\tau-1},s_{\tau},a_{\tau})) \right) \ .
\end{align}
Consequently all variance of sampling the return is 
uniformly distributed across
the sequence. Also the error of $g$ is uniformly distributed 
across the sequence.

An optimal reward redistribution implies
\begin{align}
  \label{eq:part}
  g\big(\Delta_{0:t}\big) \ &= \
  \sum_{\tau=0}^t  h(\Delta(s_{\tau-1},a_{\tau-1},s_{\tau},a_{\tau}))
  \ = \ \tilde{q}^\pi(s_t,a_t) 
\end{align} 
since the expected reward is
\begin{align}
   \EXP \left[ R_{t+1} \mid s_{t-1},a_{t-1},s_t,a_t \right] \ &= \
   h(\Delta(s_{t-1},a_{t-1},s_t,a_t)) \\ \nonumber
  &= \ \tilde{q}^\pi(s_t,a_t) \ - \  \tilde{q}^\pi(s_{t-1},a_{t-1}) 
\end{align} 
according to Eq.~\eqref{eq:AdiffQ} in Theorem~\ref{th:AzeroExp}
and 
\begin{align}
 h_0 \ &= \ h(\Delta(s_{-1},a_{-1},s_0,a_0) ) \\ \nonumber
 &= \ g\big(\Delta_{0:0}\big) \ = \ \tilde{q}^\pi(s_0,a_0) \ .
\end{align}

\subsection{Remarks on Return Decomposition}
\label{sec:Aremark}

\subsubsection{Return Decomposition for Binary Reward}
A special case is a reward that indicates success or failure by giving 
a reward of 1 or 0, respectively. 
The return is equal to the final reward $R$, which is a Bernoulli variable.
For each state $s$ or each
state-action pair $(s,a)$ the expected return can be 
considered as a Bernoulli variable with success probability
$p_R(s)$ or $p_R(s,a)$. The value function is $v^\pi(s)=\EXP_{\pi}(G \mid s)=p_R(s)$
and  the action-value is $q^\pi(s)=\EXP_{\pi}(G \mid s,a)=p_R(s,a)$ which is in 
both cases the expectation of success. 
In this case, the optimal reward redistribution tracks the success probability
\begin{align}
  R_1 \ &= \ h_0 \ = \ h(\Delta(s_{-1},a_{-1},s_0,a_0)) \ = \ \tilde{q}^\pi(s_0,a_0) \ = \ p_R(s_0,a_0)\\
  R_{t+1} \ &= \ h_t \ = \ h(\Delta(s_{t-1},a_{t-1},s_t,a_t))\ = \ \tilde{q}^\pi(s_t,a_t) \ - \  \tilde{q}^\pi(s_{t-1},a_{t-1}) \\ \nonumber &= \
  p_R(s_t,a_t) \ - \ p_R(s_{t-1},a_{t-1})
  \ \text{ for } 0<t\leq T \\
  R_{T+2} \ &= \ \tilde{R}_{T+1} \ - \ \tilde{r}_{T+1} \ = \ R \ - \ p_R(s_T,a_T)  \ .
\end{align} 
The redistributed reward is the change in the success probability. A good action
increases the success probability and obtains a positive reward while a bad action
reduces the success probability and obtains a negative reward.

\subsubsection{Optimal Reward Redistribution reduces the MDP 
to a Stochastic Contextual Bandit Problem}

The new SDP $\cP$ has a redistributed reward
with random variable $R_t$ at time $t$ distributed according to $p(r \mid s_t,a_t)$.
Theorem~\ref{th:AOptReturnDecomp} states
\begin{align}
   q^\pi(s_t,a_t) \ &= \   r(s_t,a_t)  \ .
\end{align} 
This equation looks like a contextual bandit problem, where
$r(s_t,a_t)$ is an estimate of the mean reward for action $a_t$
for state or context $s_t$.
Contextual bandits \cite[p. 208]{Lattimore:18} are characterized by 
a conditionally $\sigma$-subgaussian noise (Def.~5.1 \cite[p. 68]{Lattimore:18}).
We define the zero mean noise variable $\eta$ by
\begin{align}
  \eta_t \ &= \ \eta(s_t,a_t) \ = \  R_t \ - \ r(s_t,a_t)  \ , 
\end{align} 
where we assume that $\eta_t$ is a conditionally $\sigma$-subgaussian noise variable. 
Therefore, $\eta$ is distributed 
according to $p(r - r(s_t,a_t) \mid s_t,a_t)$ and fulfills
\begin{align}
  \EXP\left[ \eta(s_t,a_t) \right] \ &= \  0 \ , \\
  \EXP\left[ \exp(\lambda \eta(s_t,a_t) \right] \ &\leq \  \exp(\lambda^2 \sigma^2 /2) \ .
\end{align} 
Subgaussian random variables have tails that decay almost as fast as a Gaussian.
If the reward $r$ is bounded by $|r|<B$, then $\eta$ is bounded by $|\eta|<B$ 
and, therefore, a $B$-subgaussian. 
For binary rewards it is of interest that 
a Bernoulli variable is 0.5-subgaussian \cite[p. 71]{Lattimore:18}.
In summary, an optimal reward redistribution reduces the MDP 
to a stochastic contextual bandit problem.


\subsubsection{Relation to ''Backpropagation through a Model´´}
The relation of reward redistribution if applied to policy gradients
and ''Backpropagation through a Model´´ is discussed here.
For a delayed reward that is only received at the end of an episode, 
we decompose the return $\tilde{r}_{T+1}$ into
\begin{align}
  &g(\Delta_{0:T}) \ = \ \tilde{r}_{T+1} \ = \ \sum_{t=0}^T h(\Delta(s_{t-1},a_{t-1},s_t,a_t)) \ . 
\end{align}
The policy gradient for an optimal reward redistribution is
\begin{align}
  &\EXP_{\pi}\left[ \nabla_{\theta} \log \pi(a_t \mid
   s_t;\Bth) \ h(\Delta(s_{t-1},a_{t-1},s_t,a_t))  \right]  \ . 
\end{align} 
Summing up the gradient for one episode, the gradient becomes 
\begin{align}
  &\EXP_{\pi}\left[ \sum_{t=0}^T \nabla_{\theta} \log \pi(a_t \mid
   s_t;\Bth) \ h(\Delta(s_{t-1},a_{t-1},s_t,a_t))  \right] \\ \nonumber
   &=  \ \EXP_{\pi}\left[  \BJ_{\theta}(\log \pi( \Ba \mid \Bs;\Bth)) \ 
   \Bh(\Delta(\Bs',\Ba',\Bs,\Ba))  \right]  \ , 
\end{align} 
where  $\Ba'=(a_{-1},a_0,a_1,\ldots,a_{T-1})$ 
and $\Ba=(a_0,a_1,\ldots,a_T)$ are the sequences of actions, 
$\Bs'=(s_{-1},s_0,s_1,\ldots,s_{T-1})$ and $\Bs=(s_0,s_1,\ldots,s_T)$ 
are the sequences of states,
$\BJ_{\theta}(\log \pi)$ is the Jacobian of the $\log$-probability of the 
state sequence with respect to the parameter vector $\Bth$,
and $\Bh(\Delta(\Bs',\Ba',\Bs,\Ba))$ is the vector with
entries $h(\Delta(s_{t-1},a_{t-1},s_t,a_t))$. 

An alternative approach via sensitivity analysis is ''Backpropagation through a Model´´, 
where $g(\Delta_{0:T})$ is maximized, that is, the return is maximized.
Continuous actions are directly fed into $g$ while probabilistic actions are 
sampled before entering $g$. Analog to gradients used for Restricted Boltzmann Machines,
for probabilistic actions the $\log$-likelihood of the actions is used to construct a gradient.
The likelihood can also be formulated as the cross-entropy
between the sampled actions and the action probability.
The gradient for ''Backpropagation through a Model´´ is
\begin{align}
  &\EXP_{\pi}\left[ \BJ_{\theta}(\log \pi( \Ba \mid \Bs;\Bth)) \ 
  \nabla_{a}  g(\Delta_{0:T})  \right] \ , 
\end{align} 
where $\nabla_{a}  g(\Delta_{0:T})$ is the gradient of $g$ with respect to the action sequence $\Ba$. 

If for ''Backpropagation through a Model´´ 
the model gradient with respect to actions
is replaced by the vector of contributions of actions in the model, 
then we obtain 
redistribution applied to policy gradients.

\section{Bias-Variance Analysis of MDP Q-Value Estimators}
\label{sec:AbiasVariance}

Bias-variance investigations have been done for $Q$-learning.
Gr{\"{u}}new{\"{a}}lder \& Obermayer \cite{Grunewalder:11}
investigated the bias of temporal
difference learning (TD), Monte Carlo estimators (MC), and least-squares temporal
difference learning (LSTD).
Mannor et al.\ \cite{Mannor:07} and O'Donoghue et al.\ \cite{ODonoghue:17}
derived bias and variance expressions for updating
$Q$-values.

The true, but unknown, action-value function $q^\pi$
is the expected future return.
We assume to have the data $D$,
which is a set of state-action sequences with return,
that is a set of episodes with return.
Using data $D$, $q^\pi$ is 
estimated by $\hat{q}^\pi = \hat{q}^\pi (D)$, which
is an estimate with bias and variance.
For bias and variance we have to compute
the expectation $\EXP_D \left[ . \right]$ over the data $D$.
The mean squared error (MSE) of an estimator $\hat{q}^\pi(s,a)$ is
\begin{align}
\label{eq:mse}
  \mse \hat{q}^\pi(s,a) \ &= \ \EXP_D \left[ \big( \hat{q}^\pi(s,a)\ - \ q^\pi(s,a) \big)^2  \right] \ .
\end{align}
The bias of an estimator $\hat{q}^\pi(s,a)$ is
\begin{align}
\label{eq:bias}
  \bias \hat{q}^\pi(s,a) \ &= \ \EXP_D \left[ \hat{q}^\pi(s,a) \right] \ - \ q^\pi(s,a) \ .
\end{align}
The variance of an estimator $\hat{q}^\pi(s,a)$ is
\begin{align}
\label{eq:var}
 \var \hat{q}^\pi(s,a) \ &= \ \EXP_D \left[ \big( \hat{q}^\pi(s,a) \ - \
  \EXP_D \left[\hat{q}^\pi(s,a) \right] \big)^2  \right] \ .
\end{align}
The bias-variance decomposition of the MSE of an estimator $\hat{q}^\pi(s,a)$ is
\begin{align}
\label{eq:mseDecomp}
 \mse \hat{q}^\pi(s,a) \ &= \ \var \hat{q}^\pi(s,a) \ + \  \big( \bias \hat{q}^\pi(s,a)
 \big)^2 \ .
\end{align}

The bias-variance decomposition of the MSE of an estimator
$\hat{q}^\pi$ as a vector is
\begin{align}
  \label{eq:mseVec}
  \mse \hat{q}^\pi  \ &= \ \EXP_D \left[ \sum_{s,a} \big(
    \hat{q}^\pi(s,a)\ - \ q^\pi(s,a) \big)^2  \right] \ = \ \EXP_D
  \left[ \| \hat{q}^\pi \ - \ q^\pi \|^2  \right] \ , \\ 
  \bias \hat{q}^\pi \ &= \ \EXP_D \left[ \hat{q}^\pi \right] \ - \
  q^\pi \ , \\
 \var \hat{q}^\pi \ &= \ \EXP_D \left[ \sum_{s,a} \big( \hat{q}^\pi(s,a) \ - \
  \EXP_D \left[\hat{q}^\pi(s,a) \right] \big)^2  \right] \ = \ \TR \VAR_D
 \left[ \hat{q}^\pi \right] \ , \\ 
 \mse \hat{q}^\pi  \ &= \ \var \hat{q}^\pi  \ + \
 \big( \bias \hat{q}^\pi \big)^T \bias \hat{q}^\pi \ .
\end{align}

\subsection{Bias-Variance for MC and TD Estimates of the Expected Return}
\label{sec:Abias_variance_estimator}

{\bf Monte Carlo (MC)} computes the arithmetic mean
$\hat{q}^\pi(s,a)$ of $G_t$ for $(s_t=s,a_t=a)$ over the episodes
given by the data.

For {\bf temporal difference (TD)} methods,
like SARSA, with learning rate $\alpha$ the updated estimate of ${q}^\pi(s_t,a_t)$ is:
\begin{align} \nonumber
  \left(\hat{q}^\pi\right)^{\nn}(s_t,a_t) \ &= \ 
  \hat{q}^\pi(s_t,a_t) \ - \
  \alpha \ \big( \hat{q}^\pi(s_t,a_t) \ - \ R_{t+1} \ - \
  \gamma \ \hat{q}^\pi(s_{t+1}, a_{t+1}) \big) \\
  &= \ (1 \ - \ \alpha) \ \hat{q}^\pi(s_t,a_t) \ + \
  \alpha \ \big( R_{t+1} \ + \ \gamma \ \hat{q}^\pi(s_{t+1}, a_{t+1})\big) \ .
\end{align}
Similar updates are used for expected SARSA and $Q$-learning, where
only $a_{t+1}$ is chosen differently. 
Therefore,
for the estimation of $\hat{q}^\pi(s_t,a_t)$, SARSA and $Q$-learning perform 
an exponentially weighted arithmetic mean of  $(R_{t+1} +  \gamma  \hat{q}^\pi(s_{t+1}, a_{t+1}))$.
If for the updates $\hat{q}^\pi(s_{t+1}, a_{t+1})$ is fixed on some data, then
SARSA and $Q$-learning perform an exponentially weighted arithmetic mean of the immediate
reward $R_{t+1}$ plus averaging over which $\hat{q}^\pi(s_{t+1},
a_{t+1})$ (which $(s_{t+1}, a_{t+1})$) is chosen.
In summary, TD methods like SARSA and $Q$-learning are biased via
$\hat{q}^\pi(s_{t+1}, a_{t+1})$ and perform an exponentially weighted arithmetic mean of the
immediate reward $R_{t+1}$ and the next (fixed) $\hat{q}^\pi(s_{t+1}, a_{t+1})$.

{\bf Bias-Variance for Estimators of the Mean.}
Both Monte Carlo and TD methods, like SARSA and $Q$-learning, respectively, estimate
$q^\pi(s,a)= \EXP\left[ G_t \mid s, a\right]$, which is the expected
future return. The expectations are estimated
by either an arithmetic mean over samples
with Monte Carlo or an exponentially weighted arithmetic mean over samples with TD methods.
Therefore, we are interested in computing the bias and variance of
these estimators of the expectation.
In particular, we consider the arithmetic mean and the
exponentially weighted arithmetic mean.

We assume $n$ samples for a state-action pair $(s,a)$. However, the expected
number of samples depends on the probabilistic number of visits of
$(s,a)$ per episode. 

{\bf Arithmetic mean.}
For $n$ samples $\{X_1,\ldots,X_n\}$ from a distribution with mean
$\mu$ and variance $\sigma^2$, the arithmetic mean, its bias and and its variance are:
\begin{align}
  \hat{\mu}_n \ &= \ \frac{1}{n} \sum_{i=1}^{n} X_i \ , \quad
  \bias(\hat{\mu}_n) \ = \ 0 \ , \quad
  \var (\hat{\mu}_n) \ = \ \frac{\sigma^2}{n} \ .
\end{align}
The estimation variance of the arithmetic mean is determined by $\sigma^2$,
the variance of the distribution the samples are drawn from.

{\bf Exponentially weighted arithmetic mean.}
For $n$ samples $\{X_1,\ldots,X_n\}$ from a distribution with mean
$\mu$ and variance $\sigma$, the variance of the exponential mean with initial value $\mu_0$ is
\begin{align}
  \hat{\mu}_0 \ &= \ \mu_0 \ , \quad
  \hat{\mu}_k \ = \ (1 \ - \ \alpha) \ \hat{\mu}_{k-1} \ + \ \alpha \ X_k \ , 
\end{align}
which gives
\begin{align}
  \hat{\mu}_n \ &= \  \alpha \ \sum_{i=1}^{n}  (1 \  - \ \alpha)^{n-i} \ X_i \ + \ (1 \ - \ \alpha)^{n} \ \mu_0 \ . 
\end{align}

This is a weighted arithmetic mean with 
exponentially decreasing weights, 
since the coefficients sum up to one:
\begin{align}
   &\alpha \ \sum_{i=1}^{n}  (1 \  - \ \alpha)^{n-i} \ + \ (1 \ - \ \alpha)^{n} 
  \ = \ \alpha \ \frac{1-(1-\alpha)^n}{1-(1-\alpha)} +\ (1-\alpha)^n \\ \nonumber
   &= 1 \ - \ (1 \ - \ \alpha)^n \ + \ (1 \ - \ \alpha)^n \ = \ 1 \ .
\end{align}

The estimator $\hat{\mu}_n$ is biased, since:
\begin{align}
  &\bias(\hat{\mu}_n) \ = \ \EXP \left[\hat{\mu}_n \right] \ - \ \mu \ = \
  \EXP \left[  \alpha \ \sum_{i=1}^{n}  (1 \ - \ \alpha)^{n-i} \ X_i
  \right] \ + \ (1 \ - \ \alpha)^{n} \ \mu_0 \ - \ \mu \\ \nonumber
  &= \ \alpha \ \sum_{i=1}^{n}  (1 \ - \ \alpha)^{n-i} \EXP \left[X_i\right] \ + \ (1 \ - \ \alpha)^{n} \ \mu_0
 \ - \ \mu \\ \nonumber
  &= \ \mu \ \alpha \ \sum_{i=0}^{n-1}  (1 \ - \ \alpha)^i \ + \ (1 \ - \ \alpha)^{n} \ \mu_0
 \ - \ \mu \\ \nonumber
 &= \ \mu \ \left(1 \ - \  (1 \ - \ \alpha)^{n} \right)  \ + \ (1 \ - \ \alpha)^{n} \ \mu_0 \ - \ \mu
 \ = \   (1 \ - \ \alpha)^{n} \ (\mu_0 \ - \ \mu) \ .
\end{align}
Asymptotically ($n \to \infty$) the estimate is unbiased.
The variance is
\begin{align}
  &\var (\hat{\mu}_n) \ = \ \EXP \left[\hat{\mu}_n^2 \right] \ - \  \EXP^2
  \left[\hat{\mu}_n \right] \\ \nonumber
  &= \ \EXP \left[  \alpha^2 \sum_{i=1}^{n} \sum_{j=1}^{n}
    (1 \ - \ \alpha)^{n-i} \ X_i\ (1 \ - \ \alpha)^{n-j} \ X_j\right]
    \\ \nonumber 
    &+ \ \EXP \left[2\ (1 \ - \ \alpha)^{n} \ \mu_0 \ \alpha \ \sum_{i=1}^{n}  (1 \  - \ \alpha)^{n-i} \ X_i  \right] \ + \ (1 \ - \ \alpha)^{2n} \ \mu_0^2 
    \\ \nonumber &- \
     \left( (1 \ - \ \alpha)^{n} \ (\mu_0 \ - \ \mu) \ + \ \mu \right)^2 \\ \nonumber
  &= \ \alpha^2 \ \EXP \left[ \sum_{i=1}^{n}
    (1 \ - \ \alpha)^{2(n-i)} \ X_i^2 \ + \  \sum_{i=1}^{n} \sum_{j=1,j\not=i}^{n}
    (1 \ - \ \alpha)^{n-i} \ X_i\ (1 \ - \ \alpha)^{n-j} \ X_j\right]
    \\ \nonumber 
    &+ \ 2\ (1 \ - \ \alpha)^{n} \ \mu_0 \ \mu \ \alpha \ \sum_{i=1}^{n}  (1 \  - \ \alpha)^{n-i}  \ + \ (1 \ - \ \alpha)^{2n} \ \mu_0^2 
      \\ \nonumber &- \ \left( (1 \ - \ \alpha)^{n} \ \mu_0 \ + \ (1 \ - \  (1 \ - \ \alpha)^{n} ) \ \mu  \right)^2 \\ \nonumber
  &= \ \alpha^2 \ \left( \sum_{i=1}^{n}
    (1 \ - \ \alpha)^{2(n-i)} \ \right(\sigma^2 \ + \ \mu^2 \left) \ + \
    \sum_{i=1}^{n} \sum_{j=1,j\not=i}^{n}
    (1 \ - \ \alpha)^{n-i} \ (1\ - \ \alpha)^{n-j} \ \mu^2 \right) 
    \\ \nonumber 
    &+ \ 2\ (1 \ - \ \alpha)^{n} \ \mu_0 \ \mu \ (1 \ - \  (1 \  - \ \alpha)^{n})  \ + \ (1 \ - \ \alpha)^{2n} \ \mu_0^2     \\ \nonumber 
    &- \ (1 \ - \ \alpha)^{2n} \ \mu_0^2 \ - \ 2 \ (1 \ - \ \alpha)^{n} \ \mu_0 \ (1 \ - \  (1 \ - \ \alpha)^{n} ) \ \mu \ - \ (1 \ - \  (1 \ - \ \alpha)^{n} )^2 \ \mu^2 \\ \nonumber
  &= \  \sigma^2 \ \alpha^2 \ \sum_{i=0}^{n-1}
    \big((1\ - \ \alpha)^2\big)^{i} \ + \ \mu^2   \ \alpha^2 \ \left(\sum_{i=0}^{n-1}
    (1 \ - \ \alpha)^i \right)^2 \ - \ (1 \ - \  (1 \ - \ \alpha)^{n} )^2 \ \mu^2 
 \\ \nonumber
  \\ \nonumber
  &= \  \sigma^2 \ \alpha^2 \
  \frac{1-(1\ - \ \alpha)^{2n}}{1-(1 \ - \ \alpha)^2} \ = \ \sigma^2 \ 
  \frac{\alpha \ (1 \ - \ (1 \ - \ \alpha)^{2n})}{2-\alpha} \ .
\end{align}
Also the estimation variance of the exponentially weighted arithmetic mean is proportional to $\sigma^2$,
which is the variance of the distribution the samples are drawn from.

The deviation of random variable $X$ from its
mean $\mu$ can be analyzed with Chebyshev's inequality.
Chebyshev's inequality \cite{Bienayme:53,Chebyshev:67} states that for 
a random variable $X$ with expected value $\mu$
and variance $\tilde{\sigma}^2$ and for any real number $\epsilon > 0$:
\begin{align}
 \PR \left[|X-\mu| \ \geq \ \epsilon \ \tilde{\sigma} \right] \ &\leq \ \frac{1}{\epsilon^2} 
\end{align}
or, equivalently,
\begin{align}
 \PR \left[|X-\mu| \ \geq \ \epsilon  \right] \ &\leq \ \frac{\tilde{\sigma}^2}{\epsilon^2}
 \ .
\end{align}

For $n$ samples $\{X_1,\ldots,X_n\}$ from a distribution with expectation
$\mu$ and variance $\sigma$ we compute the arithmetic mean
$\frac{1}{n} \sum_{i=1}^n X_i$.
If $X$ is the arithmetic mean, then
$\tilde{\sigma}^2=\sigma^2/n$ and we obtain
\begin{align}
  \PR \left[\left|\frac{1}{n} \sum_{i=1}^n X_i \ - \ \mu \right|
    \ \geq \ \epsilon  \right] \ &\leq \
 \frac{\sigma^2}{n \ \epsilon^2} \ .
\end{align}

Following Gr{\"{u}}new{\"{a}}lder and Obermayer \cite{Grunewalder:11},
Bernstein's inequality can be used to describe the deviation of
the arithmetic mean (unbiased estimator of $\mu$)
from the expectation $\mu$
(see Theorem~6 of G{\'{a}}bor Lugosi's
lecture notes \cite{Lugosi:03}):
\begin{align}
  \PR \left[\left|\frac{1}{n} \sum_{i=1}^n X_i \ - \ \mu \right|
    \ \geq \ \epsilon \right] \ &\leq \
  2 \ \exp \left( - \ \frac{\epsilon^2 \ n}
    {2\ \sigma^2 \ + \ \frac{2\ M \ \epsilon}{3}} \right) \ ,
\end{align}
where $|X  -  \mu |<M$.

\subsection{Mean and Variance of an MDP Sample of the Return}
\label{sec:Abias_variance_sample}

Since the variance of the estimators of the expectations (arithmetic mean and 
exponentially weighted arithmetic mean) is governed by the variance of the samples, 
we compute mean and variance of the return estimate $q^\pi(s,a)$. 
We follow
\cite{Sobel:82,Tamar:12,Tamar:16} for deriving the mean and variance.

We consider an MDP with finite horizon $T$, 
that is, each episode has length $T$.
The finite horizon MDP can be generalized to an MDP 
with absorbing (terminal) state $s=\mathtt{E}$.
We only consider proper policies, that is there exists an integer $n$
such that from any initial state the probability of achieving
the terminal state $\mathtt{E}$ after $n$ steps is strictly positive.
$T$ is the time to the
first visit of the terminal state: $T=\min{k \mid s_k=\mathtt{E}}$.
The return $G_0$ is:
\begin{align}
   G_0 \ &= \ \sum_{k=0}^{T}  \gamma^k \ R_{k+1} \ .  
\end{align}

The action-value function, the $Q$-function, is the expected return
\begin{align}
  G_t \ &= \ \sum_{k=0}^{T-t} \gamma^k \  R_{t+k+1}  
\end{align}
if starting in state $S_t=s$ and action $A_t=a$:
\begin{align}
  q^{\pi}(s,a) \ &= \ \EXP_{\pi} \left[ G_t \mid s, a \right] \ .  
\end{align}

The second moment of the return is:
\begin{align}
  M^{\pi}(s,a) \ &= \ \EXP_{\pi} \left[ G_t^2 \mid s, a
  \right] \ .  
\end{align}

The variance of the return is:
\begin{align}
  V^{\pi}(s,a) \ &= \ \VAR_{\pi} \left[ G_t \mid s, a
  \right] \ = \ M^{\pi}(s,a) \ - \  \big( q^{\pi}(s,a) \big)^2 \ .  
\end{align}

Using  $\EXP_{s',a'}(f(s',a'))= \sum_{s'} p(s'\mid s,a) \sum_{a'} \pi(a' \mid
s') f(s',a')$, and analogously $\VAR_{s',a'}$ and $\VAR_r$, the
next Theorem~\ref{th:Avar} gives mean and variance
$V^{\pi}(s,a) =  \VAR_{\pi} \left[ G_t \mid s, a \right]$
of sampling returns from an MDP.

\begin{theoremA}
  \label{th:Avar}
  The mean  $q^\pi$ and variance $V^{\pi}$ 
  of sampled returns from an MDP are
\begin{align} \nonumber
  &q^\pi(s,a) =   \sum_{s',r} p(s',r\mid s,a) 
  \left(r + \gamma \sum_{a'} \pi(a' \mid s')  q^\pi(s',a')
  \right) =   r(s,a)  +  \gamma 
  \EXP_{s',a'} \left[q^\pi(s', a') \mid s,a\right] , \\ \label{eq:Avarq}
  &V^{\pi}(s,a) = \VAR_r \left[ r \mid s,a \right]  +  \gamma^2   \left(
    \EXP_{s',a'}  \left[V^{\pi}(s',a') \mid s,a\right]
  +   \VAR_{s',a'} \left[q^\pi(s', a') \mid s,a\right]  \right)  .  
\end{align}
\end{theoremA}

\begin{proof}
\label{proof:Avar}
The Bellman equation for $Q$-values is
\begin{align}
  q^\pi(s,a) \ &= \  \sum_{s',r} p(s',r\mid s,a) \ 
  \left(r \ + \ \gamma \ \sum_{a'} \pi(a' \mid s') \ q^\pi(s',a')  \right) \\
  \nonumber &= \
  r(s,a) \ + \ \gamma \
  \EXP_{s',a'} \left[q^\pi(s', a') \mid s,a\right]\ . 
\end{align}
This equation gives the mean if drawing one sample.
We use
\begin{align}
  r(s,a) \ &= \ \sum_{r} r \ p(r \mid s,a) \ , \\
  r^2(s,a) \ &= \ \sum_{r} r^2 \ p(r \mid s,a) \ . 
\end{align}
For the second moment, we obtain \cite{Tamar:12}:
\begin{align}
  \label{eq:secM}
  &M^{\pi}(s,a) \ = \ \EXP_{\pi} \left[ G_t^2 \mid s, a \right]
  \\ \nonumber &= \ 
  \EXP_{\pi} \left[ \left( \sum_{k=0}^{T-t} \gamma^k  R_{t+k+1} \right)^2 \mid s, a \right]
  \\ \nonumber &= \ 
  \EXP_{\pi} \left[ \left( R_{t+1} \ + \ \sum_{k=1}^{T-t} \gamma^k
      \ R_{t+k+1} \right)^2 \mid s, a \right]
  \\ \nonumber &= \ 
  r^2(s,a) \ + \ 2 \ r(s,a)\ \EXP_{\pi} \left[  \sum_{k=1}^{T-t} \gamma^k
      \ R_{t+k+1}  \mid s, a \right] \\ \nonumber
  &+ \ 
  \EXP_{\pi} \left[ \left(\sum_{k=1}^{T-t} \gamma^k
      \ R_{t+k+1} \right)^2 \mid s, a \right] 
  \\ \nonumber &= \ 
  r^2(s,a)  \ + \ 2 \gamma  \ r(s,a) \ \sum_{s'} p(s'\mid s,a) \ \sum_{a'}
  \pi(a' \mid s') \ q^\pi(s',a') \\\nonumber
  &+ \ \gamma^2  \  \sum_{s'} p(s'\mid s,a) \ \sum_{a'}
  \pi(a' \mid s') \ M^{\pi}(s',a') 
  \\ \nonumber &= \ 
  r^2(s,a)  \ + \ 2 \gamma  \ r(s,a)  \   \EXP_{s',a'} \left[ q^\pi(s',a')\mid s,a\right]
  \ + \ \gamma^2  \  \EXP_{s',a'} \left[M^{\pi}(s',a') \mid s,a\right] \ .
\end{align}

For the variance, we obtain:
\begin{align}
  \label{eq:Avarq1}
  &V^{\pi}(s,a) \ = \  M^{\pi}(s,a) \ - \  \big( q^{\pi}(s,a)
  \big)^2 \\\nonumber
  &= \ r^2(s,a) \ - \ (r(s,a))^2  \ + \ \gamma^2  \  \EXP_{s',a'} \left[M^{\pi}(s',a') \mid s,a\right]
  \ - \  \gamma^2 \
  \EXP_{s',a'}^2 \left[q^\pi(s', a') \mid s,a\right]
 \\\nonumber
  &= \ \VAR_r \left[ r \mid s,a \right] \ + \ \gamma^2  \  \left(
    \EXP_{s',a'}  \left[M^{\pi}(s',a') \ - \
      \big( q^\pi (s', a')\big)^2\mid s,a\right] \right. \\ \nonumber
   &\left. - \  
  \EXP_{s',a'}^2 \left[q^\pi(s', a') \mid s,a\right] \ + \ 
  \EXP_{s',a'} \left[ \big( q^\pi (s', a')\big)^2\mid s,a\right] \right)
 \\\nonumber
  &= \ \VAR_r \left[ r \mid s,a \right] \ + \ \gamma^2  \  \left(
    \EXP_{s',a'}  \left[V^{\pi}(s',a') \mid s,a\right]
  \ + \   \VAR_{s',a'} \left[q^\pi(s', a') \mid s,a\right]  \right)
  \ .  
\end{align}
\end{proof}

For deterministic reward, that is,
$\VAR_r \left[ r \mid s,a \right]=0$,
the corresponding result is given as
Equation~(4) in Sobel 1982 \cite{Sobel:82}
and as Proposition~3.1 (c) in Tamar et al.\ 2012 \cite{Tamar:12}.

For temporal difference (TD) learning,
the next $Q$-values are fixed to $\hat{q}^\pi(s', a')$ when drawing a sample.
Therefore, TD is biased, that is, both SARSA and $Q$-learning are biased.
During learning with according updates of $Q$-values, $\hat{q}^\pi(s', a')$
approaches $q^\pi(s', a')$, and the bias is reduced.
However, this reduction of the bias is exponentially small in the number of time
steps between reward and updated $Q$-values, as we will see later.
The reduction of the bias is exponentially small for eligibility
traces, too.

The variance recursion  Eq.~\eqref{eq:Avarq} of sampled returns consists of three parts:
\begin{itemize}
\item (1) the immediate variance
$\VAR_r \left[ r \mid s,a \right]$ of the
immediate reward stemming from the
probabilistic reward $p(r \mid s,a)$,
\item (2) the local
variance $\gamma^2\VAR_{s',a'} \left[q^\pi(s', a') \mid  s,a\right]$ 
from state transitions $p(s' \mid s,a)$ and new actions $\pi(a' \mid s')$,
\item (3) the expected
variance $\gamma^2\EXP_{s',a'}  \left[V^{\pi}(s',a') \mid s,a\right]$
of the next $Q$-values.
\end{itemize}
For different settings the following parts may be zero:
\begin{itemize}
\item (1) the immediate variance
$\VAR_r \left[ r \mid s,a \right]$ is zero for deterministic
immediate reward,
\item (2) the local
variance $\gamma^2\VAR_{s',a'} \left[q^\pi(s', a') \mid  s,a\right]$ 
is zero for (i) deterministic state transitions and deterministic policy
and for (ii) $\gamma=0$ (only immediate reward),
\item (3) the expected
variance $\gamma^2\EXP_{s',a'}  \left[V^{\pi}(s',a') \mid s,a\right]$
of the next $Q$-values is zero for (i) temporal difference (TD) learning,
since the next $Q$-values are fixed and set to their current estimates
(if just one sample is drawn)
and for (ii) $\gamma=0$ (only immediate reward).
\end{itemize}

The local variance $\VAR_{s',a'} \left[q^\pi(s', a') \mid s,a\right]$
is the variance of a linear combination of $Q$-values
weighted by a multinomial distribution
$\sum_{s'} p(s'\mid s,a) \ \sum_{a'} \pi(a' \mid s')$.
The local variance is
\begin{align}
  &\VAR_{s',a'} \left[q^\pi(s', a') \mid s,a\right] \ = \
  \sum_{s'} p(s'\mid s,a) \ \sum_{a'} \pi(a' \mid s')\
  \big( q^\pi(s', a')\big)^2 \\\nonumber
  &- \
    \left( \sum_{s'} p(s'\mid s,a) \ \sum_{a'} \pi(a' \mid
    s') \ q^\pi(s', a')\right)^2 \ .
\end{align}
This result is Equation~(6) in Sobel 1982 \cite{Sobel:82}.
Sobel derived these formulas also for finite horizons and
an analog formula if the reward depends
also on the next state, that is, for $p(r \mid s,a,s')$.

Monte Carlo uses the accumulated future rewards for updates, therefore
its variance is given by the recursion in Eq.~\eqref{eq:Avarq}.
TD, however, fixes $q^\pi(s', a')$ to the current
estimates $\hat{q}^\pi(s', a')$, which do not change in the current
episode. 
Therefore, TD has $\EXP_{s',a'} \left[V^{\pi}(s',a') \mid s,a\right]=0$ and
only the local variance
$\VAR_{s',a'} \left[q^\pi(s', a') \mid s,a\right]$
is present.
For $n$-step TD, the recursion in Eq.~\eqref{eq:Avarq} must be applied
$(n-1)$ times. Then, the expected next variances are zero since
the future reward is estimated by $\hat{q}^\pi(s', a')$.

{\bf Delayed rewards}.
For TD and delayed rewards, information on new data is only captured by the
last step of an episode that receives a reward. This reward is used
to update the estimates of the $Q$-values of the
last state $\hat{q}(s_T,a_T)$.
Subsequently, the reward information is propagated
one step back via the estimates $\hat{q}$ for each sample.
The drawn samples (state action sequences)
determine where information is propagated back.
Therefore, delayed reward introduces a large bias for TD over a long
period of time,
since the estimates $\hat{q}(s,a)$ need a long time to reach their true $Q$-values.

For Monte Carlo and delayed rewards,
the immediate variance $\VAR_r \left[ r \mid s,a \right]=0$ except for the
last step of the episode. 
The delayed reward
increases the variance of $Q$-values according to Eq.~\eqref{eq:Avarq}.

{\bf Sample Distribution Used by Temporal Difference and
  Monte Carlo.}
  \label{sec:ASDTDMC}
Monte Carlo (MC) sampling uses the true mean and true variance, where the true mean is
\begin{align}
  &q^\pi(s,a) \ = \  
  r(s,a) \ + \ \gamma \
  \EXP_{s',a'} \left[q^\pi(s', a') \mid s,a\right] 
\end{align}
and the true variance is
\begin{align}
  &V^{\pi}(s,a) \ = \ \VAR_r \left[ r \mid s,a \right] \ + \ \gamma^2  \  \left(
    \EXP_{s',a'}  \left[V^{\pi}(s',a') \mid s,a\right]
  \ + \   \VAR_{s',a'} \left[q^\pi(s', a') \mid s,a\right]  \right)
  \ .  
\end{align}

Temporal difference (TD) methods replace
$q^\pi(s', a')$ by $\hat{q}^\pi(s', a')$ which does not depend on
the drawn sample.
The mean which is used by temporal difference is
\begin{align}
  q^\pi(s,a) \ &= \  
  r(s,a) \ + \ \gamma \
  \EXP_{s',a'} \left[\hat{q}^\pi(s', a') \mid s,a\right]\ . 
\end{align}
This mean is biased by
\begin{align}
  &\gamma \ \left(
    \EXP_{s',a'} \left[\hat{q}^\pi(s', a') \mid s,a\right]
    \ - \
     \EXP_{s',a'} \left[q^\pi(s', a') \mid s,a\right] 
  \right) \ . 
\end{align}
The variance used by temporal difference is
\begin{align}
  V^{\pi}(s,a) \ &= \ \VAR_r \left[ r \mid s,a \right] \ + \ \gamma^2  \  
    \VAR_{s',a'} \left[\hat{q}^\pi(s', a') \mid s,a\right] \ ,  
\end{align}
since $V^{\pi}(s',a')=0$ if $\hat{q}^\pi(s', a')$ is used instead of
the future reward of the sample.
The variance of TD is smaller than for MC, since variances are not
propagated back.

\subsection{TD corrects Bias exponentially slowly with Respect to
  Reward Delay}
\label{sec:ATDslow}

\paragraph{Temporal Difference.}
We show that TD updates for delayed rewards are exponentially small,
fading exponentially with the number of delay steps.
$Q$-learning with learning rates $1/i$ at the $i$th update
leads to an arithmetic mean as estimate, which was shown to be
exponentially slow \cite{Beleznay:99}. 
If for a fixed learning rate the agent always travels along 
the same sequence of states, then TD is superquadratic \cite{Beleznay:99}.
We, however, consider the general case where the agent travels 
along random sequences due to a random environment or due to exploration.
For a fixed learning rate,
the information of the delayed reward has to be propagated back
either through the Bellman error or via eligibility traces.
We first consider backpropagation of reward information via the
Bellman error.
For each episode
the reward information is propagated back one step at visited
state-action pairs via the TD update rule.
We denote the $Q$-values of episode $i$ as $q^i$ and assume that the
state action pairs $(s_t,a_t)$ are the most visited ones.
We consider the update of $q^i(s_t,a_t)$ of
a state-action pair $(s_t,a_t)$ that is visited at time $t$ in
the $i$th episode: 
\begin{align}
 q^{i+1}(s_t,a_t) \ &= \ q^i(s_t,a_t) \ + \
 \alpha \ \delta_t \ , \\
\delta_t  \ &= \ r_{t+1} \ + \ \max_{a'} q^i(s_{t+1},a') \ - \
  q^i(s_t,a_t) \ \text{~($Q$-learning)}\\ 
\delta_t  \ &= \ r_{t+1} \ + \ \sum_{a'} \pi(a' \mid s_{t+1}) \ q^i(s_{t+1},a') \ - \
  q^i(s_t,a_t) \ \text{~(expected SARSA)} \ .
\end{align} 

\paragraph{Temporal Difference with Eligibility Traces.}
Eligibility traces have been introduced to propagate
back reward information of an episode and are now standard for
TD($\lambda$) \cite{Singh:96}.
However, the eligibility traces are exponentially decaying when
propagated back.
The accumulated trace is defined as \cite{Singh:96}:
\begin{align}
  e_{t+1}(s,a) \ &= \
  \begin{cases}
    \gamma \ \lambda \ e_t(s,a) & \text{for } s \not=s_t \ \text{or }
    a \not=a_t \ , \\
    \gamma \ \lambda \ e_t(s,a) \ + \ 1 &
    \text{for } s =s_t \ \text{and } a =a_t \ , 
  \end{cases}
\end{align} 
while the replacing trace is defined as \cite{Singh:96}:
\begin{align}
  e_{t+1}(s,a) \ &= \
  \begin{cases}
    \gamma \ \lambda \ e_t(s,a) & \text{for } s \not=s_t \ \text{or }
    a \not=a_t \ , \\
    1 & \text{for } s =s_t \ \text{and } a =a_t \ . 
  \end{cases}
\end{align} 

With eligibility traces using $\lambda \in
[0,1]$, the $\lambda$-return $G_t^{\lambda}$ is \cite{Sutton:18book}
\begin{align}
  G_t^{\lambda} \ &= \ (1-\lambda) \sum_{n=1}^{\infty} \lambda^{n-1} \ G_t^{(n)} \ , \\
  G_t^{(n)} \ &= \ r_{t+1} \ + \ \gamma \ r_{t+2} \ + \ \ldots \ + \ \gamma^{n-1} r_{t+n} \ + \ 
  \gamma^{n-1} \ V(s_{t+n}) \ .
\end{align} 
We obtain
\begin{align}
  G_t^{\lambda} \ &= \ (1-\lambda) \sum_{n=1}^{\infty} \lambda^{n-1} \ G_t^{(n)}  \\ \nonumber
  &= \ 
  (1-\lambda) \left(r_{t+1} \ + \ \gamma \ V(s_{t+1}) \ + \ 
  \sum_{n=2}^{\infty} \lambda^{n-1} \ G_t^{(n)}\right) \\\nonumber
  &= \ 
  (1-\lambda) \left(r_{t+1} \ + \ \gamma \ V(s_{t+1}) \ + \ 
  \sum_{n=1}^{\infty} \lambda^{n} \ G_t^{(n+1)} \right) \\\nonumber
  &= \ 
  (1-\lambda) \left(r_{t+1} \ + \ \gamma \ V(s_{t+1}) \ + \ 
  \lambda \ \gamma \sum_{n=1}^{\infty} \lambda^{n-1} \ G_{t+1}^{(n)} \ + \  \sum_{n=1}^{\infty} \lambda^{n} \ r_{t+1} \right)\\\nonumber
  &= \ 
  (1-\lambda)  \sum_{n=0}^{\infty} \lambda^{n} \ r_{t+1} \ + \ 
  (1-\lambda) \gamma \ V(s_{t+1}) \ + \ \lambda \ \gamma \ G_{t+1}^{\lambda}\\\nonumber
  &= \  r_{t+1} \ + \ 
  (1-\lambda) \gamma \ V(s_{t+1}) \ + \ \lambda \ \gamma \ G_{t+1}^{\lambda} \ .
\end{align}

We use the naive $Q(\lambda)$, where eligibility traces are not set to
zero.
In contrast, Watkins' $Q(\lambda)$ \cite{Watkins:89}
zeros out eligibility traces
after non-greedy actions, that is, if not the $\max_a$ is chosen.
Therefore, the decay is even stronger for  Watkin's $Q(\lambda)$.
Another eligibility trace method is Peng's $Q(\lambda)$ \cite{Peng:96}
which also does not zero out eligibility traces.

The next Theorem~\ref{th:AexponDecay} 
states that the decay of TD is exponential for
$Q$-value updates in an MDP with delayed reward, even for eligibility traces.
Thus, for delayed rewards TD requires exponentially many updates to correct the bias, 
where the number of updates is exponential in the delay steps.  
\begin{theoremA}
\label{th:AexponDecay}
For initialization  $q^0(s_t,a_t)=0$ and
delayed reward with $r_t=0$ for $t \leq T$, 
$q(s_{T-i},a_{T-i})$ receives its 
first update not earlier than at episode $i$ via
$q^i(s_{T-i},a_{T-i}) =  \alpha^{i+1} r_{T+1}^1$,
where $r_{T+1}^1$ is the reward of episode 1.
Eligibility traces with $\lambda \in [0,1)$ 
lead to an exponential decay of $(\gamma \lambda)^k$ when the reward
is propagated $k$ steps back.
\end{theoremA}

\begin{proof}
\label{proof:ATD}

If we assume that $Q$-values are initialized with zero, then
$q^0(s_t,a_t)=0$ for all $(s_t,a_t)$.
For delayed rewards we have $r_t=0$ for $t \leq T$.
The $Q$-value  $q(s_{T-i},a_{T-i})$ at time $T-i$ can receive an
update for the first time at episode $i$. Since all $Q$-values have
been initialized with zero, the update is
\begin{align}
  q^i(s_{T-i},a_{T-i}) \ &= \ \alpha^{i+1} \ r_{T+1}^1 \ ,
\end{align} 
where $r_{T+1}^1$ is the reward at time $T+1$ for episode 1.

We move on to eligibility traces, where the
update for a state $s$ is
\begin{align}
 q_{t+1}(s,a) \ &= \ q_t(s,a) \ + \
 \alpha \ \delta_t \  e_t(s,a) \ , \\
\delta_t  \ &= \ r_{t+1} \ + \ \max_{a'} q_t(s_{t+1},a') \ - \
  q_t(s_t,a_t) \ .
\end{align} 

If states are not revisited,
the eligiblity trace at time $t+k$ for a visit of state $s_t$ at time $t$ is:
\begin{align}
  e_{t+k}(s_t,a_t) \ &= \  \big(\gamma \ \lambda \big)^k \ .
\end{align} 
If all $\delta_{t+i}$ are zero except for $\delta_{t+k}$,
then the update of $q(s,a)$ is
\begin{align}
 q_{t+k+1}(s,a) \ &= \ q_{t+k}(s,a) \ + \
 \alpha \ \delta_{t+k} \  e_{t+k}(s,a) \ = \
 q_{t+k}(s,a) \ + \
 \alpha \  \big(\gamma \ \lambda \big)^k \ \delta_{t+k} \ .
\end{align} 
\end{proof}

A learning rate of $\alpha=1$ does not work since it would imply to
forget all previous learned estimates, and therefore no averaging over
episodes would exist.
Since $\alpha<1$, we observe exponential decay backwards in time for online updates.

\subsection{MC affects the Variance of Exponentially Many Estimates
  with Delayed Reward }
  \label{sec:AMCvariance}

The variance for Monte Carlo is
\begin{align}
  &V^{\pi}(s,a) \ = \ \VAR_r \left[ r \mid s,a \right] \ + \ \gamma^2  \  \left(
    \EXP_{s',a'}  \left[V^{\pi}(s',a') \mid s,a\right]
  \ + \   \VAR_{s',a'} \left[q^\pi(s', a') \mid s,a\right]  \right)
  \ .  
\end{align}
This is a Bellman equation of the variance. For undiscounted reward $\gamma=1$, we obtain
\begin{align}
  V^{\pi}(s,a) \ &= \ \VAR_r \left[ r \mid s,a \right] \ + \ 
    \EXP_{s',a'}  \left[V^{\pi}(s',a') \mid s,a\right]
  \ + \   \VAR_{s',a'} \left[q^\pi(s', a') \mid s,a\right]    \ .  
\end{align}

If we define the ``on-site'' variance $\omega$ as
\begin{align}
  \omega(s,a) \ &= \ \VAR_r \left[ r \mid s,a \right] \ + \
  \VAR_{s',a'} \left[q^\pi(s', a') \mid s,a\right] \ ,
\end{align}
we get
\begin{align}
  V^{\pi}(s,a) \ &= \  \omega(s,a)  \ + \ 
    \EXP_{s',a'}  \left[V^{\pi}(s',a') \mid s,a\right]  \ .  
\end{align}
This is the solution of the
general formulation of the
Bellman operator. The Bellman operator
is defined component-wise for any variance $V$ as
\begin{align}
  \rT^\pi \left[ V \right] (s,a) \ &= \
  \omega(s,a)  \ + \ 
    \EXP_{s',a'}  \left[V(s',a') \mid s,a\right]  \ . 
\end{align}
According to the results in Section~\ref{sec:ApropPoly},
for proper policies $\pi$ a unique fixed point $V^\pi$ exists:
\begin{align}
  V^\pi \ &= \ \rT^\pi \left[ V^\pi\right] \\
  V^\pi \ &= \ \lim_{k\to \infty} \left(\rT^\pi\right)^k V  \ ,
\end{align}
where $V$ is any initial variance.
In Section~\ref{sec:ApropPoly} it was shown that
the operator $\rT^\pi$ is
continuous, monotonically increasing (component-wise larger or smaller),
and a contraction mapping for
a weighted sup-norm. If we define the operator $\rT^\pi$ as depending on the on-site
variance $\omega$, that is
 $\rT^\pi_{\omega}$, then it is monotonically in $\omega$. 
We obtain component-wise for $\omega>\tilde{\omega}$:
\begin{align}
  &\rT^\pi_{\omega} \left[ q \right](s,a)  \ - \  \rT^\pi_{\tilde{\omega}}
  \left[ q \right](s,a)   \\ \nonumber
  &= \ \left( \omega(s,a)
  \ + \ \EXP_{s',a'} \left[q(s',a')  \right] \right)  \ - \ \left( \tilde{\omega}(s,a)
  \ + \ \EXP_{s',a'} \left[q(s',a')  \right] \right)\\ \nonumber
  &= \ \omega(s,a) \ - \  \tilde{\omega}(s,a) \ \geq \ 0 \ .
\end{align}

It follows for the fixed points $V^{\pi}$ of $\rT^\pi_{\omega}$ and
$\widetilde{V}^{\pi}$ of $\rT^\pi_{\tilde{\omega}}$: 
\begin{align}
  V^{\pi}(s,a) \ &\geq \  \widetilde{V}^{\pi}(s,a) \ .
\end{align}
Therefore if 
\begin{align}
  \omega(s,a) \ &= \ \VAR_r \left[ r \mid s,a \right] \ + \
  \VAR_{s',a'} \left[q^\pi(s', a') \mid s,a\right]  \ \geq \\ \nonumber
  \tilde{\omega}(s,a) \ &= \ \widetilde{\VAR}_r \left[ r \mid s,a \right] \ + \
  \widetilde{\VAR}_{s',a'} \left[q^\pi(s', a') \mid s,a\right] 
\end{align}
then 
\begin{align}
  V^{\pi}(s,a) \ &\geq \  \widetilde{V}^{\pi}(s,a) \ .
\end{align}

\begin{theoremA}
Starting from the sequence end at $t=T$,
as long as $\omega(s_t,a_t) \geq \tilde{\omega}(s_t,a_t)$ holds also
the following holds:
\begin{align}
  V(s_t,a_t)  \ &\geq \  \widetilde{V}(s_t,a_t)   \ .
\end{align}
If for $(s_t,a_t)$ the strict inequality
$\omega(s_t,a_t) > \tilde{\omega}(s_t,a_t)$ 
holds, then we have the strict inequality
\begin{align}
  V(s_t,a_t)  \ &> \  \widetilde{V}(s_t,a_t)   \ .
\end{align}
If $p(s_t,a_t \mid s_{t-1},a_{t-1}) \not=0$ for some  $(s_{t-1},a_{t-1})$ then
\begin{align}
 \EXP_{s_t,a_t} \left[V(s_t,a_t)  \mid s_{t-1},a_{t-1} \right]  \ &> \ 
 \EXP_{s_t,a_t} \left[\widetilde{V}(s_t,a_t)  \mid s_{t-1},a_{t-1} \right]  \ .
\end{align}
Therefore, the strict inequality 
$\omega(s_t,a_t) > \tilde{\omega}(s_t,a_t)$ 
is propagated back as a strict inequality of variances.
\end{theoremA}

\begin{proof}
Proof by induction:
Induction base:  $V(s_{T+1},a_{T+1})  = \widetilde{V}(s_{T+1},a_{T+1})=0$ and
$\omega(s_T,a_T) = \tilde{\omega}(s_T,a_T)=0$.

Induction step ($(t+1)\rightarrow t$):
The induction hypothesis is that for all $(s_{t+1},a_{t+1})$ we have
\begin{align}
  V(s_{t+1},a_{t+1})  \ &\geq \  \widetilde{V}(s_{t+1},a_{t+1})   
\end{align}
and  $\omega(s_t,a_t) \geq \tilde{\omega}(s_t,a_t)$.
It follows that
\begin{align}
 \EXP_{s_{t+1},a_{t+1}} \left[V(s_{t+1},a_{t+1})  \right]  \ &\geq \ 
 \EXP_{s_{t+1},a_{t+1}} \left[\widetilde{V}(s_{t+1},a_{t+1})  \right]  \ .
\end{align}
We obtain
\begin{align}
  & V(s_t,a_t)  \ - \  \widetilde{V}(s_t,a_t)   \\ \nonumber
  &= \ \left( \omega(s_t,a_t)
  \ + \ \EXP_{s_{t+1},a_{t+1}} \left[V(s_{t+1},a_{t+1})  \right] \right)  
  \ - \ \left( \tilde{\omega}(s_t,a_t)
  \ + \ \EXP_{s_{t+1},a_{t+1}} \left[\widetilde{V}(s_{t+1},a_{t+1})
  \right] \right)\\ \nonumber
  &= \ \omega(s_t,a_t) \ - \  \tilde{\omega}(s_t,a_t) \ \geq \ 0 \ .
\end{align}
If for $(s_t,a_t)$ the strict inequality 
$\omega(s_t,a_t) > \tilde{\omega}(s_t,a_t)$ 
holds, then we have the strict inequality
$V(s_t,a_t)  >  \widetilde{V}(s_t,a_t)$.
If $p(s_t,a_t \mid s_{t-1},a_{t-1}) \not=0$ for some  $(s_{t-1},a_{t-1})$ then
\begin{align}
 \EXP_{s_t,a_t} \left[V(s_t,a_t)  \mid s_{t-1},a_{t-1} \right]  \ &> \ 
 \EXP_{s_t,a_t} \left[\widetilde{V}(s_t,a_t)  \mid s_{t-1},a_{t-1} \right]  \ .
\end{align}
Therefore, the strict inequality 
$\omega(s_t,a_t) > \tilde{\omega}(s_t,a_t)$ 
is propagated back as a strict inequality of variances
as long as  $p(s_t,a_t \mid s_{t-1},a_{t-1}) \not=0$ for some  $(s_{t-1},a_{t-1})$.

The induction goes through as long as 
$\omega(s_t,a_t) \geq \tilde{\omega}(s_t,a_t)$.
\end{proof}

In Stephen Patek's PhD thesis, \cite{Patek:97} Lemma~5.1 on page 88-89 and proof
thereafter state that if  $\tilde{\omega}(s,a)=\omega(s,a)-\lambda$, then
the solution $\widetilde{V}^{\pi}$ is continuous and decreasing in $\lambda$.
From the inequality above it follows that
\begin{align}
  &V^{\pi}(s,a) \ - \  \widetilde{V}^{\pi}(s,a)
  \ = \ \left(\rT^\pi_{\omega} V^{\pi}\right) (s,a) \ - \
  \left(\rT^\pi_{\tilde{\omega}} \widetilde{V}^{\pi}\right) (s,a) \\
  \nonumber &= \
   \omega(s,a)  \ - \ \tilde{\omega}(s,a) \ + \ 
    \EXP_{s',a'}  \left[V^{\pi}(s',a') \ - \ \widetilde{V}^{\pi}(s',a')\mid s,a\right] \\
  \nonumber &\geq \
   \omega(s,a)  \ - \ \tilde{\omega}(s,a) \ .
\end{align}

\paragraph{Time-Agnostic States.}

We defined a Bellman operator as
\begin{align}
  \rT^\pi \left[ \BV^{\pi} \right] (s,a)   \ &= \ \omega(s,a)
  \ + \  \sum_{s'} p(s'\mid s,a) \ \sum_{a'}
  \pi(a' \mid s') \ \BV^{\pi}(s',a')  \\ \nonumber
  &= \  \omega(s,a) \ + \ \left(  \BV^{\pi} \right)^T  \Bp(s,a) \ ,
\end{align}
where $\BV^{\pi}$ is the vector with value $V^{\pi}(s',a')$ at position $(s',a')$
and $\Bp(s,a)$ is the vector with value
$p(s'\mid s,a)\pi(a' \mid s')$ at position $(s',a')$. The fixed point equation is known as
the {\em Bellman equation}. In vector and matrix notation
the Bellman equation reads
\begin{align}
  \rT^\pi \left[ \BV^{\pi} \right]    \ &= \ \Bom
  \ + \  \BP \ \BV^{\pi}  \ ,
\end{align}
where $\BP$ is the row-stochastic matrix with
$p(s'\mid s,a)\pi(a' \mid s')$ at position $((s,a),(s',a'))$.
We assume that the set of state-actions $\{(s,a)\}$ is equal to the set
of next state-actions $\{(s',a')\}$, therefore $\BP$ is a square
row-stochastic matrix.
This Bellman operator has the same characteristics as the Bellman
operator for the action-value function $q^{\pi}$.

Since $\BP$ is a row-stochastic matrix, the
Perron-Frobenius theorem says
that (1) $\BP$ has as largest eigenvalue 1 for which
the eigenvector corresponds to the steady state and
(2) the absolute value of each (complex) eigenvalue is smaller equal 1.
Only the eigenvector to eigenvalue 1 has purely positive real components.
Equation~7 of Bertsekas and Tsitsiklis, 1991, \cite{Bertsekas:91}
states that
\begin{align}
  \label{eq:tfold}
  \left(\rT^\pi \right)^t \left[ \BV^{\pi} \right]    \ &= \
  \sum_{k=0}^{t-1}  \BP^k \ \Bom
  \ + \  \BP^t \ \BV^{\pi}  \ .
\end{align}

Applying the operator $\rT^\pi$ recursively
$t$ times can be written as \cite{Bertsekas:91}:
\begin{align}
  \left(\rT^\pi \right)^t \left[ \BV^{\pi} \right]    \ &= \
  \sum_{k=0}^{t-1}  \BP^k \ \Bom
  \ + \  \BP^t \ \BV^{\pi}  \ .
\end{align}
In particular for $\BV^{\pi}=\BZe$, we obtain
\begin{align}
  \left(\rT^\pi \right)^t \left[ \BZe \right]    \ &= \
  \sum_{k=0}^{t-1}  \BP^k \ \Bom \ .
\end{align}
For finite horizon MDPs,
the values $\BV^{\pi}=\BZe$ are correct for time step $T+1$ since
no reward for $t>T+1$ exists. Therefore, the 
``backward induction algorithm'' \cite{Puterman:90,Puterman:05} gives
the correct solution:
\begin{align}
  \BV^{\pi} \ &= \   \left(\rT^\pi \right)^T \left[ \BZe \right]    \ = \
  \sum_{k=0}^{T-1}  \BP^k \ \Bom \ .
\end{align}
The product of square stochastic matrices is a stochastic matrix,
therefore $\BP^k$ is a stochastic matrix.
Perron-Frobenius theorem states that the spectral radius
$\cR(\BP^k)$ of the stochastic matrix $\BP^k$ is: $\cR(\BP^k)=1$.
Furthermore, the largest eigenvalue is 1 and all
eigenvalues have absolute values smaller or equal one.
Therefore, $\Bom$ can have large influence on $\BV^{\pi}$ at every time step.

\paragraph{Time-Aware States.}

Next we consider time-aware MDPs, where transitions occur only from
states $s_t$ to $s_{t+1}$. The transition matrix from
states $s_t$ to $s_{t+1}$ is denoted by $\BP_t$. We assume that
$\BP_t$ are row-stochastic matrices which are rectangular, that is
$\BP_t\in \dR^{m\times n}$.
\begin{definitionA}
A row-stochastic matrix $\BA \in \dR^{m\times n}$ has
non-negative entries and the entries of each row sum up to one.
\end{definitionA}
It is known that the product of
square stochastic matrices  $\BA \in \dR^{n\times n}$ is a
stochastic matrix. We show in next theorem
that this holds also for rectangular matrices.
\begin{lemmaA}
\label{th:ArowStochastic}
  The product $\BC=\BA \BB$ with $\BC \in \dR^{m\times k}$
  of a row-stochastic matrix $\BA \in \dR^{m\times n}$
and a row-stochastic matrix $\BB \in \dR^{n\times k}$ is row-stochastic.
\end{lemmaA}

\begin{proof}
  All entries of $\BC$ are non-negative since they are sums and
  products of non-negative entries of $\BA$ and
  $\BB$.
  The row-entries of $\BC$ sum up to one:
  \begin{align}
  \sum_{k} C_{ik} \ &= \ \sum_{k} \ \sum_{j} A_{ij} \ B_{jk} \ = \
   \sum_{j} A_{ij} \ \sum_{k} B_{jk} \ = \  \sum_{j} A_{ij} \ = \ 1 \ .
\end{align} 
\end{proof}

We will use the $\infty$-norm and the
1-norm of a matrix, which are defined based on the
$\infty$-norm $\|\Bx\|_{\infty}=\max_i|x_i|$ and 1-norm $\|\Bx\|_1=\sum_i|x_i|$
of a vector $\Bx$.
\begin{definitionA}
The $\infty$-norm of a matrix is the maximum absolute row sum: 
\begin{align}
  \left\|\BA\right\|_{\infty}  \ &= \ \max_{\|\Bx\|_{\infty}=1}
  \left\|\BA \ \Bx \right\|_{\infty} \ = \
  \max_i \ \sum_j \left| A_{ij} \right| \ .
\end{align} 
The 1-norm of a matrix is the maximum absolute column sum: 
\begin{align}
  \left\|\BA\right\|_1  \ &= \ \max_{\|\Bx\|_1=1} \left\|\BA \ \Bx \right\|_1  \ = \
  \max_j \ \sum_i \left| A_{ij} \right| \ .
\end{align} 
\end{definitionA}

The statements of next theorem are known as Perron-Frobenius theorem
for square stochastic matrices $\BA \in \dR^{n\times n}$, e.g.\ that
the spectral radius $\cR$ is $\cR(\BA) = 1$.
We extend the theorem to a ``$\infty$-norm equals one'' property for
rectangular stochastic matrices $\BA \in \dR^{m\times n}$. 
\begin{lemmaA}[Perron-Frobenius]
\label{th:Aperron}
If $\BA \in \dR^{m\times n}$ is a row-stochastic matrix, then
\begin{align}
  \left\|\BA\right\|_{\infty} \ &= \ 1 \ , \quad   \left\|\BA^T
  \right\|_1 \ = \ 1 \ , \text{ and for } n=m  \quad \cR (\BA) \ = \ 1 \ .
\end{align} 
\end{lemmaA}

\begin{proof}
$\BA \in \dR^{m\times n}$ is a row-stochastic matrix, therefore
$A_{ij}=\left| A_{ij} \right|$. Furthermore, the rows of $\BA$ sum up
to one. Thus,  $\left\|\BA\right\|_{\infty} = 1$.
Since the column sums of $\BA^T$ are the row sums of $\BA$, it follows
that  $\left\|\BA^T \right\|_1 = 1$.

For square stochastic matrices, that is $m=n$,
Gelfand's Formula (1941) says that
for any matrix norm $\| . \|$, for the spectral norm
$\cR(\BA)$ of a matrix $\BA \in \dR^{n\times n}$ we obtain:
\begin{align}
 \cR (\BA) \ &= \ \lim_{k\to \infty } \left\|\BA^k\right\|^{1/ k} \ .
\end{align}

Since the product of row-stochastic matrices is a row-stochastic
matrix, $\BA^k$ is a row-stochastic matrix.
Consequently $\left\|\BA^k\right\|_{\infty}=1$ and $\left\|\BA^k\right\|_{\infty}^{1/ k}=1$. 
Therefore, the spectral norm
$\cR(\BA)$ of a row-stochastic matrix $\BA \in \dR^{n\times n}$
is
\begin{align}
 \cR (\BA) \ &= \  1 \ .
\end{align} 

The last statement follows from
Perron-Frobenius theorem, which says that the spectral radius of $\BP$ is 1. 
\end{proof}

Using random matrix theory, we can guess how much the spectral radius
of a rectangular matrix deviates from that of a square matrix.
Let $\BA \in \dR^{m \times n}$ be a matrix whose entries are independent copies of some random
variable with zero mean, unit variance, and finite fourth moment.
The Marchenko-Pastur quarter circular law for rectangular
matrices says that for $n=m$ the maximal singular value
is $2 \sqrt{m}$ \cite{Marchenko:67}.
Asymptotically we have for the maximal
singular value $s_{\max}(\BA) \propto \sqrt{m} + \sqrt{n}$ \cite{Rudelson:10}. 
A bound on the largest singular value is given by \cite{Soshnikov:02}:
\begin{align}
 s_{\max}^2(\BA) \ &\leq  \ (\sqrt{m} \ + \ \sqrt{n})^2 \ + \ O(\sqrt{n}
\ \log(n))  \ \text{a.s.} 
\end{align}
Therefore, a rectangular matrix modifies the largest singular value by a factor of 
$a = 0.5 (1 + \sqrt{n/m})$ compared to a $m\times m$ square matrix.
In the case that tstates are time aware, transitions only occur from
states $s_t$ to $s_{t+1}$. The transition matrix from
states $s_t$ to $s_{t+1}$ is denoted by $\BP_t$.

{\bf States affected by the on-site variance $\Bom_k$ (reachable states).}
Typically,   states in $s_t$ have only few predecessor states in
$s_{t-1}$ compared to $N_{t-1}$, the number of possible states in $s_{t-1}$.
Only for those states in $s_{t-1}$ the transition probability to the
state in $s_t$ is larger than zero.
That is, each $i \in s_{t+1}$ has only few $j \in s_t$ for which
$p_t(i \mid j)>0$.
We now want to know how many states have increased variance due to
$\Bom_k$, that is how many states are affected by $\Bom_k$.
In a general setting, we assume random connections.

Let $N_t$ be the number of all states $s_t$ that are reachable after $t$ time steps
of an episode.
$\bar{N}=1/k \sum_{t=1}^k N_{t}$ is the arithmetic mean of $N_t$.
Let $c_t$ be the
average connectivity of a state in $s_t$ to states in $s_{t-1}$ and
$\bar{c}= \big(\prod_{t=1}^k c_t\big)^{1/k}$ the geometric mean of
the $c_t$.
Let $n_t$ be the number of states in $s_t$ that are affected by the
on-site variance $\Bom_k$ at time $k$ 
for $t \leq k$.
The number of states affected by $\Bom_k$ is $a_k=\sum_{t=0}^k n_t$.
We assume that $\Bom_k$ only has one component larger than zero, that
is, only one state at time $t=k$ is affected: $n_k=1$.
The number of affected edges from $s_t$ to $s_{t-1}$ is $c_t n_t$.
However, states in $s_{t-1}$ may be affected multiple times by different
affected states in $s_t$.
Figure~\ref{fig:affectedStates} shows examples of how affected states affect states in a
previous time step. The left panel shows no overlap
since affected states in $s_{t-1}$ connect only to
one affected state in $s_t$. The right panel shows some overlap
since affected states in $s_{t-1}$ connect to multiple
affected states in $s_t$.

\begin{figure}[htp]
\centering
\includegraphics[angle=0,width=0.49\textwidth]{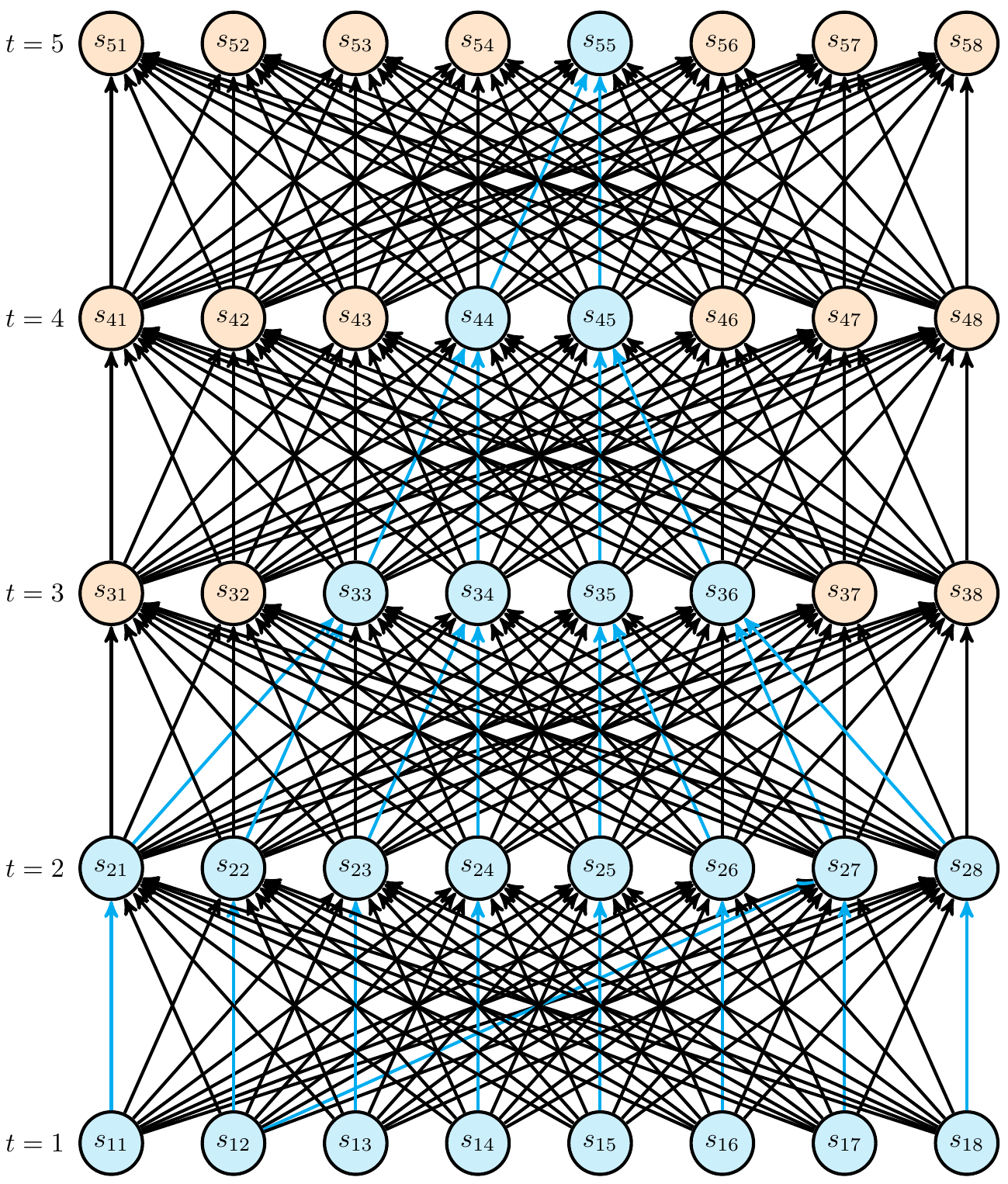}
\includegraphics[angle=0,width=0.49\textwidth]{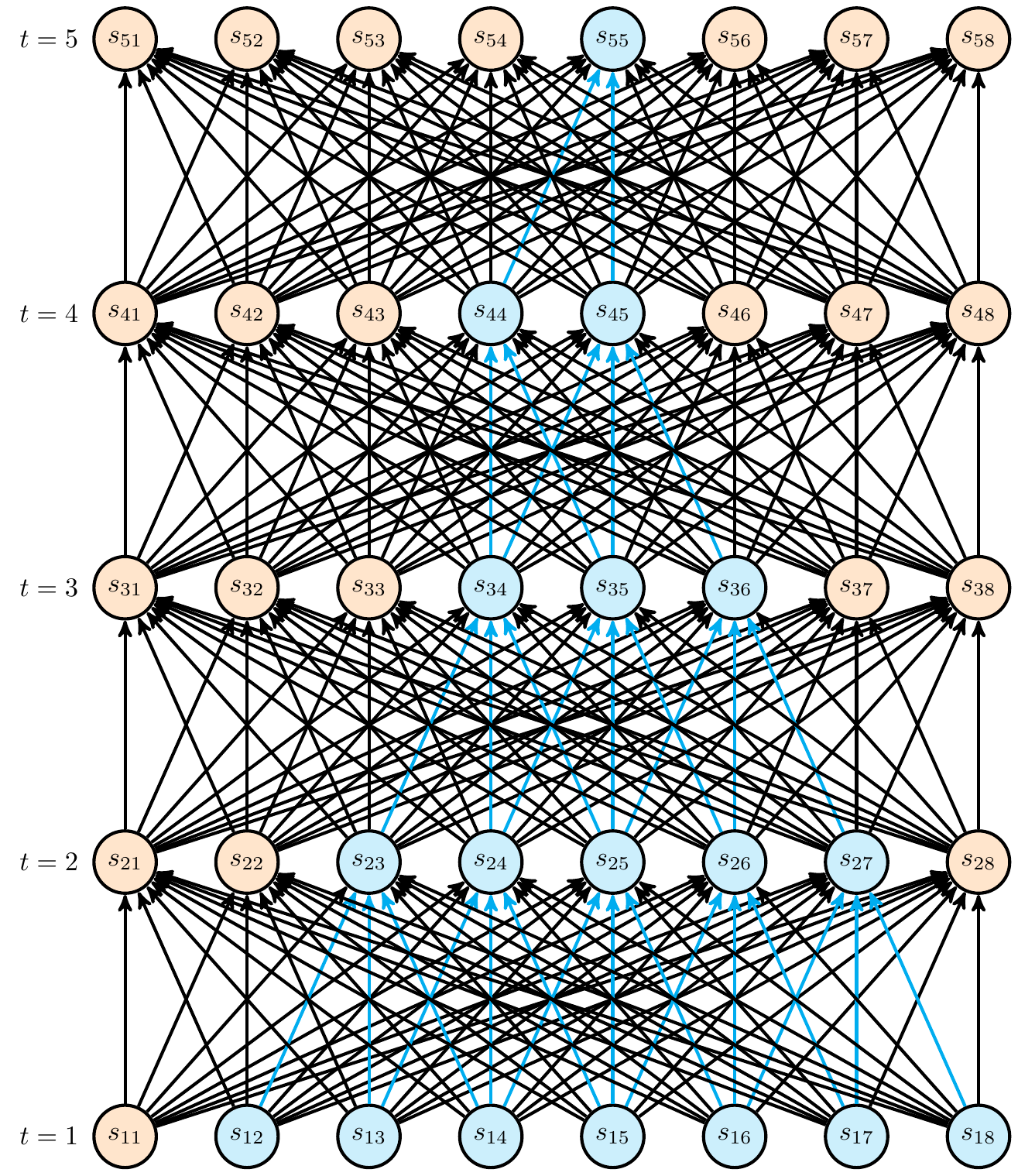}
\caption{Examples of how affected states (cyan) affect states in a
  previous time step (indicated by cyan edges) starting with $n_5=1$
  (one affected state).  The left panel shows no overlap
  since affected states in $s_{t-1}$ connect only to
  one affected state in $s_t$. The right panel shows some overlap
  since affected states in $s_{t-1}$ connect to multiple
  affected states in $s_t$.
\label{fig:affectedStates}}
\end{figure}

\begin{figure}[htb]
\centering
\includegraphics[angle=0,width=0.8\textwidth]{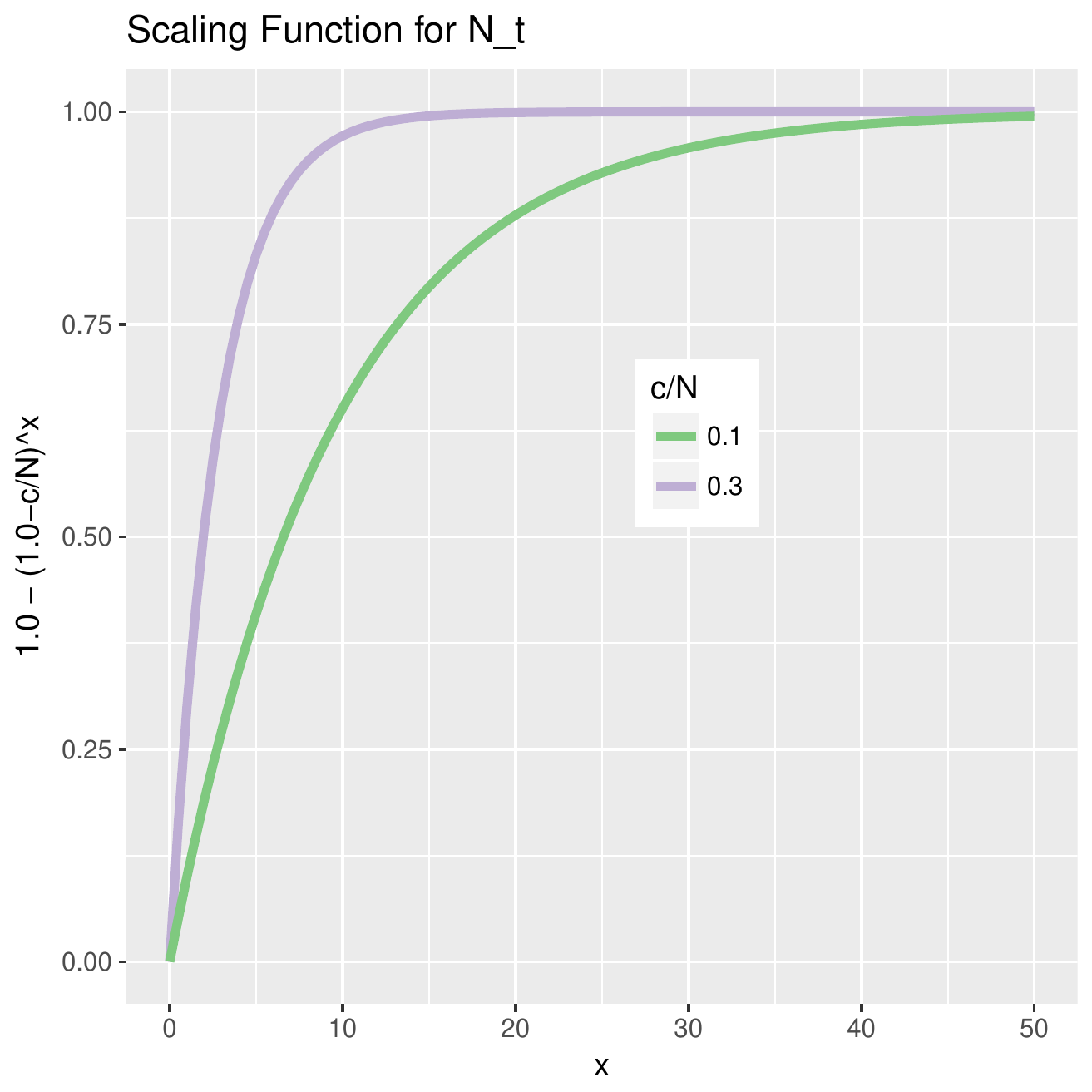}
\caption{The function $\left( 1 -   \left( 1 -
      \frac{c_t}{N_{t-1}} \right)^{n_t} \right)$ which scales $N_{t-1}$ in
    Theorem~\ref {th:Aaffect}. This function determines the growth of
    $a_k$, which is exponentially at the beginning, and then linearly
    when the function approaches 1.
\label{fig:growth}}
\end{figure}

The next theorem states that the on-site
variance $\Bom_k$ can have large effect on the variance of
each previous state-action pair. 
Furthermore, for small $k$ the number of affected
states grows exponentially, while for large $k$ it grows only 
linearly after some time $\hat{t}$. Figure~\ref{fig:growth} shows the
function which determines how much $a_k$ grows with $k$. 

\begin{theoremA}
\label{th:Aaffect}
  For $t \leq k$, $\Bom_k$ contributes to $\BV_t^{\pi}$
  by the term $\BP_{t \leftarrow k}  \ \Bom_k$, where $\left\| \BP_{t\leftarrow k} \right\|_{\infty}=1$.

The number $a_k$ of states affected by the on-site variance
$\Bom_k$ is
\begin{align}
 a_k   \ &= \ \sum_{t=0}^k \left( 1 \ - \   \left( 1 \ - \ \frac{c_t}{N_{t-1}} \right)^{n_t}
  \right) \ N_{t-1} \ .
\end{align}
\end{theoremA}

\begin{proof}
The ``backward induction algorithm'' \cite{Puterman:90,Puterman:05}
gives with $\BV_{T+1}^{\pi}=\BZe$ and on-site variance $\Bom_{T+1} =\BZe$:
\begin{align}
  \BV_t^{\pi} \ &= \   \sum_{k=t}^{T}  \prod_{\tau=t}^{k-1} \BP_{\tau} \
  \Bom_k \ ,
\end{align}
where we define $\prod_{\tau=t}^{t-1} \BP_{\tau}= \BI$ and $[\Bom_k]_{(s_k,a_k)}=\omega(s_k,a_k)$.

Since the product of two row-stochastic matrices is a row-stochastic
matrix according to Lemma~\ref{th:ArowStochastic},
$\BP_{t\leftarrow k}= \prod_{\tau=t}^{k-1} \BP_{\tau}$ is a row-stochastic matrix.
Since $\left\| \BP_{t\leftarrow k} \right\|_{\infty}=1$ according to Lemma~\ref{th:Aperron},
each on-site variance $\Bom_k$ with $t \leq k$ can have large effects on
$\BV_t^{\pi}$.
Using the row-stochastic matrices $\BP_{t\leftarrow k}$, we can
reformulate the variance:
\begin{align}
  \BV_t^{\pi} \ &= \   \sum_{k=t}^{T}  \BP_{t \leftarrow k} \ \Bom_k \ ,
\end{align}
with $\left\| \BP_{t\leftarrow k} \right\|_{\infty}=1$.
The on-site variance $\Bom_k$ at step $k$
increases all variances $\BV_t^{\pi}$ with $t \leq k$.

Next we proof the second part of the theorem, which considers the
growth of $a_k$.
To compute $a_k$ we first have
to know $n_t$. For computing $n_{t-1}$ from $n_t$, we want to know how
many states are affected in $s_{t-1}$ if $n_t$
states are affected in $s_t$.
The answer to this question is the expected
coverage when searching a document collection
using a set of independent computers \cite{Cox:09}.  
We follow the approach of Cox~et~al.\ \cite{Cox:09}.
The minimal number of affected states in $s_{t-1}$ is $c_t$, where
each of the $c_t$ affected states in $s_{t-1}$ 
connects to each of the $n_t$ states in $s_t$ (maximal overlap).
The maximal number of affected states in $s_{t-1}$ is
$c_t n_t$, where each affected state in $s_{t-1}$
connects to only one affected state in $s_t$ (no overlap).
We consider a single state in $s_t$.
The probability of a state in $s_{t-1}$ being connected to this single
state in $s_t$ is $c_t/N_{t-1}$ and being not connected to this state in
$s_t$ is $1 - c_t/N_{t-1}$.
The probability of a state in $s_{t-1}$
being not connected to any of the $n_t$ affected states in $s_t$ is
\begin{align}
  \left( 1 \ - \ \frac{c_t}{N_{t-1}} \right)^{n_t} \ .
\end{align}
The probability of a state in $s_{t-1}$
being at least connected to one of the $n_t$ affected states in $s_t$ is
\begin{align}
  1 \ - \   \left( 1 \ - \ \frac{c_t}{N_{t-1}} \right)^{n_t} \ .
\end{align}
Thus, the expected number of distinct states in $s_{t-1}$ being
connected to one of the $n_t$ affected states in $s_t$ is
\begin{align}
 n_{t-1}   \ &= \ \left( 1 \ - \   \left( 1 \ - \ \frac{c_t}{N_{t-1}} \right)^{n_t}
  \right) \ N_{t-1} \ .
\end{align}
The number $a_k$ of affected states by $\Bom_k$ is
\begin{align}
 a_k   \ &= \ \sum_{t=0}^k \left( 1 \ - \   \left( 1 \ - \ \frac{c_t}{N_{t-1}} \right)^{n_t}
  \right) \ N_{t-1} \ .
\end{align}
\end{proof}

\begin{corollaryA}
\label{th:Acoro}
 For small $k$, the number $a_k$ of states affected by
the on-site variance $\Bom_k$ at step $k$
grows exponentially with $k$ by a factor of $\bar{c}$: 
\begin{align}
  a_k \ &> \ \bar{c}^k \ .
\end{align}
For large $k$ and after some time $t>\hat{t}$,
the number $a_k$ of states affected by
$\Bom_k$ grows linearly with $k$ with a factor of $\bar{N}$: 
\begin{align}
  a_k \ &\approx \  a_{\hat{t}-1} \ + \ (k-\hat{t}+1) \ \bar{N} \ .
\end{align}
\end{corollaryA}

\begin{proof}
For small $n_t$ with $\frac{c_t  n_t}{N_{t-1}} \ll 1$, we have
\begin{align}
  \left( 1 \ - \ \frac{c_t}{N_{t-1}} \right)^{n_t} \ &\approx \ 1 \ -
  \ \frac{c_t \ n_t}{N_{t-1}} \ ,
\end{align}
thus
\begin{align}
 n_{t-1}   \ &\approx \  c_t \ n_t \ .
\end{align}

For large $N_{t-1}$ compared to the number of connections $c_t$ of a
single state in $s_t$ to states in $s_{t-1}$,
we have the approximation
\begin{align}
  \left( 1 \ - \ \frac{c_t}{N_{t-1}} \right)^{n_t} \ &= \
  \left( \left( 1 \ + \ \frac{-c_t}{N_{t-1}} \right)^{N_{t-1}} \right)^{n_t/N_{t-1}} \ \approx \
  \exp(- (c_t \ n_t)/N_{t-1}) \ .
\end{align}
We obtain
\begin{align}
n_{t-1}   \ &= \ \left( 1 \ - \  \exp(-(c_t \ n_t)/N_{t-1}) \right) \ N_{t-1} \ .
\end{align}
For small $n_t$, we again have
\begin{align}
  n_{t-1}   \ &\approx \ c_t \ n_t \ .
\end{align}
Therefore, for small $k-t$, we obtain
\begin{align}
  n_{t}   \ &\approx \ \prod_{\tau=t}^k c_{\tau}  \ \approx \ \bar{c}^{k-t}\ .
\end{align}
Thus, for small $k$ the number $a_k$ of states affected by $\Bom_k$ is
\begin{align}
  a_k \ &= \ \sum_{t=0}^k n_t \ \approx \  \sum_{t=0}^k \bar{c}^{k-t}
  \ = \ \sum_{t=0}^k \bar{c}^t \ = \ \frac{\bar{c}^{k+1}-1}{\bar{c}-1} \ >
  \ \bar{c}^k \ .
\end{align}
Consequently, for small $k$ the number $a_k$ of states affected by $\Bom_k$ grows exponentially with $k$ by a factor of $\bar{c}$.  
For large $k$, at a certain time $t>\hat{t}$, $n_t$ has grown such that
$c_t n_t > N_{t-1}$, yielding
$\exp(-(c_t n_t)/N_{t-1}) \approx 0$, and thus
\begin{align}
  n_t   \ &\approx \  N_t \ .
\end{align}
Therefore 
\begin{align}
  a_k \ - \ a_{\hat{t}-1} \ &= \ \sum_{t=\hat{t}}^k n_t \ \approx \  \sum_{t=\hat{t}}^k N_{t} \
  \approx \ (k-\hat{t}+1) \ \bar{N} \ .
\end{align}
Consequently, for large $k$ the number $a_k$ of states affected by
$\Bom_k$ grows linearly with $k$ by a factor of $\bar{N}$.  
\end{proof}

Therefore, we aim for decreasing the on-site variance $\Bom_k$ for large $k$, in order to
reduce the variance.
In particular, we want to avoid delayed rewards and provide the reward
as soon as possible in each episode.
Our goal is to give the reward as early as possible in each episode to reduce the
variance of action-values that are affected by late rewards and
their associated immediate and local variances.

\pagebreak
\section{Experiments}

\subsection{Artificial Tasks}
\label{sec:Aexp}
This section provides more details for the artificial tasks (I), (II) and (III) in the main paper. Additionally, we include artificial task (IV) characterized by deterministic reward and state transitions, and artificial task (V) which is solved using policy gradient methods.

\subsubsection{Task (I): Grid World}
\label{sec:AGW}
This environment is characterized by probabilistic delayed rewards.
It illustrates a situation, where a time bomb
explodes at episode end.
The agent has to defuse the bomb
and then run away as far as possible since
defusing fails with a certain probability.
Alternatively, the agent can immediately run away, 
which, however, leads to less reward on average 
since the bomb always explodes.
The Grid World is a quadratic $31\times 31$ grid with
{\em bomb} at coordinate $[30,15]$ and
{\em start} at $[30-d, 15]$, where $d$ is the delay of the task.
The agent can move in four different directions 
({\em up}, {\em right}, {\em left}, and {\em down}).
Only moves are allowed that keep the agent on the grid. 
The episode finishes after $1.5 d$ steps.
At the end of the episode, 
with a given probability of 0.5, 
the agent receives a reward of 1000
if it has visited {\em bomb}.
At each time step the agent receives 
an immediate reward of $c\cdot t \cdot h$, 
where the factor $c$ depends on the chosen action, 
$t$ is the current time step, and $h$ is the Hamming distance to {\em bomb}.
Each move of the agent, which reduces the Hamming distance to {\em bomb}, 
is penalized by the immediate reward via $c=-0.09$. 
Each move of the agent, which increases the Hamming distance to {\em bomb}, 
is rewarded by the immediate reward via $c=0.1$.
The agent is forced to learn the $Q$-values precisely,
since the immediate reward of directly running away hints at a sub-optimal policy.

For non-deterministic reward, the agent receives the delayed reward for
having visited {\em bomb} with probability $p(r_{T+1}=100 \mid s_T,a_T)$. 
For non-deterministic transitions, the probability of transiting to next state $s'$ 
is $p(s'  \mid s,a)$. 
For the deterministic environment these probabilities were either 1 or zero.

\paragraph{Policy evaluation: learning the action-value estimator for a fixed policy.} 
First, the theoretical statements on bias and variance of 
estimating the action-values by TD in Theorem~\ref{th:AexponDecay} and 
by MC in Theorem~\ref{th:Aaffect} are experimentally verified for a fixed policy.
Secondly, we consider the bias and variance of TD and MC estimators 
of the transformed MDP with optimal reward redistribution according to Theorem~\ref{th:AOptReturnDecomp}.

The new MDP with an optimal reward redistribution has advantages over
the original MDP both for TD and MC. For TD, the new MDP corrects the bias exponentially faster and 
for MC it has fewer number of action-values with high variance. 
Consequently, estimators for the new MDP learn faster 
than the same estimators in the original MDP.

Since the bias-variance analysis is defined for a particular number of samples 
drawn from a fixed distribution, we need to fix the policy for sampling. 
We use an $\epsilon$-greedy version of the optimal policy, where
$\epsilon$ is chosen such that on average in 10\% of the episodes the agent 
visits {\em bomb}. For the analysis, the delay ranges from 5 to 30 in steps of 5.
The true $Q$-table for each delay is computed 
by backward induction and we use 10 different action-value estimators 
for computing bias and variance.

For the TD update rule we use the exponentially weighted arithmetic mean that is sample-updates,
with initial value $q^0(s,a)=0$.
We only monitor the mean and the variance for action-value 
estimators at the first time step, 
since we are interested in the time required for correcting the bias.  
10 different estimators are run for 10,000 episodes. Figure~\ref{fig:A_TDbias} shows the bias correction for different delays, normalized by the first error.

For the MC update rule we use the arithmetic mean for policy evaluation
(later we will use constant-$\alpha$ MC for learning the optimal policy).
For each delay, a test set of state-actions for each delay is generated 
by drawing 5,000 episodes with the $\epsilon$-greedy optimal policy. 
For each action-value estimator the mean and the variance 
is monitored every 10 visits. If every action-value has 500 updates (visits), 
learning is stopped. Bias and variance are computed based on 10 different 
action-value estimators.
As expected from Section~\ref{sec:Abias_variance_estimator}, in Figure~\ref{fig:A_MCvar}
the variance decreases by $1/n$, where $n$ is the number of samples.
Figure~\ref{fig:A_MCvar} shows that the number of state-actions with a variance larger than a threshold 
increases exponentially with the delay.
This confirms the statements of Theorem~\ref{th:Aaffect}. 

\begin{figure}[htp]
 \centering
 \begin{subfigure}{.49\textwidth}
  \centering%
  \resizebox{1\textwidth}{!}{\includegraphics[]{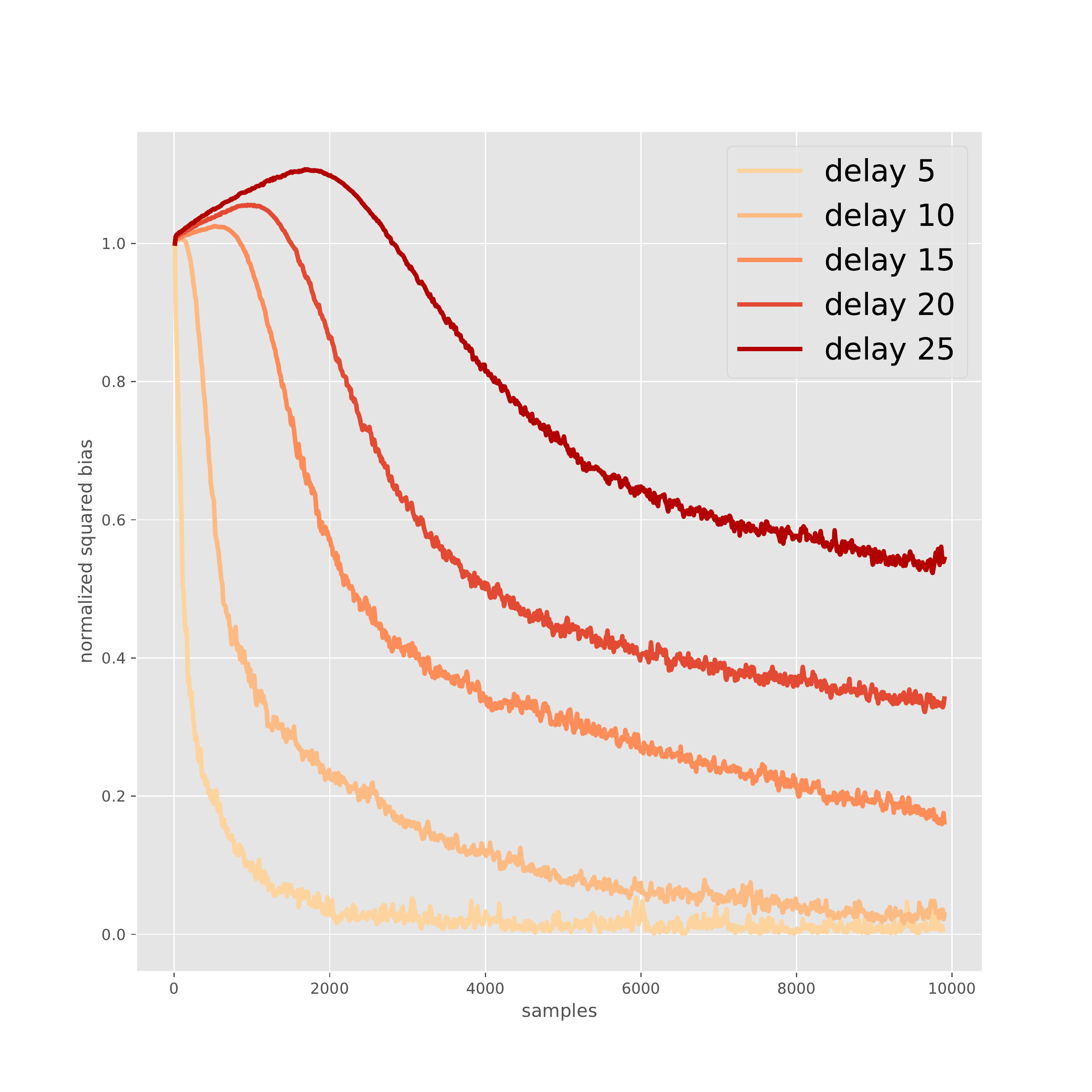} }
  \caption{\label{fig:A_TDbias}}
 \end{subfigure}
 \begin{subfigure}{.49\textwidth}
  \centering%
  \resizebox{1\textwidth}{!}{\includegraphics[]{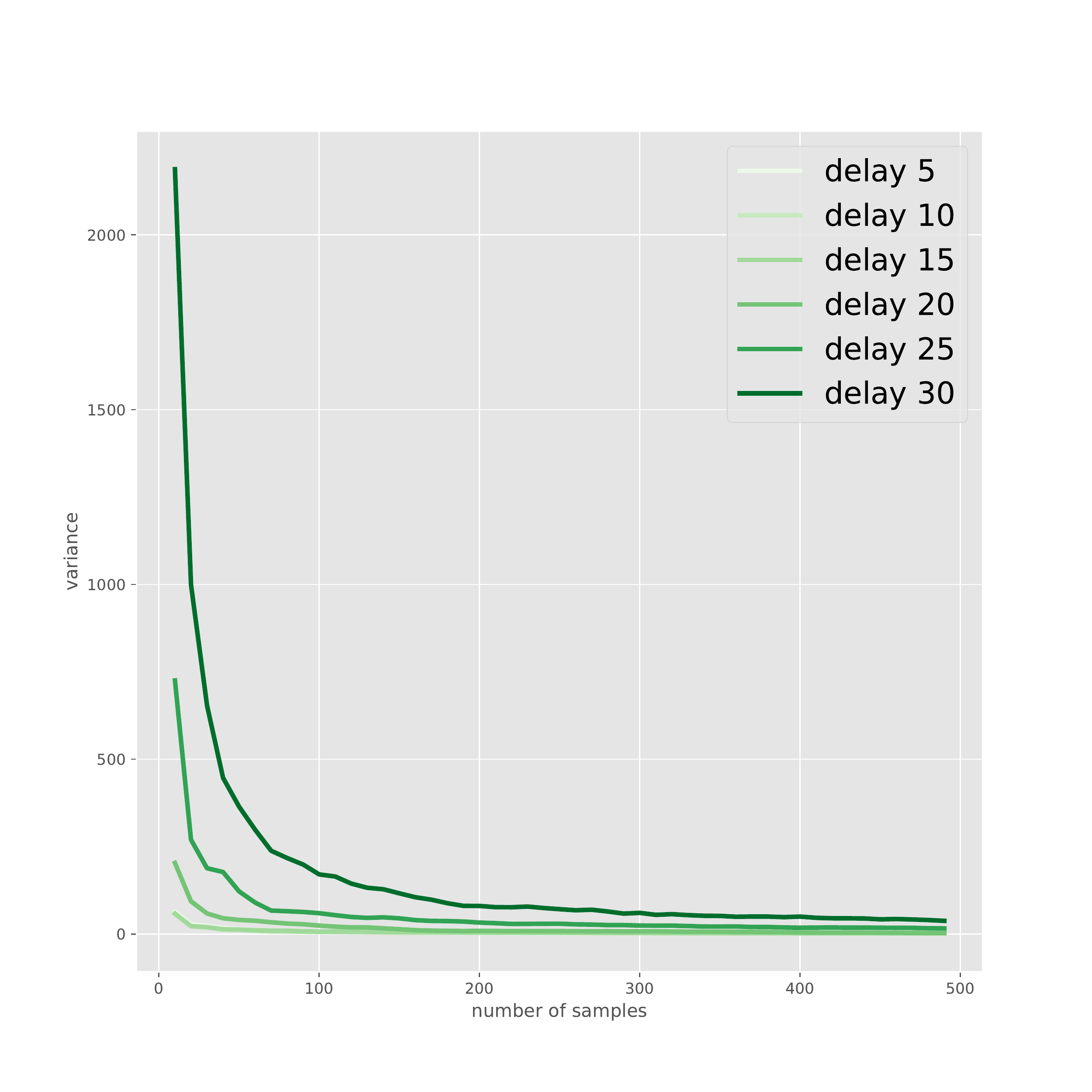}}
  \caption{\label{fig:A_MCvar}}
 \end{subfigure}%
 \caption{(a) Experimental evaluation of bias and variance of
          different $Q$-value estimators on the Grid World.
          (b) Normalized bias reduction for different delays.
          Right: Average variance reduction for the 10th highest values.
          }
          
\end{figure}

\paragraph{Learning the optimal policy.} 
For finding the optimal policy for the Grid World task, 
we apply Monte Carlo Tree Search (MCTS), 
$Q$-learning, and Monte Carlo (MC). 
We train until the greedy policy reaches 90\% of the return of the optimal policy. The learning time is measured by the number of episodes.
We use {\em sample updates} for $Q$-learning and MC \cite{Sutton:18book}. 
For MCTS the greedy policy uses $0$ for the exploration constant in UCB1 \cite{Kocsis:06}.
The greedy policy is evaluated in 100 episodes intervals. 
The MCTS selection step begins in the start state, 
which is the root of the game tree 
that is traversed using UCB1 \cite{Kocsis:06} as the tree policy. 
If a tree-node gets visited the first time, 
it is expanded with an initial value obtained by $100$ simulated trajectories that start at this node. 
These simulations use a uniform random policy whose average Return is calculated. 
The backpropagation step uses the MCTS(1) update rule \cite{Khandelwal:16}. 
The tree policies exploration constant is $\sqrt{2}$.
$Q$-learning and MC use a learning rate of $0.3$ and an 
$\epsilon$-greedy policy with $\epsilon=0.3$.
For RUDDER the optimal reward redistribution using 
a return decomposition as stated 
in Section~\ref{sec:Aopt_rew_red}
is used. For each {\em delay} and each method, 300 runs with different seeds are performed to obtain statistically relevant results.

\paragraph{Estimation of the median learning time and quantiles.}
\label{sec:Alr_estimation}
The performance of different methods is measured by 
the median learning time in terms of episodes.
We stop training at $100$ million episodes. 
Some runs, especially for long delays, have taken too long 
and have thus been stopped. To resolve this bias the quantiles of the 
learning time are estimated by fitting a distribution using right censored data~\cite{gijbels2010censored} .
The median is still robustly estimated 
if more than 50\% of runs have finished, 
which is the case for all plotted datapoints.
We find that for delays where all runs have finished 
the learning time follows a Log-normal distribution. 
Therefore, we fit a Log-normal distribution on the right censored data. 
We estimate the median from the existing data, 
and use maximum likelihood estimation to obtain 
the second distribution parameter $\sigma^2$. 
The start value of the $\sigma^2$ estimation is calculated by 
the measured variance of the existing data 
which is algebraically transformed to get the $\sigma$ parameter.

\subsubsection{Task (II): The Choice}
\label{c:thechoice}
In this experiment we compare RUDDER, temporal difference (TD) and Monte Carlo (MC) in 
an environment with delayed deterministic reward 
and probabilistic state transitions to investigate how reward information
is transferred back to early states.
This environment is a variation of our introductory
pocket watch example and reveals problems of TD and MC,
while contribution analysis excels. 
In this environment, only the first action at the very beginning 
determines the reward at the end of the episode.

\paragraph{ The environment} is an MDP consisting of two actions $a \in \cA = \{+,- \}$,
an initial state $s^{0}$, two {\em charged} states $s^+$, $s^-$, 
two {\em neutral} states $s^{\oplus}$, $s^{\ominus}$, and a final state $s^{\Rf}$.
After the first action $a_0 \in \cA = \{+,- \}$ in state $s^{0}$,
the agent transits to state $s^+$ for action $a_0 =+$ and to $s^-$ for 
action $a_0=-$. 
Subsequent state transitions are probabilistic and independent on actions. 
With probability $p_C$ the agent stays in the {\em charged} states $s^+$ or $s^-$,
and with probability $(1-p_C)$ it transits from $s^+$ or $s^-$
to the {\em neutral} states $s^{\oplus}$ or $s^{\ominus}$, respectively. 
The probability to go from {\em neutral} states 
to {\em charged} states is $p_C$, 
and the probability to stay in {\em neutral} states is $(1-p_C)$. 
Probabilities to transit from 
$s^+$ or $s^{\oplus}$ to $s^-$ or $s^{\ominus}$ or vice versa are zero.
Thus, the first action determines whether that agent stays in "$+$"-states or 
"$-$"-states.
The reward is determined by how many times the agent visits {\em charged} states 
plus a bonus reward depending on the agent's first action.
The accumulative reward is given at sequence end and is deterministic.
After $T$ time steps, the agent is in the final state $s^{\Rf}$, 
in which the reward $R_{T+1}$ is provided.
 $R_{T+1}$ is the sum of 3 deterministic terms: 
 \begin{enumerate}
     \item $R_0$, the baseline reward associated to the first action; 
     \item $R_C$, the collected reward across states, 
     which depends on the number of visits $n$ to the {\em charged} states;
     \item $R_b$, a bonus if the first action $a_0=+$.
 \end{enumerate}
The expectations of the accumulative rewards for $R_0$ and $R_C$ 
have the same absolute value but opposite signs, 
therefore they cancel in expectation over episodes.
Thus, the expected return of an episode is the expected reward $R_b$: $p(a_0=+) b$. 
The rewards are defined as follows:
\begin{align}
    c_0 \ &=  \ \begin{cases} 
     ~~~1 & \text{if }  \ a_0 = + \\
     -1 & \text{if }  \ a_0 = - \ ,
     \end{cases} \\
     R_b \ &=  \ \begin{cases} 
     b & \text{if }  \ a_0 = + \\
     0 & \text{if }  \ a_0 = - \ ,
     \end{cases} \\
    R_C \ &= \ c_0 \ C \ n \ ,  \\
    R_0 \ &= \ - \ c_0 \ C \ p_C  \ T \ , \\
    R_{T+1} \ &= \ R_C \ + \ R_0 \ + \ R_b \ ,
\end{align}
where $C$ is the baseline reward for {\em charged} states, 
and $p_C$ the probability of staying in or transiting to {\em charged} states.
The expected visits of charged states is $\EXP[n]= p_C  T$ and
$\EXP[R_{T+1}]= \EXP[R_b]=p(a_0=+) b$.

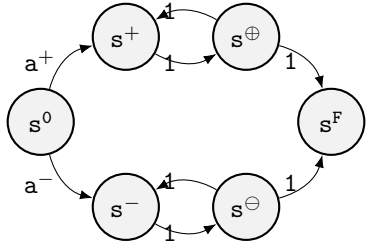
\begin{figure}
\centering
\begin{tikzpicture}[node distance=1.6cm]
  \tikzstyle{every state}=[thick, fill=gray!10]
  \tikzset{edge/.style = {->,> = latex'}}
  \tikzstyle{initial text}=[" "]
  \node[state] (s100)                    {$\mathtt{s^{0}}$};
  \node[state]         (s201) [above right of=s100] {$\mathtt{s^+}$};
  \node[state]         (s202) [below right of=s100] {$\mathtt{s^-}$};
  \node[state]         (s301) [right of=s201] {$\mathtt{s^{\oplus}}$};
  \node[state]         (s302) [right of=s202] {$\mathtt{s^{\ominus}}$};
 \node[state]         (s400) [below right of=s301] {$\mathtt{s^{F}}$};
 
\path   (s201) edge [-Latex,bend right]      node[left]  {$\mathtt{1}$} (s301) ;
\path   (s301) edge [-Latex,bend right]      node[left]  {$\mathtt{1}$} (s201) ;
\path   (s202) edge [-Latex,bend right]      node[left]  {$\mathtt{1}$} (s302) ;
\path   (s302) edge [-Latex,bend right]      node[left]  {$\mathtt{1}$} (s202) ;
\path   (s100) edge [-Latex,bend left]      node[ left] {$\mathtt{a^+}$}  (s201);
\path   (s100) edge [-Latex,bend right]      node[left]  {$\mathtt{a^-}$} (s202);
\path   (s302) edge [-Latex,bend right]      node[left]  {$\mathtt{1}$} (s400) ; 
\path   (s301) edge [-Latex,bend left]      node[left]  {$\mathtt{1}$} (s400) ;
\end{tikzpicture}
\caption{State transition diagram for The Choice task. 
The diagram is a simplification of the actual MDP.}
\end{figure}

\paragraph{Methods compared: }
The following methods are compared:
\begin{enumerate}
    \item $Q$-learning with eligibility traces according to Watkins \cite{Watkins:89}, 
   \item Monte Carlo,
    \item RUDDER with reward redistribution.
\end{enumerate}

For RUDDER, we use an LSTM without lessons buffer 
and without safe exploration. 
Contribution analysis is realized by differences of return predictions.
For MC, $Q$-values are the exponential moving average of the episode return.
For RUDDER, the $Q$-values are estimated by an exponential moving average of the reward redistribution.

\paragraph{Performance evaluation and results.}

The task is considered as solved 
when the exponential moving average of the selection of the desired action 
at time $t = 0$  is equal to $1-\epsilon$, where $\epsilon$ is the exploration rate.
The performances of the compared methods are measured by the average learning time
in the number of episodes required to solve the task.
A Wilcoxon signed-rank test is performed between the learning time of RUDDER 
and those of the other methods.
Statistical significance p-values are obtained by Wilcoxon signed-rank test.
RUDDER with reward redistribution is significantly faster than all other methods with p-values $<10^{-8}$. 
Table~\ref{tab:res1} reports the number of episodes required by different
methods to solve the task.  
RUDDER with reward redistribution clearly outperforms all other methods. 


\newpage
\begin{landscape}
\begin{table}[htp]
\begin{flushleft}
\caption{Number of episodes required by different 
 methods to solve the grid world task with delayed reward. Numbers give the  
 mean and the standard deviation over 100 trials.
 RUDDER with reward redistribution clearly outperforms all other TD methods.}
\label{tab:res1}%
\begin{tabular}{*{1}{>{\raggedright}p{6em}}*{1}{>{\columncolor{mColor1}\raggedleft}p{4em}}*{1}{>{\columncolor{mColor1}\raggedleft}p{5em}}*{1}{>{\columncolor{mColor1}\raggedright}p{5em}}*{1}{>{\raggedleft}p{4em}}*{1}{>{\raggedleft}p{5em}}*{1}{>{\raggedright}p{5em}}*{1}{>{\columncolor{mColor1}\raggedleft}p{4em}}*{1}{>{\columncolor{mColor1}\raggedleft}p{5em}}*{1}{>{\columncolor{mColor1}\raggedright}p{5em}}*{1}{>{\raggedright}p{0.01em}}}
\toprule[1pt]
\addlinespace[2pt]
{\bf Method} &\multicolumn{3}{c}{\bf Delay 10} &\multicolumn{3}{c}{\bf Delay 15} &\multicolumn{3}{c}{\bf Delay 20} &\\
\toprule[1pt]
RUDDER & 3520.06 & {\small $\pm$ 2343.79} & {\small p = 5.00E-01} & 3062.07 & {\small $\pm$ 1278.92} & {\small p = 5.00E-01} & 3813.96 & {\small $\pm$ 2738.18} & {\small p = 5.00E-01} &  \\
MC & 10920.64 & {\small $\pm$ 7550.04} & {\small p = 5.03E-24} & 17102.89 & {\small $\pm$ 12640.09} & {\small p = 1.98E-30} & 22910.85 & {\small $\pm$ 19149.02} & {\small p = 1.25E-28} &  \\
Q & 66140.76 & {\small $\pm$ 1455.33} & {\small p = 1.28E-34} & 115352.25 & {\small $\pm$ 1962.20} & {\small p = 1.28E-34} & 171571.94 & {\small $\pm$ 2436.25} & {\small p = 1.28E-34} &  \\
\addlinespace[1pt]

\end{tabular}
\end{flushleft}

\begin{flushleft}
\begin{tabular}{*{1}{>{\raggedright}p{6em}}*{1}{>{\raggedleft}p{4em}}*{1}{>{\raggedleft}p{5em}}*{1}{>{\raggedright}p{5em}}*{1}{>{\columncolor{mColor1}\raggedleft}p{4em}}*{1}{>{\columncolor{mColor1}\raggedleft}p{5em}}*{1}{>{\columncolor{mColor1}\raggedright}p{5em}}*{1}{>{\raggedleft}p{4em}}*{1}{>{\raggedleft}p{5em}}*{1}{>{\raggedright}p{5em}}*{1}{>{\raggedright}p{0.01em}}}
\toprule[1pt]
\addlinespace[2pt]
{\bf Method} &\multicolumn{3}{c}{\bf Delay 25} &\multicolumn{3}{c}{\bf Delay 30} &\multicolumn{3}{c}{\bf Delay 35} &\\
\toprule[1pt]
MC & 39772 & {\small $\pm$ 47460} & {\small p < 1E-29} & 41922 & {\small $\pm$ 36618} & {\small p < 1E-30} & 50464 & {\small $\pm$ 60318} & {\small p < 1E-30} &  \\
Q & 234912 & {\small $\pm$ 2673} & {\small p < 1E-33} & 305894 & {\small $\pm$ 2928} & {\small p < 1E-33} & 383422 & {\small $\pm$ 4346} & {\small p < 1E-22} &  \\
RUDDER & 4112 & {\small $\pm$ 3769} &  & 3667 & {\small $\pm$ 1776} &  & 3850 & {\small $\pm$ 2875} & &  \\
\addlinespace[1pt]
\end{tabular}
\end{flushleft}

\begin{flushleft}
\begin{tabular}{*{1}{>{\raggedright}p{6em}}*{1}{>{\columncolor{mColor1}\raggedleft}p{4em}}*{1}{>{\columncolor{mColor1}\raggedleft}p{5em}}*{1}{>{\columncolor{mColor1}\raggedright}p{5em}}*{1}{>{\raggedleft}p{4em}}*{1}{>{\raggedleft}p{5em}}*{1}{>{\raggedright}p{5em}}*{1}{>{\columncolor{mColor1}\raggedleft}p{4em}}*{1}{>{\columncolor{mColor1}\raggedleft}p{5em}}*{1}{>{\columncolor{mColor1}\raggedright}p{5em}}*{1}{>{\raggedright}p{0.01em}}}
\toprule[1pt]
\addlinespace[2pt]
{\bf Method} &\multicolumn{3}{c}{\bf Delay 40} &\multicolumn{3}{c}{\bf Delay 45} &\multicolumn{3}{c}{\bf Delay 50} &\\
\toprule[1pt]
MC & 56945 & {\small $\pm$ 54150} & {\small p < 1E-30} & 69845 & {\small $\pm$ 79705} & {\small p < 1E-31} & 73243 & {\small $\pm$ 70399} & {\small p = 1E-31} &  \\
Q & 466531 & {\small $\pm$ 3515} & {\small p = 1E-22} &  &  &  &  &  &  &  \\
RUDDER & 3739 & {\small $\pm$ 2139} &  & 4151 & {\small $\pm$ 2583} &  & 3884 & {\small $\pm$ 2188} &  &  \\
\addlinespace[1pt]
\end{tabular}
\end{flushleft}

\begin{flushleft}
\begin{tabular}{*{1}{>{\raggedright}p{14em}}*{1}{>{\columncolor{mColor1}\raggedleft}p{4em}}*{1}{>{\columncolor{mColor1}\raggedleft}p{5em}}*{1}{>{\columncolor{mColor1}\raggedright}p{5em}}*{1}{>{\raggedleft}p{4em}}*{1}{>{\raggedleft}p{5em}}*{1}{>{\raggedright}p{5em}}*{1}{>{\raggedright}p{0.01em}}}
\toprule[1pt]
\addlinespace[2pt]
{\bf Method} &\multicolumn{3}{c}{\bf Delay 100} &\multicolumn{3}{c}{\bf Delay 500} &\\
\toprule[1pt]
MC & 119568 & {\small $\pm$ 110049} & {\small p < 1E-11} & 345533 & {\small $\pm$ 320232} & {\small p < 1E-16} &  \\
RUDDER & 4147 & {\small $\pm$ 2392} &  & 5769 & {\small $\pm$ 4309} &  &  \\
\addlinespace[1pt]
\end{tabular}
\end{flushleft}
\end{table}

\end{landscape}

\newpage
\subsubsection{Task(III): Trace-Back}
This section supports the artificial task (III) -- {\bf Trace-Back} -- in the main paper.
RUDDER is compared to potential-based reward shaping methods.
In this experiment, we compare reinforcement learning methods that
have to transfer back information about a delayed reward.
These methods comprise RUDDER, TD($\lambda$) and potential-based reward shaping approaches.
For potential-based reward shaping we compare the original 
{\em reward shaping} \cite{Ng:99}, {\em look-forward advice}, 
and {\em look-back advice}~\cite{Wiewiora:03}
with three different potential functions.
Methods that transfer back reward information
are characterized by low variance estimates
of the value function or the action-value function, 
since they use an estimate of the future return instead of the
future return itself.
To update the estimates of the future returns, 
reward information has to be transferred back.
The task in this experiment can be solved by Monte Carlo
estimates very fast, which do not transfer back information but use samples of 
the future return for the estimation instead. 
However, Monte Carlo methods have high variance, which is
not considered in this experiment. 
\paragraph{ The environment } is a 15$\times$15 grid, where 
actions move the agent from its current position 
in 4 adjacent positions ({\em up, down, left, right}), 
except the agent would be moved outside the grid. 
The number of steps (moves) per episode is $T=20$. 
The starting position is $(7,7)$ in the middle of the grid. 
The maximal return is a combination of negative immediate reward 
and positive delayed reward.
To obtain the maximum return, 
the policy must move the agent {\em up} in the time step $t=1$ and  
{\em right} in the following time step $t=2$. In this case, the 
agent receives an immediate reward of -50 at $t=2$ and 
a delayed reward of 150 at the end of the episode at $t=20$, 
that is, a return of 100.
Any other combination of actions gives the agent immediate reward of 50 at $t=2$ without
any delayed reward, that is, a return of 50.
To ensure Markov properties the position of the agent, the time, as well as the delayed reward
are coded in the state. The future reward discount rate $\gamma$ is set to 1.
The state transition probabilities are deterministic for the first two moves.
For $t>2$ and for each action, state transition probabilities 
are equal for each possible next state (uniform distribution), 
meaning that 
actions after $t=2$ do not influence the return.
For comparisons of long delays, 
both the size of the grid and the length of the episode are increased. 
For a delay of $n$, a $(3n/4) \times (3n/4)$ grid is used
with an episode length of $n$, 
and starting position $(3n/8, 3n/8)$.

\paragraph{Compared methods.}
We compare different TD($\lambda$) and potential-based reward shaping methods.
For TD($\lambda$), the baseline is $Q(\lambda)$, 
with eligibility traces $\lambda=0.9$ and $\lambda=0$ and Watkins' implementation \cite{Watkins:89}. 
The potential-based reward shaping methods are the original reward shaping, 
look-ahead advice as well as look-back advice. 
For look-back advice, we use SARSA($\lambda$)~\cite{Rummery:94} instead of $Q(\lambda)$ 
as suggested by the authors~\cite{Wiewiora:03}. 
$Q$-values are represented by a state-action table, that is, we consider only
tabular methods.
In all experiments an $\epsilon$-greedy policy with $\epsilon = 0.2$ is used.
All three reward shaping methods 
require a potential function $\phi$,
which is based on the reward redistribution ($\tilde{r}_t$) 
in three different ways:

(I) The Potential function $\phi$ is the difference of LSTM predictions, 
which is the redistributed reward $R_t$:
\begin{align}
    \phi (s_t) \ &= \ \EXP \left[ R_{t+1} \mid s_t \right]\ \ \mbox{ or } \\ 
    \phi (s_t,a_t) \ &= \ \EXP \left[ R_{t+1} \mid s_t , a_t \right] \ .
\end{align}
(II) The potential function $\phi$ is the sum of 
future redistributed rewards, i.e.\ 
the q-value of the redistributed rewards. 
In the optimal case, this coincides with implementation (I):
\begin{align}
    \phi (s_t) \ &= \ \EXP  \left[ \sum^{T}_{\tau = t} R_{\tau +1} \mid s_t \right] 
    \ \ \mbox{ or } \\  
    \phi (s_t,a_t) \ &= \ \EXP  \left[ \sum^{T}_{\tau = t} R_{\tau+1} \mid s_t, a_t \right] \ .
\end{align}
(III) The potential function $\phi$ corresponds to the LSTM predictions. In the optimal case this corresponds
to the accumulated reward up to $t$ plus the q-value of the delayed MDP:
\begin{align}
    \phi (s_t) \ &= \ \EXP  \left[ \sum^{T}_{\tau = 0} \tilde{R}_{\tau +1} \mid s_t \right] 
    \ \ \mbox{ or } \\ 
    \phi (s_t,a_t) \ &= \ \EXP  \left[ \sum^{T}_{\tau = 0} \tilde{R}_{\tau+1} \mid s_t, a_t \right] \ .
\end{align}


The following methods are compared:
\begin{enumerate}
    \item $Q$-learning with eligibility traces according to Watkins ($Q(\lambda)$), 
    \item SARSA with eligibility traces (SARSA($\lambda$)),
    \item Reward Shaping with potential functions (I), (II), or (III) 
    according to $Q$-learning and eligibility traces according to Watkins,
    \item Look-ahead advise with potential functions (I), (II), or (III) 
    with $Q(\lambda)$,
    \item Look-back advise with potential functions (I), (II), or (III) 
    with SARSA($\lambda$),
    \item RUDDER with reward redistribution for $Q$-value estimation and RUDDER applied on top of $Q$-learning.
\end{enumerate}

RUDDER is implemented with an LSTM architecture 
without output gate nor forget gate. For this experiments, 
RUDDER does not use lessons buffer nor safe exploration. 
For contribution analysis we use differences of return predictions.
For RUDDER, the $Q$-values are estimated by an exponential moving average 
(RUDDER $Q$-value estimation) or alternatively by $Q$-learning. 

\paragraph{ Performance evaluation: }
The task is considered solved when the exponential moving average of the return 
is above 90, which is 90\% of the maximum return. 
Learning time is the number of episodes required to solve the task. 
The first evaluation criterion is the average learning time.
The $Q$-value differences at time step $t=2$ are monitored. 
The $Q$-values at $t=2$ are the most important ones, since they have to predict 
whether the maximal return will be received or not.
At $t=2$ the immediate reward acts as a distraction since it is -50 for the action leading  
to the maximal return ($a^+$) and 50 for all other actions ($a^-$).  
At the beginning of learning, the $Q$-value difference 
between $a^+$ and $a^-$ is about -100, since the immediate reward is -50 and 50, respectively.
Once the $Q$-values converge to the optimal policy, the difference approaches 50. 
However, the task will already be correctly solved as soon as this difference is positive.
The second evaluation criterion is the $Q$-value differences at time step $t=2$,
since it directly shows to what extend the task is solved.

\paragraph{ Results: }
Table~\ref{tab:res1} reports the number of episodes required by different
methods to solve the task. 
The mean and the standard deviation over 100 trials are given.
A Wilcoxon signed-rank test is performed between the learning time of RUDDER 
and those of the other methods.
Statistical significance p-values are obtained by Wilcoxon signed-rank test.
RUDDER with reward redistribution is significantly faster than all other methods with p-values $<10^{-17}$. 
Tables~\ref{tab:Ares1},\ref{tab:Ares2} report the results for all methods.

\newpage
\begin{landscape}
\begin{table}[htp]
\begin{center}
\caption{Number of episodes required by different 
 methods to solve the Trace-Back task with delayed reward. The numbers represent the  
 mean and the standard deviation over 100 trials.
 RUDDER with reward redistribution significantly outperforms all other methods.
 }
\label{tab:Ares1}%
\begin{tabular}{*{1}{>{\raggedright}p{14em}}*{1}{>{\columncolor{mColor1}\raggedleft}p{4em}}*{1}{>{\columncolor{mColor1}\raggedleft}p{5em}}*{1}{>{\columncolor{mColor1}\raggedright}p{5em}}*{1}{>{\raggedleft}p{4em}}*{1}{>{\raggedleft}p{5em}}*{1}{>{\raggedright}p{5em}}*{1}{>{\columncolor{mColor1}\raggedleft}p{4em}}*{1}{>{\columncolor{mColor1}\raggedleft}p{5em}}*{1}{>{\columncolor{mColor1}\raggedright}p{5em}}*{1}{>{\raggedright}p{0.01em}}}
\toprule[1pt]
\addlinespace[2pt]
{\bf Method} &\multicolumn{3}{c}{\bf Delay 6} &\multicolumn{3}{c}{\bf Delay 8} &\multicolumn{3}{c}{\bf Delay 10} &\\
\toprule[1pt]
Look-back I & 6074 & {\small $\pm$ 952} & {\small p = 1E-22} & 13112 & {\small $\pm$ 2024} & {\small p = 1E-22} & 21715 & {\small $\pm$ 4323} & {\small p = 1E-06} &  \\
Look-back II & 4584 & {\small $\pm$ 917} & {\small p = 1E-22} & 9897 & {\small $\pm$ 2083} & {\small p = 1E-22} & 15973 & {\small $\pm$ 4354} & {\small p = 1E-06} &  \\
Look-back III & 4036.48 & {\small $\pm$ 1424.99} & {\small p = 5.28E-17} & 7812.72 & {\small $\pm$ 2279.26} & {\small p = 1.09E-23} & 10982.40 & {\small $\pm$ 2971.65} & {\small p = 1.03E-07} &  \\
Look-ahead I & 14469.10 & {\small $\pm$ 1520.81} & {\small p = 1.09E-23} & 28559.32 & {\small $\pm$ 2104.91} & {\small p = 1.09E-23} & 46650.20 & {\small $\pm$ 3035.78} & {\small p = 1.03E-07} &  \\
Look-ahead II & 12623.42 & {\small $\pm$ 1075.25} & {\small p = 1.09E-23} & 24811.62 & {\small $\pm$ 1986.30} & {\small p = 1.09E-23} & 43089.00 & {\small $\pm$ 2511.18} & {\small p = 1.03E-07} &  \\
Look-ahead III & 16050.30 & {\small $\pm$ 1339.69} & {\small p = 1.09E-23} & 30732.00 & {\small $\pm$ 1871.07} & {\small p = 1.09E-23} & 50340.00 & {\small $\pm$ 2102.78} & {\small p = 1.03E-07} &  \\
Reward Shaping I & 14686.12 & {\small $\pm$ 1645.02} & {\small p = 1.09E-23} & 28223.94 & {\small $\pm$ 3012.81} & {\small p = 1.09E-23} & 46706.50 & {\small $\pm$ 3649.57} & {\small p = 1.03E-07} &  \\
Reward Shaping II & 11397.10 & {\small $\pm$ 905.59} & {\small p = 1.09E-23} & 21520.98 & {\small $\pm$ 2209.63} & {\small p = 1.09E-23} & 37033.40 & {\small $\pm$ 1632.24} & {\small p = 1.03E-07} &  \\
Reward Shaping III & 12125.48 & {\small $\pm$ 1209.59} & {\small p = 1.09E-23} & 23680.98 & {\small $\pm$ 1994.07} & {\small p = 1.09E-23} & 40828.70 & {\small $\pm$ 2748.82} & {\small p = 1.03E-07} & \\
$Q(\lambda)$ & 14719.58 & {\small $\pm$ 1728.19} & {\small p = 1.09E-23} & 28518.70 & {\small $\pm$ 2148.01} & {\small p = 1.09E-23} & 44017.20 & {\small $\pm$ 3170.08} & {\small p = 1.03E-07} &  \\
SARSA($\lambda$) & 8681.94 & {\small $\pm$ 704.02} & {\small p = 1.09E-23} & 23790.40 & {\small $\pm$ 836.13} & {\small p = 1.09E-23} & 48157.50 & {\small $\pm$ 1378.38} & {\small p = 1.03E-07} &  \\
RUDDER $Q(\lambda)$ & 726.72 & {\small $\pm$ 399.58} & {\small p = 3.49E-04} & 809.86 & {\small $\pm$ 472.27} & {\small p = 3.49E-04} & 906.13 & {\small $\pm$ 514.55} & {\small p = 3.36E-02} &  \\
RUDDER & 995.59 & {\small $\pm$ 670.31} & {\small p = 5.00E-01} & 1128.82 & {\small $\pm$ 741.29} & {\small p = 5.00E-01} & 1186.34 & {\small $\pm$ 870.02} & {\small p = 5.00E-01} &  \\
\end{tabular}

\begin{tabular}{*{1}{>{\raggedright}p{14em}}*{1}{>{\raggedleft}p{4em}}*{1}{>{\raggedleft}p{5em}}*{1}{>{\raggedright}p{5em}}*{1}{>{\columncolor{mColor1}\raggedleft}p{4em}}*{1}{>{\columncolor{mColor1}\raggedleft}p{5em}}*{1}{>{\columncolor{mColor1}\raggedright}p{5em}}*{1}{>{\raggedleft}p{4em}}*{1}{>{\raggedleft}p{5em}}*{1}{>{\raggedright}p{5em}}*{1}{>{\raggedright}p{0.01em}}}
\toprule[1pt]
\addlinespace[2pt]
{\bf Method} &\multicolumn{3}{c}{\bf Delay 12} &\multicolumn{3}{c}{\bf Delay 15} &\multicolumn{3}{c}{\bf Delay 17} &\\
\toprule[1pt]
Look-back I & 33082.56 & {\small $\pm$ 7641.57} & {\small p = 1.09E-23} & 49658.86 & {\small $\pm$ 8297.85} & {\small p = 1.28E-34} & 72115.16 & {\small $\pm$ 21221.78} & {\small p = 1.09E-23} &  \\
Look-back II & 23240.16 & {\small $\pm$ 9060.15} & {\small p = 1.09E-23} & 29293.94 & {\small $\pm$ 7468.94} & {\small p = 1.28E-34} & 42639.38 & {\small $\pm$ 17178.81} & {\small p = 1.09E-23} &  \\
Look-back III & 15647.40 & {\small $\pm$ 4123.20} & {\small p = 1.09E-23} & 20478.06 & {\small $\pm$ 5114.44} & {\small p = 1.28E-34} & 26946.92 & {\small $\pm$ 10360.21} & {\small p = 1.09E-23} &  \\
Look-ahead I & 66769.02 & {\small $\pm$ 4333.47} & {\small p = 1.09E-23} & 105336.74 & {\small $\pm$ 4977.84} & {\small p = 1.28E-34} & 136660.12 & {\small $\pm$ 5688.32} & {\small p = 1.09E-23} &  \\
Look-ahead II & 62220.56 & {\small $\pm$ 3139.87} & {\small p = 1.09E-23} & 100505.05 & {\small $\pm$ 4987.16} & {\small p = 1.28E-34} & 130271.88 & {\small $\pm$ 5397.61} & {\small p = 1.09E-23} &  \\
Look-ahead III & 72804.44 & {\small $\pm$ 4232.40} & {\small p = 1.09E-23} & 115616.59 & {\small $\pm$ 5648.99} & {\small p = 1.28E-34} & 149064.68 & {\small $\pm$ 7895.48} & {\small p = 1.09E-23} &  \\
Reward Shaping I & 68428.04 & {\small $\pm$ 3416.12} & {\small p = 1.09E-23} & 107399.17 & {\small $\pm$ 5242.88} & {\small p = 1.28E-34} & 137032.14 & {\small $\pm$ 6663.12} & {\small p = 1.09E-23} &  \\
Reward Shaping II & 56225.24 & {\small $\pm$ 3778.86} & {\small p = 1.09E-23} & 93091.44 & {\small $\pm$ 5233.02} & {\small p = 1.28E-34} & 122224.20 & {\small $\pm$ 5545.63} & {\small p = 1.09E-23} &  \\
Reward Shaping III & 60071.52 & {\small $\pm$ 3809.29} & {\small p = 1.09E-23} & 99476.40 & {\small $\pm$ 5607.08} & {\small p = 1.28E-34} & 130103.50 & {\small $\pm$ 6005.61} & {\small p = 1.09E-23} &  \\
$Q(\lambda)$ & 66952.16 & {\small $\pm$ 4137.67} & {\small p = 1.09E-23} & 107438.36 & {\small $\pm$ 5327.95} & {\small p = 1.28E-34} & 135601.26 & {\small $\pm$ 6385.76} & {\small p = 1.09E-23} &  \\
SARSA($\lambda$) & 78306.28 & {\small $\pm$ 1813.31} & {\small p = 1.09E-23} & 137561.92 & {\small $\pm$ 2350.84} & {\small p = 1.28E-34} & 186679.12 & {\small $\pm$ 3146.78} & {\small p = 1.09E-23} &  \\
RUDDER $Q(\lambda)$ & 1065.16 & {\small $\pm$ 661.71} & {\small p = 3.19E-01} & 972.73 & {\small $\pm$ 702.92} & {\small p = 1.13E-04} & 1101.24 & {\small $\pm$ 765.76} & {\small p = 1.54E-01} &  \\
RUDDER & 1121.70 & {\small $\pm$ 884.35} & {\small p = 5.00E-01} & 1503.08 & {\small $\pm$ 1157.04} & {\small p = 5.00E-01} & 1242.88 & {\small $\pm$ 1045.15} & {\small p = 5.00E-01} &  \\
\end{tabular}
\end{center}
\end{table}
\end{landscape}

\begin{landscape}
\begin{table}[htp]
\begin{center}
\caption{Cont. Number of episodes required by different 
 methods to solve the Trace-Back task with delayed reward. The numbers represent the  
 mean and the standard deviation over 100 trials.
 RUDDER with reward redistribution significantly outperforms all other methods.
 }
\label{tab:Ares2}%
\begin{tabular}{*{1}{>{\raggedright}p{14em}}*{1}{>{\columncolor{mColor1}\raggedleft}p{4em}}*{1}{>{\columncolor{mColor1}\raggedleft}p{5em}}*{1}{>{\columncolor{mColor1}\raggedright}p{5em}}*{1}{>{\raggedleft}p{4em}}*{1}{>{\raggedleft}p{5em}}*{1}{>{\raggedright}p{5em}}*{1}{>{\raggedright}p{0.01em}}}
\toprule[1pt]
\addlinespace[2pt]
{\bf Method} &\multicolumn{3}{c}{\bf Delay 20} &\multicolumn{3}{c}{\bf Delay 25} &\\
\toprule[1pt]
Look-back I & 113873.30 & {\small $\pm$ 31879.20} & {\small p = 1.03E-07} &  &  &  &  \\
Look-back II & 56830.30 & {\small $\pm$ 19240.04} & {\small p = 1.03E-07} & 111693.34 & {\small $\pm$ 73891.21} & {\small p = 1.09E-23} &  \\
Look-back III & 35852.10 & {\small $\pm$ 11193.80} & {\small p = 1.03E-07} &  &  &  &  \\
Look-ahead I & 187486.50 & {\small $\pm$ 5142.87} & {\small p = 1.03E-07} &  &  &  &  \\
Look-ahead II & 181974.30 & {\small $\pm$ 5655.07} & {\small p = 1.03E-07} & 289782.08 & {\small $\pm$ 11984.94} & {\small p = 1.09E-23} &  \\
Look-ahead III & 210029.90 & {\small $\pm$ 6589.12} & {\small p = 1.03E-07} &  &  &  &  \\
Reward Shaping I & 189870.30 & {\small $\pm$ 7635.62} & {\small p = 1.03E-07} & 297993.28 & {\small $\pm$ 9592.30} & {\small p = 1.09E-23} &  \\
Reward Shaping II & 170455.30 & {\small $\pm$ 6004.24} & {\small p = 1.03E-07} & 274312.10 & {\small $\pm$ 8736.80} & {\small p = 1.09E-23} &  \\
Reward Shaping III & 183592.60 & {\small $\pm$ 6882.93} & {\small p = 1.03E-07} & 291810.28 & {\small $\pm$ 10114.97} & {\small p = 1.09E-23} &  \\
$Q(\lambda)$ & 186874.40 & {\small $\pm$ 7961.62} & {\small p = 1.03E-07} &  &  &  &  \\
SARSA($\lambda$) & 273060.70 & {\small $\pm$ 5458.42} & {\small p = 1.03E-07} & 454031.36 & {\small $\pm$ 5258.87} & {\small p = 1.09E-23} &  \\
RUDDER I & 1048.97 & {\small $\pm$ 838.26} & {\small p = 5.00E-01} & 1236.57 & {\small $\pm$ 1370.40} & {\small p = 5.00E-01} &  \\
RUDDER II & 1159.30 & {\small $\pm$ 731.46} & {\small p = 8.60E-02} & 1195.75 & {\small $\pm$ 859.34} & {\small p = 4.48E-01} &  \\
\end{tabular}

\end{center}
\end{table}
\end{landscape}





\subsubsection{Task (IV): Charge-Discharge}
\label{sec:Acharge_discharge}
The Charge-Discharge task depicted in Figure~\ref{fig:charge}
is characterized by
deterministic reward and state transitions.
The environment consists of two states: 
{\em charged} $\mathtt{C}$ / {\em discharged} $\mathtt{D}$ 
and two actions {\em charge} $\mathtt{c}$ / {\em discharge}
$\mathtt{d}$.
The deterministic reward is $r(\mathtt{D},\mathtt{d})=1, r(\mathtt{C},\mathtt{d})=10,
r(\mathtt{D},\mathtt{c})=0$,
and $r(\mathtt{C},\mathtt{c})=0$. 
The reward $r(\mathtt{C},\mathtt{d})$ is accumulated 
for the whole episode and given only at time $T+1$,
where $T$ corresponds to the maximal delay of the reward.
The optimal policy alternates between charging and discharging
to accumulate a reward of 10 every other time step. 
The smaller immediate reward of $1$ distracts the agent from 
the larger delayed reward. The distraction forces the agent to learn 
the value function well enough to distinguish between the contribution of the 
immediate and the delayed reward to the final return.  
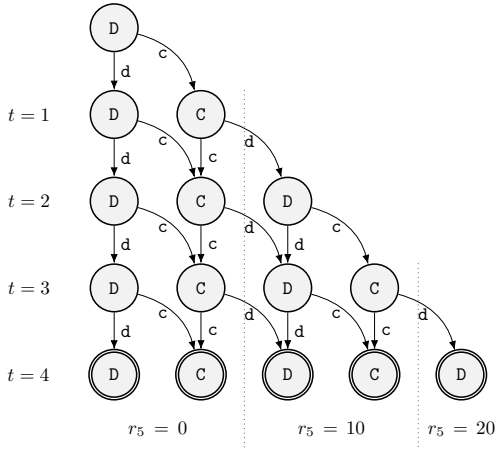
\begin{figure}[htp]%
 \centering%
 \resizebox{0.5\textwidth}{!}{\begin{tikzpicture}[node distance=1.6cm]
  \tikzstyle{every state}=[thick, fill=gray!10]
  \tikzset{edge/.style = {->,> = latex'}}
  \tikzstyle{initial text}=[" "]

  \node[state] (s100)                    {$\mathtt{D}$};
  \node[state]         (s101) [below of=s100] {$\mathtt{D}$};
  \node[state]         (s201) [right of=s101] {$\mathtt{C}$};
  \node[state]         (s102) [below of=s101] {$\mathtt{D}$};
  \node[state]         (s202) [right of=s102] {$\mathtt{C}$};
  \node[state]         (s112) [right of=s202] {$\mathtt{D}$};
  \node[state]         (s103) [below of=s102] {$\mathtt{D}$};
  \node[state]         (s203) [right of=s103]       {$\mathtt{C}$};
  \node[state]         (s113) [right of=s203] {$\mathtt{D}$};
  \node[state]         (s213) [right of=s113] {$\mathtt{C}$};
  \node[state,accepting]          (s104) [below of=s103] {$\mathtt{D}$};
  \node[state,accepting]          (s204) [right of=s104]       {$\mathtt{C}$};
  \node[state,accepting]          (s114) [right of=s204] {$\mathtt{D}$};
  \node[state,accepting]          (s214) [right of=s114] {$\mathtt{C}$};
  \node[state,accepting]         (s124) [right of=s214] {$\mathtt{D}$};

  
\path 
    (s100.south) -- coordinate (t1) (s100.south|-s101.north);
\path
    (s101.south) -- coordinate (t2) (s101.south|-s102.north);
\path
    (s102.south) -- coordinate (t3) (s102.south|-s103.north);
\path
    (s103.south) -- coordinate (t4) (s103.south|-s104.north);
\path
    (s203.west) -- coordinate (c1) (s203.west-|s113.east);
    \draw[dotted,opacity=0.8] (s101.north-|c1) -- ([yshift=-1cm]s104.south-|c1);   
\path
    (s213.west) -- coordinate (c2) (s213.west-|s124.east);
    \draw[dotted,opacity=0.8] (s103.north-|c2) -- ([yshift=-1cm]s104.south-|c2);   
    
\path   (s100) edge [-Latex,bend left]      node[ left] {$\mathtt{c}$}  (s201)
        (s201) edge  [-Latex,bend left]      node[left]  {$\mathtt{d}$} (s112)
        (s112) edge  [-Latex,bend left]      node[left]  {$\mathtt{c}$} (s213)
        (s213) edge  [-Latex,bend left]      node[left]  {$\mathtt{d}$} (s124)
        
        (s201) edge [-Latex]      node[right]  {$\mathtt{c}$} (s202)
        (s202) edge [-Latex]      node[right]  {$\mathtt{c}$} (s203)
        (s203) edge [-Latex]      node[right]  {$\mathtt{c}$} (s204)

        (s101) edge  [-Latex, bend left]      node[left] {$\mathtt{c}$} (s202)
        (s202) edge  [-Latex, bend left]     node[left] {$\mathtt{d}$} (s113)
        (s113) edge  [-Latex, bend left]      node[left] {$\mathtt{c}$} (s214)
        
        (s102) edge  [-Latex, bend left]      node[left] {$\mathtt{c}$} (s203)
        (s203) edge  [-Latex, bend left]      node[left] {$\mathtt{d}$} (s114)
        
        (s103) edge  [-Latex, bend left]      node[left] {$\mathtt{c}$} (s204)
        
        (s100) edge [-Latex]      node[right]  {$\mathtt{d}$} (s101)
        (s101) edge [-Latex]      node[right]  {$\mathtt{d}$} (s102)
        (s102) edge [-Latex]      node[right]  {$\mathtt{d}$} (s103)
        (s103) edge [-Latex]      node[right]  {$\mathtt{d}$} (s104)

        (s112) edge [-Latex]      node[right]  {$\mathtt{d}$} (s113)
        (s113) edge [-Latex]      node[right]  {$\mathtt{d}$} (s114)
        
        (s213) edge [-Latex]      node[right]  {$\mathtt{c}$} (s214)
;

\node[text width=1cm] at ([xshift=-1cm]s101.west) {$t =1$};
\node[text width=1cm] at ([xshift=-1cm]s102.west) {$t =2$};
\node[text width=1cm] at ([xshift=-1cm]s103.west) {$t=3$};
\node[text width=1cm] at ([xshift=-1cm]s104.west) {$t=4$};

\path
    (s103.west) -- coordinate (cd1) (s103.west-|s203.east);
\path
    (s114.west) -- coordinate (cd2) (s114.west-|s214.east);
\path
    (s124.west) -- coordinate (cd3) (s124.west-|s124.east);

\node[text width=2cm,align=center] at ([yshift=-1cm]s104-|cd1) { $r_5=0$};
\node[text width=2cm, align=center] at ([yshift=-1cm]s104-|cd2) {$r_5=10$};
\node[text width=2cm, align=center] at ([yshift=-1cm]s104-|cd3) {$r_5=20$};

\end{tikzpicture}}%
 \caption{The Charge-Discharge task with two basic states: {\em charged} $\mathtt{C}$ and 
 {\em discharged} $\mathtt{D}$.
 In each state the actions charge $\mathtt{c}$ leading to 
 the charged state $\mathtt{C}$ 
 and discharge $\mathtt{d}$ leading to discharged state $\mathtt{D}$ are possible. 
 Action $\mathtt{d}$ in the discharged state $\mathtt{D}$ leads to a small immediate reward of $1$ and in the charged state $\mathtt{C}$ 
 to a delayed reward of $10$. 
 After sequence end $T=4$, the accumulated delayed reward $r_{T+1}=r_5$ is given. \label{fig:charge}}
\end{figure}%

For this task, 
the RUDDER backward analysis is based on 
monotonic LSTMs
and on layer-wise relevance propagation (LRP). 
The reward redistribution provided by RUDDER uses an 
LSTM which consists of $5$ memory cells 
and is trained with Adam and a learning rate of $0.01$. 
The reward redistribution is used to learn an optimal policy by $Q$-learning and by MC
with a learning rate of $0.1$ and an exploration rate of $0.1$. 
Again, we use {\em sample updates} for $Q$-learning and MC \cite{Sutton:18book}.
The learning is stopped either if the agent achieves $90$\% of the reward of the optimal policy or after a maximum number of $10$ million episodes. 
For each $T$ and each method, 100 runs with different seeds are performed to obtain statistically relevant results. For delays with runs which did not finish within 100m episodes we estimate parameters like described in Paragraph~\ref{sec:Alr_estimation}.

\subsubsection{Task (V): Solving Trace-Back using policy gradient methods}
In this experiment, we compare policy gradient methods instead of $Q$-learning based methods. 
These methods comprise RUDDER on top of PPO with and without GAE, and a baseline PPO using GAE.
The environment and performance evaluation are the same as reported in Task III.
Again, RUDDER is exponentially faster than PPO. RUDDER on top of PPO is slightly better with GAE than without.

\begin{figure}[h]
 \centering%
 \resizebox{0.7\linewidth}{!}{%
    \input{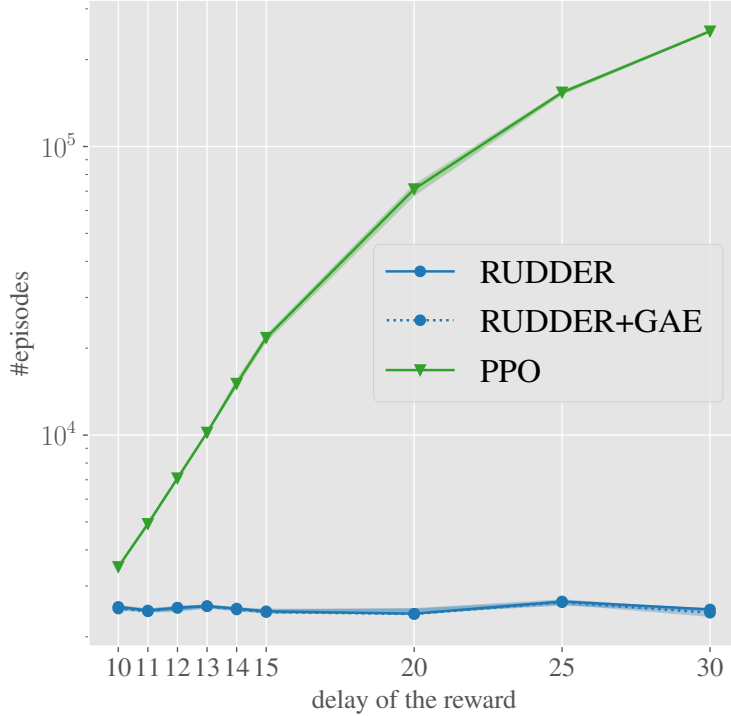}%
 }
\caption{Comparison of performance of RUDDER with GAE (RUDDER+GAE) and without GAE (RUDDER) and PPO with GAE (PPO) on artificial task V
with respect to the learning time in episodes (median of 100 trials) in log scale
vs.\ the delay of the reward.
The shadow bands indicate the $40\%$ and $60\%$ quantiles.
Again, RUDDER significantly outperforms all other methods.\label{fig:ppo}}%
\end{figure}%










\subsection{Atari Games}
\label{sec:Aatari}
In this section we describe the implementation of RUDDER for Atari games.
The implementation is largely based on the {\em OpenAI baselines}
package \cite{Dhariwal:17} for the RL components and our package
for the LSTM reward redistribution model, which will be announced upon publication.
If not specified otherwise, standard input processing, such as skipping $3$ frames and stacking $4$ frames, 
is performed by the {\em OpenAI baselines} package.

We consider the 52 Atari games that were compatible with OpenAI baselines,
Arcade Learning Environment (ALE) \cite{Bellemare:13}, and OpenAI Gym \cite{Brockman:16}.
Games are divided into episodes,
i.e.\ the loss of a life or the exceeding of 108k frames
trigger the start of a new episode without resetting the environment.
Source code will be made available at upon publication.

\subsubsection{Architecture}
\label{sec:Aatari-arch}
We use a modified PPO architecture and a separate reward redistribution model. While parts of the two could be combined, this separation allows for better comparison between the PPO baseline with and without RUDDER.

\paragraph{PPO architecture.}
The design of the policy and the value network relies 
on the {\em ppo2} implementation \cite{Dhariwal:17}, which  
is depicted in Figure~\ref{fig:atari-arch} and 
summarized in Table~\ref{tab:atari-arch}.
The network input, 4 stacked Atari game frames \cite{Mnih:15},
is processed by 3 convolution layers with ReLU activation functions,
followed by a fully connected layer with ReLU activation functions.
For PPO with RUDDER,
2 output units, for the original and redistributed reward value function,
and another set of output units for the policy prediction are applied.
For the PPO baseline without RUDDER,
the output unit for the redistributed reward value function is omitted.

\paragraph{Reward redistribution model.}
Core of the reward redistribution model
is an LSTM layer containing 64 memory cells
with sigmoid gate activations, tanh input nonlinearities, 
and identity output activation functions, 
as illustrated in Figure~\ref{fig:atari-arch} 
and summarized in Table~\ref{tab:atari-arch}.
This LSTM implementation omits output gate and forget gate to simplify the network dynamics.
Identity output activation functions were chosen to support 
the development of linear counting dynamics within the LSTM layer, 
as is required to count the reward pieces during an episode chunk.
Furthermore, the input gate is only connected recurrently to other LSTM blocks and the
cell input is only connected to forward connections from the lower layer.
For the vision system the same architecture was used as with the PPO network,
with the first convolution layer being doubled to process
$\Delta$ frames and full frames separately in the first layer.
Additionally, the memory cell layer receives the vision feature activations
of the PPO network, the current action, and the approximate in-game time as inputs.
No gradients from the reward redistribution network are propagated over the connections to the PPO network.
After the LSTM layer, the reward redistribution model
has one output node for the prediction $\widehat{\returnrealization}$
of the return realization $\returnrealization$ of the return variable $G_0$.
The reward redistribution model has 4 additional
output nodes for the auxiliary tasks as described in Section~\ref{sec:Aatari-taupdate}.

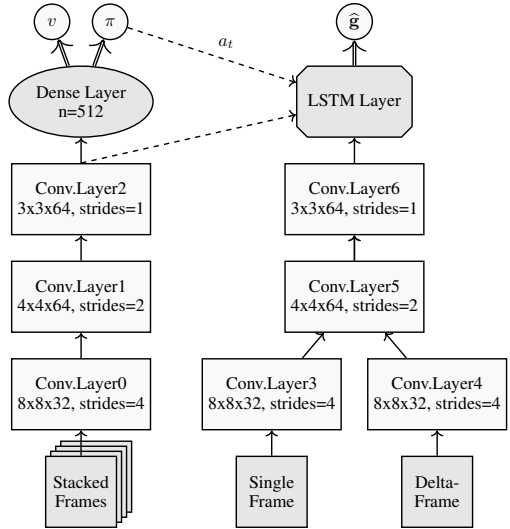
\begin{figure}[htp]%
 \centering%
 \resizebox{0.5\textwidth}{0.5\textwidth}{\newcommand{\blkh}{4em}
\newcommand{\blkw}{4em}
\newcommand{\hsep}{0.5cm}
\newcommand{\vsep}{0.5cm}

\tikzset{%
  conv/.style    = {draw, thick, rectangle, align=flush center, anchor=mid, fill=gray!5, minimum height = \blkh, minimum width = \blkw},
  lstm/.style    = {draw, thick, chamfered rectangle, align=flush center, anchor=mid, fill=gray!20, minimum height = \blkh, minimum width = \blkw},
  dense/.style    = {draw, thick, ellipse, rounded corners, align=flush center, anchor=mid, fill=gray!20, minimum height = \blkh, minimum width = \blkw},
  mp/.style    = {draw, thick, shape aspect=2, diamond, rounded corners, align=flush center, anchor=mid, fill=gray!10, minimum height = 0.01*\blkh, minimum width = \blkw},
  input_frame/.style    = {draw, thick, rectangle, align=flush center, anchor=mid, fill=gray!20, minimum height = \blkh, minimum width = \blkw},
  output_node/.style    = {draw, thick, circle, align=flush center, anchor=bottom, fill=gray!5, minimum height = 0.5*\blkh, minimum width = 0.5*\blkw},
  c/.style    = {coordinate}
}

\begin{tikzpicture}[auto, thick, node distance= \vsep and \hsep]
\draw
	node at (0,0) [input_frame, name=a2c_inp]{Stacked\\Frames} 
	node [input_frame, above right =0.3 and 0.3 of a2c_inp.south west, behind path, name=a2c_inp_s3]{} 
	node [input_frame, above right =0.2 and 0.2 of a2c_inp.south west, behind path, name=a2c_inp_s2]{}
	node [input_frame, above right =0.1 and 0.1 of a2c_inp.south west, behind path, name=a2c_inp_s1]{}  
	node [conv, above = of a2c_inp, name=a2c_c1]{Conv.Layer0\\8x8x32, strides=4} 
	node [conv, above = of a2c_c1, name=a2c_c2]{Conv.Layer1\\4x4x64, strides=2} 
	node [conv, above = of a2c_c2, name=a2c_c3]{Conv.Layer2\\3x3x64, strides=1} 
	node [dense, above = of a2c_c3, name=a2c_d1]{Dense Layer\\n=512} 
	
	node [conv, right=2*\vsep  of a2c_c1, name=f_c1]{Conv.Layer3\\8x8x32, strides=4}
	
	node [conv, right=\vsep of f_c1, name=d_c1]{Conv.Layer4\\8x8x32, strides=4}
	
	node [c, name=c] at ($(f_c1)!0.5!(d_c1)$){}
	node [conv, above= of d_c1.north-|c, name=rr_c0] {Conv.Layer5\\4x4x64, strides=2}
	node [conv, above= of rr_c0, name=rr_c1] {Conv.Layer6\\3x3x64, strides=1}
	
	
	node [lstm, above = of rr_c1, name=lstm] {LSTM Layer} 
	
	node [c, name=R_height] at (lstm.north){}
	node [c, name=vf_height] at ($(a2c_d1.north-|a2c_d1.west)!0.3!(a2c_d1.north-|a2c_d1.east)$){}
	node [c, name=a_height] at ($(a2c_d1.north-|a2c_d1.west)!0.7!(a2c_d1.north-|a2c_d1.east)$){}
	node [output_node, above = of R_height, name=rr_c1, name=R]{$\widehat{\returnrealization}$} 
	node [output_node, above = of a_height, name=a]{$\pi$}
	node [output_node, above = of vf_height, name=vf]{$v$}
	
	node [input_frame, below= of f_c1, name=f_inp]{Single\\Frame} 
	node [input_frame, below= of d_c1, name=d_inp]{Delta-\\Frame} 
	
%
;
      
    
	\draw[->](a2c_inp) -- node {}(a2c_c1);
	\draw[->](a2c_c1) -- node {}(a2c_c2);
	\draw[->](a2c_c2) -- node {}(a2c_c3);
	\draw[->](a2c_c3) -- node {}(a2c_d1);
	
	\draw[double, ->](a2c_d1) -- node {}(a);
	\draw[double, ->](a2c_d1) -- node {}(vf);
	
	\draw[->](f_inp) -- node {}(f_c1);
	\draw[->](f_c1) -- node {}(rr_c0);
	\draw[->](d_inp) -- node {}(d_c1);
	\draw[->](d_c1) -- node {}(rr_c0);
	
	\draw[->](rr_c0) -- node {}(rr_c1);
	\draw[->](rr_c0) -- node {}(rr_c1);
	\draw[->](rr_c1) -- node {}(lstm);
	
	\draw[double, ->](lstm) -- node {}(R);
	
	\draw[->, dashed](a2c_c3.north) -- node {}(lstm);
	
	\draw[->, dashed](a) -- node {$a_{t}$}(lstm);

\end{tikzpicture}}
 \caption{RUDDER architecture for Atari games as described in Section~\ref{sec:Aatari-arch}. 
 Left: The {\em ppo2} implementation \cite{Dhariwal:17}.
 Right: LSTM reward redistribution architecture.
 The reward redistribution network has access to the PPO vision features (dashed lines)
 but no gradient is propagated between the networks.
 The LSTM layer receives the current action and an approximate 
 in-game-time as additional input.
 The PPO outputs $v$ for value function prediction and $\pi$ for policy prediction each represent multiple output nodes: the original and redistributed reward value function prediction for $v$ and the outputs for all of the available actions for $\pi$.
 Likewise, the reward redistribution network output $\widehat{\returnrealization}$ represents multiple outputs, 
 as described in Section~\ref{sec:Aatari-taupdate}
 Details on layer configuration are given in Table~\ref{tab:atari-arch}.\label{fig:atari-arch}}%
\end{figure}%

\begin{table}[htp]%
  \resizebox{\linewidth}{!}{%
  \begin{tabular}[t]{lllllll}
    \toprule
Layer&Specifications&&Layer&Specifications&\\
    \midrule
Conv.Layer 0&features&32&Conv.Layer 4&features&32\\
&kernelsize&8x8&&kernelsize&8x8\\
&striding&4x4&&striding&4x4\\
&act&ReLU&&act&ReLU\\
&initialization&orthogonal, \lstinline{gain}=$\sqrt{2}$&&initialization&orthogonal, \lstinline{gain}=$0.1$\\
Conv.Layer 1&features&64&Conv.Layer 5&features&64\\
&kernelsize&4x4&&kernelsize&4x4\\
&striding&2x2&&striding&2x2\\
&act&ReLU&&act&ReLU\\
&initialization&orthogonal, \lstinline{gain}=$\sqrt{2}$&&initialization&orthogonal, \lstinline{gain}=$0.1$\\
Conv.Layer 2&features&64&Conv.Layer 6&features&64\\
&kernelsize&3x3&&kernelsize&3x3\\
&striding&1x1&&striding&1x1\\
&act&ReLU&&act&ReLU\\
&initialization&orthogonal, \lstinline{gain}=$\sqrt{2}$&&initialization&orthogonal, \lstinline{gain}=$0.1$\\
Dense Layer&features&512&LSTM Layer&cells&64\\
&act&ReLU&&gate act.&sigmoid\\
&initialization&orthogonal, \lstinline{gain}=$\sqrt{2}$&&ci act.&tanh\\
Conv.Layer 3&features&32&&output act.&linear\\
&kernelsize&8x8&&bias ig&trunc.norm., \lstinline{mean}$=-5$\\
&striding&4x4&&bias ci&trunc.norm., \lstinline{mean}$=0$\\
&act&ReLU&&fwd.w. ci&trunc.norm., \lstinline{scale}$=0.0001$\\
&initialization&orthogonal, \lstinline{gain}=$0.1$&&fwd.w. ig&omitted\\
&&&&rec.w. ci&omitted\\
&&&&rec.w. ig&trunc.norm., \lstinline{scale}$=0.001$\\
&&&&og&omitted\\
&&&&fg&omitted\\
    \bottomrule
  \end{tabular}}%
  \caption{Specifications of PPO and RUDDER architectures 
as shown in Figure~\ref{fig:atari-arch}.
Truncated normal initialization has the default values 
\lstinline{mean}$=0$, \lstinline{stddev}$=1$ and is optionally multiplied by a factor \lstinline{scale}.
\label{tab:atari-arch}}%
\end{table}%

\subsubsection{Lessons Replay Buffer}
\label{sec:Alessonbuffer}
The lessons replay buffer is realized as
a priority-based buffer
containing up to $128$ samples.
New samples are added to the buffer if (i) the buffer is not filled or
if (ii) the new sample is considered more important than the least important sample in the buffer,
in which case the new sample replaces the least important sample.

Importance of samples for the buffer is determined based on a combined ranking of
(i) the reward redistribution model error and
(ii) the difference of the sample return to the mean return of all samples in the lessons buffer.
Each of these two rankings contributes equally to the final ranking of the sample. Samples with higher loss and greater difference to the mean return achieve a higher ranking.

Sampling from the lessons buffer is performed as a sampling from a softmax function on the sample-losses in the buffer. Each sample is a sequence of $512$ consecutive transitions, as described in the last paragraph of Section~\ref{sec:Aatari-taupdate}.

\subsubsection{Game Processing, Update Design, and Target Design}\label{sec:Aatari-taupdate}
Reward redistribution is performed in an online fashion as new transitions are sampled from the environment. This allows to keep the original update schema of the PPO baseline, while still using the redistributed reward for the PPO updates. Training of the reward redistribution model is done separately on the lessons buffer samples from Section~\ref{sec:Alessonbuffer}. These processes are described in more detail in the following paragraphs.

\paragraph{Reward Scaling.}
As described in the main paper,
rewards for the PPO baseline and RUDDER are scaled based on the maximum return per sample encountered during training so far.
With $i$ samples sampled from the environment
and a maximum return of
$\returnrealization_{i}^{\max} = \max_{1\leq j \leq i}\{\ABS{\returnrealization_j}\}$
encountered, the scaled reward $r_{\nn}$ is
\begin{align}
    r_{\nn} \ &= \ \frac{10 \ r}{\returnrealization_{i}^{\max}} \ .
\end{align}

Goal of this scaling is to normalize the reward $r$ to range $[-10, 10]$ 
with a linear scaling, suitable for training the PPO and reward redistribution model.
Since the scaling is linear, the original proportions between rewards are kept.
Downside to this approach is that if a new
maximum return is encountered,
the scaling factor is updated, and the models have to readjust.

\paragraph{Reward redistribution.}
Reward redistribution is performed using differences of return predictions of the LSTM network.
That is, the differences of the reward redistribution model prediction $\widehat{\returnrealization}$
at time step $t$ and $t-1$ serve as contribution analysis and thereby give the redistributed reward $r_t = \widehat{\returnrealization}_t - \widehat{\returnrealization}_{t-1}$.
This allows for online reward redistribution on the sampled transitions
before they are used to train the PPO network,
without waiting for the game sequences to be completed.

To assess the current quality of the reward redistribution model,
a quality measure based on the relative absolute error of 
the prediction $\widehat{\returnrealization}_T$ at the last time step $T$ is introduced:
\begin{align}
    \text{\tt quality} \ &= \ 1 \ - \ \frac{\ABS{\returnrealization - \widehat{\returnrealization}_T}}{\mu} \  \frac{1}{1-\epsilon},
\end{align}
with $\epsilon$ as quality threshold of $\epsilon=80\%$ and the maximum
possible error $\mu$ as $\mu=10$ due to the reward scaling applied.
{\tt quality} is furthermore clipped to be within range $[0, 1]$.

\paragraph{PPO model.}
\label{c:ppomodel}
The {\em ppo2} implementation \cite{Dhariwal:17}
samples from the environment using multiple agents in parallel.
These agents play individual environments but share all weights,
i.e.\ they are distinguished by random effects in the environment
or by exploration.
The value function and policy network is trained online on a batch of transitions
sampled from the environment.
Originally, the policy/value function network updates are adjusted using a policy loss,
a value function loss, and an entropy term, each with dedicated scaling factors \cite{Schulman:17}.
To decrease the number of hyperparameters,
the entropy term scaling factor is adjusted automatically using Proportional Control
to keep the policy entropy in a predefined range.

We use two value function output units to predict the value functions of 
the original and the redistributed reward.
For the PPO baseline without RUDDER, the output unit for the redistributed reward is omitted.
Analogous to the {\em ppo2} implementation, these two value function predictions 
serve to compute the advantages used to scale the policy gradient updates.
For this, the advantages for original reward $a_{o}$ and redistributed reward $a_{r}$ are combined
as a weighted sum $a = a_{o} \  (1 - \text{\tt qualityv}) + a_{r} \ \text{\tt quality}$.
The PPO value function loss term $L_{v}$ is replaced by the sum of the value function $v_o$
loss $L_{o}$ for the original reward  and
the scaled value function $v_r$ loss $L_{r}$ for the redistributed reward,
such that $L_{v} = L_{o} + L_{r} \ \text{\tt quality}$.
Parameter values were taken from
the original paper \cite{Schulman:17} and implementation \cite{Dhariwal:17}.
Additionally, a coarse hyperparameter search was performed with value function coefficients 
$\{0.1, 1, 10\}$
and replacing the static entropy coefficient by a Proportional Control scaling of 
the entropy coefficient.
The Proportional Control target entropy was linearly decreased from $1$ to $0$ over 
the course of training.
PPO baseline hyperparamters were used for PPO with RUDDER without changes.

Parameter values are listed in Table~\ref{tab:atari_params}.

\paragraph{Reward redistribution model.}
The loss of the reward redistribution model for a sample is composed of four parts.
(i) The main loss $L_m$,
which is the squared prediction loss of $\returnrealization$ at the last time step $T$ of the episode
\begin{align}
    L_m \ &= \ \left(\returnrealization \ - \ \widehat{\returnrealization}_T\right)^2 \ ,
\end{align}
(ii) the continuous prediction loss $L_c$ of $\returnrealization$ at each time step
\begin{align}
    L_c \ &= \ \frac{1}{T+1} \ \sum_{t=0}^T\left(\returnrealization ´
    \ -  \ \widehat{\returnrealization}_t\right)^2\ ,
\end{align}
(iii) the loss $L_e$ of the prediction of the output at $t+10$ at each time step $t$
\begin{align}
    L_e \ &= \ \frac{1}{T-9} \ \sum_{t=0}^{T-10}\left(\widehat{\returnrealization}_{t+10}
    \ - \ \widehat{\left(\widehat{\returnrealization}_{t+10}\right)}_t\right)^2 \ ,
\end{align}
as well as (iv) the loss on 3 auxiliary tasks.
At every time step $t$,
these auxiliary tasks are
(1) the prediction of the action-value function $\widehat{q}_t$,
(2) the prediction of the accumulated original reward $\tilde{r}$ in the next 10 frames 
$\sum_{i=t}^{t+10} \tilde{r}_i$,
and (3) the prediction of the accumulated reward in the next 50 frames
$\sum_{i=t}^{t+50} \tilde{r}_i$,
resulting in the final auxiliary loss $L_a$ as
\begin{align}
    L_{a1} \ &= \ \frac{1}{T+1} \ \sum_{t=0}^T\left(q_t \ - \ \widehat{q}_t\right)^2 \ ,\\
    L_{a2} \ &= \ \frac{1}{T-9} \ \sum_{t=0}^{T-10}\left(\sum_{i=t}^{t+10} \tilde{r}_i 
    \ - \ \widehat{\left(\sum_{i=t}^{t+10} \tilde{r}_i\right)}_t\right)^2 \ ,\\
    L_{a3} \ &= \ \frac{1}{T-49} \ \sum_{t=0}^{T-50}\left(\sum_{i=t}^{t+50} \tilde{r}_i 
    \ - \ \widehat{\left(\sum_{i=t}^{t+50} \tilde{r}_i\right)}_t\right)^2\ ,\\
    L_a \ &= \ \frac{1}{3} \ \left(L_{a1} \ + \ L_{a2}\ +\ L_{a3}\right)\ .
\end{align}

The final loss for the reward redistribution model is then computed as
\begin{align}
    L \ &= \ L_m \ + \ \frac{1}{10} \ \left(L_c \ + \ L_e \ + \ L_a \right) \ .
\end{align}

The continuous prediction and earlier prediction losses $L_c$ and $L_e$
push the reward redistribution model toward performing an optimal reward redistribution.
This is because important events that are redundantly encoded 
in later states are stored as early as possible.
Furthermore, the auxiliary loss $L_a$ speeds up learning 
by adding more information about the original immediate rewards to the updates.

The reward redistribution model is only trained on the lessons buffer.
Training epochs on the lessons buffer are performed every $10^4$ PPO updates or if a new sample was added to the lessons buffer.
For each such training epoch, 8 samples are sampled from the lessons buffer.
Training epochs are repeated until the reward redistribution quality 
is sufficient ($\text{\tt quality} > 0$) for all replayed samples in the last 5 training epochs.

The reward redistribution model is not trained or used until
the lessons buffer contains at least 32 samples
and samples with different return have been encountered.

Parameter values are listed in Table~\ref{tab:atari_params}.

\begin{table}[htp]
\begin{center}
\begin{tabular}{lrclr}
\toprule
\multicolumn{2}{c}{PPO} & & \multicolumn{2}{c}{RUDDER}\\
learning rate & $2.5 \cdot 10^{-4}$ & & learning rate & $10^{-4}$ \\
policy coefficient & $1.0$ & & $L_2$ weight decay  & $10^{-7}$ \\
initial entropy coefficient & $0.01$ & & gradient clipping  & $0.5$ \\
value function coefficient & $1.0$ & & optimization & ADAM \\
\bottomrule
\end{tabular}
\end{center}
\caption{Left: Update parameters for PPO model.
Entropy coefficient is scaled via Proportional Control 
with the target entropy linearly annealed from $1$ to $0$ over the course of learning.
Unless stated otherwise,
default parameters of {\em ppo2} implementation \cite{Dhariwal:17} are used.
Right: Update parameters for reward redistribution model of RUDDER.
\label{tab:atari_params}}%
\end{table}

\paragraph{Sequence chunking and Truncated Backpropagation Through Time (TBPTT).}
Ideally, RUDDER would be trained on completed game sequences,
to consequently redistribute the reward within a completed game.
To shorten computational time for learning the reward redistribution model,
the model is not trained on completed game sequences
but on sequence chunks consisting of $512$ time steps.
The beginning of such a chunk is treated as beginning of a new episode for the model
and ends of episodes within this chunk reset the state of the LSTM,
so as to not redistribute rewards between episodes.
To allow for updates on sequence chunks even if the game sequence is not completed,
the PPO value function prediction is used to estimate the expected future reward at the end of the chunk.

Utilizing TBPTT to further speed up LSTM learning,
gradients for the reward redistribution LSTM are cut after every 128 time steps.

%

\subsubsection{Exploration}
\label{sec:Aatari-exploration}
Safe exploration to increase the likelihood of observing 
delayed rewards is an important feature of RUDDER.
We use a safe exploration strategy, which is realized by normalizing the output 
of the policy network to range $[0,1]$ and randomly picking one of the actions 
that is above a threshold $\theta$. 
Safe exploration is activated once per sequence at a random sequence position 
for a random duration between 0 and the average game length $\bar{l}$. 
Thereby we encourage long but safe off-policy trajectories within parts of the game sequences.
Only 2 of the 8 parallel actors use safe exploration with $\theta_1 = 0.001$ and $\theta_1 = 0.5$, respectively.
All actors sample from the softmax policy output.

To avoid policy lag during safe exploration transitions,
we use those transitions only to update
the reward redistribution model but not the PPO model.

\subsubsection{Results}
Training curves for 3 random seeds for PPO baseline and PPO with RUDDER are shown in Figure~\ref{fig:atari_training} and scores are listed in Table~\ref{tab:full_atari_scores} for all 52 Atari games. Training was conducted over 200M game frames (including skipped frames), as described in the experiments section of the main paper. 

We investigated failures and successes of RUDDER in different Atari games.
RUDDER failures were observed to be mostly due to LSTM failures and comprise e.g. slow learning in Breakout, explaining away in Double Dunk, spurious redistributed rewards in Hero, overfitting to the first levels in Qbert, and exploration problems in MontezumaRevenge.
RUDDER successes were observed to be mostly due to redistributing rewards to important key actions that would otherwise not receive reward, such as moving towards the built igloo in Frostbite, diving up for refilling oxygen in Seaquest, moving towards the treasure chest in Venture, and shooting at the shield of the enemy boss UFO, thereby removing its shield.

\begin{figure}[htp]
 \centering%
 \resizebox{\textwidth}{!}{\input{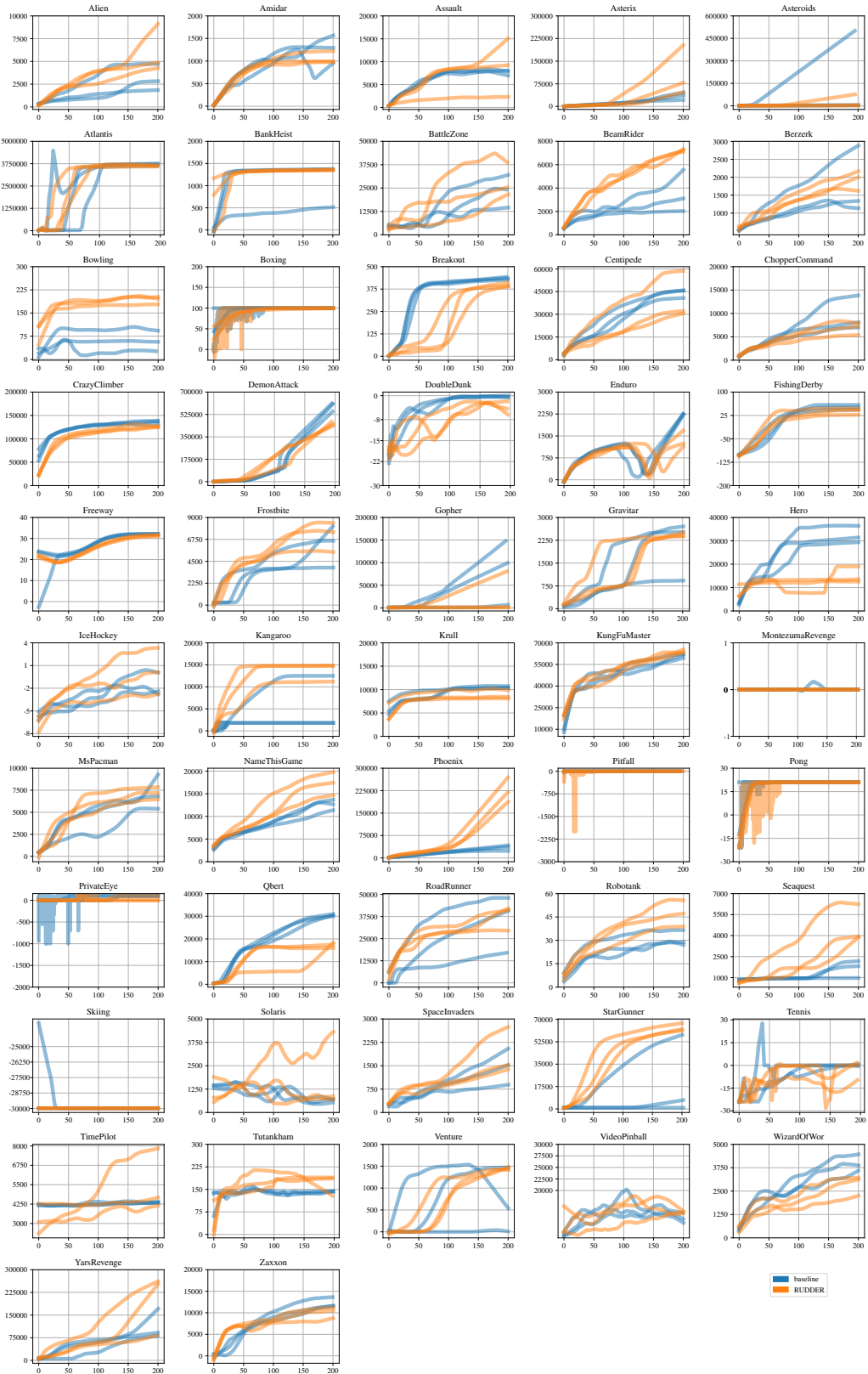}}
 \caption{Training curves for PPO baseline and PPO with RUDDER over 200M game frames, 3 runs with different random seeds each. Curves show scores during training of a single agent that does not use safe exploration, smoothed using Locally Weighted Scatterplot Smoothing (y-value estimate using 20\% of data with 10 residual-based re-weightings).\label{fig:atari_training}}%
\end{figure}%


\begin{table}[htp]
\begin{center}
\resizebox{0.9\textwidth}{!}{%
\begin{tabular}{lrrrrrr}
\toprule
 & \multicolumn{3}{c}{{\em average}} & \multicolumn{3}{c}{{\em final}}\\
 & baseline & RUDDER & \% & baseline & RUDDER & \%\\
Alien & 1,878 & 3,087 & 64.4 & 3,218 & 5,703 & 77.3 \\
Amidar & 787 & 724 & -8.0 & 1,242 & 1,054 & -15.1 \\
Assault & 5,788 & 4,242 & -26.7 & 10,373 & 11,305 & 9.0 \\
Asterix & 10,554 & 18,054 & 71.1 & 29,513 & 102,930 & 249 \\
Asteroids & 22,065 & 4,905 & -77.8 & 310,505 & 154,479 & -50.2 \\
Atlantis & 1,399,753 & 1,655,464 & 18.3 & 3,568,513 & 3,641,583 & 2.0 \\
BankHeist & 936 & 1,194 & 27.5 & 1,078 & 1,335 & 23.8 \\
BattleZone & 12,870 & 17,023 & 32.3 & 24,667 & 28,067 & 13.8 \\
BeamRider & 2,372 & 4,506 & 89.9 & 3,994 & 6,742 & 68.8 \\
Berzerk & 1,261 & 1,341 & 6.4 & 1,930 & 2,092 & 8.4 \\
Bowling & 61.5 & 179 & 191 & 56.3 & 192 & 241 \\
Boxing & 98.0 & 94.7 & -3.4 & 100 & 99.5 & -0.5 \\
Breakout & 217 & 153 & -29.5 & 430 & 352 & -18.1 \\
Centipede & 25,162 & 23,029 & -8.5 & 53,000 & 36,383 & -31.4 \\
ChopperCommand & 6,183 & 5,244 & -15.2 & 10,817 & 9,573 & -11.5 \\
CrazyClimber & 125,249 & 106,076 & -15.3 & 140,080 & 132,480 & -5.4 \\
DemonAttack & 28,684 & 46,119 & 60.8 & 464,151 & 400,370 & -13.7 \\
DoubleDunk & -9.2 & -13.1 & -41.7 & -0.3 & -5.1 & -1,825 \\
Enduro & 759 & 777 & 2.5 & 2,201 & 1,339 & -39.2 \\
FishingDerby & 19.5 & 11.7 & -39.9 & 52.0 & 36.3 & -30.3 \\
Freeway & 26.7 & 25.4 & -4.8 & 32.0 & 31.4 & -1.9 \\
Frostbite & 3,172 & 4,770 & 50.4 & 5,092 & 7,439 & 46.1 \\
Gopher & 8,126 & 4,090 & -49.7 & 102,916 & 23,367 & -77.3 \\
Gravitar & 1,204 & 1,415 & 17.5 & 1,838 & 2,233 & 21.5 \\
Hero & 22,746 & 12,162 & -46.5 & 32,383 & 15,068 & -53.5 \\
IceHockey & -3.1 & -1.9 & 39.4 & -1.4 & 1.0 & 171 \\
Kangaroo & 2,755 & 9,764 & 254 & 5,360 & 13,500 & 152 \\
Krull & 9,029 & 8,027 & -11.1 & 10,368 & 8,202 & -20.9 \\
KungFuMaster & 49,377 & 51,984 & 5.3 & 66,883 & 78,460 & 17.3 \\
MontezumaRevenge & 0.0 & 0.0 & 38.4 & 0.0 & 0.0 & 0.0 \\
MsPacman & 4,096 & 5,005 & 22.2 & 6,446 & 6,984 & 8.3 \\
NameThisGame & 8,390 & 10,545 & 25.7 & 10,962 & 17,242 & 57.3 \\
Phoenix & 15,013 & 39,247 & 161 & 46,758 & 190,123 & 307 \\
Pitfall & -8.4 & -5.5 & 34.0 & -75.0 & 0.0 & 100 \\
Pong & 19.2 & 18.5 & -3.9 & 21.0 & 21.0 & 0.0 \\
PrivateEye & 102 & 34.1 & -66.4 & 100 & 33.3 & -66.7 \\
Qbert & 12,522 & 8,290 & -33.8 & 28,763 & 16,631 & -42.2 \\
RoadRunner & 20,314 & 27,992 & 37.8 & 35,353 & 36,717 & 3.9 \\
Robotank & 24.9 & 32.7 & 31.3 & 32.2 & 47.3 & 46.9 \\
Seaquest & 1,105 & 2,462 & 123 & 1,616 & 4,770 & 195 \\
Skiing & -29,501 & -29,911 & -1.4 & -29,977 & -29,978 & 0.0 \\
Solaris & 1,393 & 1,918 & 37.7 & 616 & 1,827 & 197 \\
SpaceInvaders & 778 & 1,106 & 42.1 & 1,281 & 1,860 & 45.2 \\
StarGunner & 6,346 & 29,016 & 357 & 18,380 & 62,593 & 241 \\
Tennis & -13.5 & -13.5 & 0.2 & -4.0 & -5.3 & -32.8 \\
TimePilot & 3,790 & 4,208 & 11.0 & 4,533 & 5,563 & 22.7 \\
Tutankham & 123 & 151 & 22.7 & 140 & 163 & 16.3 \\
Venture & 738 & 885 & 20.1 & 820 & 1,350 & 64.6 \\
VideoPinball & 19,738 & 19,196 & -2.7 & 15,248 & 16,836 & 10.4 \\
WizardOfWor & 3,861 & 3,024 & -21.7 & 6,480 & 5,950 & -8.2 \\
YarsRevenge & 46,707 & 60,577 & 29.7 & 109,083 & 178,438 & 63.6 \\
Zaxxon & 6,900 & 7,498 & 8.7 & 12,120 & 10,613 & -12.4 \\

\bottomrule
\end{tabular}
}%
\caption{Scores on all 52 considered Atari games for the PPO baseline and PPO with RUDDER and the improvement by using RUDDER in percent (\%).
Agents are trained for 200M game frames (including skipped frames) with {\em no-op starting condition},
i.e.\ a random number of up to 30 no-operation actions at the start of
each game.
Episodes are prematurely terminated if a maximum of 108K frames is reached.  
Scoring metrics are (a) {\em average},
the average reward per completed game throughout training, which favors fast learning \cite{Schulman:17} 
and (b) {\em final}, 
the average over the last 10 consecutive games at the end of training, 
which favors consistency in learning.
Scores are shown for one agent without safe exploration.}%
\label{tab:full_atari_scores}%
\end{center}
\end{table}

\clearpage
{\bf Visual Confirmation of Detecting Relevant Events by Reward Redistribution.} 
We visually confirm a meaningful and helpful redistribution
of reward in both Bowling and Venture during training.
As illustrated in Figure~\ref{fig:ventureexample},
RUDDER is capable of redistributing a reward to key events in a game,
drastically shortening the delay of the reward and quickly steering
the agent toward good policies.
Furthermore, it enriches sequences that were sparse in reward
with a dense reward signal. Video demonstrations are available at \url{https://goo.gl/EQerZV}.
\begin{figure}[htp]
\includegraphics[angle=0,width=0.49\textwidth]{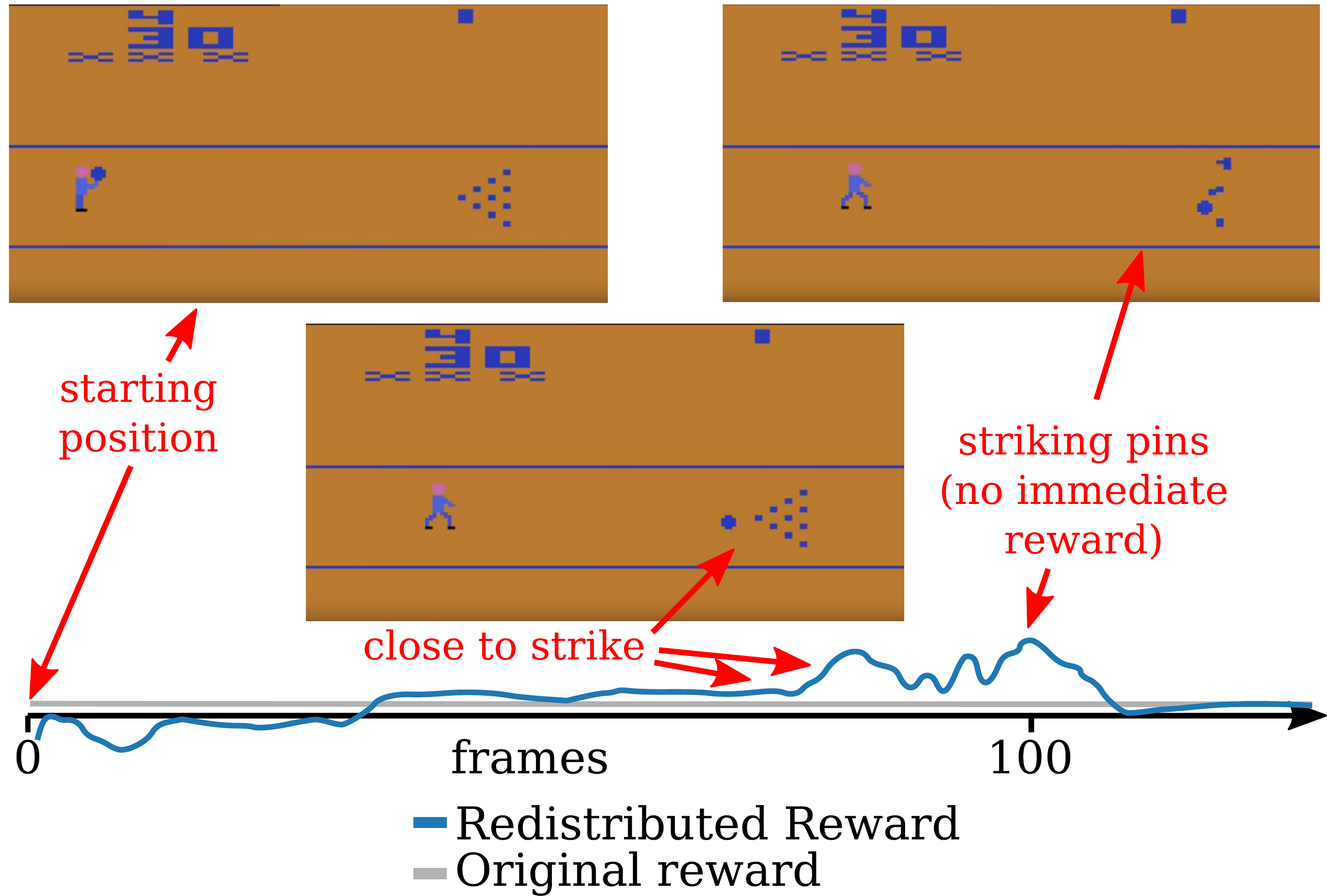} \hfill
\includegraphics[angle=0,width=0.49\textwidth]{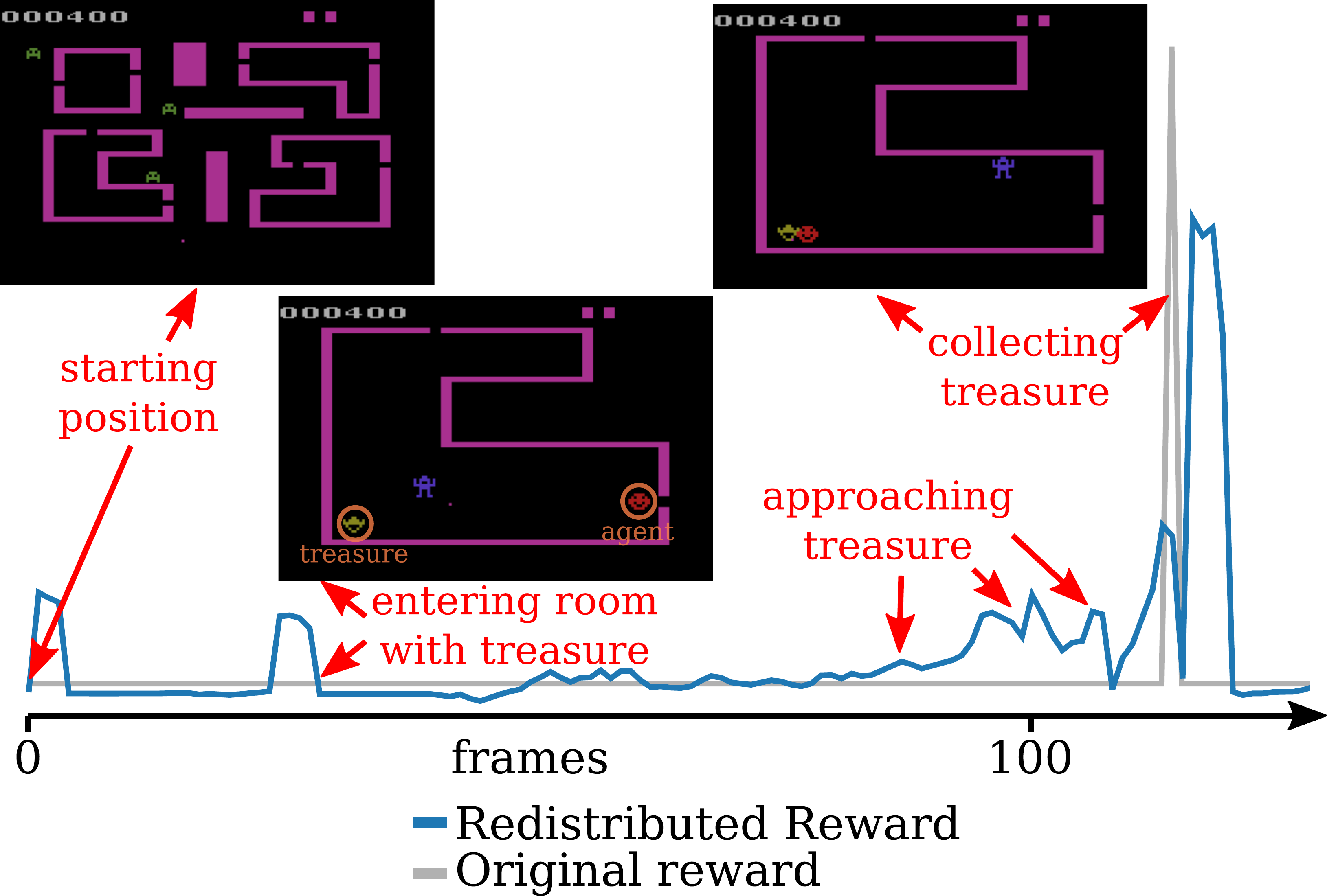}
\caption{Observed return decomposition by RUDDER in two Atari games with long delayed rewards.
{\bf Left:} In the game Bowling, reward is only given after a turn which consist of multiple rolls.
RUDDER identifies the actions that 
guide the ball in the right direction to hit all pins. 
Once the ball hit the pins, RUDDER detects the delayed reward 
associated with striking the pins down. 
In the figure only 100 frames are represented 
but the whole turn spans more than 200 frames. 
In the original game, the reward is given only at the end of the turn.
{\bf Right:} In the game Venture, reward is only obtained after picking the treasure.
RUDDER guides the agent (red) towards the treasure (golden) via reward redistribution.
Reward is redistributed to entering a room with treasure.
Furthermore, the redistributed reward gradually increases 
as the agent approaches the treasure.
For illustration purposes, the green curve shows the
return redistribution before applying lambda. 
The environment only gives reward at the event of 
collecting treasure (blue). \label{fig:ventureexample}}
\end{figure}

\clearpage
\pagebreak
\section{Discussion and Frequent Questions}

\paragraph{RUDDER and reward rescaling.}
RUDDER works with no rescaling, various rescalings, and sign function 
as we have confirmed in additional experiments.
Rescaling ensures similar reward magnitudes across different Atari games, 
therefore the same hyperparameters can be used for all games. 
For LSTM and PPO, we only scale the original return by a constant factor, 
therefore do not change the problem and do not simplify it. 
The sign function, in contrast, may simplify the problem 
but may change the optimal policy.

\paragraph{RUDDER for infinite horizon: Continual Learning.}
RUDDER assumes a finite horizon problem.
For games and for most tasks in real world these assumptions apply: 
did you solve the task? (make tax declaration, convince a customer to buy, 
design a drug, drive a car to a location, 
assemble a car, build a building, clean the room, cook a meal, pass the Turing test). 
In general our approach can be extended to continual learning with discounted reward. 
Only the transformation of an immediate reward MDP 
to an MDP with episodic reward is no longer possible. 
However the delayed reward problem becomes more obvious and also 
more serious when not discounting the reward.

\paragraph{Is the LSTM in RUDDER a state-action value function?}
For reward redistribution we assume an MDP with one reward (=return) 
at sequence end which can be predicted from the last state-action pair. 
When introducing the $\Delta$-states, 
the reward cannot be predicted from the last $\Delta$ 
and the task is no longer Markov. 
However the return can be predicted from the sequence of $\Delta$s. 
Since the $\Delta$s are mutually independent, 
the contribution of each $\Delta$ to the return must be stored 
in the hidden states of the LSTM to predict the final reward. 
The $\Delta$ can be generic as states and actions can be numbered 
and then the difference of this numbers can be used for $\Delta$.

In the applications like Atari with immediate rewards 
we give the accumulated reward at the end of the episode without enriching the states. 
This has a similar effect as using $\Delta$. 
We force the LSTM to build up an internal state 
which tracks the already accumulated reward. 

True, the LSTM is the value function at time $t$ based on the 
$\Delta$ sub-sequence up to t. 
The LSTM prediction can be decomposed into two sub-predictions. 
The first sub-prediction is the contribution of the already known 
$\Delta$ sub-sequence up to t to the return (backward view). 
The second sub-prediction is the expected contribution 
of the unknown future sequence from t+1 onwards 
to the return (forward view). 
However, we are not interested in the second sub-prediction 
but only in the contribution of $\Delta_t$ to the prediction of the expected return. 
The second sub-prediction is irrelevant for our approach. 
We cancel the second sub-prediction via the differences of predictions. 
The difference at time t gives the contribution of $\Delta_t$ to the expected return.

Empirical confirmation:
Four years ago, 
we started this research project with using LSTM as a value function, but we failed. 
This was the starting point for RUDDER. 
In the submission, we used LSTM predictions in artificial task (IV) 
as potential function for reward shaping, look-ahead advice, and look-back advice. 
Furthermore, we investigated LSTM as a value function for artificial task (II) 
but these results have not been included.
At the time where RUDDER already solved the task, 
the LSTM error was too large to allow learning via a value function. 
Problem is the large variance of the returns at the beginning of the sequence 
which hinders LSTM learning (forward view). 
RUDDER LSTM learning was initiated by propagating back prediction errors at the sequence end, 
where the variance of the return is lower (backward view). 
These late predictions initiated the storing of key events at the sequence beginning 
even with high prediction errors. 
The redistributed reward at the key events led RUDDER solve the task. 
Concluding: at the time RUDDER solved the task, 
the early predictions are not learned due to the high variance of the returns. 
Therefore using the predictions as value function does not help (forward view).

Example: The agent has to take a key to open the door. 
Since it is an MDP, 
the agent is always aware to have the key indicated by a key bit to be on. 
The reward can be predicted in the last step. 
Using differences $\Delta$ the key bit is zero, 
except for the step where the agent takes the key. 
Thus, the LSTM has to store this event and will transfer reward to it.

\paragraph{Compensation reward.} 
The compensation corrects for prediction errors of $g$ ($g$ is the sum of $h$). 
The prediction error of $g$ can have two sources: 
(1) the probabilistic nature of the reward, 
(2) an approximation error of g for the expected reward. 
We aim to make (2) small and then the correction is only for the probabilistic nature of the reward.
The compensation error depends on $g$, which, in turn, depends on the whole sequence. 
The dependency on state-action pairs from $t=0$ to $T-1$ is viewed as random effect, 
therefore the compensation reward only depends on the last state-action pair. 

That $h_t$ and $R_{t+1}$ depends only on $(s_t,a_t,s_{t-1},a_{t-1})$ 
is important to prove Theorem 3. 
Then $a_{t-1}$ cancels and the advantage function remains the same.

\paragraph{Connection theory and algorithms.}
Theorem 1 and Theorem 2 
ensure that the algorithms are correct since 
the optimal policies do not change even for non-optimal return decompositions. 
In contrast to TD methods which are biased, 
Theorem 3 
shows that the update rule $Q$-value estimation 
is unbiased when assuming optimal decomposition. 
Theorem 4 
explicitly derives optimality conditions 
for the expected sum of delayed rewards “kappa” 
and measures the distance to optimality. 
This “kappa” is used for learning and 
is explicitly estimated to correct learning 
if an optimal decomposition cannot be assured. 
The theorems are used to justify following learning methods (A) and (B):

(A) Q-value estimation: 
(i) Direct Q-value estimation (not Q-learning) according to Theorem 3 
is given in Eq. (9) 
when an optimal decomposition is assumed.
(ii) Q-value estimation with correction by kappa according to Theorem 4, 
when optimal decomposition is not assumed. 
Here kappa is learned by TD as given in Eq. (10). 
(iii) Q-value estimation using eligibility traces.
(B) Policy gradient: Theorems are used as for Q-value estimation as in (A) 
but now the Q-values serve for policy gradient.
(C) Q-learning: Here the properties in Theorem 3 
and Theorem 4 
are ignored.

We also shows variants (not in the main paper) on page 31 and 32 
of using kappa “Correction of the reward redistribution” 
by reward shaping with kappa and 
“Using kappa as auxiliary task in predicting the return for return decomposition”.

\paragraph{Optimal Return Decomposition, contributions and policy.} 
The Q-value $q^\pi$ depends on a particular policy $\pi$. 
The function h depends on policy $\pi$ since 
h predicts the expected return ($E_{\pi}[\tilde R_{T+1}]$) 
which depends on $\pi$. 
Thus, both return decomposition and optimal return decomposition 
are defined for a particular policy $\pi$. 
A reward redistribution from a return decomposition 
leads to a return equivalent MDP. 
Return equivalent MDPs are defined via all policies 
even if the reward redistribution was derived from a particular policy. 
A reward redistribution depends only on the state-action sequence 
but not on the policy that generated this sequence. 
Also $\Delta$ does not depend on a policy.

\paragraph{Optimal policies are preserve for every state.}
We assume all states are reachable via at least one non-zero transition probability 
to each state and policies that have a non-zero probability 
for each action due to exploration. 
For an MDP being optimal in the initial state 
is the same as being optimal in every reachable state. 
This follows from recursively applying the Bellman optimality equation 
to the initial value function. 
The values of the following states must be optimal 
otherwise the initial value function is smaller. 
Only states to which the transition probability is zero 
the Bellman optimality equation does not determine the optimality.

All RL algorithms are suitable. 
For example we applied TD, Monte Carlo, Policy Gradient, 
which all work faster with the new MDP.

\paragraph{Limitations.}
In all of the experiments reported in this manuscript, we show that RUDDER significantly outperforms other methods for delayed reward problems. However, RUDDER might not be effective when the reward is not delayed since LSTM learning takes extra time and has problems with very long sequences. Furthermore, reward redistribution may introduce disturbing spurious reward signals.

\clearpage
\pagebreak
\section{Additional Related Work}

\paragraph{Delayed Reward.}

To learn delayed rewards there are three phases to consider:
(i) discovering the delayed reward,
(ii) keeping information about the delayed reward,
(iii) learning to receive the delayed reward to secure it for the future.
Recent successful reinforcement learning methods provide solutions to one
or more of these phases.
Most prominent are
Deep $Q$-Networks (DQNs) \cite{Mnih:13,Mnih:15}, which
combine $Q$-learning with convolutional neural networks for 
visual reinforcement learning \cite{Koutnik:13}.
The success of DQNs is attributed to
{\em experience replay} \cite{Lin:93}, which stores
observed state-reward transitions and then samples from them.
Prioritized experience replay \cite{Schaul:15,Horgan:18} advanced the
sampling from the replay memory.
Different policies perform exploration in parallel for the Ape-X DQN 
and share a prioritized experience replay memory \cite{Horgan:18}.
DQN was extended to double DQN (DDQN) \cite{Hasselt:10,Hasselt:16}
which helps exploration as the overestimation bias is reduced.
Noisy DQNs \cite{Fortunato:18} explore by
a stochastic layer in the policy network (see \cite{Hochreiter:90,Schmidhuber:90diff}).
Distributional $Q$-learning \cite{Bellemare:17} profits from noise since means that
have high variance are more likely selected.
The dueling network architecture \cite{Wang:15,Wang:16} separately
estimates state values and action advantages,
which helps exploration in unknown states.
Policy gradient approaches \cite{Williams:92} explore via parallel
policies, too. 
A2C has been improved by IMPALA through parallel actors and
correction for policy-lags between actors and learners \cite{Espeholt:18}.
A3C with asynchronous gradient descent \cite{Mnih:16}
and  Ape-X DPG \cite{Horgan:18} also rely on parallel policies.
Proximal policy optimization (PPO) extends A3C by a surrogate
objective and a trust region optimization that is realized by clipping or
a Kullback-Leibler penalty \cite{Schulman:17}.

Recent approaches aim to solve learning problems caused
by delayed rewards.
Function approximations of value functions or critics \cite{Mnih:15,Mnih:16}
bridge time intervals if states associated
with rewards are similar to states that were encountered many steps earlier.
For example, assume a function that has learned to predict a large reward
at the end of an episode
if a state has a particular feature. 
The function can generalize this
correlation to the beginning of an episode and predict already
high reward for states possessing the same feature.
Multi-step temporal difference (TD) learning \cite{Sutton:88td,Sutton:18book} 
improved both DQNs
and policy gradients \cite{Hessel:17,Mnih:16}.
AlphaGo and AlphaZero learned to play Go and Chess better than
human professionals using Monte Carlo Tree Search (MCTS) \cite{Silver:16,Silver:17}.
MCTS simulates games from a time point until the end of 
the game or an evaluation point and therefore captures long delayed rewards.
Recently, world models using an evolution strategy were successful \cite{Ha:18}.
These forward view approaches are not
feasible in probabilistic environments with a high branching factor of state transition.

\paragraph{Backward View.}

We propose learning from a backward view, which 
either learns a separate model or analyzes a forward model. 
Examples of learning a separate model are 
to trace back from known goal states
\cite{Edwards:18} or from high reward states \cite{Goyal:18}.
However, learning a backward model is very challenging.
When analyzing a forward model that predicts the return then either
sensitivity analysis or contribution analysis may be utilized.
The best known backward view approach
is sensitivity analysis (computing the gradient)
like ''Backpropagation through a Model´´
\cite{Munro:87,Robinson:89,RobinsonFallside:89,Werbos:90,Bakker:07}. 
Sensitivity analysis has several drawbacks:
local minima, instabilities, exploding or vanishing
gradients, and proper exploration
\cite{Hochreiter:90,Schmidhuber:90diff}.
The major drawback is that
the relevance of actions is missed
since sensitivity analysis does not consider their contribution to 
the output but
only their effect on the output when slightly perturbing them.

We use contribution analysis
since sensitivity analysis has serious drawbacks.
Contribution analysis determines how much a state-action pair
contributes to the final prediction.
To focus on state-actions which are most relevant for learning 
is known from prioritized sweeping 
for model-based reinforcement learning \cite{Moore:93}. 
Contribution analysis can be done 
by computing differences of return predictions when adding another input,
by zeroing out an input and then compute the change in the prediction,
by contribution-propagation \cite{Landecker:13},
by a contribution approach \cite{Poulin:06},
by excitation backprop \cite{Zhang:16},
by layer-wise relevance propagation (LRP) \cite{Bach:15},
by Taylor decomposition \cite{Bach:15,Montavon:17taylor}, or 
by integrated gradients (IG) \cite{Sundararajan:17}.

\paragraph{LSTM.}

LSTM was already used in reinforcement learning \cite{Schmidhuber:15}
for advantage learning \cite{Bakker:02}, for constructing a potential 
function for reward shaping by representing the return by a sum of 
LSTM outputs across an episode \cite{Su:15}, and
learning policies \cite{Hausknecht:15,Mnih:16,Heess:16}.

\paragraph{Reward Shaping, Look-Ahead Advice, Look-Back Advice.}

Redistributing the reward is fundamentally different from
reward shaping \cite{Ng:99,Wiewiora:03}, look-ahead advice and
look-back advice \cite{Wiewiora:03icml}. 
However, these methods can be viewed as a special case of 
reward redistribution that result in an MDP that is return-equivalent 
to the original MDP as is shown in Section~\ref{sec:AReturnEquivalent}. 
On the other hand every reward function can be expressed as
look-ahead advice \cite{Harutyunyan:15}.
In contrast to these methods, 
reward redistribution is not limited to potential functions, where
the additional reward is the potential difference, therefore it is a more
general concept than shaping reward or look-ahead/look-back advice.
The major difference of reward redistribution to reward shaping,
look-ahead advice, and look-back advice 
is that the last three keep the original rewards. 
Both look-ahead advice and look-back advice 
have not been designed for replacing for the original rewards.
Since the original reward is kept, the reward redistribution is
not optimal according to Section~\ref{sec:Aopt_rew_red}.
The original rewards 
may have long delays that cause an exponential slow-down of learning. 
The added reward improves sampling but a delayed original reward must
still be transferred to the $Q$-values of early states that caused the reward.
The concept of return-equivalence of SDPs resulting from 
reward redistributions allows to 
eliminate the original reward completely.
Reward shaping can replace the original reward.
However, it only depends on states but not on actions, 
and therefore, it cannot identify
relevant actions without the original reward.


\clearpage
\pagebreak

\section{Markov Decision Processes with Undiscounted Rewards}

We focus on Markov Decision Processes (MDPs)
with undiscounted rewards, since
the relevance but also the problems
of a delayed reward can be considerably decreased by discounting it.
Using discounted rewards both the bias correction in TD as well as the variance 
of MC are greatly reduced. The correction amount decreases exponentially with 
the delay steps, and also the variance contribution to one state 
decreases exponentially 
with the delay of the reward.

MDPs with undiscounted rewards are either finite horizon or process absorbing 
states without reward. The former can always be described by the latter.

\subsection{Properties of the Bellman Operator in MDPs with Undiscounted Rewards}
\label{sec:ApropPoly}

At each time $t$ the environment is in some state $s=s_t \in \cS$. The agent
takes an action $a=a_t \in \cA$ according to policy $\pi$, which causes a transition of
the environment to state $s'=s_{t+1} \in \cS$
and a reward $r=r_{t+1} \in \cR$ for the agent
with probability $p(s',r\mid s,a)$.

The Bellman operator maps a action-value function $q=q(s,a)$ to another
action-value function. We do not require that $q$ are $Q$-values and
that $r$ is the actual reward. 
We define the Bellman operator $\rT^\pi$ for policy $\pi$ as:
\begin{align}
  \rT^\pi \left[ q \right] (s,a)   \ &= \ \sum_{s',r} p(s',r\mid s,a) \ 
  \left[r \ + \ \sum_{a'} \pi(a' \mid s') \ q(s',a')  \right] \ .
\end{align} 

We often rewrite the operator as
\begin{align}
  \rT^\pi \left[ q \right] (s,a)   \ &= \ r(s,a)
\ + \ \EXP_{s',a'} \left[q(s',a')  \right] \ ,
\end{align} 
where
\begin{align}
  r(s,a) \ &=  \
  \sum_{r} r \ p(r\mid s,a) \ , \\ 
  \EXP_{s',a'} \left[q(s',a')  \right] \ &= \
  \sum_{s'} p(s'\mid s,a) \ \sum_{a'} \pi(a' \mid s') \ q(s',a') \ . 
\end{align} 
We did not explicitly express the dependency on the policy $\pi$ and
the state-action pair $(s,a)$ in the
expectation $\EXP_{s',a'}$. A more precise way would be to write
$\EXP^\pi_{s',a'}\left[ . \mid s,a \right]$.

More generally, we have
\begin{align}
\label{eq:bellGen}
  \rT^\pi \left[ q \right] (s,a)   \ &= \ g(s,a)
 \ + \ \EXP_{s',a'} \left[q(s',a')  \right] \ .
\end{align} 
In the following we show properties for this general formulation.

\subsubsection{Monotonically Increasing and Continuous}
\label{sec:ApropPolyFP}

We assume the general formulation Eq.~\eqref{eq:bellGen} of the
Bellman operator.
Proposition 2.1 on pages 22-23 in
Bertsekas and Tsitsiklis, 1996, \cite{Bertsekas:96} shows that
a fixed point $q^\pi$ of the Bellman operator exists and that for every
$q$:
\begin{align}
  q^\pi \ &= \ \rT^\pi \left[ q^\pi\right] \\
  q^\pi \ &= \ \lim_{k\to \infty} \left(\rT^\pi\right)^k q  \ .
\end{align} 
The fixed point equation 
\begin{align}
  q^\pi \ &= \ \rT^\pi \left[ q^\pi\right] 
\end{align}
is called {\em Bellman equation} or {\em Poisson equation}.
For the Poisson equation see
Equation~33 to Equation~37 for the undiscounted case and Equation~34 and Equation~43 for the
discounted case in Alexander Veretennikov, 2016,
\cite{Veretennikov:16}.
This form of the Poisson equation describes the Dirichlet boundary
value problem. The Poisson equation is
\begin{align}
\label{eq:poisson}
 q^\pi(s,a) \ + \ \bar{g} \ &= \  g(s,a)
 \ + \ \EXP_{s',a'} \left[q(s',a') \mid s,a  \right] \ ,
\end{align} 
where $\bar{g}$ is the long term average reward or the expected 
value of the reward for the stationary distribution:
\begin{align}
 \bar{g} \ &= \  \lim_{T\to \infty} \frac{1}{T+1} \sum_{t=0}^T g(s_t,a_t) \ .
\end{align} 
We assume $\bar{g}=0$ since after some time the agent does no longer
receive reward in MDPs with finite time horizon or in MDPs with
absorbing states that have zero reward.

$\rT^\pi$ is {\em monotonically increasing} in its arguments \cite{Bertsekas:96}.
For $q_1$ and $q_2$ with the component-wise condition $q_1 \geq q_2$, we have
\begin{align}
  &\rT^\pi \left[ q_1\right](s,a)  \ - \  \rT^\pi \left[ q_2\right](s,a)  \\ \nonumber
  &= \ \left( g(s,a)
 \ + \ \EXP_{s',a'} \left[q_1(s',a')  \right] \right)  \ - \ \left( g(s,a)
 \ + \ \EXP_{s',a'} \left[q_2(s',a')  \right] \right)\\ \nonumber
 &= \ \EXP_{s',a'} \left[q_1(s',a') \ - \ q_2(s',a') \right]
  \ \geq \ 0 \ ,
\end{align}
where ``$\geq$'' is component-wise. The last inequality follows from
the component-wise condition $q_1 \geq q_2$.

We define the norm $\| . \|_{\infty}$, which gives the maximal difference of
the $Q$-values:
\begin{align}
  \| q_1 \ - \ q_2 \|_{\infty} \ &= \
  \underset{s,a}{\max}
  \left| q_1(s,a) \ - \  q_2(s,a) \right| \ .
\end{align}

$\rT$ is a {\em non-expansion mapping} for $q_1$ and $q_2$:
\begin{align}
  &\| \rT^\pi \left[ q_1\right]  \ - \  \rT^\pi \left[ q_2\right] \|_{\infty} \ = \
  \underset{s,a}{\max}
  \left| \rT[q_1](s,a) \ - \  \rT[q_2](s,a) \right| \\ \nonumber
  &= \ \underset{s,a}{\max}
  \left| \left[ g(s,a) \ + \
  \sum_{s'} p(s'\mid s,a) \ 
   \sum_{a'} \pi(a' \mid s') \ q_1(s',a')  \right]
  \ - \ \right.  \\ \nonumber
  & \ \ \ \left. \left[ g(s,a) \ + \ \sum_{s'} p(s'\mid s,a) \sum_{a'} \pi(a' \mid
    s') \ q_2(s',a')   \right]  
  \right|  \\ \nonumber
  &= \ \underset{s,a}{\max}
  \left|
  \sum_{s'} p(s'\mid s,a) \ 
  \sum_{a'} \pi(a' \mid s') \ \left( q_1(s',a')  \ - \ q_2(s',a') \right)
  \right|  \\ \nonumber
  &\leq \ \underset{s,a}{\max}
  \sum_{s'} p(s'\mid s,a) \ 
  \sum_{a'} \pi(a' \mid s') \ \left|q_1(s',a')  \ - \ q_2(s',a') \right|
  \\ \nonumber 
  &\leq \ \underset{s',a'}{\max}
  \left|q_1(s',a')  \ - \ q_2(s',a') \right| \ = \ \| q_1 \ - \ q_2
  \|_{\infty} \ .
\end{align} 

The first inequality is valid since
the absolute value is moved into the sum.
The second inequality is valid since
the expectation depending on $(s,a)$ is replaced by a maximum that
does not depend on $(s,a)$.
Consequently, the operator $\rT^\pi$ is continuous.

\subsubsection{Contraction for Undiscounted Finite Horizon}
\label{sec:ApropPolyCon}

For time-aware states, we can define another norm with
$0<\eta<1$ which
allows for a contraction mapping:
\begin{align}
  \| q_1 \ - \ q_2 \|_{\infty,t} \ &= \
  \max_{t=0}^{T} \eta^{T-t+1} \
 \underset{s_t,a}{\max}
 \left| q_1(s_t,a) \ - \  q_2(s_t,a) \right| \ .
\end{align}

$\rT^\pi$ is a {\em contraction mapping} for $q_1$ and $q_2$ \cite{Bertsekas:96}:
\begin{align}
  &\| \rT^\pi \left[ q_1\right]  \ - \  \rT^\pi \left[ q_2\right] \|_{\infty,t} \ = \
  \max_{t=0}^{T} \ \eta^{T-t+1} \ \underset{s_t,a}{\max}
  \left| \rT[q_1](s_t,a) \ - \  \rT[q_2](s_t,a) \right| \\ \nonumber
  &= \ \max_{t=0}^{T} \ \eta^{T-t+1} \ \underset{s_t,a}{\max}
  \left|\left[g(s_t,a) \ + \
  \sum_{s_{t+1}} p(s_{t+1}\mid s_t,a) \ 
  \sum_{a'} \pi(a' \mid s') \ q_1(s_{t+1},a')  \right]
  \ - \ \right. \\\nonumber
  & \ \ \  \left. \left[g(s_t,a) \ + \ \sum_{s_{t+1}} p(s_{t+1}\mid s_t,a)
    \sum_{a'} \pi(a' \mid s') \ q_2(s_{t+1},a')  \right]  
  \right|  \\ \nonumber
  &= \ \max_{t=0}^{T} \ \eta^{T-t+1} \ \underset{s_t,a}{\max}
  \left|
  \sum_{s_{t+1}} p(s_{t+1}\mid s_t,a) \ 
  \sum_{a'} \pi(a' \mid s') \ \left[ q_1(s_{t+1},a')  \ - \  q_2(s_{t+1},a') \right]
  \right|  \\ \nonumber
  &\leq \ \max_{t=0}^{T} \ \eta^{T-t+1} \ \underset{s_t,a}{\max}
  \sum_{s_{t+1}} p(s_{t+1}\mid s_t,a) \ 
  \sum_{a'} \pi(a' \mid s') \ \left| q_1(s_{t+1},a')  \ - \  q_2(s_{t+1},a') \right|
   \\ \nonumber
  &\leq \
  \max_{t=0}^{T} \ \eta^{T-t+1} \  \underset{s_{t+1},a'}{\max}
  \ \left| q_1(s_{t+1},a')  \ - \  q_2(s_{t+1},a') \right|
   \\ \nonumber
  &\leq \
  \max_{t=0}^{T} \ \eta \  \eta^{T-(t+1)+1} \ \underset{s_{t+1},a'}{\max}
  \left|q_1(s_{t+1},a')  \ - \ q_2(s_{t+1},a') \right|\\ \nonumber
  &= \  \eta \ \max_{t=1}^{T+1} \  \eta^{T-t+1} \underset{s_{t},a'}{\max}
  \left|q_1(s_{t},a')  \ - \ q_2(s_{t},a') \right|\\ \nonumber
  &= \  \eta \ \max_{t=0}^{T} \  \eta^{T-t+1} \ \underset{s_{t},a'}{\max}
  \left|q_1(s_{t},a')  \ - \ q_2(s_{t},a') \right|\\ \nonumber
  &= \  \eta \  \| q_1 \ - \ q_2 \|_{\infty,t} \ .
\end{align} 
The equality in the last but one line stems from the fact that
all $Q$-values at $t=T+1$ are zero and that all $Q$-values at $t=1$ have
the same constant value.

Furthermore, all $q$ values are equal to
zero for additionally introduced
states at $t=T+1$ since for $t>T+1$ all rewards are zero.
We have
\begin{align}
q^\pi \ &= \ \rT^T\left[ q \right] \ ,
\end{align} 
which is correct for additionally introduced states at time $t=T+1$ since they are zero.
Then, in the next iteration $Q$-values of states at time $t=T$ are correct.
After iteration $i$, $Q$-values of states at time $t=T-i+1$ are correct.
This iteration is called the ``backward induction algorithm'' \cite{Puterman:90,Puterman:05}.
If we perform this iteration for a policy $\pi$ instead of
the optimal policy, then this procedure is called ``policy evaluation
algorithm'' \cite{Puterman:90,Puterman:05}.


\subsubsection{Contraction for Undiscounted Infinite Horizon With Absorbing States}
\label{sec:ApropPolyCon2}

A stationary policy is {\em proper} if there exists an integer $n$
such that from any initial state $x$ the probability of achieving
the terminal state after $n$ steps is strictly positive.

If all terminal states are absorbing and cost/reward free and
if all stationary policies are proper the Bellman operator is a contraction
mapping with respect to a weighted sup-norm.

The fact that the Bellman operator is a contraction
mapping with respect to a weighted sup-norm
has been proved in Tseng, 1990, in Lemma 3 with equation (13) and
text thereafter \cite{Tseng:90Journal}. Also Proposition 1
in Bertsekas and Tsitsiklis, 1991, \cite{Bertsekas:91},
Theorems 3 and 4(b) \& 4(c) in Tsitsiklis, 1994, \cite{Tsitsiklis:94},
and Proposition 2.2 on pages 23-24 in
Bertsekas and Tsitsiklis, 1996, \cite{Bertsekas:96} have proved
the same fact.

\subsubsection{Fixed Point of Contraction is Continuous wrt Parameters}
\label{sec:ApropPolyFPcon}

The mean  $q^\pi$ and variance $V^{\pi}$ are continuous with respect to $\pi$, that is
$\pi(a' \mid s')$, with
respect to the reward distribution $p(r \mid s,a)$ and with respect to
the transition probabilities $p(s' \mid s,a)$. 

A complete metric space or a Cauchy space is a space where every
Cauchy sequence of points has a limit in the space, that is,
every Cauchy sequence converges in the space.
The Euclidean space $\dR^n$  with the usual distance metric is complete.
Lemma 2.5 in Jachymski, 1996, is \cite{Jachymski:96}:
\begin{theoremA}[Jachymski: complete metric space]
Let $(X, d)$ be a complete metric space, and let $(P, d_P)$ be a
metric space. Let $F: P \times X \to X$
be continuous in the first variable and
contractive in the second variable with the same Lipschitz constant
$\alpha <1$. For
$\Bp \in P$, let $\Bx^*(\Bp)$
be the unique fixed point of the map $\Bx \to F(\Bp, \Bx)$. Then the
mapping $\Bx^*$ is continuous.
\end{theoremA}
This theorem is Theorem 2.3 in Frigon, 2007,  \cite{Frigon:07}.
Corollary 4.2 in Feinstein, 2016, generalized the theorem
to set valued operators, that is, these operators may have
more than one fixed point \cite{Feinstein:16}
(see also \cite{Kirr:97}).
All mappings $F(p,.)$ must have the same Lipschitz constant
$\alpha <1$.

A locally compact space is a space where  every point has a compact neighborhood.
$\dR^n$ is locally compact as a consequence of the Heine-Borel theorem.
Proposition 3.2 in Jachymski, 1996, is \cite{Jachymski:96}:
\begin{theoremA}[Jachymski: locally compact complete metric space]
Let $(X, d)$ be a locally compact complete metric space,
and let $(P, d_P)$ be a
metric space. Let $F: P \times X \to X$
be continuous in the first variable and
contractive in the second variable with not necessarily the
same Lipschitz constant.
For $\Bp \in P$, let $\Bx^*(\Bp)$
be the unique fixed point of the map $\Bx \to F(\Bp, \Bx)$. Then the
mapping $\Bx^*$ is continuous.
\end{theoremA}
This theorem is Theorem 2.5 in Frigon, 2007,  \cite{Frigon:07}
and Theorem 2 in Kwiecinski, 1992, \cite{Kwiecinski:92}.
The mappings $F(p,.)$ can have different Lipschitz constants.

\subsubsection{t-fold Composition of the Operator}
\label{sec:At-fold}

We define the Bellman operator as
\begin{align}
  \rT^\pi \left[ q \right] (s,a)   \ &= \ g(s,a)
  \ + \  \sum_{s'} p(s'\mid s,a) \ \sum_{a'}
  \pi(a' \mid s') \ q(s',a')  \\ \nonumber
  &= \  g(s,a) \ + \ \Bq^T  \Bp(s,a) \ ,
\end{align}
where $\Bq$ is the vector with value $q(s',a')$ at position $(s',a')$
and $\Bp(s,a)$ is the vector with value
$p(s'\mid s,a)\pi(a' \mid s')$ at position $(s',a')$.

In vector notation we obtain
the {\em Bellman equation} or {\em Poisson equation}.
For the Poisson equation see
Equation~33 to Equation~37 for the undiscounted case and Equation~34 and Equation~43 for the
discounted case in Alexander Veretennikov, 2016,
\cite{Veretennikov:16}.
This form of the Poisson equation describes the Dirichlet boundary
value problem.
The {\em Bellman equation} or {\em Poisson equation} is
\begin{align}
  \rT^\pi \left[ \Bq \right]    \ &= \ \Bg
  \ + \  \BP \ \Bq  \ ,
\end{align}
where $\BP$ is the row-stochastic matrix with
$p(s'\mid s,a)\pi(a' \mid s')$ at position $((s,a),(s',a'))$.

The Poisson equation is
\begin{align}
\label{eq:poisson1}
 \Bq^\pi \ + \ \bar{g} \BOn \ &= \  \Bg \ + \  \BP \ \Bq  \ ,
\end{align} 
where $\BOn$ is the vector of ones and
$\bar{\Bg}$ is the long term average reward or the expected 
value of the reward for the stationary distribution:
\begin{align}
 \bar{g} \ &= \  \lim_{T\to \infty} \frac{1}{T+1} \sum_{t=0}^T g(s_t,a_t) \ .
\end{align} 
We assume $\bar{g}=0$ since after some time the agent does no longer
receive reward for MDPs with finite time horizon or MDPs with
absorbing states that have zero reward.

Since $\BP$ is a row-stochastic matrix, the
Perron-Frobenius theorem says that (1) $\BP$ has as largest eigenvalue 1 for which
the eigenvector corresponds to the steady state and
(2) the absolute value of each (complex) eigenvalue is smaller or equal 1.
Only the eigenvector to the eigenvalue 1 has purely positive real components.

Equation~7 of Bertsekas and Tsitsiklis, 1991, \cite{Bertsekas:91}
states
\begin{align}
  \label{eq:tfoldA}
  \left(\rT^\pi \right)^t \left[ \Bq \right]    \ &= \
  \sum_{k=0}^{t-1}  \BP^k \ \Bg
  \ + \  \BP^t \ \Bq  \ .
\end{align}

If $\Bp$ is the stationary distribution vector for $\BP$, that is,
\begin{align}
  \lim_{k \to \infty} \BP^k \ &= \ \BOn \ \Bp^T \\
  \lim_{k \to \infty} \Bp_0^T \BP^k \ &= \  \Bp^T 
\end{align}
then
\begin{align}
  \lim_{k \to \infty} \frac{1}{k} \sum_{i=0}^{k-1}  \BP^i \ &= \ \BOn
  \ \Bp^T \\
  \lim_{k \to \infty} \frac{1}{k} \sum_{i=0}^{k-1}  \Bp_0^T \BP^i \ &=
  \ \Bp^T \ .
\end{align}

\subsection{Q-value Transformations: Shaping Reward, Baseline, and Normalization}
\label{sec:qtransform}

The Bellman equation for the action-value function $q^\pi$ is
\begin{align}
  q^\pi(s,a) \ &= \ \sum_{s',r} p(s',r \mid s,a) \ \left[r \ + \ \sum_{a'} \pi(a' \mid s')
    \ q^\pi(s',a') \right] \ .
\end{align}

The expected return at time $t=0$ is:
\begin{align}
  v_0 \ &= \ \sum_{s_0} p(s_0) \ v(s_0)   \ .
\end{align}
As introduced for the REINFORCE algorithm,
we can subtract a baseline $v_0$ from the return.
We subtract the baseline $v_0$ from the last reward.
Therefore, for the new reward $\bar{R}$
we have $\bar{R}_t=R_t$ for $t \leq T$ and
$\bar{R}_{T+1} = R_{T+1} - v_0$.
Consequently, $\bar{q}(s_t,a_t)=q(s_t,a_t)-v_0$ for $t \leq T$.

The TD update rules are:
\begin{align}
  q(s_t,a_t) \ &\longleftarrow \ q(s_t,a_t) \ + \ \alpha \left(r_t \ +
    \ \sum_{a} \pi(a \mid s_{t+1}) \
    q(s_{t+1},a) \ - \ q(s_t,a_t) \right) \ .
\end{align}
The $\delta$-errors are
\begin{align}\nonumber
  &R_{t+1} \ + \ \sum_{a} \pi(a \mid s_{t+1}) \
    q(s_{t+1},a) \ - \ q(s_t,a_t)  \\ \nonumber
   &= \ R_{t+1} \ + \ \sum_{a} \pi(a \mid s_{t+1}) \
    (q(s_{t+1},a)-v_0) \ - \ (q(s_t,a_t) -v_0) \\
  &= \
  \bar{R}_{t+1} \ + \ \sum_{a} \pi(a \mid s_{t+1}) \
    \bar{q}(s_{t+1},a) \ - \ \bar{q}(s_t,a_t)
\end{align}
and for the last step
\begin{align}
  &R_{T+1} \ - \ q(s_T,a_T)  \ = \
  (R_{T+1} \ - \ v_0) \ - \ ( q(s_T,a_T) \ - \ v_0) \\\nonumber
  &= \ \bar{R}_{T+1}   \ - \ \bar{q}(s_T,a_T) \ .
\end{align}

If we set 
\begin{align}
  \bar{q}(s_t,a_t) \ &= \
  \begin{cases}
    q(s_t,a_t) \ - \ v_0 \ , & \text{for } t \leq T  \ .
  \end{cases}  \\
  \bar{R}_t \ &= \
  \begin{cases}
    R_t \ , & \text{for } t \leq T  \\
    R_{T+1} \ - \ v_0 \ , & \text{for } t = T+1  \ ,
  \end{cases}
\end{align}
then the $\delta$-errors and the updates remain the same for $q$ and $\bar{q}$.
We are equally far away from the optimal solution in both cases.

Removing the offset $v_0$ at the end by $\bar{R}_{T+1}=R_{T+1}-v_0$,
can also be derived via reward shaping.
However, the offset has to be added at the beginning: $\bar{R}_1=R_1+v_0$.
Reward shaping requires for the shaping reward $F$ and a potential
function $\Phi$ \cite{Ng:99,Wiewiora:03}:
\begin{align}
F(s_t,a_t,s_{t+1}) \ &= \ \Phi(s_{t+1}) - \Phi(s_t) \ .
\end{align}

For introducing a reward of $c$ at time $t=k$ and removing
it from time $t=m<k$ we set:
\begin{align}
  \Phi(s_t)  \ &= \
  \begin{cases}
    0 \ , & \text{for } t \leq m \ , \\
    -c \ , & \text{for } m+1 \leq t \leq k \ ,\\
    0 \ , & \text{for } t > k \ ,
  \end{cases} 
\end{align}
then the shaping reward is
\begin{align}
  F \big( s_t,a_t,s_{t+1} \big)  \ &= \
  \begin{cases}
    0 \ , & \text{for } t < m \ , \\
    - c \ , & \text{for } t = m  \ , \\
    0 \ , & \text{for } m+1 \leq t < k \ , \\
    c \ , & \text{for } t = k  \ , \\
    0 \ , & \text{for } t > k \ .
  \end{cases} 
\end{align}

For $k=T$, $m=0$, and $c=-v_0$ we obtain above situation but with
$\bar{R}_1=R_1+v_0$ and $\bar{R}_{T+1}=R_{T+1}-v_0$,
that is, $v_0$ is removed at the end and
added at the beginning. All $Q$-values except $q(s_0,a_0)$ are
decreased by $v_0$.
In the general case,
all $Q$-values $q(s_t,a_t)$ with
$m+1 \leq t \leq k$ are increased by $c$.

{\bf $Q$-value normalization}:
We apply reward shaping \cite{Ng:99,Wiewiora:03} for normalization of
the $Q$-values.
The potential $\Phi(s)$ defines the shaping reward
$F(s_t,a_t,s_{t+1})= \Phi(s_{t+1}) - \Phi(s_t)$.
The optimal policies do not change and the $Q$-values become
\begin{align}
   q^{\nn}(s_t,a_t)  \ &= \  q(s_t,a_t) \ - \ \Phi(s_t) \ .
\end{align}
We change the $Q$-values for all $1\leq t \leq T$, but not
for $t=0$ and $t=T+1$. The first and the last $Q$-values are not normalized.
All the shaped reward is added/subtracted to/from the initial and the
last reward.

\begin{itemize}
\item The maximal $Q$-values are zero and the non-optimal $Q$-values are negative
for all $1\leq t \leq T$:
\begin{align}
  \Phi(s_t)   \ &= \ \max_{a} q(s_t,a) \ .
\end{align}
  
\item The minimal $Q$-values are zero and all others $Q$-values are positive
for all $1\leq t \leq T-1$:
\begin{align}
 \Phi(s_t)  \ &= \ \min_{a} q(s_t,a)  \ .
\end{align}
\end{itemize}

\subsection{Alternative Definition of State Enrichment}
\label{sec:AStateEnrich}
Next, we define state-enriched processes $\tilde{\cP}$ compared to $\cP$.
The state $\tilde{s}$ of $\tilde{\cP}$ is enriched with a
deterministic information compared to a state $s$ of $\cP$.
The enriched information
in $\tilde{s}$ can be computed from the state-action pair $(\tilde{s},a)$
and the reward $r$.
Enrichments may be the accumulated reward, count of the time step,
a count how often a certain action
has been performed, a count how often a certain state has been
visited, etc.
Givan et~al.\ have already shown that state-enriched Markov decision
processes (MDPs) preserve the optimal action-value and action sequence properties
as well as the optimal policies of the model \cite{Givan:03}.
Theorem 7 and Corollary 9.1 in Givan et~al.\ proved
these properties \cite{Givan:03} by bisimulations
(stochastically bisimilar MDPs).
A homomorphism between MDPs maps a MDP 
to another one with corresponding reward and transitions probabilities.
Ravindran and Barto have shown that solving the original MDP can be
done by solving a homomorphic image \cite{Ravindran:03}.
Therefore, Ravindran and Barto have also shown that 
state-enriched MDPs preserve the optimal action-value and action sequence properties.
Li et al.\ give an overview over state abstraction or state aggregation for
MDPs, which covers state-enriched MDPs \cite{Li:06}.

\begin{definitionA}
  A decision process $\tilde{\cP}$ is {\em state-enriched} compared to
  a decision process $\cP$ if following conditions hold.
  If $\tilde{s}$ is the state of $\tilde{\cP}$,
  then there exists a function $f: \tilde{s} \to s $ with $f(\tilde{s})=s$,
  where  $s$ is the state of $\cP$.
  There exists a function $g: \tilde{s} \to \cR$, where $g(\tilde{s})$
  gives the additional information of state $\tilde{s}$ compared to $f(\tilde{s})$.
  There exists a function $\nu$ with
  $\nu(f(\tilde{s}),g(\tilde{s}))=\tilde{s}$, that is, the state $\tilde{s}$ can be
  constructed from the original state and the additional information.
  There exists a function $H$ with $h(\tilde{s}')=H(r,\tilde{s},a)$, where $\tilde{s}'$
  is the next state and $r$ the reward.
  $H$ ensures that  $h(\tilde{s}')$ of the next state $\tilde{s}'$
  can be computed from
  reward $r$, actual state $\tilde{s}$, and the actual action
  $a$. Consequently, $\tilde{s}'$ can be computed from $(r,\tilde{s},a)$.
  For all $\tilde{s}$ and $\tilde{s}'$ following holds: 
 \begin{align}
   \tilde{p}(\tilde{s}',r\mid \tilde{s},a) \ &= \ p(f(\tilde{s}') ,r
   \mid f(\tilde{s}),a)  \ , \\ 
   \tilde{p}_0(\tilde{s}_0) \ &= \  p_0(f(\tilde{s}_0)) \ ,
 \end{align}
 where $\tilde{p}_0$ and $p_0$ are the probabilities of the initial states
 of $\tilde{\cP}$ and $\cP$, respectively.
\end{definitionA}
If the reward is deterministic, then
$\tilde{p}(\tilde{s}',r\mid \tilde{s},a)=
\tilde{p}(\tilde{s}'\mid \tilde{s},a)$ and
$\tilde{p}_0(\tilde{s}_0,r)=\tilde{p}_0(\tilde{s}_0)$.

We proof the following theorem, even if it has been proved several
times as mention above.
\begin{theoremA}
  If decision process $\tilde{\cP}$ is state-enriched compared to
  $\cP$, then for each optimal policy $\tilde{\pi}^*$ of $\tilde{\cP}$ there
  exists an equivalent optimal policy $\pi^*$ of $\cP$, and vice
  versa, with $\tilde{\pi}^*(\tilde{s})=\pi^*(f(\tilde{s}))$. The
  optimal return is the same for $\tilde{\cP}$ and $\cP$.
\end{theoremA}

\begin{proof}
  We proof by induction that
  $\tilde{q}^{\tilde{\pi}}(\tilde{s},a)=q^{\pi}(f(\tilde{s}),a)$ if
  $\tilde{\pi}(\tilde{s})=\pi(f(\tilde{s}))$.

  {\bf Basis}: The end of the sequence.
  For $t\geq T$ we have
  $\tilde{q}^{\tilde{\pi}}(\tilde{s},a)=q^{\pi}(f(\tilde{s}),a)=0$, since
  no policy receives reward for $t\geq T$.
  
{\bf Inductive step ($t \to t-1$)}: Assume $\tilde{q}^{\tilde{\pi}}(\tilde{s}',a')=q^{\pi}(f(\tilde{s}'),a')$ for the next state $\tilde{s}'$ and next action $a'$.
\begin{align} \nonumber 
  \tilde{q}^{\tilde{\pi}}(\tilde{s},a) \ &= \ \EXP_{\tilde{\pi}} \left[
    \tilde{G}_t \mid  \tilde{s}_t=\tilde{s}, A_t=a \right]
  \ = \   \sum_{\tilde{s}',r} \tilde{p}(\tilde{s}',r\mid \tilde{s},a) \ 
  \left[r \ + \ \sum_{a'} \tilde{\pi}(a'\mid \tilde{s}')  \
    \tilde{q}^{\tilde{\pi}}(\tilde{s}',a')  \right] \\   
  &= \ \sum_{f(\tilde{s}'),g(\tilde{s}'),r} \tilde{p}(\tilde{s}',r\mid \tilde{s},a) \ 
  \left[r \ + \ \sum_{a'} \tilde{\pi}(a'\mid \tilde{s}')  \
    \tilde{q}^{\tilde{\pi}}(\tilde{s}',a')  \right] \\ \nonumber    
  &= \ \sum_{f(\tilde{s}'),G(r,\tilde{s},a),r} \tilde{p}(\tilde{s}',r\mid \tilde{s},a) \ 
  \left[r \ + \ \sum_{a'} \tilde{\pi}(a'\mid \tilde{s}')  \
    \tilde{q}^{\tilde{\pi}}(\tilde{s}',a')  \right] \\ \nonumber    
  &= \ \sum_{f(\tilde{s}'),r} \tilde{p}(\tilde{s}',r\mid \tilde{s},a) \ 
  \left[r \ + \ \sum_{a'} \tilde{\pi}(a'\mid \tilde{s}')  \
    \tilde{q}^{\tilde{\pi}}(\tilde{s}',a')  \right] \\ \nonumber    
  &= \ \sum_{f(\tilde{s}'),r} p(f(\tilde{s}'),r\mid f(\tilde{s}),a) \ 
  \left[r \ + \ \sum_{a'} \pi(a'\mid f(\tilde{s}'))  \
    \tilde{q}^{\tilde{\pi}}(\tilde{s}',a')  \right] \\ \nonumber    
  &= \ \sum_{f(\tilde{s}'),r} p(f(\tilde{s}'),r\mid f(\tilde{s}),a) \ 
  \left[r \ + \ \sum_{a'} \pi(a'\mid f(\tilde{s}'))  \
    q^{\pi}(f(\tilde{s}'),a')  \right] \\ \nonumber    
  &= \ q^{\pi}(f(\tilde{s}),a) \ .
\end{align}
For the induction step $1 \to 0$ we use
$\tilde{p}_0(\tilde{s}_0,r) =  p_0(f(\tilde{s}_0),r)$
instead of $\tilde{p}(\tilde{s}',r\mid \tilde{s},a) =  p(f(\tilde{s}') ,r \mid f(\tilde{s}),a)$.

It follows that
$\tilde{q}^*(\tilde{s},a)=q^*(f(\tilde{s}),a)$, and therefore
\begin{align}
\tilde{\pi}^{*}(\tilde{s}) \ &= \ \argmax_a \  \tilde{q}^*(\tilde{s},a)
\ = \ \argmax_a \ q^*(f(\tilde{s}),a) \ = \ \pi^{*}(f(\tilde{s})) \ .
\end{align}

Using Bellman's optimality equation would give the same result,
where in above equation
both $\sum_{a'} \pi(a'\mid f(\tilde{s}'))$ and $\sum_{a'} \tilde{\pi}(a'\mid \tilde{s}')$ are
replaced by $\max_{a'}$.
\end{proof}

\begin{theoremA}
  If a Markov decision process $\tilde{\cP}$ is state-enriched compared to
  the MDP $\cP$, then for each optimal policy $\tilde{\pi}^*$ of $\tilde{\cP}$ there
  exists an equivalent optimal policy $\pi^*$ of $\cP$, and vice
  versa, with $\tilde{\pi}^*(f(s))=\pi^*(s)$. The
  optimal return is the same for $\tilde{\cP}$ and $\cP$.
\end{theoremA}
\begin{proof}
  The MDP $\tilde{\cP}$ is a homomorphic image of $\cP$.
  For state-enrichment, the mapping $g$ is bijective, therefore the
  optimal policies in $\tilde{\cP}$ and $\cP$ are equal according to
  Lemma~\ref{th:Arav}. The optimal return is also equal since it does
  not change via state-enrichment.
\end{proof}

\subsection{Variance of the Weighted Sum of a Multinomial Distribution}

State transitions are multinomial distributions and the future
expected reward is a weighted sum of multinomial distributions.
Therefore, we are interested in the variance of the weighted sum
of a multinomial distribution.
Since we have
\begin{align}
  \EXP_{s',a'} \left[ q^\pi(s',a') \mid s,a \right] =
  \sum_{s'} p(s'\mid s,a) \sum_{a'} \pi(a' \mid s')  \ q^\pi(s',a')
  \ ,
\end{align}
the variance of $\EXP_{s',a'} \left[ q^\pi(s', a')\right]$
is determined by the variance of the multinomial distribution
$p(s'\mid s,a)$. In the following we derive
the variance of the estimation of a
linear combination of variables of a multinomial distribution like
$\sum_{s'} p(s'\mid s,a) f(s')$.

A multinomial distribution with parameters $(p_1, \ldots, p_N)$ as
event probabilities satisfying $\sum_{i=1}^N p_i = 1$
and support $x_i \in \{0, \dots, n\}$, $i \in \{1,\dots,N\}$ for $n$
trials, that is $\sum x_i = n$, has
\begin{align}
  \text{pdf} \ &\text{~~~~~~} \frac{n!}{x_1!\cdots x_k!} \ p_1^{x_1}
  \cdots p_k^{x_k} \ , \\
  \text{mean} \ &\text{~~~~~~} \EXP[X_i] \ = \ n \ p_i \ , \\
  \text{variance} \ &\text{~~~~~~} \VAR[X_i] \ = \ n \ p_i \ (1-p_i) \
  , \\
  \text{covariance} \ &\text{~~~~~~} \COV[X_i,X_j] \ =
  \ - \ n \ p_i\ p_j \ , \ \ (i\neq j) \ ,
\end{align}
where $X_i$ is the random variable and $x_i$ the actual count.

A linear combination of random variables has variance
\begin{align}
 \VAR \left[ \sum_{i=1}^N a_i \ X_i\right] \ &= \ \sum_{i,j=1}^{N} a_i
 \ a_j\ \COV \left[ X_i,X_j \right] \\ \nonumber
 &=\sum_{i=1}^N a_i^2 \ \VAR \left[ X_i\right] \ + \ \sum_{i\not=j}a_i
 \ a_j\ \COV\left[ X_i,X_j \right] \ .
\end{align}

The variance of estimating the mean $X$ of independent random variables
$(X_1,\ldots,X_n)$ that all have variance $\sigma^2$ is:
\begin{align}
 \VAR \left[X \right] \ &= \
 \VAR \left[ \frac{1}{n} \sum_{i=1}^n X_i\right] \\ \nonumber
 &= \  \frac{1}{n^2} \sum_{i=1}^n\VAR \left[  X_i\right]
 \ = \  \frac{1}{n^2} \sum_{i=1}^n \sigma^2 \ = \ \frac{\sigma^2}{n} \ .
\end{align}

When estimating the mean $\bar{y}$ over $n$ samples
of a linear combination of variables of a
multinomial distribution $y=\sum_{i=1}^N a_i  X_i$, where each $y$ has
$n_y$ trials, we obtain:
\begin{align}
  \VAR \left[ \bar{y} \right] \ &= \ \frac{\sigma_y^2}{n} \ = \
  \frac{1}{n} \ \left( \sum_{i=1}^N a_i^2 \ n_y \ p_i \ (1-p_i) \ - \
  \sum_{i\not=j}a_i \ a_j \ n_y \ p_i \ p_j \right) \\ \nonumber
  &= \  \frac{n_y}{n} \ \left( \sum_{i=1}^N a_i^2  \ p_i \ (1-p_i) \ - \
  \sum_{i\not=j}a_i \ a_j \ p_i \ p_j \right) \\ \nonumber
  &= \ \frac{n_y}{n}  \ \left( \sum_{i=1}^N a_i^2 \ p_i  \ - \
    \sum_{(i,j)=(1,1)}^{(N,N)} a_i \
  a_j \ p_i \ p_j \right) \\ \nonumber
  &= \ \frac{n_y}{n}  \ \left( \sum_{i=1}^N a_i^2 \ p_i  \ - \
  \left( \sum_{i=1}^{N} a_i \  p_i\right)^2  \right) \ .
\end{align}

\section{Long Short-Term Memory (LSTM)}

\subsection{LSTM Introduction}

Recently, {\em Long Short-Term Memory} (LSTM; \cite{Hochreiter:91,Hochreiter:95,Hochreiter:97})
networks have emerged as the best-performing technique in speech and language processing.
LSTM networks have been overwhelming successful in different speech and language applications,
including handwriting recognition \cite{Graves:09}, generation of writings
\cite{Graves:13}, language modeling and identification \cite{Gonzalez-Dominguez:14,Zaremba:14arxiva},
automatic language translation \cite{Sutskever:14nips}, 
speech recognition \cite{Sak:14,Geiger:14}
analysis of audio data \cite{Marchi:14}, analysis, annotation, and
description of video data \cite{Donahue:14,Venugopalan:14,Srivastava:15}. 
LSTM has facilitated recent benchmark records in TIMIT phoneme recognition (Google),
optical character recognition, text-to-speech synthesis (Microsoft),
language identification (Google), large vocabulary speech recognition (Google),
English-to-French translation (Google), audio onset detection, social signal classification,
image caption generation (Google), video-to-text description, end-to-end speech recognition (Baidu),
and semantic representations. In the proceedings of the flagship conference {\em ICASSP 2015}
(40\textsuperscript{th} IEEE International Conference on Acoustics, Speech and Signal
Processing, Brisbane, Australia, April 19--24, 2015), 13 papers had ``LSTM'' in their
title, yet many more contributions described computational approaches that make use of LSTM.

The key idea of LSTM is the use of memory cells that allow for constant error flow
during training. Thereby, LSTM avoids the {\em vanishing gradient problem}, that is,
the phenomenon that training errors are decaying when they are back-propagated through time
\cite{Hochreiter:91,Hochreiter:00}.
The vanishing gradient problem severely impedes {\em credit assignment} in recurrent neural
networks, i.e.\ the correct identification of relevant events whose effects are not
immediate, but observed with possibly long delays.
LSTM, by its constant error flow, avoids vanishing gradients and, hence, allows for
{\em uniform credit assignment}, i.e.\ all input signals obtain a similar error signal.
Other recurrent neural networks are not able to assign the same credit to all input signals,
therefore they are very limited concerning the solutions they will
find. Uniform credit assignment enabled LSTM networks to excel in speech and
language tasks: if a sentence is analyzed, then the first word can be as important as
the last word. Via uniform credit assignment, LSTM networks regard all words of a sentence equally.
Uniform credit assignment enables to consider all input information
at each phase of learning, no matter where it is located in the input
sequence. Therefore, uniform credit assignment reveals many more
solutions to the learning algorithm which would otherwise remain hidden. 

\begin{figure}[htb]
\centering
\includegraphics[angle=0,width=1.0\textwidth]{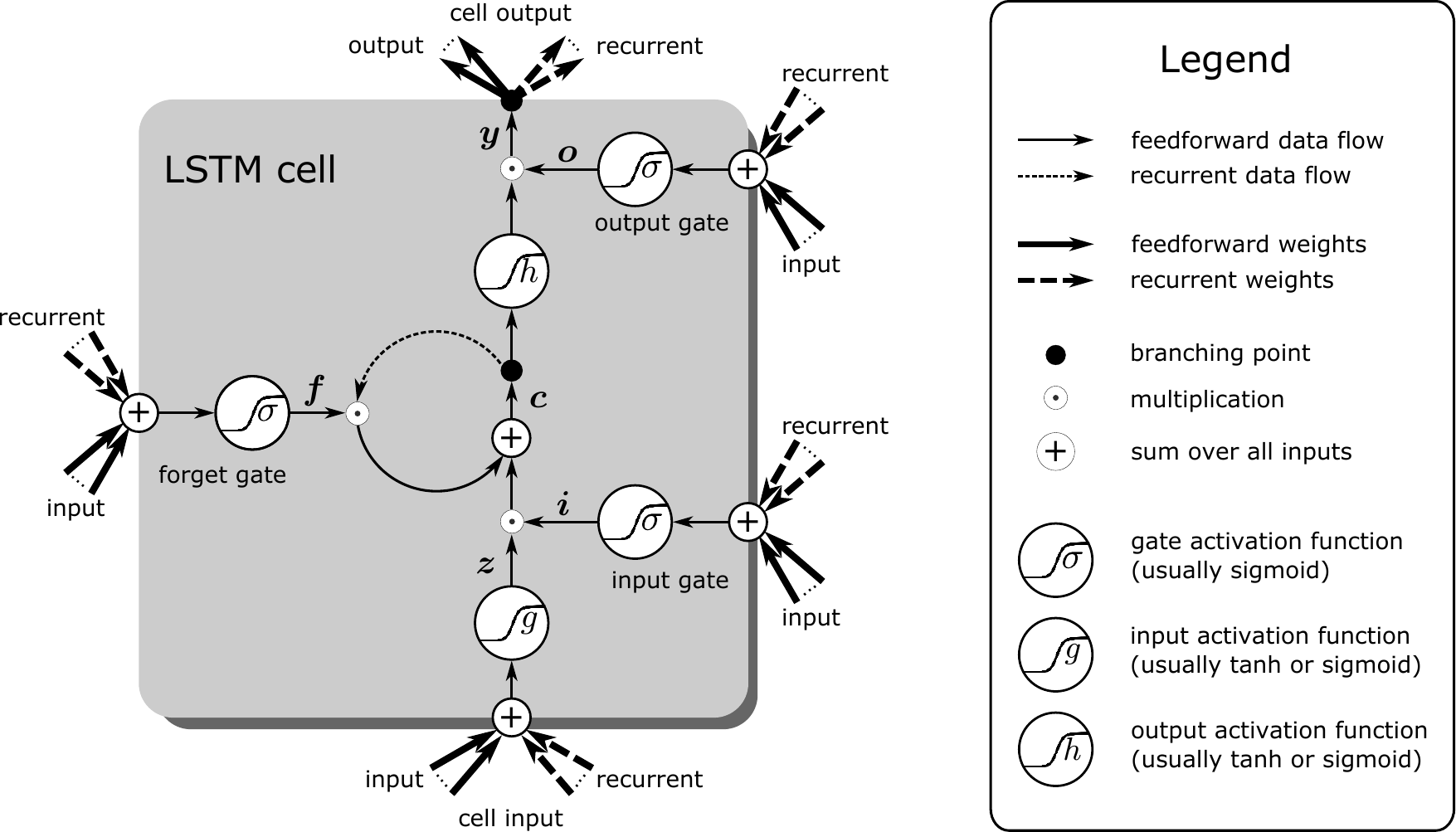}
\caption{LSTM memory cell without peepholes. 
$\Bz$ is the vector of cell input
activations, $\Bi$ is the vector of input gate
activations,  $\Bf$ is the vector of forget gate
activations,  $\Bc$ is the vector of memory cell states,
$\Bo$ is the vector of output gate
activations, and $\By$ is the vector of cell output 
activations. The activation functions are $g$ for the cell input, $h$ for the cell
state, and $\sigma$ for the gates. Data flow is either ``feed-forward''
without delay or ``recurrent'' with an one-step delay.
``Input'' connections are from the
external input to the LSTM network, while ``recurrent'' connections take inputs
from other memory cells and hidden units of the LSTM network with a delay of one time step.  
\label{fig:cellFB}}
\end{figure}

\subsection{LSTM in a Nutshell}

The central processing and storage unit for LSTM recurrent networks is
the {\em memory cell}. As already mentioned, it avoids vanishing gradients and allows for
uniform credit assignment.
The most commonly used LSTM memory cell architecture in the 
literature \cite{Graves:05,Schmidhuber:15} 
contains forget gates \cite{Gers:99a,Gers:00}
and peephole connections \cite{Gers:00a}. 
In our previous work \cite{Hochreiter:01,Hochreiter:07}, 
we found that peephole connections are 
only useful for modeling time series, but not for 
language, meta-learning, or biological sequences. 
That peephole connections can be removed without performance decrease, 
was recently confirmed in a large assessment, where 
different LSTM architectures have been tested \cite{Greff:15}.
While LSTM networks are highly successful in various applications, 
the central memory cell architecture was not modified since 2000 \cite{Schmidhuber:15}.
A memory cell architecture without peepholes is depicted in
Figure~\ref{fig:cellFB}. 

In our definition of a LSTM network, all units of one kind are
pooled to a vector: $\Bz$ is the vector of cell input
activations, $\Bi$ is the vector of input gate
activations,  $\Bf$ is the vector of forget gate
activations,  $\Bc$ is the vector of memory cell states,
$\Bo$ is the vector of output gate
activations, and $\By$ is the vector of cell output 
activations.
We assume to have an input sequence, where the input vector at 
time $t$ is $\Bx^t$. The matrices $\BW_{\Bz}$, $\BW_{\Bi}$,
$\BW_{\Bf}$, and $\BW_{\Bo}$ correspond to the
weights of the connections between inputs and cell input, input gate, forget gate, and
output gate, respectively.
The vectors  $\Bb_{\Bz}$, $\Bb_{\Bi}$,
$\Bb_{\Bf}$, and $\Bb_{\Bo}$ are the bias vectors of cell input, input gate, forget gate, and
output gate, respectively.
The activation functions are $g$ for the cell input, $h$ for the cell
state, and $\sigma$ for the gates, where these functions are evaluated in a
component-wise manner if they are applied to vectors.
Typically, either the sigmoid $\frac{1}{1+\exp(-x)}$ or
$\tanh$ are used as activation functions.
$\odot$ denotes the point-wise multiplication
of two vectors. Without peepholes, the LSTM memory cell forward pass rules
are (see Figure~\ref{fig:cellFB}):
\begin{align}
\Bz^t \ &= \ g \left( \BW_{\Bz} \ \Bx^t \ + \
   \Bb_{\Bz}\right) & \text{cell input} \\
\Bi^t \ &= \ \sigma \left( \BW_{\Bi} \ \Bx^t \ + \
    \Bb_{\Bi} \right) & \text{input gate} \\
\Bf^t \ &= \ \sigma \left( \BW_{\Bf} \ \Bx^t \ + \
   \Bb_{\Bf} \right) & \text{forget gate} \\
\Bc^t \ &= \  \Bi^t \odot \Bz^t \ + \ 
\Bf^t \odot \Bc^{t-1} & \text{cell state} \\
\Bo^t \ &= \ \sigma \left( \BW_{\Bo} \ \Bx^t \ + \
  \Bb_{\Bo} \right) & \text{output gate} \\
\By^t \ &= \ \Bo^t \odot h\left( \Bc^t \right) &
\text{cell output}
\end{align}

\subsection{Long-Term Dependencies vs.\ Uniform Credit Assignment}

The LSTM network has been proposed with the aim
to learn {\em long-term dependencies} in sequences
which span over long intervals
\cite{Hochreiter:97,Hochreiter:97e,Hochreiter:97f,Hochreiter:98}. 
However, besides extracting long-term dependencies, 
LSTM memory cells have another, even
more important, advantage in sequence learning:
as already described in the early 1990s,
LSTM memory cells allow for {\em uniform credit assignment}, that is,
the propagation of errors back to inputs without 
scaling them \cite{Hochreiter:91}. 
For uniform credit assignment of current LSTM architectures,
the forget gate $\Bf$ must be one or close to one.  
A memory cell without an input gate $\Bi$ just sums up all the squashed inputs it
receives during scanning the input sequence.
Thus, such a memory cell is equivalent to a unit that sees all sequence
elements at the same time, as has been shown via 
the ``Ersatzschaltbild'' \cite{Hochreiter:91}.
If an output error occurs only at the end of the sequence,
such a memory cell, via backpropagation, supplies
the same delta error at the cell input unit $\Bz$ at every time
step.
Thus, all inputs obtain the same credit for producing the correct
output and are treated on an equal level and, consequently, the incoming weights to a memory cell 
are adjusted by using the same delta error at the input unit $\Bz$.

In contrast to LSTM memory cells, standard recurrent networks scale
the delta error and assign different credit to different inputs.
The more recent the input, the more credit it obtains.
The first inputs of the sequence are hidden from the final states of
the recurrent network.
In many learning tasks, however, important information is distributed over
the entire length of the sequence and can even occur at the very beginning. For
example, in language- and text-related tasks, 
the first words are often important for the meaning of a sentence. 
If the credit assignment is not uniform along the input sequence, then
learning is very limited. Learning would start by trying to improve
the prediction solely by using the most recent inputs.
Therefore, the solutions that can be found are restricted to those
that can be constructed if the last inputs are considered first.
Thus, only those solutions are found that are accessible by gradient
descent from regions in the parameter space that only use the most recent input information.
In general, these limitations lead to sub-optimal solutions, since 
learning gets trapped in local optima. 
Typically, these local optima correspond to solutions 
which efficiently exploit the most recent information in the input
sequence, while information way back in the past is neglected.

\subsection{Special LSTM Architectures for contribution Analysis}
\label{sec:ALSTMadjust}

\subsubsection{LSTM for Integrated Gradients}

For Integrated Gradients contribution analysis with LSTM, 
we make following assumptions:
\begin{enumerate}[label=\textbf{(A\arabic*)}]
\item $\Bf^t=1$ for all $t$. That is the forget gate is always 1 and
  nothing is forgotten. We assume uniform credit assignment, which
  is ensured by the forget gate set to one.

\item $\Bo^t=1$ for all $t$. That is the output gate is always 1 and
  nothing is forgotten. 

\item We set $h=a_h \tanh$ with $a_h=1,2,4$.

\item We set $g=a_g \tanh$ with $a_g=1,2,4$.

\item The cell input gate $\Bz$ is only connected to the input but not
  to other memory cells.  $\BW_{\Bz}$ has only connections to the
  input. 

\item The input gate $\Bi$ is not connected to the input, that is,
  $\BW_{\Bi}$ has only connections to other memory cells. This ensures
  that LRP assigns relevance only via $\Bz$ to the input.

\item The input gate $\Bi$ has a negative bias, that is,
  $\Bb_{\Bi}<0$. The negative bias
  reduces the drift effect, that is, the memory 
  content $\Bc$ either increases or decreases over time.
  Typical values are $\Bb_{\Bi}=-1,-2,-3,-4,-5$.

\item The memory cell content is initialized with zero at time $t=0$,
  that is, $\Bc^0=0$.
 
\end{enumerate}

The resulting LSTM forward pass rules for Integrated Gradients are:
\begin{align}
\Bz^t \ &= \ a_g \ \sigma \left( \BW_{\Bz} \ \Bx^t \ + \
   \Bb_{\Bz}\right) & \text{cell input} \\
\Bi^t \ &= \ \sigma \left( \BW_{\Bi} \ \Bx^t \ + \
    \Bb_{\Bi} \right) & \text{input gate} \\
\Bc^t \ &= \  \Bi^t \ \odot \ \Bz^t \ + \ \Bc^{t-1} & \text{cell state} \\
\By^t \ &= \ a_h \ \tanh\left( \Bc^t \right) &
\text{cell output}
\end{align}
See Figure~\ref{fig:cellMarkovA} which depicts these  
forward pass rules for Integrated Gradients.

\begin{figure}[htb]
\centering
\includegraphics[angle=0,width=1.0\textwidth]{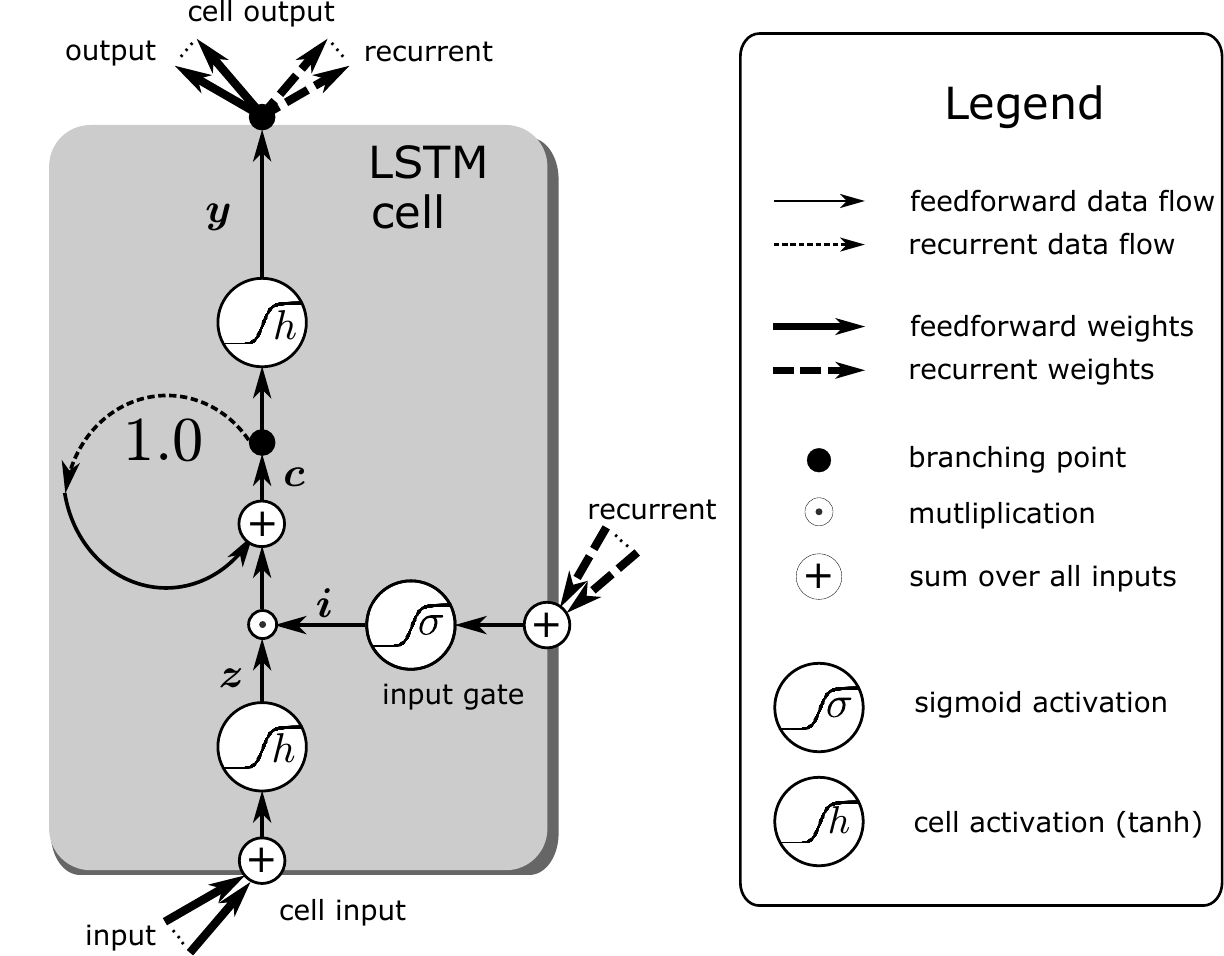}
\caption{LSTM memory cell used for Integrated Gradients (IG). 
  Forget gates and output gates are set to 1 since they
  can modify all cell inputs at times after they have been observed,
  which can make the dynamics highly nonlinear.
\label{fig:cellMarkovA}}
\end{figure}

\subsubsection{LSTM for LRP}

LRP has already been used for LSTM in order to
identify important terms in sentiment analysis \cite{Arras:17}.
In texts, positive and negative terms with respect to the topic
could be identified.

For LRP contribution analysis with LSTM, we make following assumptions:
\begin{enumerate}[label=\textbf{(A\arabic*)}]
\item $\Bf^t=1$ for all $t$. That is the forget gate is always 1 and
  nothing is forgotten. We assume uniform credit assignment, which
  is ensured by the forget gate set to one.

\item $g>0$, that is, $g$ is positive. For example we can use a sigmoid
  $\sigma(x)= a_g \frac{1}{1+\exp(-x)}$: $g(x)=a_g \sigma(x)$, with
  $a_g = 2,3,4$.
  Methods like LRP have problems with negative contributions
  which cancel with positive contributions \cite{Montavon:17}.
  With a positive $g$ all
  contributions are positive.
  The cell input $\Bz$ (the function $g$) has a negative bias, that is,
  $\Bb_{\Bz}<0$. This is important to avoid the drift effect.
  The drift effect is that the memory content only gets positive
  contributions which lead to an increase of $\Bc$ over time.
  Typical values are $\Bb_{\Bz} = -1,-2,-3,-4,-5$.

\item We want to ensure that $h(0)=0$. If the memory content is zero
  then nothing is transferred to the next layer.
  Therefore we set $h=a_h \tanh$ with $a_h=1,2,4$.

\item The cell input gate $\Bz$ is only connected to the input but not
  to other memory cells.  $\BW_{\Bz}$ has only connections to the
  input. This ensures
  that LRP assigns relevance $\Bz$ to the input and $\Bz$ is not
  disturbed by redistributing relevance to the network.

\item The input gate $\Bi$ is not connected to the input, that is,
  $\BW_{\Bi}$ has only connections to other memory cells. This ensures
  that LRP assigns relevance only via $\Bz$ to the input.

\item The output gate $\Bo$ is not connected to the input, that is,
  $\BW_{\Bo}$ has only connections to other memory cells. This ensures
  that LRP assigns relevance only via $\Bz$ to the input.

\item The input gate $\Bi$ has a negative bias, that is,
  $\Bb_{\Bi}<0$. Like with the cell input the negative bias
  avoids the drift effect.
  Typical values are $\Bb_{\Bi}=-1,-2,-3,-4$.

\item The output gate $\Bo$ may also have a negative bias, that is,
  $\Bb_{\Bo}<0$. This allows to bring in different memory cells at
  different time points. It is related to resource allocation.
  
\item The memory cell content is initialized with zero at time $t=0$,
  that is, $\Bc^0=0$. The memory cell content $\Bc^t$ 
  is non-negative $\Bc^t \geq 0$ since 
  $\Bz \geq 0$ and $\Bi\geq 0$.
 
\end{enumerate}

The resulting LSTM forward pass rules for LRP are:
\begin{align}
\Bz^t \ &= \ a_g \ \sigma \left( \BW_{\Bz} \ \Bx^t \ + \
   \Bb_{\Bz}\right) & \text{cell input} \\
\Bi^t \ &= \ \sigma \left( \BW_{\Bi} \ \Bx^t \ + \
    \Bb_{\Bi} \right) & \text{input gate} \\
\Bc^t \ &= \  \Bi^t \ \odot \ \Bz^t \ + \ \Bc^{t-1} & \text{cell state} \\
\Bo^t \ &= \ \sigma \left( \BW_{\Bo} \ \Bx^t \ + \
  \Bb_{\Bo} \right) & \text{output gate} \\
\By^t \ &= \ \Bo^t \ \odot \ a_h \ \tanh\left( \Bc^t \right) &
\text{cell output}
\end{align}
See Figure~\ref{fig:cellLRP} which depicts these  
forward pass rules for LRP. However, gates may be used while no
relevance is given to them which may lead to inconsistencies.

\begin{figure}[htb]
\centering
\includegraphics[angle=0,width=1.0\textwidth]{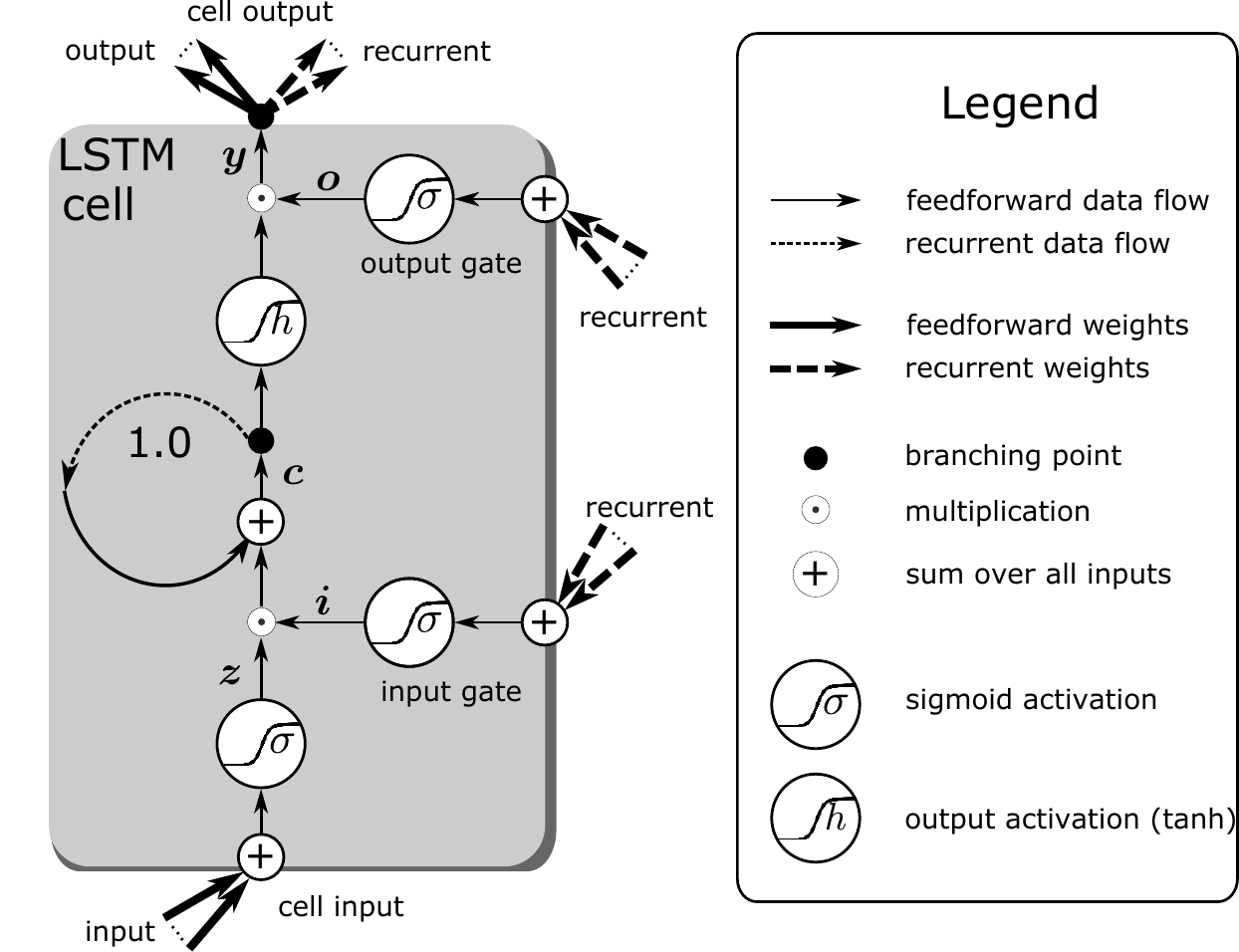}
\caption{LSTM memory cell used for Layer-Wise Relevance Propagation (LRP). 
$\Bz$ is the vector of cell input
activations, $\Bi$ is the vector of input gate
activations,  $\Bc$ is the vector of memory cell states,
$\Bo$ is the vector of output gate
activations, and $\By$ is the vector of cell output 
activations. The activation functions are
the sigmoid $\sigma(x)=a_g \frac{1}{1+\exp(-x)}$
and the cell state activation $h(x)=a_h \tanh(x)$. 
Data flow is either ``feed-forward''
without delay or ``recurrent'' with an one-step delay.
External input reaches the LSTM network 
only via the cell input $\Bz$. All gates only receive
recurrent input, that is, from other memory cells.
\label{fig:cellLRP}}
\end{figure}

\paragraph{LRP and Contribution Propagation for LSTM.}

We analyze Layer-wise Relevance Propagation (LRP) and Contribution Propagation
for LSTM networks.
A single memory cell can be described by:
\begin{align}
c^t \ &= \  i^t \ z^t \ + \ c^{t-1} \ .
\end{align}
Here we treat $i^t$ like a weight for $z^t$ and $c^{t-1}$ has weight 1.

For positive values of $i^t$,  $z^t$, and $c^{t-1}$,
both LRP and contribution propagation leads to
\begin{align}
    R_{c^t \leftarrow y^t}\ &= \  R_{y^t} \\
    R_{c^t}\ &= \   R_{c^t \leftarrow c^{t+1}} \ + \ R_{c^t \leftarrow y^t} \\
    R_{c^{t-1} \leftarrow c^t}\ &= \  \frac{c^{t-1}}{c^t} \ R_{c^t} \\
    R_{z^t \leftarrow c^t}\ &= \  \frac{ i^t \ z^t }{c^t} \ R_{c^t} \ .
\end{align}
Since we predict only at the last step $t=T$, we have
$R_{y^t}=0$ for $t<T$. For $t=T$ we obtain $R_{c^T}=R_{y^T}$, since
$R_{c^T \leftarrow c^{T+1}} =0$.

We obtain for $t=1 \ldots T$:
\begin{align}
    R_{c^T}\ &= \   R_{y^T} \\
    R_{c^{t-1}}\ &= \  \frac{c^{t-1}}{c^t} \ R_{c^t} 
\end{align}
which gives
\begin{align}
    R_{c^t}\ &= \   R_{y^T} \ \prod_{\tau=t+1}^{T}
               \frac{c^{\tau-1}}{c^{\tau}}
  \ = \ \frac{c^{t}}{c^{T}} \ R_{y^T} 
\end{align}
and consequently as $c^0=0$ we obtain
\begin{align}
   R_{c^0}\ &= \ 0 \ , \\ 
    R_{z^t}\ &= \   \frac{i^t \ z^t }{c^{T}} \ R_{y^T} \ . 
\end{align}
Since we assume $c^0=0$, we have
\begin{align}
   c^{T}\ &= \   \sum_{t=1}^{T} i^t \ z^t 
\end{align}
and therefore
\begin{align}
    R_{z^t}\ &= \   \frac{i^t \ z^t }{\sum_{\tau=1}^{T} i^{\tau} \ z^{\tau}} \ R_{y^T} \ . 
\end{align}

Therefore the relevance $R_{y^T}$ is distributed across the inputs
$z^t$ for $t=1 \ldots T-1$, where input $z^t$ obtains relevance $R_{z^t}$.

\subsubsection{LSTM for Nondecreasing Memory Cells}

contribution analysis is made simpler if memory cells are nondecreasing since
the contribution of each input to each memory cells
is well defined. The problem that a 
negative and a positive input cancels each other is avoided. 
For nondecreasing memory cells 
and contribution analysis with LSTM, 
we make following assumptions:
\begin{enumerate}[label=\textbf{(A\arabic*)}]
\item $\Bf^t=1$ for all $t$. That is the forget gate is always 1 and
  nothing is forgotten. We assume uniform credit assignment, which
  is ensured by the forget gate set to one.

\item $g>0$, that is, $g$ is positive. For example we can use a sigmoid
  $\sigma(x)= a_g \frac{1}{1+\exp(-x)}$: $g(x)=a_g \sigma(x)$, with
  $a_g = 2,3,4$.
  With a positive $g$ all
  contributions are positive.
  The cell input $\Bz$ (the function $g$) has a negative bias, that is,
  $\Bb_{\Bz}<0$. This is important to avoid the drift effect.
  The drift effect is that the memory content only gets positive
  contributions which lead to an increase of $\Bc$ over time.
  Typical values are $\Bb_{\Bz} = -1,-2,-3,-4,-5$.

\item We want to ensure that $h(0)=0$. If the memory content is zero
  then nothing is transferred to the next layer.
  Therefore we set $h=a_h \tanh$ with $a_h=1,2,4$.

\item The cell input gate $\Bz$ is only connected to the input but not
  to other memory cells.  $\BW_{\Bz}$ has only connections to the
  input. 

\item The input gate $\Bi$ is not connected to the input, that is,
  $\BW_{\Bi}$ has only connections to other memory cells. 

\item The output gate $\Bo$ is not connected to the input, that is,
  $\BW_{\Bo}$ has only connections to other memory cells. 

\item The input gate $\Bi$ has a negative bias, that is,
  $\Bb_{\Bi}<0$. Like with the cell input the negative bias
  avoids the drift effect.
  Typical values are $\Bb_{\Bi}=-1,-2,-3,-4$.

\item The output gate $\Bo$ may also have a negative bias, that is,
  $\Bb_{\Bo}<0$. This allows to bring in different memory cells at
  different time points. It is related to resource allocation.
  
\item The memory cell content is initialized with zero at time $t=0$,
  that is, $\Bc^0=0$. We ensured via the architecture that $\Bc^t \geq 0$
  and $\Bc^{t+1} \geq \Bc^{t}$, that is, 
  the memory cells are positive and nondecreasing.
 
\end{enumerate}

The resulting LSTM forward pass rules for nondecreasing memory cells are:
\begin{align}
\Bz^t \ &= \ a_g \ \sigma \left( \BW_{\Bz} \ \Bx^t \ + \
   \Bb_{\Bz}\right) & \text{cell input} \\
\Bi^t \ &= \ \sigma \left( \BW_{\Bi} \ \Bx^t \ + \
    \Bb_{\Bi} \right) & \text{input gate} \\
\Bc^t \ &= \  \Bi^t \ \odot \ \Bz^t \ + \ \Bc^{t-1} & \text{cell state} \\
\Bo^t \ &= \ \sigma \left( \BW_{\Bo} \ \Bx^t \ + \
  \Bb_{\Bo} \right) & \text{output gate} \\
\By^t \ &= \ \Bo^t \ \odot \ a_h \ \tanh\left( \Bc^t \right) &
\text{cell output}
\end{align}
See Figure~\ref{fig:cellLRPMarkov} for a LSTM memory cell that is
nondecreasing.

\begin{figure}[htb]
\centering
\includegraphics[angle=0,width=1.0\textwidth]{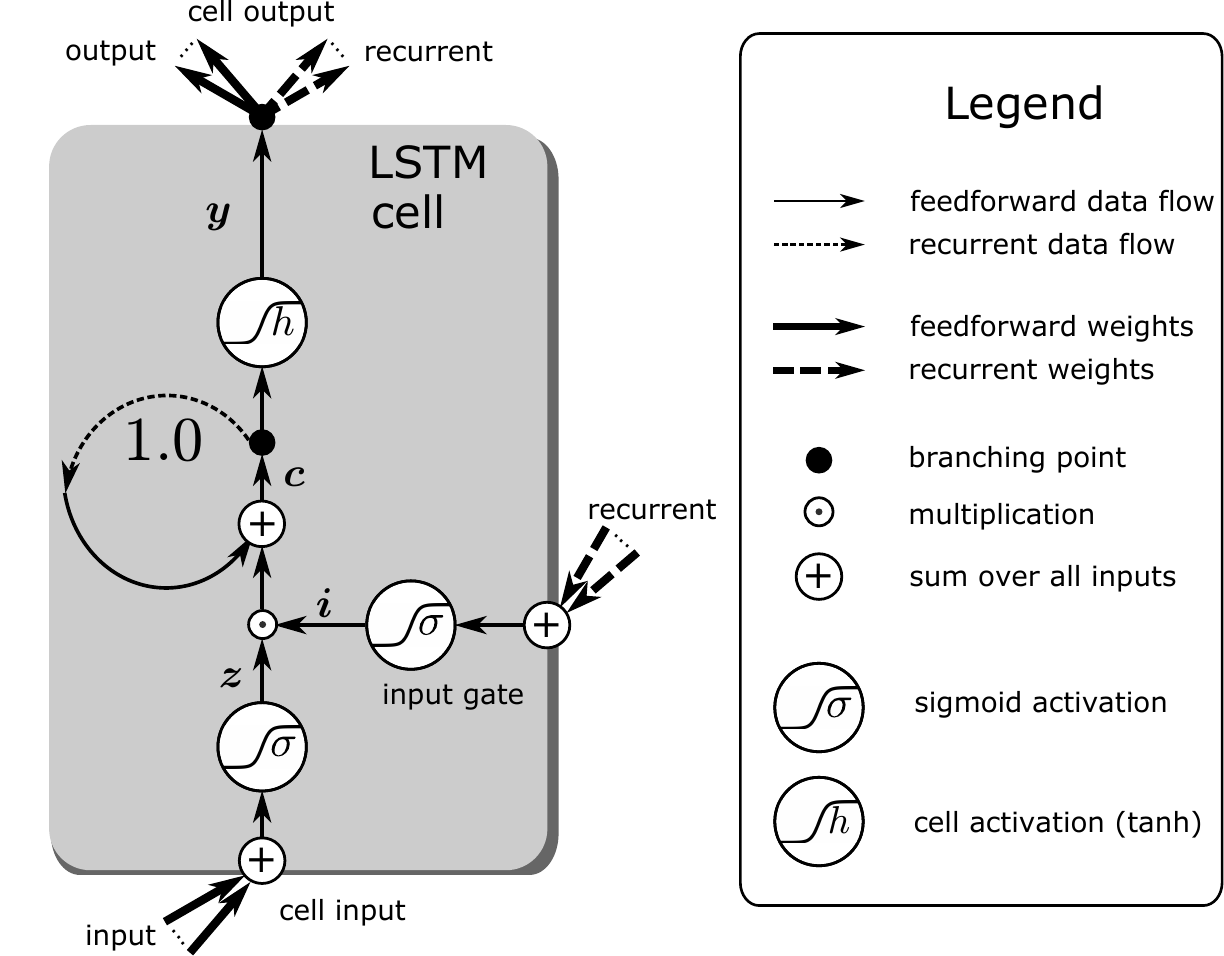}
\caption{A nondecreasing LSTM memory cell. 
\label{fig:cellLRPMarkov}}
\end{figure}

\subsubsection{LSTM without Gates}

The most simple LSTM architecture for contribution analysis does not
use any gates. Therefore complex dynamics that have to be treated 
in the contribution analysis are avoided.
For LSTM without gates, 
we make following assumptions:
\begin{enumerate}[label=\textbf{(A\arabic*)}]
\item $\Bf^t=1$ for all $t$. That is the forget gate is always 1 and
  nothing is forgotten. 

\item $\Bo^t=1$ for all $t$. That is the output gate is always 1.

\item $\Bi^t=1$ for all $t$. That is the input gate is always 1.

\item $g>0$, that is, $g$ is positive. For example we can use a sigmoid
  $\sigma(x)= a_g \frac{1}{1+\exp(-x)}$: $g(x)=a_g \sigma(x)$, with
  $a_g = 2,3,4$.
  With a positive $g$ all
  contributions are positive.
  The cell input $\Bz$ (the function $g$) has a negative bias, that is,
  $\Bb_{\Bz}<0$. This is important to avoid the drift effect.
  The drift effect is that the memory content only gets positive
  contributions which lead to an increase of $\Bc$ over time.
  Typical values are $\Bb_{\Bz} = -1,-2,-3,-4,-5$.

\item We want to ensure that $h(0)=0$. If the memory content is zero
  then nothing is transferred to the next layer.
  Therefore we set $h=a_h \tanh$ with $a_h=1,2,4$.

\item The memory cell content is initialized with zero at time $t=0$,
  that is, $\Bc^0=0$. 
 
\end{enumerate}

The resulting LSTM forward pass rules are:
\begin{align}
\Bz^t \ &= \ a_g \ \sigma \left( \BW_{\Bz} \ \Bx^t \ + \
   \Bb_{\Bz}\right) & \text{cell input} \\
\Bc^t \ &= \  \Bz^t \ + \ \Bc^{t-1} & \text{cell state} \\
\By^t \ &= \ a_h \ \tanh\left( \Bc^t \right) &
\text{cell output}
\end{align}
See Figure~\ref{fig:cellLRPNoGate} for a LSTM memory cell without
gates which perfectly distributes the relevance across the input.

\begin{figure}[htb]
\centering
\includegraphics[angle=0,width=1.0\textwidth]{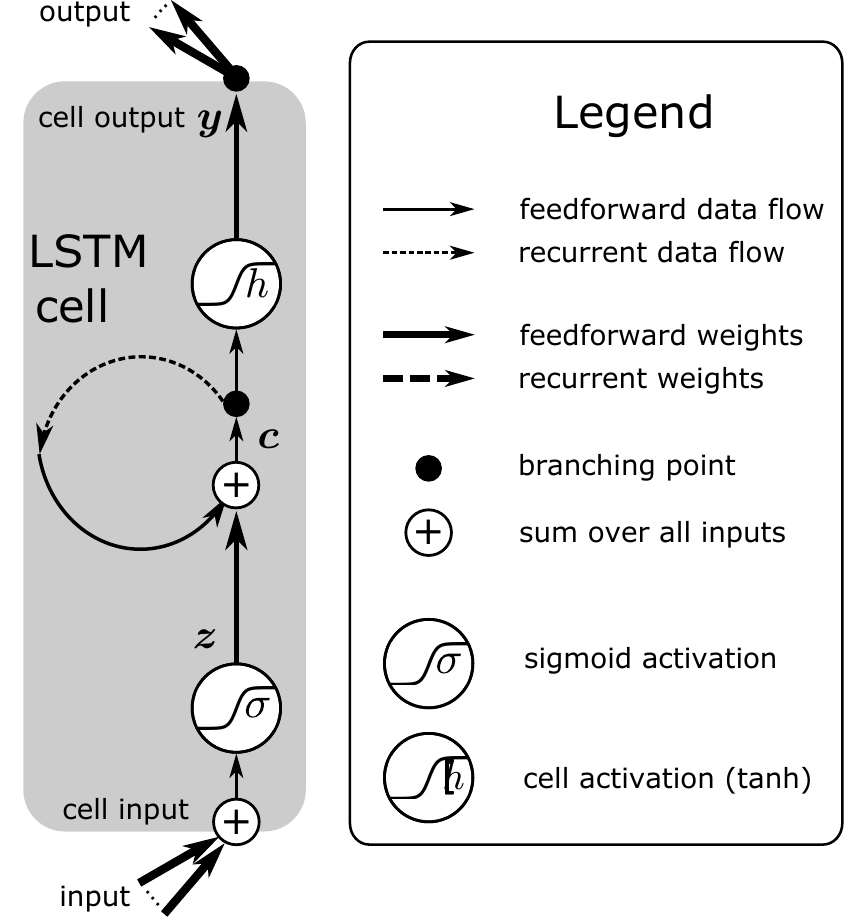}
\caption{LSTM memory cell without gates. 
\label{fig:cellLRPNoGate}}
\end{figure}

\clearpage

\pagebreak

\section{Contribution Analysis}
\label{sec:Aback}

\subsection{Difference of Consecutive Predictions for Sequences}

\paragraph{General Approach.}

The idea is to assess the information gain that is induced by an input at a particular time step. This information gain is used for predicting the target at sequence end
by determining the change in prediction.
The input to a recurrent neural network is the sequence
$\Bx=(x_1,\ldots,x_d)$ with target $y_d$, which is only given at
sequence end.
The prefix sequence $\Bx_t$ of length $t\leq d$ is $\Bx_t=(x_1,\ldots,x_t)$.
$F$ predicts the target $y_d$ at every time step $t$:
\begin{align}
 F(\Bx_t) \ &= \ y_d \ . 
\end{align} 
We can define the decomposition of $F$ through contributions at different time steps
\begin{align}
 h_0 \ &= \ F(\Bx_0)  \ , \\
 h_t \ &= \ F(\Bx_t) \ - \ F(\Bx_{t-1}) \ \text{ for } t > 0 \ ,
\end{align} 
where $F(\Bx_0)$ is a predefined constant.
We have
\begin{align}
 F(\Bx_t) \ &= \ \sum_{\tau=0}^t h_{\tau} \ . 
\end{align} 

We assume a loss function for $F$ that is minimal if $F \equiv F_{\min}$ predicts the expected
$y_d$
\begin{align}
 F_{\min}(\Bx_t) \ &= \ \EXP \left[ y_d \mid \Bx_t \right] \ . 
\end{align} 
Then 
\begin{align}
 h_0 \ &= \   \EXP \left[ y_d  \right] \ , \\ 
 h_t \ &= \ \EXP \left[ y_d \mid \Bx_t \right] \ - \ 
 \EXP \left[ y_d \mid \Bx_{t-1} \right] \ \text{ for } t > 0 \ .
\end{align} 
In this case, the contributions are the change in the expectation of 
the target that will be observed at sequence end.
The contribution can be viewed as the information gain in time step $t$
for predicting the target.
If we cannot ensure that $F$ predicts the target at every time step,
then other contribution analysis methods must be employed.
For attributing the prediction of a deep
network to its input features several contribution analysis
methods have been proposed.
We consider Input Zeroing, 
Integrated Gradients (IG), and Layer-Wise Relevance Propagation (LRP).

\paragraph{Linear Models and Coefficient of Determination.}

We consider linear models and the average gain of information about
the reward at sequence end if we go one time step further in the 
input sequence.
By adding a variable, that is, another sequence element, the mean squared error (MSE)
decreases, which is the amount by which the expectation improves due to new information.
But by what amount does the MSE decrease in average? 
Here, we consider linear models.
For linear models we are interested in how much the coefficient of 
determination increases if we add another variable, that is, if we
see another input.

We consider the feature vector $\Bx=(x_1,x_2,\ldots,x_k)^T$ from which
the target $y$ (the reward at sequence end) has to be predicted. 
We assume to have $n$ pairs $(\Bx_i,y_i), 1\leq i \leq n$,
as training set. 
The prediction or estimation of $y_i$ from $\Bx_i$ is $\hat{y}_i$ with 
$\hat{y}_i=F(\Bx_i)$. The vector of all training labels is $\By=(y_1,\ldots,y_n)$ and
the training feature matrix is $\BX=(\Bx_1,\ldots,\Bx_n)$.
We define the mean squared error (MSE) as
\begin{align}
 \mathrm{mse}(\By,\BX) \ &= \ \frac{1}{n-1}\ 
 \sum_{i=1}^n \left( \hat{y}_i \ - \ y_i  \right)^2 \ . 
\end{align}

The {\em coefficient of determination} $R^2$
is equal to the correlation between the target $y$ and
its prediction $\hat{y}$.
$R^2$ is given by:
\begin{align}
 R^2 \ &= \ 1 \ - \ \frac{\frac{1}{n-1}\ \sum_{i=1}^n \left( \hat{y}_i \ - \ y_i
  \right)^2}{\frac{1}{n-1}\ \sum_{i=1}^n \left( y_i \ - \ \bar{y}
  \right)^2} \ = \ 1 \ - \ \frac{\mathrm{mse}(\By,\BX)}{s_y^2}\ . 
\end{align}
Therefore, $R^2$ is one minus the ratio of the mean squared
error divided by the mean total sum of squares.
$R^2$ is a strict monotonically decreasing function of the
mean squared error.

We will give a breakdown of the factors that determine how much
each variable adds to $R^2$ \cite[chapter 10.6, p.~263]{Rencher:08}.
The feature vector $\Bx$ is expanded by one additional feature $z$:
$\Bw=(x_1,x_2,\ldots,x_k,z)^T=(\Bx^T,z)^T$.
We want to know the increase in $R^2$ due to adding $z$.
Therefore, we decompose $\Bw$ into $\Bx$ and $z$.  
The difference in coefficients of determination is the difference 
of the according MSEs divided by the empirical variance of $y$:
\begin{align}
 R^2_{y\Bw} \ - \ R^2_{y\Bx} \ &= \ 
 \frac{\mathrm{mse}(\By,\BW) \ - \ \mathrm{mse}(\By,\BX)}{s_y^2}  \ . 
\end{align}

We further need definitions: 
\begin{itemize}
\item $\Bx=(x_1,x_2,\ldots,x_k)^T$.
\item $\Bw=(x_1,x_2,\ldots,x_k,z)^T=(\Bx^T,z)^T$.
\item The sample covariance between $y$ and $x$
  is $s_{yx}=\sum_{i=1}^{n} (x_i-\bar{x})(y_i-\bar{y})/(n-1)$, where
  $\bar{x}=\sum_{i=1}^{n} x_i / n$ and $\bar{y}=\sum_{i=1}^{n} y_i /
  n$ are the sample means. The variance of $x$ is $s_{xx}$ often
  written as $s_x^2$, the standard deviation squared: $s_x:=\sqrt{s_{xx}}$.
\item The correlation between $y$ and $x$ is $r_{yx}=
  s_{yx}/(s_{x} s_{y})$.
\item The covariance matrix $\BS_{xx}$ of a vector $\Bx$ is the matrix
  with entries $[\BS_{xx}]_{ij}=s_{x_i x_j}$.
\item The covariance matrix $\BR_{xx}$ of a vector $\Bx$ is the matrix
  with entries $[\BR_{xx}]_{ij}=r_{x_i x_j}$.
\item The diagonal matrix $\BD_x=[\mathrm{diag}(\BS_{xx})]^{1/2}$
  has a $i$th diagonal entry
  $\sqrt{s_{x_i}}$ and is the diagonal matrix of standard deviations of the components
  of $\Bx$.
\item $R^2_{yw}$ is the squared multiple correlation between $y$ and $\Bw$.
\item $R^2_{yx}$ is the squared multiple correlation between $y$ and $\Bx$.
\item $R^2_{zx}=\Bs_{zx}^T \BS_{xx}^{-1}\Bs_{zx} / s_z^2=\Br_{zx}^T \BR_{xx}^{-1}\Br_{zx}$  is
  the squared multiple correlation between $z$ and $\Bx$.
\item $r_{yz}$ is the simple correlation between $y$ and $z$: $r_{yz}=
  s_{yz}/(s_{y} s_{z})$.
\item $\Br_{yx}=(r_{yx_1},r_{yx_2},\ldots,r_{yx_k})^T=s_y^{-1}
  \BD_x^{-1} \BS_{yx}$ is
  the vector of correlations between $y$ and $\Bx$.
\item $\Br_{zx}=(r_{zx_1},r_{zx_2},\ldots,r_{zx_k})^T=s_z^{-1}
  \BD_x^{-1} \BS_{zx}$ is
  the vector of correlations between $z$ and $\Bx$.
\item $\hat{\Bbe}^*_{zx}=\BR_{xx}^{-1} \Br_{zx}$ is
  the vector of standardized regression coefficients (beta weights)
of $z$ regressed on $\Bx$.
\item The parameter vector is partitioned into the constant $\beta_0$
  and $\Bbe_1$ via
  $\Bbe=(\beta_0,\beta_1,\ldots,\beta_m)^T=(\beta_0,\Bbe_1^T)^T$.
  We have for the maximum likelihood estimate
 \begin{align}
   \hat{\beta}_0 \ &= \ \bar{y} \ - \ \Bs_{yx}^T \BS_{xx}^{-1}
                     \bar{\Bx} \ , \\
   \hat{\Bbe}_1 \ &= \ \BS_{xx}^{-1} \Bs_{yx} \ .
 \end{align}
 The offset $\hat{\beta}_0$ guarantees $\bar{\hat{y}}=\bar{y}$,
 therefore, $\By^{T}\bar{\By} =  \hat{\By}^{T} \bar{\By}$,
 since $\bar{\By}=\bar{y} \BOn$:
 \begin{align}
  \bar{\hat{y}} \ &= \    \frac{1}{n} \
  \sum_{i=1}^{n} \hat{y}_i \ = \ \frac{1}{n} \
  \sum_{i=1}^{n} \big( \hat{\beta}_0 \ + \ \hat{\Bbe}_1^T \Bx_i \big) \\ \nonumber
   &= \ \bar{y} \ - \ \Bs_{yx}^T \BS_{xx}^{-1} \bar{\Bx} \ + \ \frac{1}{n} \ \sum_{i=1}^{n} 
                              \hat{\Bbe}_1^T \Bx_i\\ \nonumber
   &= \ \bar{y} \ - \ \Bs_{yx}^T \BS_{xx}^{-1} \bar{\Bx} \ + \
     \Bs_{yx}^T \BS_{xx}^{-1} \bar{\Bx}\\ \nonumber
   &= \ \bar{y} \ .
 \end{align}

\item The vector of {\em standardized coefficients} $\hat{\Bbe}^{*}_1$ are
 \begin{align}
   \hat{\Bbe}^{*}_1   \ &= \ \frac{1}{s_y} \ \BD_x \ \hat{\Bbe}_1 \ =
  \ \BR_{xx}^{-1} \ \Br_{yx} \ .
\end{align}
\end{itemize}

The next theorem is Theorem~10.6 in Rencher and Schaalje \cite{Rencher:08}
and gives a breakdown of the factors that determine how much
each variable adds to $R^2$ \cite[Chapter 10.6, p.~263]{Rencher:08}.
\begin{theorem}[Rencher Theorem 10.6]
The increase in $R^2$ due to $z$ can be expressed as
\begin{align}
\label{eq:increase}
  R^2_{y\Bw} \ - \ R^2_{y\Bx} \ &= \ \frac{(\hat{r}_{yz} \ -  \ 
  r_{yz})^2}{1 \ - \ R^2_{zx}} \ ,
\end{align}
where $\hat{r}_{yz}=(\hat{\Bbe}^*_{zx})^T \Br_{yx}$
is a ``predicted'' value of $r_{yz}$ based on the relationship of $z$ to
the $\Bx$'s.
\end{theorem}

The following equality shows
that $\hat{r}_{yz}=(\hat{\Bbe}^*_{zx})^T \Br_{yx}$ is indeed a
prediction of $r_{yz}$:
\begin{align}
\label{eq:empCorr}
 \left(\hat{\Bbe}^*_{zx} \right)^T \Br_{yx} \ &= \
 \frac{1}{s_z} \ \BD_x \  \hat{\Bbe}_{zx}^T \frac{1}{s_y} \ \BD_x^{-1} \Bs_{yx} \\ \nonumber
  &= \ \frac{1}{s_z \ s_y} \ \hat{\Bbe}_{zx}^T \frac{1}{n-1} \
    \sum_{i=1}^{n} (\Bx_i-\bar{\Bx})(y_i-\bar{y}) \\ \nonumber
  &= \ \frac{1}{s_z \ s_y} \ \frac{1}{n-1} \ \sum_{i=1}^{n}
    (\hat{\Bbe}_{zx}^T \Bx_i-\hat{\Bbe}_{zx}^T \bar{\Bx})(y_i-\bar{y}) \\ \nonumber
  &= \ \frac{1}{s_z \ s_y} \ \frac{1}{n-1} \ \sum_{i=1}^{n}
    (\hat{z_i}-\bar{\hat{z}})(y_i-\bar{y}) \\ \nonumber
  &= \ \frac{1}{s_z \ s_y} \ \hat{s}_{yz} \ = \ \hat{r}_{yz} \ .
\end{align}

If $z$ is orthogonal to $\Bx$
(i.e., if $\Br_{zx}=\BZe$), then $\hat{\Bbe}^*_{zx}=\BZe$, which implies
that $\hat{r}_{yz}=0$ and $R^2_{zx}=0$. In this case,
Eq.~\eqref{eq:increase} can be written as
\begin{align}
 R^2_{y\Bw} \ - \ R^2_{y\Bx} \ &= \  r_{yz}^2 \ .
\end{align}
Consequently, if all $x_i$ are independent from each other, then
\begin{align}
R^2_{y\Bx} \ &= \  \sum_{j=1}^{k} r_{y x_j}^2 \ .
\end{align}
The contribution of $z$ to $R^2$ can either be less than or
greater than $r_{yz}$.
If the correlation $r_{yz}$ can be predicted from $\Bx$, then
$\hat{r}_{yz}$ is close to $r_{yz}$ and, therefore, $z$ has
contributes less to $R^2$ than $r_{yz}^2$.

Next, we compute the contribution of $z$ to $R^2$ explicitly. 
The correlation between $y$ and $z$ is 
\begin{align}
\label{eq:corrYZ}
  r_{yz}  \ &= \ \frac{1}{s_z \ s_y} \ \frac{1}{n-1} \ \sum_{i=1}^{n}
    (z_i-\bar{z})(y_i-\bar{y}) \ = \ \frac{1}{s_z \ s_y} \ s_{yz} \ . 
\end{align}
We assume that $\bar{z}=\bar{\hat{z}}$.
We want to express the information gain using the mean squared error (MSE)
$1/(n-1) \sum_{i=1}^n (\hat{z_i}- z_i)^2$.
We define the error $e_i:= \hat{z_i}- z_i$ at sample $i$ 
with $\bar{e} =  \bar{\hat{z}} - \bar{z} = 0$. Therefore,
the MSE is equal to the empirical variance
$s_e^2 = 1/(n-1) \sum_{i=1}^n e_i^2$.
The correlation $r_{ey}$ between the target $y$ and the
error $e$ is
\begin{align}
  r_{ey} \ &= \ \frac{1}{s_y \ s_e}\ \frac{1}{n-1}\ \sum_{i=1}^{n}
    (e_i \ - \ \bar{e}) \ (y_i-\bar{y}) \ .
\end{align}
Using Eq.~\eqref{eq:empCorr} and Eq.~\eqref{eq:corrYZ},
we can express the difference between the estimate $\hat{r}_{yz}$ and
the true correlation $r_{yz}$ by:
\begin{align}
  \hat{r}_{yz} \ - \ r_{yz}  \ &= \ \frac{1}{s_z \ s_y} \ \frac{1}{n-1} \ \sum_{i=1}^{n}
    (\hat{z_i}-\bar{\hat{z}})(y_i-\bar{y})  \ - \ \frac{1}{s_z \ s_y} \ \frac{1}{n-1} \ \sum_{i=1}^{n}
  (z_i-\bar{z})(y_i-\bar{y}) \\ \nonumber
  &= \ \frac{1}{s_z \ s_y} \ \frac{1}{n-1} \ \sum_{i=1}^{n}
    (\hat{z_i}- z_i)(y_i-\bar{y}) \ .
\end{align}

The information gain can now be expressed by
the correlation $r_{ey}$ between the target $y$ and the
error $e$:
\begin{align}
  R^2_{y\Bw} \ - \ R^2_{y\Bx} \ &= \
  \frac{(\hat{r}_{yz} \ - \ r_{yz})^2}{1 \ - \ R^2_{zx}} \ = \
  \frac{\frac{1}{s_z^2 \ s_y^2} \ \frac{1}{(n-1)^2} \ \left( \sum_{i=1}^{n}
    (\hat{z_i}- z_i)(y_i-\bar{y}) \right)^2}
  {\frac{\frac{1}{n-1}\ \sum_{i=1}^n \left( \hat{z}_i \ - \ z_i
  \right)^2}{\frac{1}{n-1}\ \sum_{i=1}^n \left( z_i \ - \ \bar{z}
  \right)^2}} \\ \nonumber
  &= \ \frac{\frac{1}{s_y^2} \ \frac{1}{(n-1)^2} \ \left( \sum_{i=1}^{n}
    (\hat{z_i}- z_i)(y_i-\bar{y}) \right)^2}
  {\frac{1}{n-1}\ \sum_{i=1}^n \left( \hat{z}_i \ - \ z_i \right)^2} \\ \nonumber
  &= \ r_{ey}^2 \ .
\end{align}
The information gain is the squared
correlation $r_{ey}^2$ between the target $y$ and the
error $e$.
{\bf The information gain is the information in $z$ about $y$, which is
not contained in $\Bx$.}

\subsection{Input Zeroing}

The simplest contribution analysis method is Input Zeroing, where just
an input is set to zero to determine its contribution to the output.
Input Zeroing sets a particular input $x_i$ to zero and then computes
the network's output. For the original input $\Bx=(x_1,\ldots,x_d)$ and the input with
$x_i=0$, i.e.\ $\tilde{\Bx}_i=(x_1,\ldots,x_{i-1},0,x_{i+1},\ldots,x_d)$,
we compute $\Delta x_i=F(\Bx) \ - \ F(\tilde{\Bx}_i)$ to obtain the
contribution of $x_i$. 
We obtain for the difference of $F(\Bx)$ to the baseline
of average zeroing $\frac{1}{d} \sum_{i=1}^d F(\tilde{\Bx}_i)$:
\begin{align}
  &F(\Bx)  \ - \ \frac{1}{d} \ \sum_{i=1}^d F(\tilde{\Bx}_i) \ = \
  \frac{1}{d} \ \sum_{i=1}^d \Delta x_i \ .
\end{align}
The problem is that the $F(\tilde{\Bx}_i)$ have to be computed $d$-times, that
is, for each input component zeroed out.

Input Zeroing does not recognize redundant inputs, i.e.\ each one of the 
inputs is sufficient 
to produce the output but if all inputs are missing at the same time 
then the output changes. 
In contrast, Integrated Gradients (IG) and Layer-Wise Relevance Propagation (LRP)
detect the relevance of an input even if it is redundant.

\subsection{Integrated Gradients}

Integrated gradients is a recently introduced method
\cite{Sundararajan:17}.
Integrated gradients decomposes the difference $F(\Bx)-F(\tilde{\Bx})$
between the network output $F(\Bx)$ and a baseline $F(\tilde{\Bx})$:
\begin{align}
\label{eq:intgrad}
 F(\Bx) \ - \ F(\tilde{\Bx}) \ &= \ \sum_{i=1}^d (x_i \ - \
 \tilde{x}_i) \ \int_{t=0}^{1}
 \frac{\partial F}{\partial x_i}(\tilde{\Bx}+ t (\Bx-\tilde{\Bx}))
 \, \Rd t \\ \nonumber
 &\approx \
 \sum_{i=1}^d (x_i \ - \ \tilde{x}_i) \ \frac{1}{m} \ \sum_{k=1}^m
 \frac{\partial F}{\partial x_i}(\tilde{\Bx}+ (k/m)
 (\Bx-\tilde{\Bx})) \ .
\end{align}
In contrast to previous approaches, we have $F$ and its derivative to
evaluate only $m$-times, where $m<d$.

The equality can be seen if we define $\Bh=\Bx-\tilde{\Bx}$ and
\begin{align}
\begin{cases}
  g:[0,1]\to \dR \\
  g(t)=F(\Bx+t \Bh) \ .
\end{cases}
\end{align}
Consequently, we have
\begin{align}
  &F(\Bx+\Bh)-F(\Bx) \ = \ g(1)-g(0) \ = \
  \int _{0}^{1} g'(t)\, \Rd t\\
  &= \ \int _{0}^{1} \left(\sum_{i=1}^{d} \
    \frac{\partial F}{\partial x_i}(\Bx+t \Bh) \ h_i\right) \, \Rd t
  \ = \ \sum _{i=1}^{d} \ \left(\int _{0}^{1}
    \frac{\partial F}{\partial x_i}(\Bx+t \Bh)\, \Rd t\right) \ h_i \ .
\end{align}

For the final reward decomposition, we obtain
\begin{align}
 F(\Bx)  \ &= \ 
 \sum_{i=1}^d \left( (x_i \ - \ \tilde{x}_i) \ \frac{1}{m} \ \sum_{k=1}^m
   \frac{\partial F}{\partial x_i}(\tilde{\Bx}+ (k/m)
 (\Bx-\tilde{\Bx})) 
  \ + \ \frac{1}{d} \ F(\tilde{\Bx}) \right) \ .
 \end{align}

\subsection{Layer-Wise Relevance Propagation}
\label{sec:ALRP}

Layer-Wise Relevance Propagation (LRP) \cite{Bach:15}
has been introduced to interpret machine learning models.
LRP is an extension of
the contribution-propagation algorithm \cite{Landecker:13} based on
the contribution approach \cite{Poulin:06}.
Recently ``excitation backprop'' was proposed \cite{Zhang:16},
which is like LPR but uses only positive weights and shifts the
activation function to have non-negative values.
Both algorithms assign a relevance or importance value
to each node of a neural network which describes how much
it contributed to generating the network output.
The relevance or importance is recursively propagated back:
A neuron is important to the network output if it has been important to its
parents, and its parents have been important to the network output.
LRP moves through a neural network like backpropagation:
it starts at the output, redistributes the relevance scores of one
layer to the previous layer until the input layer is reached.
The redistribution procedure satisfies a local relevance
conservation principle. All relevance values that a node obtains from its parents will be
redistributed to its children.
This is analog to Kirchhoff's first law for the conservation of electric charge
or the continuity equation in physics for transportation in general form.
LRP has been used for deep neural networks (DNN) \cite{Montavon:17} and
for recurrent neural networks like LSTM \cite{Arras:17}.

We consider a neural network with activation $x_i$ for neuron $i$.
The weight from neuron $l$ to neuron $i$ is denoted by $w_{il}$.
The activation function is $g$ and $\net_i$ is the netinput to
neuron $i$ with bias $b_i$.
We have following forward propagation rules:
\begin{align}
  \net_i \ &= \ \sum_l w_{il} \ x_l \ , \\
  x_i \ &= \ f_i(\net_i)  \ = \ g( \net_i + b_i) \ .
\end{align}

Let $R_i$ be the relevance for neuron $i$ and $R_{i \leftarrow k}$
the share of relevance $R_k$ that flows from neuron $k$ in the higher
layer to neuron $i$ in the lower layer.
The parameter $z_{ik}$ is a weighting for the share of $R_k$ of neuron
$k$ that flows to neuron $i$. We define $R_{i \leftarrow k}$ as
\begin{align}
    R_{i \leftarrow k}\ &= \  \frac{z_{ik}}{\sum_l z_{lk}} \ R_k \ .
\end{align}
The relative contributions $z_{ik}$ are
previously defined as \cite{Bach:15,Montavon:17,Arras:17}:
\begin{align}
  z_{ik}\ &= \ w_{ik} \ x_k \ .
\end{align}
Here, $z_{ik}$ is the contribution of $x_k$ to the netinput value $\net_i$.
If neuron $k$ is removed from the network,
then  $z_{ik}$ will be the difference to the original $\net_i$.

The relevance $R_i$ of neuron $i$ is the sum of relevances it obtains
from its parents $k$ from a layer above: 
\begin{align}
   R_{i}\ &= \ \sum_k R_{i \leftarrow k} \ . 
\end{align}
Furthermore, a unit $k$ passes on all its relevance $R_k$ to its children,
which are units $i$ of the layer below:
\begin{align}
   R_k \ &= \  \sum_i  R_{i \leftarrow k} \ .
\end{align}
It follows the {\em conservation of relevance}.
The sum of relevances $R_k$
of units $k$ in a layer is equal to the sum of
relevances $R_i$ of units $i$
of a layer below:
\begin{align}
  \label{eq:cons}
  \sum_k R_k \ &= \  \sum_k  \ \sum_i  R_{i \leftarrow k} \ = \
  \sum_i  \ \sum_k  R_{i \leftarrow k}  \ = \ \sum_i R_i \ .
\end{align}

The scalar output $g(\Bx)$ 
of a neural network with input $\Bx=(x_1,\ldots,x_d)$
is considered as relevance $R$ which is
decomposed into contributions $R_i$ of the inputs $x_i$:
\begin{align}
  \sum_i R_i \ &= \  R \ = \ g(\Bx) \ .
\end{align} 

The decomposition is valid for recurrent neural networks,
where the relevance at the output
is distributed across the sequence elements of the
input sequence.

\subsubsection{New Variants of LRP}
\label{sec:ALRPvariants}

An alternative definition of $z_{ik}$ is
\begin{align}
  z_{ik}\ &= \ w_{ik} \ (x_k \ - \ \bar{x}_k ) \ ,
\end{align} 
where $\bar{x}_k$ is the mean of $x_k$ across samples.
Therefore, $(x_k \ - \ \bar{x}_k )$ is the contribution of the
actual sample to the variance of $x_k$. This in turn is related to
the information carried by $x_k$.
Here, $z_{ik}$ is the contribution of $x_k$ to the variance of
$\net_i$. However, we can have negative values of
$(x_k \ - \ \bar{x}_k )$ which may lead to negative contributions even
if the weights are positive.

Another alternative definition of $z_{ik}$ is
\begin{align}
  z_{ik} \ &= \ f_i(\net_i) \ - \ f_i(\net_i \ - \
            w_{ik} \ x_k) \ .
\end{align} 
Here, $z_{ik}$ is the contribution of $x_k$ to the activation
value $x_i=f_i(\net_i)$.
If neuron $k$ is removed from the network,
then  $z_{ik}$ will be the difference to the original $x_i$.
If $f_i$ is strict monotone increasing and $x_k>0$, then
positive weights $w_{ik}$ will lead to positive values and negative weights
$w_{ik}$ to negative values.

\noindent Preferred Solution:\newline
A definition of $z_{ik}$ is
\begin{align}
  z_{ik}\ &= \ w_{ik} \ (x_k \ - \ x_{\min} ) \ ,
\end{align} 
where $x_{\min}$ is the minimum of $x_k$ either across samples
(mini-batch) or across time steps.
The difference $(x_k  -  x_{\min} )$ is always positive.
Using this definition,
activation functions with negative values are possible, like for
excitation backprop \cite{Zhang:16}.
The minimal value is considered as default off-set, which can be
included into the bias.

\subsubsection{LRP for Products}
\label{sec:ALRPproduct}

Here we define relevance propagation for products of two units.
We assume that $z= x_1  x_2$ with $x_1>0$ and $x_2>0$. 
We view $x_1$ and $x_2$ as units of a layer below the layer in
which $z$ is located. Consequently, $R_z$ has to be divided between $x_1$ and $x_2$,
which gives the conservation rule
\begin{align}
R_z  \ &= \    R_{x_1 \leftarrow z} \ + \ R_{x_2 \leftarrow z}   \ .
\end{align}

Alternative 1:
\begin{align}
  R_{x_1 \leftarrow z}   \ &= \  0.5 \ R_z \\
  R_{x_2 \leftarrow z}   \ &= \  0.5 \ R_z   \ .
\end{align}
The relevance is equally distributed.

\noindent Preferred Solution:\newline
Alternative 2:
The contributions according to the deep Taylor decomposition around
$(a,a)$ are
\begin{align}
  \frac{\partial z}{\partial x_1}\bigg\rvert_{(a,a)} \ (x_1 -a)   \ &= \ (x_1-a)  \ a  \ , \\
  \frac{\partial z}{\partial x_2}\bigg\rvert_{(a,a)} \ (x_2 -a)   \ &= \ a  \ (x_2-a) \ .
\end{align}
We compute the relative contributions:
\begin{align}
  \frac{(x_1-a)  \ a}{(x_1-a)  \ a \ + \ a  \ (x_2-a)}  \ &= \
  \frac{x_1-a}{(x_1\ + \ x_2 \ - \ 2 \ a)}  \ , \\
  \frac{(x_2-a)  \ a}{(x_1-a)  \ a \ + \ a  \ (x_2-a)}  \ &= \
  \frac{x_2-a}{(x_1\ + \ x_2 \ - \ 2 \ a)}  \ .
\end{align}
For $\lim_{a\to 0}$ we obtain $x_1/(x_1+x_2)$ and $x_2/(x_1+x_2)$
as contributions.

We use this idea but scale $x_1$ and $x_2$ to the range $[0,1]$:
\begin{align}
  R_{x_1 \leftarrow z}   \ &= \
  \frac{\frac{x_1-x_{\min}}{x_{\max}-x_{\min}}}{\frac{x_1-x_{\min}}{x_{\max}-x_{\min}}  \ + \ \frac{x_2-x_{\min}}{x_{\max}-x_{\min}}} \ R_z \\
  R_{x_2 \leftarrow z}   \ &= \
  \frac{\frac{x_2-x_{\min}}{x_{\max}-x_{\min}}}{\frac{x_1-x_{\min}}{x_{\max}-x_{\min}}
                             \ + \  \frac{x_2-x_{\min}}{x_{\max}-x_{\min}}} \ R_z  \ .
\end{align}
The relevance is distributed according to how close the maximal value
is achieved and how far away it is from the minimal value.

Alternative 3:
\begin{align}
  R_{x_1 \leftarrow z}   \ &= \
  \frac{\ln\left(1-\frac{x_1-x_{\min}}{x_{\max}-x_{\min}} \right)}{\ln\left(1-\frac{x_1-x_{\min}}{x_{\max}-x_{\min}} \right) \ + \ \ln\left(1-\frac{x_2-x_{\min}}{x_{\max}-x_{\min}} \right)} \ R_z \\
  R_{x_2 \leftarrow z}   \ &= \
  \frac{\ln\left(1-\frac{x_2-x_{\min}}{x_{\max}-x_{\min}} \right)}{\ln\left(1-\frac{x_1-x_{\min}}{x_{\max}-x_{\min}} \right) \ + \ \ln\left(1-\frac{x_2-x_{\min}}{x_{\max}-x_{\min}} \right)} \ R_z  \ .
\end{align}
All $\ln$-values are negative, therefore the fraction in front of
$R_z$ is positive. $x_1= x_{\min}$ leads to a zero relevance for $x_1$.
The ratio of the relevance for $x_1$
increases to 1 when $x_1$ approaches $x_{\max}$. 
The relevance is distributed according to how close the maximal value
is achieved. We assume that the maximal value is a saturating value,
therefore we use $\ln$, the natural logarithm.

\subsection{Variance Considerations for contribution Analysis}

We are interested how the redistributed reward affects the variance 
of the estimators. 
We consider
(A) the difference of consecutive predictions is the redistributed reward,
(B) integrated gradients (IG), and
(C) layer-wise relevance propagation (LRP).

For (A) the difference of consecutive predictions is the redistributed reward,
all variance is moved to the final correction. However imperfect $g$ and variance
cannot be distinguished.

For (B) integrated gradients (IG) the redistributed rewards depend on future values.
Therefore the variance can even be larger than in the original MDP.

For (C) layer-wise relevance propagation (LRP) the variance is propagated back without
decreasing or increasing if the actual return is used as relevance. If the prediction is
used as relevance and a final correction is used then the variance is moved to the final
prediction but new variance is injected since rewards depend on the future path.

\clearpage
\pagebreak
\section{Reproducibility Checklist}
We followed the reproducibility checklist \cite{Pineau:18} and point to relevant sections. 
\paragraph{For all models and algorithms presented, check if you include:}
\begin{itemize}
    \item \textbf{A clear description of the mathematical setting, algorithm, and/or model.}
    
    Description of mathematical settings starts at paragraph ~\nameref{c:def}.
    
    Description of novel learning algorithms starts at paragraph \nameref{c:novel}.
    \item \textbf{An analysis of the complexity (time, space, sample size) of any algorithm.}
    
    Plots in Figure~\ref{fig:test} show the number of episodes, i.e.\ the sample size, which are needed for convergence to the optimal policies. They are evaluated for different algorithms and delays in all artificial tasks.
    For Atari games, the number of samples corresponds to the number of game frames. See paragraph \nameref{c:Atari}.
    We further present a bias-variance analysis of TD and MC learning in Section~\ref{sec:Abias_variance_estimator} and Section~\ref{sec:Abias_variance_sample} in the appendix.
    \item \textbf{A link to a downloadable source code, with specification of all dependencies, including external libraries.}
    
    \href{https://github.com/ml-jku/baselines-rudder}{https://github.com/ml-jku/baselines-rudder}
\end{itemize}

\paragraph{For any theoretical claim, check if you include:}
\begin{itemize}
    \item \textbf{A statement of the result.}
    
    The main theorems:
    \begin{itemize}
        \item Theorem \ref{th:EquivT}
        \item Theorem \ref{th:zeroExp}
        \item Theorem \ref{th:OptReturnDecomp}
    \end{itemize}
    
    Additional supporting theorems can be found in the proof section of the appendix \ref{c:RR}.
    \item \textbf{A clear explanation of any assumptions.}
    
    The proof section \ref{c:RR} in the appendix covers all the assumptions for the main theorems.
    \item \textbf{A complete proof of the claim.}
    
Proof of the main theorems are moved to the appendix.
\begin{itemize}
    \item Proof of Theorem 1 can be found after Theorem \ref{th:AEquivT} in the appendix. 
    \item Proof of Theorem 2 can be found after Theorem \ref{th:AzeroExp} in the appendix. 
    \item Proof of Theorem 3 can be found after Theorem \ref{th:AOptReturnDecomp} in the appendix. 
\end{itemize}
Proofs for additional theorems can also be found in this appendix.
\end{itemize}

\paragraph{For all figures and tables that present empirical results, check if you include:}
\begin{itemize}
    \item \textbf{A complete description of the data collection process, including sample size.}
    
    For artificial tasks the environment descriptions can be found in section \nameref{sec:Aexp} in the main paper. For Atari games, we use the standard sampling procedures as in OpenAI Gym \cite{Brockman:16} (description can be found in paragraph \nameref{para:Atari}).
    \item \textbf{A link to a downloadable version of the dataset or simulation environment.}
    
    Link to our repository:
    \href{https://github.com/ml-jku/baselines-rudder}{https://github.com/ml-jku/rudder}
    
    \item \textbf{An explanation of any data that were excluded, description of any pre-processing step}
    
    For Atari games, we use the standard pre-processing described in \cite{Mnih:16}.
    \item \textbf{An explanation of how samples were allocated for training / validation / testing.}
    
    For artificial tasks, description of training and evaluation are included in section \ref{sec:Aexp} . For Atari games, description of training and evaluation are included Section~\ref{sec:Aexp}.
    \item \textbf{The range of hyper-parameters considered, method to select the best hyper-parameter configuration, and specification of all hyper-parameters used to generate results.}
    
    A description can be found at paragraph \nameref{c:ppomodel} in the appendix. 
    \item \textbf{The exact number of evaluation runs.}
    
    For artificial tasks evaluation was performed during training runs. See Figure~\ref{fig:test}.
    For Atari games see paragraph \nameref{para:Atari}.
    We also provide a more detailed description in Section~\ref{sec:Aexp} and Section~\ref{sec:Aatari} in the appendix.
    \item \textbf{A description of how experiments were run}.
    For artificial task, description can be found at \ref{sec:Mexperiments}.
    
    For Atari games, description starts at paragraph \nameref{para:Atari}.
    We also provide a more detailed description in Section~\ref{sec:Aexp} and Section~\ref{sec:Aatari} in the appendix.
    \item \textbf{A clear definition of the specific measure or statistics used to report results.}
    
For artificial tasks, see section \ref{sec:Mexperiments}.
For Atari games, see section \ref{sec:Aatari} and the caption of Table 1. 
We also provide a more detailed description in Section~\ref{sec:Aexp} and Section~\ref{sec:Aatari} in the appendix.
    \item \textbf{Clearly defined error bars.}
    
For artificial tasks, see caption of Figure~\ref{fig:test}, second line.
For Atari games we show all runs in Figure~\ref{fig:atari_training} in the appendix.
    \item \textbf{A description of results with central tendency (e.g.\ mean) \& variation (e.g.\ stddev).}
    
An exhaustive description of the results including mean, variance and significant test, is included in Table~\ref{tab:res1}, Table~\ref{tab:Ares1} and Table~\ref{tab:Ares2} in Section~\ref{sec:Aexp} in the appendix.
    \item \textbf{A description of the computing infrastructure used.}
    
We distributed all runs across 2 CPUs per run and 1 GPU per 4 runs for Atari experiments. We used various GPUs including GTX 1080 Ti, TITAN X, and TITAN V.
    Our algorithm takes approximately 10 days.
\end{itemize}

\vfill
\pagebreak
\section{References}
\label{sec:Areferences}
\renewcommand{\section}[2]{}%
\bibliographystyle{plain}
\bibliography{lrp}

\begin{thebibliography}{100}

\bibitem{Arras:17}
L.~Arras, G.~Montavon, K.-R. M{\"{u}}ller, and W.~Samek.
\newblock Explaining recurrent neural network predictions in sentiment
  analysis.
\newblock {\em CoRR}, abs/1706.07206, 2017.

\bibitem{Aytar:18}
Y.~Aytar, T.~Pfaff, D.~Budden, T.~Le Paine, Z.~Wang, and N.~de~Freitas.
\newblock Playing hard exploration games by watching {YouTube}.
\newblock {\em ArXiv}, 2018.

\bibitem{Bach:15}
S.~Bach, A.~Binder, G.~Montavon, F.~Klauschen, K.-R. M{\"{u}}ller, and
  W.~Samek.
\newblock On pixel-wise explanations for non-linear classifier decisions by
  layer-wise relevance propagation.
\newblock {\em PLoS ONE}, 10(7):e0130140, 2015.

\bibitem{Bakker:02}
B.~Bakker.
\newblock Reinforcement learning with long short-term memory.
\newblock In T.~G. Dietterich, S.~Becker, and Z.~Ghahramani, editors, {\em
  Advances in Neural Information Processing Systems 14}, pages 1475--1482. MIT
  Press, 2002.

\bibitem{Bakker:07}
B.~Bakker.
\newblock Reinforcement learning by backpropagation through an lstm
  model/critic.
\newblock In {\em IEEE International Symposium on Approximate Dynamic
  Programming and Reinforcement Learning}, pages 127--134, 2007.

\bibitem{Barreto:18}
A.~Barreto, D.~Borsa, J.~Quan, T.~Schaul, D.~Silver, M.~Hessel, D.~Mankowitz,
  A.~Z{\'{\i}}dek, and R.~Munos.
\newblock Transfer in deep reinforcement learning using successor features and
  generalised policy improvement.
\newblock In {\em 35th International Conference on Machine Learning}, volume~80
  of {\em Proceedings of Machine Learning Research}, pages 501--510, 2018.
\newblock ArXiv 1901.10964.

\bibitem{Barreto:17}
A.~Barreto, W.~Dabney, R.~Munos, J.~Hunt, T.~Schaul, H.~P. vanHasselt, and
  D.~Silver.
\newblock Successor features for transfer in reinforcement learning.
\newblock In {\em Advances in Neural Information Processing Systems 30}, pages
  4055--4065, 2017.
\newblock ArXiv 1606.05312.

\bibitem{Barto:04}
A.~G. Barto and T.~G. Dietterich.
\newblock {\em Handbook of Learning and Approximate Dynamic Programming},
  chapter Reinforcement Learning and Its Relationship to Supervised Learning,
  pages 45--63.
\newblock IEEE Press, John Wiley \& Sons, 2015.

\bibitem{Beleznay:99}
F.~Beleznay, T.~Grobler, and C.~Szepesv\'{a}ri.
\newblock Comparing value-function estimation algorithms in undiscounted
  problems.
\newblock Technical Report TR-99-02, Mindmaker Ltd., 1999.

\bibitem{Bellemare:17}
M.~G. Bellemare, W.~Dabney, and R.~Munos.
\newblock A distributional perspective on reinforcement learning.
\newblock In D.~Precup and Y.~W. Teh, editors, {\em Proceedings of the 34th
  International Conference on Machine Learning}, volume~70 of {\em Proceedings
  of Machine Learning Research (ICML)}, pages 449--458. PMLR, 2017.

\bibitem{Bellemare:13}
M.~G. Bellemare, Y.~Naddaf, J.~Veness, and M.~Bowling.
\newblock The {Arcade} learning environment: An evaluation platform for general
  agents.
\newblock {\em Journal of Artificial Intelligence Research}, 47:253--279, 2013.

\bibitem{Berthelot:17}
D.~Berthelot, T.~Schumm, and L.~Metz.
\newblock {BEGAN:} boundary equilibrium generative adversarial networks.
\newblock {\em ArXiv e-prints}, 2017.

\bibitem{Bertsekas:91}
D.~P. Bertsekas and J.~N. Tsitsiklis.
\newblock An analysis of stochastic shortest path problems.
\newblock {\em Math. Oper. Res.}, 16(3), 1991.

\bibitem{Bertsekas:96}
D.~P. Bertsekas and J.~N. Tsitsiklis.
\newblock {\em Neuro-dynamic programming}.
\newblock Athena Scientific, Belmont, MA, 1996.

\bibitem{Bienayme:53}
I.-J. Bienaym{\'{e}}.
\newblock Consid{\'{e}}rations {\`{a}}l'appui de la d{\'{e}}couverte de
  laplace.
\newblock {\em Comptes Rendus de l'Acad{\'{e}}mie des Sciences}, 37:309--324,
  1853.

\bibitem{Bolton:15}
W.~Bolton.
\newblock {\em Instrumentation and Control Systems}, chapter Chapter 5 -
  Process Controllers, pages 99--121.
\newblock Newnes, 2 edition, 2015.

\bibitem{Borkar:97}
V.~S. Borkar.
\newblock Stochastic approximation with two time scales.
\newblock {\em Systems \& Control Letters}, 29(5):291--294, 1997.

\bibitem{Brockman:16}
G.~Brockman, V.~Cheung, L.~Pettersson, J.~Schneider, J.~Schulman, J.~Tang, and
  W.~Zaremba.
\newblock Openai gym.
\newblock {\em ArXiv}, 2016.

\bibitem{Cox:09}
I.~J. Cox, R.~Fu, and L.~K. Hansen.
\newblock Probably approximately correct search.
\newblock In {\em Advances in Information Retrieval Theory}, pages 2--16.
  Springer, Berlin, Heidelberg, 2009.

\bibitem{Dayan:92}
P.~Dayan.
\newblock The convergence of {TD($\lambda$)} for general $\lambda$.
\newblock {\em Machine Learning}, 8:341, 1992.

\bibitem{Dhariwal:17}
P.~Dhariwal, C.~Hesse, O.~Klimov, A.~Nichol, M.~Plappert, A.~Radford,
  J.~Schulman, S.~Sidor, and Y.~Wu.
\newblock Openai baselines.
\newblock \url{https://github.com/openai/baselines}, 2017.

\bibitem{Donahue:14}
J.~Donahue, L.~A. Hendricks, S.~Guadarrama, M.~Rohrbach, S.~Venugopalan,
  K.~Saenko, and T.~Darrell.
\newblock Long-term recurrent convolutional networks for visual recognition and
  description.
\newblock {\em ArXiv}, 2014.

\bibitem{Edwards:18}
A.~D. Edwards, L.~Downs, and J.~C. Davidson.
\newblock Forward-backward reinforcement learning.
\newblock {\em ArXiv}, 2018.

\bibitem{Espeholt:18}
L.~Espeholt, H.~Soyer, R.~Munos, K.~Simonyan, V.~Mnih, T.~Ward, Y.~Doron,
  V.~Firoiu, T.~Harley, I.~Dunning, S.~Legg, and K.~Kavukcuoglu.
\newblock {IMPALA:} {S}calable distributed {Deep-RL} with importance weighted
  actor-learner architectures.
\newblock In J.~Dy and A.~Krause, editors, {\em Proceedings of the 35th
  International Conference on Machine Learning}, 2018.
\newblock ArXiv: 1802.01561.

\bibitem{Feinstein:16}
Z.~Feinstein.
\newblock Continuity properties and sensitivity analysis of parameterized fixed
  points and approximate fixed points.
\newblock Technical report, Operations Research and Financial Engineering
  Laboratory, Washington University in St. Louis, 2016.
\newblock preprint.

\bibitem{Fortunato:18}
M.~Fortunato, M.~G. Azar, B.~Piot, J.~Menick, I.~Osband, A.~Graves, V.~Mnih,
  R.~Munos, D.~Hassabis, O.~Pietquin, C.~Blundell, and S.~Legg.
\newblock Noisy networks for exploration.
\newblock {\em ArXiv}, 2018.
\newblock Sixth International Conference on Learning Representations (ICLR).

\bibitem{Frigon:07}
M.~Frigon.
\newblock Fixed point and continuation results for contractions in metric and
  {Gauge} spaces.
\newblock {\em Banach Center Publications}, 77(1):89--114, 2007.

\bibitem{Fu:18}
J.~Fu, K.~Luo, and S.~Levine.
\newblock Learning robust rewards with adversarial inverse reinforcement
  learning.
\newblock {\em ArXiv}, 2018.
\newblock Sixth International Conference on Learning Representations (ICLR).

\bibitem{Geiger:14}
J.~T. Geiger, Z.~Zhang, F.~Weninger, B.~Schuller, and G.~Rigoll.
\newblock Robust speech recognition using long short-term memory recurrent
  neural networks for hybrid acoustic modelling.
\newblock In {\em Proc. 15th Annual Conf. of the Int. Speech Communication
  Association (INTERSPEECH 2014)}, pages 631--635, Singapore, September 2014.

\bibitem{Gers:00a}
F.~A. Gers and J.~Schmidhuber.
\newblock Recurrent nets that time and count.
\newblock In {\em Proc. Int. Joint Conf. on Neural Networks (IJCNN 2000)},
  volume~3, pages 189--194. IEEE, 2000.

\bibitem{Gers:99a}
F.~A. Gers, J.~Schmidhuber, and F.~Cummins.
\newblock Learning to forget: Continual prediction with {LSTM}.
\newblock In {\em Proc. Int. Conf. on Artificial Neural Networks (ICANN '99)},
  pages 850--855, Edinburgh, Scotland, 1999.

\bibitem{Gers:00}
F.~A. Gers, J.~Schmidhuber, and F.~Cummins.
\newblock Learning to forget: Continual prediction with {LSTM}.
\newblock {\em Neural Comput.}, 12(10):2451--2471, 2000.

\bibitem{gijbels2010censored}
Ir{\`e}ne Gijbels.
\newblock Censored data.
\newblock {\em Wiley Interdisciplinary Reviews: Computational Statistics},
  2(2):178--188, 2010.

\bibitem{Givan:03}
R.~Givan, T.~Dean, and M.~Greig.
\newblock Equivalence notions and model minimization in {Markov} decision
  processes.
\newblock {\em Artificial Intelligence}, 147(1):163--223, 2003.

\bibitem{Gonzalez-Dominguez:14}
J.~Gonzalez-Dominguez, I.~Lopez-Moreno, H.~Sak, J.~Gonzalez-Rodriguez, and
  P.~Moreno.
\newblock Automatic language identification using long short-term memory
  recurrent neural networks.
\newblock In {\em Proc. 15th Annual Conf. of the Int. Speech Communication
  Association (INTERSPEECH 2014)}, pages 2155--2159, Singapore, September 2014.

\bibitem{Goyal:18}
A.~Goyal, P.~Brakel, W.~Fedus, T.~Lillicrap, S.~Levine, H.~Larochelle, and
  Y.~Bengio.
\newblock Recall traces: Backtracking models for efficient reinforcement
  learning.
\newblock {\em ArXiv}, 2018.

\bibitem{Graves:09}
A.~Graves, M.~Liwicki, S.~Fernandez, R.~Bertolami, H.~Bunke, and
  J.~Schmidhuber.
\newblock A novel connectionist system for improved unconstrained handwriting
  recognition.
\newblock {\em IEEE Trans. Pattern Anal. Mach. Intell.}, 31(5):855--868, 2009.

\bibitem{Graves:13}
A.~Graves, A.-R. Mohamed, and G.~E. Hinton.
\newblock Speech recognition with deep recurrent neural networks.
\newblock In {\em Proc. IEEE Int. Conf. on Acoustics, Speech and Signal
  Processing (ICASSP 2013)}, pages 6645--6649, Vancouver, BC, 2013.

\bibitem{Graves:05}
A.~Graves and J.~Schmidhuber.
\newblock Framewise phoneme classification with bidirectional {LSTM} and other
  neural network architectures.
\newblock {\em Neural Networks}, 18(5-6):602--610, 2005.

\bibitem{Greff:15}
K.~Greff, R.~K. Srivastava, J.~Koutn\'{i}k, B.~R. Steunebrink, and
  J.~Schmidhuber.
\newblock {LSTM}: A search space odyssey.
\newblock {\em ArXiv}, 2015.

\bibitem{Grunewalder:11}
S.~Gr{\"{u}}new{\"{a}}lder and K.~Obermayer.
\newblock The optimal unbiased value estimator and its relation to {LSTD}, {TD}
  and {MC}.
\newblock {\em Machine Learning}, 83(3):289--330, 2011.

\bibitem{Ha:18}
D.~Ha and J.~Schmidhuber.
\newblock World models.
\newblock {\em ArXiv}, 2018.

\bibitem{Harutyunyan:15}
A.~Harutyunyan, S.~Devlin, P.~Vrancx, and A.~Now'{e}.
\newblock Expressing arbitrary reward functions as potential-based advice.
\newblock In {\em Proceedings of the Twenty-Ninth AAAI Conference on Artificial
  Intelligence (AAAI'15)}, pages 2652--2658, 2015.

\bibitem{Hausknecht:15}
M.~J. Hausknecht and P.~Stone.
\newblock Deep recurrent {Q-Learning} for partially observable {MDPs}.
\newblock {\em ArXiv}, 2015.

\bibitem{Heess:16}
N.~Heess, G.~Wayne, Y.~Tassa, T.~P. Lillicrap, M.~A. Riedmiller, and D.~Silver.
\newblock Learning and transfer of modulated locomotor controllers.
\newblock {\em ArXiv}, 2016.

\bibitem{Hernandez-Leal:18}
P.~Hernandez-Leal, B.~Kartal, and M.~E. Taylor.
\newblock Is multiagent deep reinforcement learning the answer or the question?
  {A} brief survey.
\newblock {\em ArXiv}, 2018.

\bibitem{Hessel:17}
M.~Hessel, J.~Modayil, H.~van Hasselt, T.~Schaul, G.~Ostrovski, W.~Dabney,
  D.~Horgan, B.~Piot, M.~G. Azar, and D.~Silver.
\newblock Rainbow: Combining improvements in deep reinforcement learning.
\newblock {\em ArXiv}, 2017.

\bibitem{Hochreiter:90}
S.~Hochreiter.
\newblock {Implementierung und Anwendung eines `neuronalen'
  Echtzeit-Lernalgorithmus f\"{u}r reaktive Umgebungen}.
\newblock Practical work, Supervisor: J. Schmidhuber, Institut f\"{u}r
  Informatik, Technische Universit\"{a}t M\"{u}nchen, 1990.

\bibitem{Hochreiter:91}
S.~Hochreiter.
\newblock { Untersuchungen zu dynamischen neuronalen Netzen}.
\newblock Master's thesis, Technische Universit\"{a}t M\"{u}nchen, 1991.

\bibitem{Hochreiter:97f}
S.~Hochreiter.
\newblock Recurrent neural net learning and vanishing gradient.
\newblock In C.~Freksa, editor, {\em Proc. Fuzzy-Neuro-Systeme '97}, pages
  130--137, Sankt Augustin, 1997. INFIX.

\bibitem{Hochreiter:98}
S.~Hochreiter.
\newblock The vanishing gradient problem during learning recurrent neural nets
  and problem solutions.
\newblock {\em Internat. J. Uncertain. Fuzziness Knowledge-Based Systems},
  6(2):107--116, 1998.

\bibitem{Hochreiter:00}
S.~Hochreiter, Y.~Bengio, P.~Frasconi, and J.~Schmidhuber.
\newblock Gradient flow in recurrent nets: the difficulty of learning long-term
  dependencies.
\newblock In J.~F. Kolen and S.~C. Kremer, editors, {\em A Field Guide to
  Dynamical Recurrent Networks}. IEEE Press, 2001.

\bibitem{Hochreiter:07}
S.~Hochreiter, M.~Heusel, and K.~Obermayer.
\newblock Fast model-based protein homology detection without alignment.
\newblock {\em Bioinformatics}, 23(14):1728--1736, 2007.

\bibitem{Hochreiter:95}
S.~Hochreiter and J.~Schmidhuber.
\newblock Long short-term memory.
\newblock Technical Report FKI-207-95, Fakult\"{a}t f\"{u}r Informatik,
  Technische Universit\"{a}t M\"{u}nchen, 1995.

\bibitem{Hochreiter:97}
S.~Hochreiter and J.~Schmidhuber.
\newblock Long short-term memory.
\newblock {\em Neural Comput.}, 9(8):1735--1780, 1997.

\bibitem{Hochreiter:97e}
S.~Hochreiter and J.~Schmidhuber.
\newblock {LSTM} can solve hard long time lag problems.
\newblock In M.~C. Mozer, M.~I. Jordan, and T.~Petsche, editors, {\em Advances
  in Neural Information Processing Systems 9}, pages 473--479, Cambridge, MA,
  1997. MIT Press.

\bibitem{Hochreiter:01}
S.~Hochreiter, A.~Steven Younger, and Peter~R. Conwell.
\newblock Learning to learn using gradient descent.
\newblock In G.~Dorffner, H.~Bischof, and K.~Hornik, editors, {\em Proc. Int.
  Conf. on Artificial Neural Networks (ICANN 2001)}, pages 87--94. Springer,
  2001.

\bibitem{Horgan:18}
D.~Horgan, J.~Quan, D.~Budden, G.~Barth-Maron, M.~Hessel, H.~van Hasselt, and
  D.~Silver.
\newblock Distributed prioritized experience replay.
\newblock {\em ArXiv}, 2018.
\newblock Sixth International Conference on Learning Representations (ICLR).

\bibitem{Hung:18}
C.~Hung, T.~Lillicrap, J.~Abramson, Y.~Wu, M.~Mirza, F.~Carnevale, A.~Ahuja,
  and G.~Wayne.
\newblock Optimizing agent behavior over long time scales by transporting
  value.
\newblock {\em ArXiv}, 2018.

\bibitem{Hyvarinen:01}
A.~Hyv{\"a}rinen, J.~Karhunen, and E.~Oja.
\newblock {\em Independent Component Analysis}.
\newblock John Wiley \& Sons, New York, 2001.

\bibitem{Jaakkola:94}
T.~Jaakkola, M.~I. Jordan, and S.~P. Singh.
\newblock On the convergence of stochastic iterative dynamic programming
  algorithms.
\newblock {\em Neural Computation}, 6(6):1185--1201, 1994.

\bibitem{Jachymski:96}
J.~Jachymski.
\newblock Continuous dependence of attractors of iterated function systems.
\newblock {\em Journal Of Mathematical Analysis And Applications},
  198(0077):221--226, 1996.

\bibitem{John:94}
G.~H. John.
\newblock When the best move isn't optimal: $q$-learning with exploration.
\newblock In {\em Proceedings of the 10th Tenth National Conference on
  Artificial Intelligence, Menlo Park, CA, 1994. AAAI Press.}, page 1464, 1994.

\bibitem{Karmakar:17}
P.~Karmakar and S.~Bhatnagar.
\newblock Two time-scale stochastic approximation with controlled {Markov}
  noise and off-policy temporal-difference learning.
\newblock {\em Mathematics of Operations Research}, 2017.

\bibitem{Ke:18}
N.~Ke, A.~Goyal, O.~Bilaniuk, J.~Binas, M.~Mozer, C.~Pal, and Y.~Bengio.
\newblock Sparse attentive backtracking: Temporal credit assignment through
  reminding.
\newblock In {\em Advances in Neural Information Processing Systems 31}, pages
  7640--7651, 2018.

\bibitem{Khandelwal:16}
P.~Khandelwal, E.~Liebman, S.~Niekum, and P.~Stone.
\newblock On the analysis of complex backup strategies in {Monte Carlo Tree
  Search}.
\newblock In {\em International Conference on Machine Learning}, pages
  1319--1328, 2016.

\bibitem{Kirr:97}
E.~Kirr and A.~Petrusel.
\newblock Continuous dependence on parameters of the fixed point set for some
  set-valued operators.
\newblock {\em Discussiones Mathematicae Differential Inclusions}, 17:29--41,
  1997.

\bibitem{Kocsis:06}
L.~Kocsis and C.~Szepesv{\'{a}}ri.
\newblock Bandit based {Monte-Carlo} planning.
\newblock In {\em European Conference on Machine Learning}, pages 282--293.
  Springer, 2006.

\bibitem{Koutnik:13}
J.~Koutn\'{\i}k, G.~Cuccu, J.~Schmidhuber, and F.~Gomez.
\newblock Evolving large-scale neural networks for vision-based reinforcement
  learning.
\newblock In {\em Proceedings of the 15th Annual Conference on Genetic and
  Evolutionary Computation}, GECCO '13, pages 1061--1068, 2013.

\bibitem{Kwiecinski:92}
M.~Kwiecinski.
\newblock A note on continuity of fixed points.
\newblock {\em Universitatis Iagellonicae Acta Mathematica}, 29:19--24, 1992.

\bibitem{Landecker:13}
W.~Landecker, M.~D. Thomure, L.~M.~A. Bettencourt, M.~Mitchell, G.~T. Kenyon,
  and S.~P. Brumby.
\newblock Interpreting individual classifications of hierarchical networks.
\newblock In {\em IEEE Symposium on Computational Intelligence and Data Mining
  (CIDM)}, pages 32--38, 2013.

\bibitem{Lattimore:18}
T.~Lattimore and C.~Szepesv\'{a}.
\newblock {\em Bandit Algorithms}.
\newblock Cambridge University Press, 2018.
\newblock Draft of 28th July, Revision 1016.

\bibitem{Li:06}
L.~Li, T.~J. Walsh, and M.~L. Littman.
\newblock Towards a unified theory of state abstraction for {MDPs}.
\newblock In {\em Ninth International Symposium on Artificial Intelligence and
  Mathematics (ISAIM)}, 2006.

\bibitem{Lin:93}
L.~Lin.
\newblock {\em Reinforcement Learning for Robots Using Neural Networks}.
\newblock PhD thesis, Carnegie Mellon University, Pittsburgh, 1993.

\bibitem{Lugosi:03}
G.~Lugosi.
\newblock Concentration-of-measure inequalities.
\newblock In {\em Summer School on Machine Learning at the Australian National
  University,Canberra}, 2003.
\newblock Lecture notes of 2009.

\bibitem{Luoma:17}
J.~Luoma, S.~Ruutu, A.~W. King, and H.~Tikkanen.
\newblock Time delays, competitive interdependence, and firm performance.
\newblock {\em Strategic Management Journal}, 38(3):506--525, 2017.

\bibitem{Mannor:07}
S.~Mannor, D.~Simester, P.~Sun, and J.~N. Tsitsiklis.
\newblock Bias and variance approximation in value function estimates.
\newblock {\em Management Science}, 53(2):308--322, 2007.

\bibitem{Marchi:14}
E.~Marchi, G.~Ferroni, F.~Eyben, L.~Gabrielli, S.~Squartini, and B.~Schuller.
\newblock Multi-resolution linear prediction based features for audio onset
  detection with bidirectional {LSTM} neural networks.
\newblock In {\em Proc. IEEE Int. Conf. on Acoustics, Speech and Signal
  Processing (ICASSP 2014)}, pages 2164--2168, Florence, May 2014.

\bibitem{Marchenko:67}
V.~A. Mar\u{o}enko and L.~A. Pastur.
\newblock Distribution of eigenvalues or some sets of random matrices.
\newblock {\em Mathematics of the USSR-Sbornik}, 1(4):457, 1967.

\bibitem{Mnih:16}
V.~Mnih, A.~P. Badia, M.~Mirza, A.~Graves, T.~Lillicrap, T.~Harley, D.~Silver,
  and K.~Kavukcuoglu.
\newblock Asynchronous methods for deep reinforcement learning.
\newblock In M.~F. Balcan and K.~Q. Weinberger, editors, {\em Proceedings of
  the 33rd International Conference on Machine Learning (ICML)}, volume~48 of
  {\em Proceedings of Machine Learning Research}, pages 1928--1937. PMLR, 2016.

\bibitem{Mnih:13}
V.~Mnih, K.~Kavukcuoglu, D.~Silver, A.~Graves, I.~Antonoglou, D.~Wierstra, and
  M.~A. Riedmiller.
\newblock Playing {Atari} with deep reinforcement learning.
\newblock {\em ArXiv}, 2013.

\bibitem{Mnih:15}
V.~Mnih, K.~Kavukcuoglu, D.~Silver, A.~A. Rusu, J.~Veness, M.~G. Bellemare,
  A.~Graves, M.~Riedmiller, A.~K. Fidjeland, G.~Ostrovski, S.~Petersen,
  C.~Beattie, A.~Sadik, I.~Antonoglou, H.~King, D.~Kumaran, D.~Wierstra,
  S.~Legg, , and D.~Hassabis.
\newblock Human-level control through deep reinforcement learning.
\newblock {\em Nature}, 518(7540):529--533, 2015.

\bibitem{Montavon:17taylor}
G.~Montavon, S.~Lapuschkin, A.~Binder, W.~Samek, and K.-R. M{\"{u}}ller.
\newblock Explaining nonlinear classification decisions with deep {Taylor}
  decomposition.
\newblock {\em Pattern Recognition}, 65:211 -- 222, 2017.

\bibitem{Montavon:17}
G.~Montavon, W.~Samek, and K.-R. M{\"{u}}ller.
\newblock Methods for interpreting and understanding deep neural networks.
\newblock {\em Digital Signal Processing}, 73:1--15, 2017.

\bibitem{Moore:93}
A.~W. Moore and C.~G. Atkeson.
\newblock Prioritized sweeping: Reinforcement learning with less data and less
  time.
\newblock {\em Machine Learning}, 13(1):103--130, 1993.

\bibitem{Munro:87}
P.~W. Munro.
\newblock A dual back-propagation scheme for scalar reinforcement learning.
\newblock In {\em Proceedings of the Ninth Annual Conference of the Cognitive
  Science Society, Seattle, WA}, pages 165--176, 1987.

\bibitem{Ng:99}
A.~Y. Ng, D.~Harada, and S.~J. Russell.
\newblock Policy invariance under reward transformations: Theory and
  application to reward shaping.
\newblock In {\em Proceedings of the Sixteenth International Conference on
  Machine Learning (ICML'99)}, pages 278--287, 1999.

\bibitem{ODonoghue:17}
B.~O'Donoghue, I.~Osband, R.~Munos, and V.~Mnih.
\newblock The uncertainty {Bellman} equation and exploration.
\newblock {\em ArXiv}, 2017.

\bibitem{Patek:97}
S.~D. Patek.
\newblock {\em Stochastic and shortest path games: theory and algorithms}.
\newblock PhD thesis, Massachusetts Institute of Technology. Dept. of
  Electrical Engineering and Computer Science, 1997.

\bibitem{Peng:96}
J.~Peng and R.~J. Williams.
\newblock Incremental multi-step $q$-learning.
\newblock {\em Machine Learning}, 22(1):283--290, 1996.

\bibitem{Peters:07}
J.~Peters and S.~Schaal.
\newblock Reinforcement learning by reward-weighted regression for operational
  space control.
\newblock In {\em Proceedings of the 24th International Conference on Machine
  Learning}, pages 745--750, 2007.

\bibitem{Pineau:18}
J.~Pineau.
\newblock The machine learning reproducibility checklist, 2018.

\bibitem{Pohlen:18}
T.~Pohlen, B.~Piot, T.~Hester, M.~G. Azar, D.~Horgan, D.~Budden,
  G.~Barth-Maron, H.~van Hasselt, J.~Quan, M.~Ve{\v{c}}er{\'{i}}k, M.~Hessel,
  R.~Munos, and O.~Pietquin.
\newblock Observe and look further: Achieving consistent performance on
  {Atari}.
\newblock {\em ArXiv}, 2018.

\bibitem{Poulin:06}
B.~Poulin, R.~Eisner, D.~Szafron, P.~Lu, R.~Greiner, D.~S. Wishart, A.~Fyshe,
  B.~Pearcy, C.~MacDonell, and J.~Anvik.
\newblock Visual explanation of evidence in additive classifiers.
\newblock In {\em Proceedings of the 18th Conference on Innovative Applications
  of Artificial Intelligence (IAAI)}, volume~2, pages 1822--1829, 2006.

\bibitem{Puterman:90}
M.~L. Puterman.
\newblock Markov decision processes.
\newblock In {\em Stochastic Models}, volume~2 of {\em Handbooks in Operations
  Research and Management Science}, chapter~8, pages 331--434. Elsevier, 1990.

\bibitem{Puterman:05}
M.~L. Puterman.
\newblock {\em Markov Decision Processes}.
\newblock John Wiley \& Sons, Inc., 2005.

\bibitem{Rahmandad:09}
H.~Rahmandad, N.~Repenning, and J.~Sterman.
\newblock Effects of feedback delay on learning.
\newblock {\em System Dynamics Review}, 25(4):309--338, 2009.

\bibitem{Ravindran:01}
B.~Ravindran and A.~G. Barto.
\newblock Symmetries and model minimization in {Markov} decision processes.
\newblock Technical report, University of Massachusetts, Amherst, MA, USA,
  2001.

\bibitem{Ravindran:03}
B.~Ravindran and A.~G. Barto.
\newblock {SMDP} homomorphisms: An algebraic approach to abstraction in
  semi-{Markov} decision processes.
\newblock In {\em Proceedings of the 18th International Joint Conference on
  Artificial Intelligence (IJCAI'03)}, pages 1011--1016, San Francisco, CA,
  USA, 2003. Morgan Kaufmann Publishers Inc.

\bibitem{Rencher:08}
A.~C. Rencher and G.~B. Schaalje.
\newblock {\em Linear Models in Statistics}.
\newblock John Wiley \& Sons, Hoboken, New Jersey, 2 edition, 2008.
\newblock ISBN 978-0-471-75498-5.

\bibitem{Robinson:89}
A.~J. Robinson.
\newblock {\em Dynamic Error Propagation Networks}.
\newblock PhD thesis, Trinity Hall and Cambridge University Engineering
  Department, 1989.

\bibitem{RobinsonFallside:89}
T.~Robinson and F.~Fallside.
\newblock Dynamic reinforcement driven error propagation networks with
  application to game playing.
\newblock In {\em Proceedings of the 11th Conference of the Cognitive Science
  Society, Ann Arbor}, pages 836--843, 1989.

\bibitem{Romoff:18}
J.~Romoff, A.~Pich{\'{e}}, P.~Henderson, V.~Francois-Lavet, and J.~Pineau.
\newblock Reward estimation for variance reduction in deep reinforcement
  learning.
\newblock {\em ArXiv}, 2018.

\bibitem{Rudelson:10}
M.~Rudelson and R.~Vershynin.
\newblock Non-asymptotic theory of random matrices: extreme singular values.
\newblock {\em ArXiv}, 2010.

\bibitem{Rummery:94}
G.~A. Rummery and M.~Niranjan.
\newblock On-line $q$-learning using connectionist systems.
\newblock Technical Report TR 166, Cambridge University Engineering Department,
  1994.

\bibitem{Sahni:18}
H.~Sahni.
\newblock Reinforcement learning never worked, and 'deep' only helped a bit.
\newblock
  \url{himanshusahni.github.io/2018/02/23/reinforcement-learning-never-worked.html},
  2018.

\bibitem{Sak:14}
H.~Sak, A.~Senior, and F.~Beaufays.
\newblock Long short-term memory recurrent neural network architectures for
  large scale acoustic modeling.
\newblock In {\em Proc. 15th Annual Conf. of the Int. Speech Communication
  Association (INTERSPEECH 2014)}, pages 338--342, Singapore, September 2014.

\bibitem{Schaal:99}
S.~Schaal.
\newblock Is imitation learning the route to humanoid robots?
\newblock {\em Trends in Cognitive Sciences}, 3(6):233--242, 1999.

\bibitem{Schaul:15}
T.~Schaul, J.~Quan, I.~Antonoglou, and D.~Silver.
\newblock Prioritized experience replay.
\newblock {\em ArXiv}, 2015.

\bibitem{Schmidhuber:90diff}
J.~Schmidhuber.
\newblock Making the world differentiable: On using fully recurrent
  self-supervised neural networks for dynamic reinforcement learning and
  planning in non-stationary environments.
\newblock Technical Report FKI-126-90 (revised), Institut f\"{u}r Informatik,
  Technische Universit\"{a}t M\"{u}nchen, 1990.
\newblock Experiments by Sepp Hochreiter.

\bibitem{Schmidhuber:91nips}
J.~Schmidhuber.
\newblock Reinforcement learning in markovian and non-markovian environments.
\newblock In R.~P. Lippmann, J.~E. Moody, and D.~S. Touretzky, editors, {\em
  Advances in Neural Information Processing Systems 3}, pages 500--506. San
  Mateo, CA: Morgan Kaufmann, 1991.
\newblock Pole balancing experiments by Sepp Hochreiter.

\bibitem{Schmidhuber:15}
J.~Schmidhuber.
\newblock Deep learning in neural networks: An overview.
\newblock {\em Neural Networks}, 61:85--117, 2015.

\bibitem{Schulman:15icml}
J.~Schulman, S.~Levine, P.~Moritz, M.~I. Jordan, and P.~Abbeel.
\newblock Trust region policy optimization.
\newblock In {\em 32st International Conference on Machine Learning (ICML)},
  volume~37 of {\em Proceedings of Machine Learning Research}, pages
  1889--1897. PMLR, 2015.

\bibitem{Schulman:15}
J.~Schulman, P.~Moritz, S.~Levine, M.~I. Jordan, and P.~Abbeel.
\newblock High-dimensional continuous control using generalized advantage
  estimation.
\newblock {\em ArXiv}, 2015.
\newblock Fourth International Conference on Learning Representations
  (ICLR'16).

\bibitem{Schulman:17}
J.~Schulman, F.~Wolski, P.~Dhariwal, A.~Radford, and O.~Klimov.
\newblock Proximal policy optimization algorithms.
\newblock {\em ArXiv}, 2018.

\bibitem{Silver:16}
D.~Silver, A.~Huang, C.~J. Maddison, A.~Guez, L.~Sifre, G.~van~den Driessche,
  J.~Schrittwieser, I.~Antonoglou, V.~Panneershelvam, M.~Lanctot, S.~Dieleman,
  D.~Grewe, J.~Nham, N.~Kalchbrenner, I.~Sutskever, T.~P. Lillicrap, M.~Leach,
  K.~Kavukcuoglu, T.~Graepel, and D.~Hassabis.
\newblock Mastering the game of {Go} with deep neural networks and tree search.
\newblock {\em Nature}, 529(7587):484--489, 2016.

\bibitem{Silver:17}
D.~Silver, T.~Hubert, J.~Schrittwieser, I.~Antonoglou, M.~Lai, A.~Guez,
  M.~Lanctot, L.~Sifre, D.~Kumaran, T.~Graepel, T.~P. Lillicrap, K.~Simonyan,
  and D.~Hassabis.
\newblock Mastering {Chess} and {Shogi} by self-play with a general
  reinforcement learning algorithm.
\newblock {\em ArXiv}, 2017.

\bibitem{Singh:00}
S.~Singh, T.~Jaakkola, M.~Littman, and C.~Szepesv{\'{a}}ri.
\newblock Convergence results for single-step on-policy reinforcement-learning
  algorithms.
\newblock {\em Machine Learning}, 38:287--308, 2000.

\bibitem{Singh:96}
S.~P. Singh and R.~S. Sutton.
\newblock Reinforcement learning with replacing eligibility traces.
\newblock {\em Machine Learning}, 22:123--158, 1996.

\bibitem{Skinner:58}
B.~F. Skinner.
\newblock Reinforcement today.
\newblock {\em American Psychologist}, 13(3):94--99, 1958.

\bibitem{Sobel:82}
M.~J. Sobel.
\newblock The variance of discounted {Markov} decision processes.
\newblock {\em Journal of Applied Probability}, 19(4):794--802, 1982.

\bibitem{Soshnikov:02}
A.~Soshnikov.
\newblock A note on universality of the distribution of the largest eigenvalues
  in certain sample covariance matrices.
\newblock {\em J. Statist. Phys.}, 108(5-6):1033--1056, 2002.

\bibitem{Srivastava:15}
N.~Srivastava, E.~Mansimov, and R.~Salakhutdinov.
\newblock Unsupervised learning of video representations using {LSTMs}.
\newblock {\em ArXiv}, 2015.

\bibitem{Su:15}
P.-H. Su, D.~Vandyke, M.~Gasic, N.~Mrksic, T.-H. Wen, and S.~Young.
\newblock Reward shaping with recurrent neural networks for speeding up on-line
  policy learning in spoken dialogue systems.
\newblock In {\em Proceedings of the 16th Annual Meeting of the Special
  Interest Group on Discourse and Dialogue}, pages 417--421. Association for
  Computational Linguistics, 2015.

\bibitem{Sundararajan:17}
M.~Sundararajan, A.~Taly, and Q.~Yan.
\newblock Axiomatic attribution for deep networks.
\newblock {\em ArXiv}, 2017.

\bibitem{Sutskever:14nips}
I.~Sutskever, O.~Vinyals, and Q.~V.~V. Le.
\newblock Sequence to sequence learning with neural networks.
\newblock In Z.~Ghahramani, M.~Welling, C.~Cortes, N.~D. Lawrence, and K.~Q.
  Weinberger, editors, {\em Advances in Neural Information Processing Systems
  27 (NIPS'13)}, pages 3104--3112. Curran Associates, Inc., 2014.

\bibitem{Sutton:88td}
R.~S. Sutton.
\newblock Learning to predict by the methods of temporal differences.
\newblock {\em Machine Learning}, 3:9--44, 1988.

\bibitem{Sutton:18book}
R.~S. Sutton and A.~G. Barto.
\newblock {\em Reinforcement Learning: An Introduction}.
\newblock MIT Press, Cambridge, MA, 2 edition, 2018.

\bibitem{Tamar:12}
A.~Tamar, D.~DiCastro, and S.~Mannor.
\newblock Policy gradients with variance related risk criteria.
\newblock In J.~Langford and J.~Pineau, editors, {\em Proceedings of the 29th
  International Conference on Machine Learning (ICML'12)}, 2012.

\bibitem{Tamar:16}
A.~Tamar, D.~DiCastro, and S.~Mannor.
\newblock Learning the variance of the reward-to-go.
\newblock {\em Journal of Machine Learning Research}, 17(13):1--36, 2016.

\bibitem{Chebyshev:67}
P.~Tchebichef.
\newblock Des valeurs moyennes.
\newblock {\em Journal de math{\'{e}}matiques pures et appliqu{\'{e}}es 2},
  12:177--184, 1867.

\bibitem{Tseng:90Journal}
P.~Tseng.
\newblock Solving $h$-horizon, stationary {Markov} decision problems in time
  proportional to $\log(h)$.
\newblock {\em Operations Research Letters}, 9(5):287--297, 1990.

\bibitem{Tsitsiklis:94}
J.~N. Tsitsiklis.
\newblock Asynchronous stochastic approximation and $q$-learning.
\newblock {\em Machine Learning}, 16(3):185--202, 1994.

\bibitem{Hasselt:10}
H.~van Hasselt.
\newblock Double $q$-learning.
\newblock In J.~D. Lafferty, C.~K.~I. Williams, J.~Shawe-Taylor, R.~S. Zemel,
  and A.~Culotta, editors, {\em Advances in Neural Information Processing
  Systems 23}, pages 2613--2621. Curran Associates, Inc., 2010.

\bibitem{Hasselt:16}
H.~van Hasselt, A.~Guez, and D.~Silver.
\newblock Deep reinforcement learning with double $q$-learning.
\newblock In {\em Proceedings of the Thirtieth {AAAI} Conference on Artificial
  Intelligence}, pages 2094--2100. {AAAI} Press, 2016.

\bibitem{Venugopalan:14}
S.~Venugopalan, H.~Xu, J.~Donahue, M.~Rohrbach, R.~J. Mooney, and K.~Saenko.
\newblock Translating videos to natural language using deep recurrent neural
  networks.
\newblock {\em ArXiv}, 2014.

\bibitem{Veretennikov:16}
A.~Veretennikov.
\newblock Ergodic {Markov} processes and poisson equations (lecture notes).
\newblock {\em ArXiv}, 2016.

\bibitem{Wang:15}
Z.~Wang, N.~de~Freitas, and M.~Lanctot.
\newblock Dueling network architectures for deep reinforcement learning.
\newblock {\em ArXiv}, 2015.

\bibitem{Wang:16}
Z.~Wang, T.~Schaul, M.~Hessel, H.~Hasselt, M.~Lanctot, and N.~de~Freitas.
\newblock Dueling network architectures for deep reinforcement learning.
\newblock In M.~F. Balcan and K.~Q. Weinberger, editors, {\em Proceedings of
  the 33rd International Conference on Machine Learning (ICML)}, volume~48 of
  {\em Proceedings of Machine Learning Research}, pages 1995--2003. PMLR, 2016.

\bibitem{Watkins:89}
C.~J. C.~H. Watkins.
\newblock {\em Learning from Delayed Rewards}.
\newblock PhD thesis, King's College, 1989.

\bibitem{Watkins:92}
C.~J. C.~H. Watkins and P.~Dayan.
\newblock {Q-Learning}.
\newblock {\em Machine Learning}, 8:279--292, 1992.

\bibitem{Werbos:90}
P.~J. Werbos.
\newblock A menu of designs for reinforcement learning over time.
\newblock In W.~T. Miller, R.~S. Sutton, and P.~J. Werbos, editors, {\em Neural
  Networks for Control}, pages 67--95. MIT Press, Cambridge, MA, USA, 1990.

\bibitem{Wiewiora:03}
E.~Wiewiora.
\newblock Potential-based shaping and $q$-value initialization are equivalent.
\newblock {\em Journal of Artificial Intelligence Research}, 19:205--208, 2003.

\bibitem{Wiewiora:03icml}
E.~Wiewiora, G.~Cottrell, and C.~Elkan.
\newblock Principled methods for advising reinforcement learning agents.
\newblock In {\em Proceedings of the Twentieth International Conference on
  International Conference on Machine Learning (ICML'03)}, pages 792--799,
  2003.

\bibitem{Williams:92}
R.~J. Williams.
\newblock Simple statistical gradient-following algorithms for connectionist
  reinforcement learning.
\newblock {\em Machine Learning}, 8(3):229--256, 1992.

\bibitem{Zaremba:14arxiva}
W.~Zaremba, I.~Sutskever, and O.~Vinyals.
\newblock Recurrent neural network regularization.
\newblock {\em ArXiv}, 2014.

\bibitem{Zhang:16}
J.~Zhang, Z.~L. Lin, J.~Brandt, X.~Shen, and S.~Sclaroff.
\newblock Top-down neural attention by excitation backprop.
\newblock In {\em Proceedings of the 14th European Conference on Computer
  Vision (ECCV)}, pages 543--559, 2016.
\newblock part IV.

\end{thebibliography}

\end{appendices}

\end{document}